%% file: ICML_EM2MLR.tex
\def\isarxiv{1} % for icml submission, comment this line
\ifdefined\isarxiv
  \documentclass[11pt]{article}
  \usepackage[numbers]{natbib}
  \usepackage{changepage}
  \usepackage[margin=1in]{geometry}
\else
  \documentclass[twoside]{article}
  \usepackage[accepted]{icml2024}
\fi

\input{0_package.tex}

\begin{document}

\ifdefined\isarxiv
  \title{\textbf{Unveiling the Cycloid Trajectory of EM Iterations in Mixed Linear Regression}}
  \author{
  Zhankun Luo\thanks{Correspondence to: \href{mailto:luo333@purdue.edu}{luo333@purdue.edu}}
  \and 
  Abolfazl Hashemi\thanks{Correspondence to: \href{mailto:abolfazl@purdue.edu}{abolfazl@purdue.edu}\\
  \indent\texttt{Keywords:} Expectation-Maximization (EM) algorithm, mixtures of
linear regression, statistical learning.}
  }
  \date{Purdue University}
  \maketitle
\else
  \twocolumn[
  \icmltitle{Unveiling the Cycloid Trajectory of EM Iterations in Mixed Linear Regression}
  \icmlsetsymbol{equal}{*}
  \begin{icmlauthorlist}
  \icmlauthor{Zhankun Luo}{yyy}%{equal,yyy}
  \icmlauthor{Abolfazl Hashemi}{yyy}%{equal,yyy}
  \end{icmlauthorlist}
  \icmlaffiliation{yyy}{School of Electrical and Computer Engineering, Purdue University, IN, USA}
  \icmlcorrespondingauthor{Zhankun Luo}{luo333@purdue.edu}
  \icmlcorrespondingauthor{Abolfazl Hashemi}{abolfazl@purdue.edu}
  \vskip 0.3in
  ]
  \printAffiliationsAndNotice{}
\fi
\begin{abstract}
% summary
We study the trajectory of iterations and the convergence rates of the Expectation-Maximization (EM) algorithm for two-component Mixed Linear Regression (2MLR).
% background
The fundamental goal of MLR is to learn the regression models from unlabeled observations.
The EM algorithm finds extensive applications in solving the mixture of linear regressions.
% previous works
Recent results have established the super-linear convergence of EM for 2MLR in the noiseless and high SNR settings under some assumptions and its global convergence rate with random initialization has been affirmed.
% open problem
However, the exponent of convergence has not been theoretically estimated and the geometric properties of the trajectory of EM iterations are not well-understood. 
% contributions
In this paper, first, using Bessel functions we provide explicit closed-form expressions for the EM updates under all SNR regimes. Then, 
in the noiseless setting, we completely characterize the behavior of EM iterations by deriving a recurrence relation at the population level 
and notably show that all the iterations lie on a certain cycloid.
Based on this new trajectory-based analysis, we exhibit the theoretical estimate for the exponent of super-linear convergence and further improve the statistical error bound at the finite-sample level.
% novelty
Our analysis provides a new framework for studying the behavior of EM for Mixed Linear Regression.
\end{abstract}

% For TOC in appendix (https://tex.stackexchange.com/a/419290)
\doparttoc % Tell to minitoc to generate a toc for the parts
\faketableofcontents % Run a fake tableofcontents command for the partocs

\section{Introduction}\label{sec:intro}
\input{1_introduction}

\section{Problem Setup}\label{sec:setup}
\input{2_problem_setup.tex}

\section{Population EM Updates}\label{sec:updates}
\input{3_0_framework}

\section{Population Level Analysis}\label{sec:population}
\input{3_population}

\section{Finite-sample Level Analysis}\label{sec:finite}
\input{4_finite_sample}

\section{Experiments}\label{sec:experiments}
\input{5_experiments}

\section{Conclusion}\label{sec:conclusion}
\input{6_conclusion}

\section*{Acknowledgements}
\input{6.1_acknowledgements}

\section*{Impact Statement}
\input{7_impact_statement}

\ifdefined\isarxiv
  \bibliographystyle{plainnat}
\else
  \bibliographystyle{icml2024}
\fi
\bibliography{ref_EM2MLR} 
%%%%%%%%%%%%%%%%%%%%%%%%%%%%%%%%%%%
%%%%%% SUPPLEMENT (OPTIONAL) %%%%%%
%%%%%%%%%%%%%%%%%%%%%%%%%%%%%%%%%%%
\input{8_supplementary}
\end{document}

%% file: 0_package.tex
\usepackage{xcolor}
\usepackage{fancyhdr}
\usepackage{graphics}
\usepackage{graphicx} 
\usepackage{float} 
\usepackage{subcaption} 
\usepackage{comment}
\usepackage{color}
\usepackage{amsthm}
\usepackage{amsfonts}
\usepackage{amsmath}
\usepackage{bm}
\usepackage{amssymb,bbm}
\usepackage{url}
\usepackage[colorlinks=True,linkcolor=magenta,citecolor=blue,urlcolor=blue,pagebackref=true,backref=page]
{hyperref}

\usepackage{enumitem}
\usepackage[font=scriptsize]{caption}
\usepackage{mathtools}
\newcommand{\assign}{:=}
\newcommand{\mathd}{\mathrm{d}}
\newcommand{\mathe}{\mathrm{e}}

\newcommand{\tmop}[1]{\ensuremath{\operatorname{#1}}}
\newcommand{\tmstrong}[1]{\textbf{#1}}
\newcommand{\tmmathbf}[1]{\ensuremath{\boldsymbol{#1}}}

\newenvironment{itemizedot}{\begin{itemize} }{\end{itemize}}
\newcommand{\ind}{\mathrel{\perp\!\!\!\perp}} % random variable independent

\usepackage{mdframed} 
\usepackage{thmtools}

\definecolor{shadecolor}{gray}{0.90}
\declaretheoremstyle[
headfont=\normalfont\bfseries,
notefont=\mdseries, notebraces={(}{)},
bodyfont=\normalfont,
postheadspace=0.5em,
spaceabove=0.5em,
spacebelow=0.5em,
mdframed={
  skipabove=8pt,
  skipbelow=8pt,
  hidealllines=true,
  backgroundcolor={shadecolor},
  innerleftmargin=4pt,
  innerrightmargin=4pt}
]{shaded}

\declaretheorem[within=section]{definition}
\declaretheorem[style=shaded,sibling=definition]{theorem}
\declaretheorem[sibling=definition]{proposition}
\declaretheorem[style=shaded,sibling=definition]{assumption}
\declaretheorem[style=shaded,sibling=definition]{corollary}
\declaretheorem[style=shaded,sibling=definition]{lemma}
\declaretheorem[sibling=definition]{remark}
\usepackage[nohints]{minitoc}

%% file: 1_introduction.tex
%\vspace{-4mm}
%\subsection{Background}
%% significance of the problem
A mixture model of parameterized distributions, such as the Mixture of Linear Regression (MLR) and Gaussian Mixture Model (GMM), is remarkably powerful for modeling intricate relationships in practice.
It is highly suitable to address the challenges arising from data with corruptions, missing values, and latent variables~\citep{beale1975missing}. 
In this paper, we focus on the symmetrical two-component mixed linear regression (2MLR) that can be expressed as follows:
\begin{equation}\label{eq:model}
  y= (- 1)^{z+ 1} \langle \theta^{\ast}, x
     \rangle +\varepsilon,
\end{equation}
where $\varepsilon$ is the addictive noise, $s= (x,y) \in \mathbb{R}^d \times \mathbb{R}$ is a pair of the covariate and response random variable,  $z\in\{1, 2\} \sim \mathcal{C}\mathcal{A}\mathcal{T} (\pi^{\ast})$ represents the latent
variable, namely the label of data, and $\theta^{\ast}, \pi^{\ast}$ are the true values for the regression parameters and the mixing weights, respectively.

Maximum Likelihood Estimation (MLE) provides a systematic framework to study such models.
However, computing the Maximum Likelihood Estimate (MLE) for high-dimensional data is intractable due to its non-convexity and numerous spurious local maxima.
% approaches
Various approaches have been proposed to handle this intractable problem. 
~\citet{tipping1999mixtures} adopted PCA by connecting the inherent geometric property and probabilistic interpretation with Gaussian covariates and errors.
~\citet{kong20nips, kong20icml} employed a meta-learning approach to learn the parameters of MLR with small batches.
~\citet{shen2019nips} proposed an iterative variant of the least trimmed squares to handle MLR with corruptions.
Another competitor is the moment-based method, in conjunction with the gradient descent algorithm~\citep{li2018learning}.
Moreover, the Expectation Maximization (EM) method stands out for its computational efficiency and ease of practical implementation. In the context of  \eqref{eq:model}, EM estimates the regression parameters and the mixing weights from observations. It operates in two steps: E-step computes the expected log-likelihood using the current parameter estimate; M-step updates the parameters to maximize the expected log-likelihood compute in the E-step.
These steps serve to maximize the lower bound on the MLE objective iteratively until convergence.

%% background
% classic works for background
\citet{dempster1977maximum} presented the modern EM algorithm and demonstrated its likelihood to be monotonically increasing with EM updates.
Theoretically, ~\citet{wu1983convergence} established the global convergence of a unimodal likelihood under some regularity conditions.
Empirically, EM showed success in the MLR problem ~\citep{de1989mixtures,jordan1994hierarchical,jordan1995convergence}. 
Additionally, ~\citet{wedel1995mixture} introduced a framework of EM for MLR involving unknown latent variables.

%\subsection{Related works}
\noindent\textbf{Related Works.}
%% literature review
Both MLR and GMM can be viewed as instances of subspace clustering, thus sharing similarities in the analysis of EM.
% GMM
~\citet{dasgupta2007} showed parameterized well-separated spherical Gaussians can be learned to near-optimal precision using a variant of EM.
Furthermore, ~\citet{zhao2020statistical} illustrated the linear global convergence of EM with well-separated spherical Gaussians and initialization within a ball around the truth.
For GMM with $k\geq 3$ components, ~\citet{jin2016local} demonstrated the EM with a random initialization is frequently trapped in local minima
with high probability, and local maxima can exhibit arbitrarily inferior likelihood than that of any global maximum.
~\citet{chen2020likelihood} and ~\citet{qian2022spurious} characterized the only types of local minima for EM and $k$-means (EM with hard labels) in GMM under a separation condition.
By leveraging the characterized structures of local minimum, a general framework was proposed to escape local minima ~\cite{zhang2020symmetry, hong2022geometric}, unifying variants of $k$-means from a geometric perspective.
~\citet{katsevich2023likelihood} revealed the link between EM and the moment method for GMM via an asymptotic expansion of log-likelihood in low Signal-to-Noise Ratio (SNR).
So far, the specific case of GMM with $k=2$ components (2GMM) has been studied intensively. 
The global convergence of EM with random initialization for spherical 2GMM was established in ~\citep{klusowski2016statistical, xu2016global, daskalakis17b}.
~\citet{wu2021randomly} elaborated the convergence result in all SNR regimes, while constraining the initialization within a very small radius.
~\citet{qian2019global,qian2020local} extended this convergence result from the spherical Gaussian to rotation-invariant log-concave densities.
~\citet{ndaoud2018sharp} examined the phase transition threshold of SNR for the exact recovery of 2GMM.
% MLR
Similarly, EM for MLR with two components (2MLR) with random initialization converges globally.
~\citet{balakrishnan2017statistical} firstly proved the global convergence of EM for 2MLR with valid initialization within a ball around the truth.
~\citet{dana2019estimate2mix} extended the convergence result in the high SNR regime for the case where the cosine angle between the initial parameters and the truth is sufficiently large.
~\citet{kwon2022dissertation, kwon2019global} confirmed that EM for 2MLR converges from a random initialization with high probability.
~\citet{yudong2018trans} bounded the statistical error of EM for 2MLR in different SNR regimes.
~\citet{xu2020towards} illustrated the generalization error bounds of log-likelihood of the first-order EM for 2MLR.
~\citet{kwon2021minimax} further studied the statistical error and the convergence rate of EM for 2MLR under all regimes of SNR.
~\citet{yi2014alternating, yi2016solving} considered Alternating Minimization (AM), an EM variant with hard labels, for 2MLR in the setting of no noise.
Accordingly, ~\citet{ghosh20a} demonstrated a super-linear convergence rate of AM for 2MLR in the noiseless setting within a specific convergence region.
~\citet{kwon2021minimax} generalized the noiseless setting to the high SNR regime while retaining the super-linear convergence.
~\citet{kwon2020converges} provided a convergence analysis of EM for MLR with multiple components, covering the most general scenarios.

Previous works on the convergence analysis of EM for 2MLR have overlooked the existence of unbalanced mixing weights and assumed a balanced setting.
~\citet{dwivedi2020sharp,dwivedi2020unbalanced} accounted for unbalanced mixing weights and revealed a sharp contrast in statistical error and convergence rate between unbalanced and balanced cases, for the special case of no separation of parameters.
In the previous convergence analysis of EM, the location scale of GMM and the noise variance of MLR are fixed.
In light of this, ~\citet{ren2022beyond} proposed an EM variant for 2GMM with the unknown location scale to speed up the convergence of EM.
~\citet{chandrasekher2021sharp} devised a tool that demonstrated noteworthy potential, employing Gordon state evolution update ~\citep{thrampoulidis2014gaussian, thrampoulidis2015gordon} to accurately estimate both the statistical error and the convergence rate at the finite-sample level.

%\subsection{Contributions}
\noindent\textbf{Contributions.}
%% challenges
%% contributions of our work
In this paper, we propose a framework that offers explicit closed-form expressions with Bessel functions (Chapter 10 of ~\citep{olver2010nist}) for the EM updates of 2MLR, enabling the analysis of convergence rate across all SNR regimes.
Moreover, our framework includes both scenarios where the mixing weights are balanced and unbalanced.
More specifically, we focus on EM updates in the noiseless setting, and present the following contributions:
\begin{itemize}
  \item We derive the recurrence relation, and further show the cycloid trajectory of EM iterations at the population level.
  \item We establish the super-linear convergence without restrictions in previous works \citep{ghosh20a,kwon2021minimax}.%, and characterize all the stationary points for EM updates. 
  \item We conduct a finer analysis for the statistical errors in regression parameters, and explore how the error in mixing weights is influenced by the angle formed by the EM iteration and the true regression parameters, the true mixing weights.
\end{itemize}
%{\color{red} Abolfazl: add a few lines about the new technical tools we develop that enable us to make the above contributions.}
%{\color{red} Abolfazl: no need for the conference paper as there won't be room for an organization generally. We will add it if there was room.}
%% Organization
% We organize this paper as follows:
% section 1: xxx
% section 2: xxx
% section 3: xxx
% section 4: xxx
% Finally, we provide technical details of proofs in the Appendix for better readability.

%% file: 2_problem_setup.tex
\input{2.0_notation}
\input{2.1_assumptions}
\input{2.2_em_updates}

%% file: 2.0_notation.tex
\textbf{Notation.}
Consider the 2MLR model in \eqref{eq:model}.
Let $n$ denote the number of
samples $\mathcal{S} \assign \{ x_i, y_i \}_{i=1}^n$ used for each EM
update, $\{ z_i \}_{i=1}^n$ be the values of latent variable for these
samples. 
Further, $\sigma^2$ denotes the noise variance, $\eta \assign \frac{\| \theta^{\ast}
\|}{\sigma}$ is the signal-to-noise ratio (SNR), and $\bar{\theta}\assign\frac{\theta}{\sigma}, \bar{\theta}^\ast\assign\frac{\theta^\ast}{\sigma}$ are  the normalized parameters. $f(\nu)\ast g(\nu)$ stands for a convolution of $f(\nu)$ and $g(\nu)$, and $a\vee b, a\wedge b$ refer to the the least upper bound $\max(a, b)$ and greatest lower bound $\min(a, b)$ of $a,b$ respectively.
The symbol $\ind$ indicates that the random variables are independent. 
$K_0,K_1$ are the modified Bessel functions of the second kind with parameters 0,1 respectively  (Chapter 10 of ~\citep{olver2010nist}).

%% file: 2.1_assumptions.tex
\noindent\textbf{Assumptions.}
% We assume that the latent variable and the mixing weights $(z ; \pi)$ are independent of the regression parameters $\theta$, namely $(z; \pi) \ind \theta$. Furthermore, The additive noise $\varepsilon$ is independent of the covariate random variable, latent variable, the regression parameters, and the mixing weights, that is $\varepsilon \ind (x, z; \theta, \pi)$. Additionally, The covariate random variable $x$ is independent of the latent variable, the regression parameters, and the mixing weights, namely $x \ind (z ; \theta, \pi)$. Finally, Both the covariate $x$ and the noise $\varepsilon$ are Gaussians, $x\sim \mathcal{N} (0, I_d), \varepsilon \sim \mathcal{N} (0, \sigma^2)$, where $I_d$ is a $d$ by $d$ identity matrix.

\begin{assumption}\label{ass:1}
    The latent variable and the mixing weights $(z ; \pi)$ are independent
    of the regression parameters $\theta$, namely $(z; \pi) \ind \theta$.
\end{assumption}

\begin{assumption}\label{ass:2}
    The additive noise $\varepsilon$ is independent of the covariate random variable, latent variable, the regression parameters, and the mixing weights, that is $\varepsilon \ind (x, z; \theta, \pi)$.
\end{assumption}

\begin{assumption}\label{ass:3}
    The covariate random variable $x$ is independent of the latent
    variable, the regression parameters, and the mixing weights, namely $x \ind
    (z ; \theta, \pi)$.
\end{assumption}

\begin{assumption}\label{ass:4}
    Both the covariate $x$ and the noise $\varepsilon$ are Gaussians, $x
    \sim \mathcal{N} (0, I_d), \varepsilon \sim \mathcal{N} (0, \sigma^2)$,
    where $I_d$ is a $d$ by $d$ identity matrix.
\end{assumption}
%{\color{red}Todo: add interpretation for assumptions}
We leverage the assumption of the Gaussianity of the covariate, which is standard in this line of work (see, e.g., Assumption 1 in~\cite{ghosh20a} and Section 2.1 in~\cite{kwon2021minimax}). 
The above standard assumptions are necessary to derive the forthcoming results. 

%% file: 2.2_em_updates.tex
\textbf{EM Updates.}
\citet{balakrishnan2017statistical} considered the following population EM update for 2MLR given the balanced 
mixing weights $\pi = \pi^\ast = \left\{ \frac{1}{2}, \frac{1}{2} \right\}$, where $\mathbb{E}_{s\sim p(s\mid\theta^\ast, \pi^\ast)} \assign \mathbb{E}_{x \sim \mathcal{N} (0, I_d)}$ $\mathbb{E}_{y \mid x \sim \pi^\ast(1)\mathcal{N} (\langle x, \theta^{\ast} \rangle, \sigma^2)+ \pi^\ast(2)\mathcal{N} (-\langle x, \theta^{\ast} \rangle, \sigma^2)}$.
\begin{equation}
  M (\theta) =\mathbb{E}_{s\sim p(s\mid\theta^\ast, \pi^\ast)} 
   \tanh \left(
  \frac{y\langle x, \theta \rangle}{\sigma^2}\right)
  y  x\nonumber
\end{equation}
% {\color{red} Todo: 
% 1. Extend the expectation expression to Expectations with unbalanced mixing weights,
% 2. The EM update for mixing weights or $\nu$,
% 3. from population level to finite-sample level, plus Easy version}
To extend the EM update for both balanced and unbalanced mixing weights, we introduce 
\begin{equation}\label{eq:nu_def}
    \nu\assign \frac{\log \pi(1) - \log \pi(2)}{2}, \quad \pi=\{\pi(1), \pi(2)\},
\end{equation}
that is $\tanh(\nu)=\pi(1)- \pi(2)$. Thus, the population EM update rule for regression parameters $\theta$ becomes
\begin{equation}\label{eq:theta}
  M(\theta, \nu) \assign \mathbb{E}_{s\sim p(s\mid\theta^\ast, \pi^\ast)} 
  \tanh \left(
  \frac{y \langle x, \theta \rangle}{\sigma^2} +\nu\right)
  y  x ,
\end{equation}
while the corresponding EM update rule for $\tanh(\nu)$ is (see the derivations in the supplementary, Appendix~\ref{sup:derive_em})
\begin{equation}\label{eq:nu}
  N(\theta, \nu) \assign \mathbb{E}_{s\sim p(s\mid\theta^\ast, \pi^\ast)} 
  \tanh \left(
  \frac{y \langle x, \theta \rangle}{\sigma^2} +\nu\right).
\end{equation}

Subsequently, the finite-sample EM update rules are
\begin{eqnarray}\label{eq:finite}
M_n  (\theta, \nu) &=& \left( \frac{1}{n}  \sum_{i = 1}^n x_i x_i^{\top}
\right)^{-1}\nonumber\\
& &\left( \frac{1}{n}  \sum_{i = 1}^n \tanh \left(
\frac{y_i\langle x_i, \theta\rangle}{\sigma^2} +\nu \right) y_i x_i \right)\nonumber\\
N_n  (\theta, \nu) &=& \frac{1}{n}  \sum_{i = 1}^n \tanh \left(
\frac{y_i\langle x_i, \theta\rangle}{\sigma^2} +\nu\right).
\end{eqnarray}

For the ease of theoretical analysis, we will use the easy EM method as discussed in Section \ref{sec:finite} with the following update
\begin{equation}\label{eq:easyEM}
    M^{\tmop{easy}}_n  (\theta, \nu)=  \frac{1}{n}  \sum_{i = 1}^n \tanh \left(
\frac{y_i\langle x_i, \theta\rangle}{\sigma^2} +\nu \right) y_i x_i. 
\end{equation}

\textbf{Trajectory-Relevant Quantities.} 
Our trajectory-based analysis further utilizes certain angles described next.
We denote the cosine of the angle between the estimate for
parameters $\theta$ and the true value $\theta^{\ast}$ by $\rho \assign \frac{\langle \theta, \theta^{\ast} \rangle}{\| \theta \| \cdot
\| \theta^{\ast} \|}$. Furthermore, $\varphi \assign \frac{\pi}{2} - \arccos | \rho | \in [ 0,
\frac{\pi}{2} ), \phi \assign 2 \arccos | \rho | \in (0, \pi] $ are
defined accordingly. As apparent from these definitions, we refer to $\rho$ as the sub-optimality cosine and $\varphi$ and $\phi$ as \textit{sub-optimality angles}, respectively.

Further, let $\hat{e}_1 \assign \frac{\theta^{\ast}}{\| \theta^{\ast} \|}, \vec{e}_1 \assign \frac{\theta}{\| \theta \|}$ be the direction unit vectors of $\theta^{\ast}, \theta$, and $\hat{e}_2 : = \frac{\theta - \hat{e}_1
\hat{e}_1^{\top} \theta}{\| \theta - \hat{e}_1 \hat{e}_1^{\top} \theta \|},
\vec{e}_2 : = \frac{\theta^\ast - \vec{e}_1
\vec{e}_1^{\top} \theta^\ast}{\| \theta^\ast - \vec{e}_1 \vec{e}_1^{\top} \theta^\ast \|}$ 
be the unit vectors on the plane $\tmop{span} \{ \theta^{\ast}, \theta \}$ which are perpendicular to $\hat{e}_1,\vec{e}_1$ respectively.
The superscript $t$ stands for the $t$-th EM iteration. For instance, $\theta^t, \pi^t$ denote the $t$-th iteration for regression parameters and mixing weights. 
These vectors will again simplify the ensuing trajectory-based discussion (see Fig.~\ref{fig:vec} for visualization).

%% file: 3_0_framework.tex
In this section we derive closed-form expressions for the update of population EM we introduced in Eq.~\eqref{eq:theta},~\eqref{eq:nu}. Let $U(\theta, \nu) : =\mathbb{E}_{s \sim p (s \mid \theta^{\ast},
    \pi^{\ast})}\log \cosh ( \frac{y \langle x, \theta \rangle}{\sigma^2} + \nu )$. Then, Eq.~\eqref{eq:theta},~\eqref{eq:nu}  can be written as 
the following relations.
\begin{equation}
   M(\theta, \nu) = \sigma^2 \nabla_{\theta} U(\theta, \nu), \quad N(\theta, \nu) =\nabla_{\nu}U(\theta, \nu).
\end{equation}
Therefore, we need to derive a closed-form expression for $U(\theta, \nu)$  which appears in the population update of both $\theta$ and $\nu$.
In the supplementary materials (Appendix~\ref{sup:lemma}), we provide the explicit expression of this expectation 
by introducing $K_0$, the modified Bessel function of the second kind with parameter 0 (Chapter 10 of ~\citep{olver2010nist}).
Subsequently, we derive the closed-form expressions for EM update rules at the population level in Theorem~\ref{thm:em_update} below (see Appendices~\ref{sup:lemma},~\ref{sup:updates}).
\begin{theorem}{(EM Updates across All SNR)}\label{thm:em_update}
    Let $\rho \assign \frac{\langle \theta,
    \theta^{\ast} \rangle}{\| \theta \| \cdot \| \theta^{\ast} \|},
    \bar{\theta} \assign \frac{\theta}{\sigma}, \bar{\theta}^{\ast}
    \assign \frac{\theta^{\ast}}{\sigma}$, then the EM update rules for
    $\theta, \tanh(\nu)$ at Population level are
    \ifdefined\isarxiv
      \normalsize
    \else
      \scriptsize
    \fi
    \begin{equation}\nonumber
        \begin{aligned}
    &M(\theta, \nu)=\left[ - \frac{\sigma}{\pi} \cdot \frac{\| \bar{\theta}^{\ast}
    \|^2}{\| \bar{\theta} \|^2} \cdot \frac{\sqrt{1 - \rho^2}}{\left( 1 + (1 -
    \rho^2) \| \bar{\theta}^{\ast} \|^2 \right)^{\frac{3}{2}}}  \cosh^{-1}
    (\nu^{\ast}) \right]\\
    &\left\{ \tanh (\nu) \ast \nu \left[ \alpha(\nu) \left(
    \frac{\frac{\bar{\theta}}{\| \bar{\theta} \|}}{\sqrt{1 - \rho^2} \|
    \bar{\theta}^{\ast} \|^2} + \hat{e}_2 \right) + \beta(\nu)  \vec{e}_2
    \right] \right\},\\
    &N(\theta, \nu)=\frac{{(1 + (1 - \rho^2) \| \bar{\theta}^{\ast} \|^2)^{-
    \frac{1}{2}}} }{\pi \| \bar{\theta} \| \cosh (\nu^{\ast})}
    \underset{\mathbb{R}}{\int} \mathd \nu'  \tanh (\nu - \nu')\\
    &K_0 \left(
    \frac{\sqrt{1 + \| \bar{\theta}^{\ast} \|^2} \cdot \left| \frac{\nu'}{\|
    \bar{\theta} \|} \right|}{[1 + (1 - \rho^2) \| \bar{\theta}^{\ast} \|^2]}
    \right)\cosh \left( \frac{\rho \| \bar{\theta}^{\ast} \| \left(
    \frac{\nu'}{\| \bar{\theta} \|} \right)}{[1 + (1 - \rho^2) \|
    \bar{\theta}^{\ast} \|^2]} - \nu^{\ast} \right) 
        \end{aligned}
    \end{equation}

    \normalsize
    where these coefficients are defined as 
    \ifdefined\isarxiv
      \normalsize
    \else
      \scriptsize
    \fi            
    \begin{equation}\nonumber
        \begin{aligned}
    &\alpha(\nu) \assign \cosh \left( \frac{\rho \| \bar{\theta}^{\ast} \|
    \left( \frac{\nu}{\| \bar{\theta} \|} \right)}{[1 + (1 - \rho^2) \|
    \bar{\theta}^{\ast} \|^2]} - \nu^{\ast} \right)\\
    &\quad K_0 \left(
    \frac{\sqrt{1 + \| \bar{\theta}^{\ast} \|^2} \cdot \left| \frac{\nu}{\|
    \bar{\theta} \|} \right|}{[1 + (1 - \rho^2) \| \bar{\theta}^{\ast} \|^2]}
    \right),\\
    &\beta(\nu) \assign \tmop{sgn} (\nu) \frac{\sqrt{1 + \|
    \bar{\theta}^{\ast} \|^2}}{\| \bar{\theta}^{\ast} \|} \sinh \left(
    \frac{\rho \| \bar{\theta}^{\ast} \| \left( \frac{\nu}{\| \bar{\theta} \|}
    \right)}{[1 + (1 - \rho^2) \| \bar{\theta}^{\ast} \|^2]} - \nu^{\ast}
    \right)\\ 
    &\quad K_1 \left( \frac{\sqrt{1 + \| \bar{\theta}^{\ast} \|^2} \cdot
    \left| \frac{\nu}{\| \bar{\theta} \|} \right|}{[1 + (1 - \rho^2) \|
    \bar{\theta}^{\ast} \|^2]} \right).
        \end{aligned}
    \end{equation}
  \end{theorem}

  \begin{figure}[!htbp]
    \centering
    \includegraphics[width=0.35\textwidth,trim={0 5cm 14cm 0},clip]{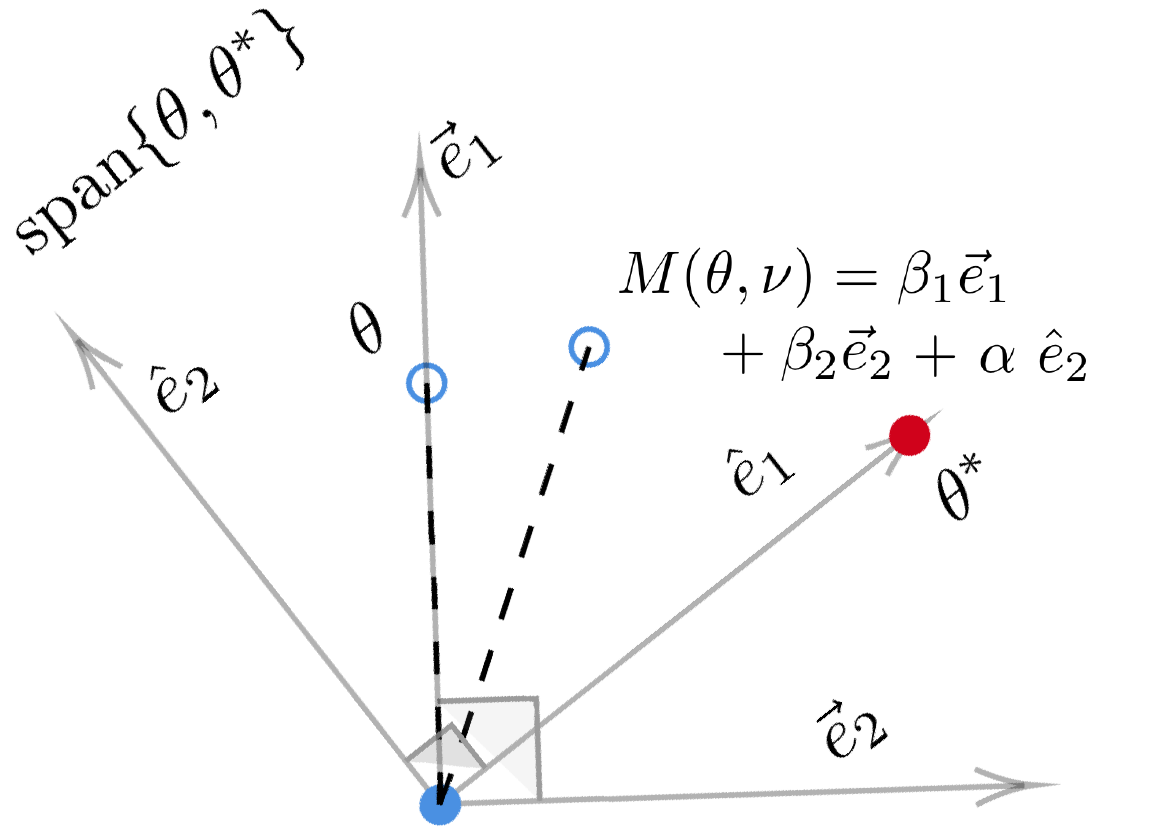}
    \caption{The EM update $M(\theta, \nu)$ for regression parameters lies on span$\{\theta, \theta^\ast\}$.}
    \label{fig:vec}
  \end{figure}
 \normalsize 

Note that in Theorem~\ref{thm:em_update} we have $\{\bar{\theta}, \hat{e}_2, \vec{e}_2\}\in\text{span}\{\theta, \theta^\ast\}$, which in turn implies the population EM update for the regression parameters satisfies $M(\theta, \nu)\in\text{span}\{\theta, \theta^\ast\}$. 
We visualize these vectors and their relations in Fig.~\ref{fig:vec}.

\textbf{Special cases.} For the special case of no separation $\tmop{SNR} \assign \frac{\| \theta^{\ast}\|}{\sigma} \rightarrow 0$ in Theorem~\ref{thm:em_update}, the population EM update rules can be simplified (see Appendix~\ref{sup:updates}, proof of Corollary~\ref{cor:no_separa}).
  \begin{corollary}{(EM Updates for No Separation Case)}\label{cor:no_separa}
    For the special case of no separation of parameters, namely $\tmop{SNR}
    \assign \frac{\| \theta^{\ast} \|}{\sigma} \rightarrow 0$, the EM update
    rules for $\theta^t, \tanh(\nu^t)$ at Population level are $\bar{\theta}^t =
    \frac{\bar{\theta}^0}{\| \bar{\theta}^0 \|} \cdot \frac{1}{\pi} 
    \int_{\mathbb{R}} \tanh (\| \bar{\theta}^{t - 1} \| x - \nu^{t - 1}) xK_0
    (|x|) \mathrm{d} x$ and $\tanh (\nu^t) = \frac{1}{\pi}  \int_{\mathbb{R}}
    \tanh (\nu^{t - 1} - \| \bar{\theta}^{t - 1} \| x) K_0 (|x|) \mathrm{d} x$,
    where $\bar{\theta}^t \assign \frac{\theta^t}{\sigma}$ and $\bar{\theta}^0
    \assign \frac{\theta^0}{\sigma}$.
  \end{corollary}
  % Commented by Abolfazl, Move to supp
  % \begin{remark}
  % Note that 
  %   %For the ease of theoretical analysis
  %   % With the identity $\tanh (\nu^{t - 1} - \| \bar{\theta}^{t - 1} \| x) +
  %   % \tanh (\nu^{t - 1} + \| \bar{\theta}^{t - 1} \| x) = \frac{2 \sinh (2 \nu^{t
  %   % - 1})}{\cosh (2 \nu^{t - 1}) + \cosh (2 \| \bar{\theta}^{t - 1} \| x)}$
  % we can rewrite the EM update rule as $\tanh (\nu^t) = \tanh (\nu^{t - 1})
  % \cdot \frac{2}{\pi}  \int_{\mathbb{R}_{\geq 0}} \frac{\cosh (2 \nu^{t - 1}) +
  % 1}{\cosh (2 \nu^{t - 1}) + \cosh (2 \| \bar{\theta}^{t - 1} \| x)} K_0 (|x|)
  % \mathrm{d} x$.
  
  % Since $\frac{2}{\pi}  \int_{\mathbb{R}_{\geq 0}} K_0 (|x|) \mathrm{d} x = 1$
  % %and $\cosh (2 \| \bar{\theta}^{t - 1} \| x) \geq 1$
  % , the EM update rule implies $| \nu^t |
  % \leq | \nu^{t - 1} |$ and $\tmop{sgn} (\nu^t) = \tmop{sgn} (\nu^{t - 1})$.
  % If we take the $\ell_2$ norm on both sides of the EM update rule for regression parameters, it follows that $\| \bar{\theta}^t
  % \| = \frac{1}{\pi}  \int_{\mathbb{R}} \tanh (\| \bar{\theta}^{t - 1} \| x -
  % \nu^{t - 1}) xK_0 (|x|) \mathrm{d} x \leq \frac{1}{\pi}  \int_{\mathbb{R}} | x
  % | K_0 (|x|) \mathrm{d} x = \frac{2}{\pi}$ is bounded.
  % \end{remark}

The second special case, which will be our main focus moving forward, arises by letting $\tmop{SNR} \assign \frac{\| \theta^{\ast}
    \|}{\sigma} \rightarrow \infty$ in Theorem~\ref{thm:em_update}. In doing so, we obtain the EM update rules at the population level in the noiseless setting described below (see Appendix~\ref{sup:updates}, proof of Corollary~\ref{cor:noiseless}).
  \begin{corollary}{(EM Updates in Noiseless Setting)}\label{cor:noiseless}
    In the noiseless setting, namely $\tmop{SNR} \assign \frac{\| \theta^{\ast}
    \|}{\sigma} \rightarrow \infty$, the EM update rules for $\theta^t, \tanh(\nu^t)$
    at the Population level are
    \begin{equation}\nonumber
    \begin{aligned}
      &\frac{\theta^t}{\| \theta^{\ast} \|} = \frac{2}{\pi} \left[ \mathrm{sgn}
      (\rho^{t - 1}) \varphi^{t-1}  \frac{\theta^{\ast}}{\| \theta^{\ast} \|} + \cos
      \varphi^{t - 1} \frac{\theta^{t - 1}}{\| \theta^{t - 1} \|} \right]\\
      &
                \tanh (\nu^t) = \mathrm{sgn} (\rho^{t - 1}) \left( \frac{2}{\pi}
      \varphi^{t - 1} \right) \cdot \tanh (\nu^{\ast}),
          \end{aligned}
    \end{equation}
    where $\rho^{t - 1} \assign \frac{\langle \theta^{t - 1}, \theta^{\ast}
    \rangle}{\| \theta^{t - 1} \| \| \theta^{\ast} \|},
    \varphi^{t -
    1} \assign \frac{\pi}{2} - \arccos | \rho^{t - 1} |$.
  \end{corollary}

%% file: 3_population.tex
In this section, we focus on the properties of EM update rules (Corollary~\ref{cor:noiseless}) in the noiseless setting at the population level and present one of our main theoretical results on the convergence rate and the estimation of error of the population EM through a trajectory-based analysis.
In particular, we show that the error of mixing weights (in $\ell_1$ norm) is proportional to the angle between the estimated regression parameters and the true parameters and establish a consistent quadratic convergence rate, which, interestingly, is independent of mixing weights.

\begin{theorem}{(Population Level Convergence)}\label{thm:population}
  If the initial sup-optimality cosine $\rho^0 \assign \frac{\langle \theta^0,
  \theta^{\ast} \rangle}{\| \theta^0 \| \cdot \| \theta^{\ast} \|} \neq 0$,
  then with the number of total iterations at most $T =\mathcal{O} \left( \log
  \frac{1}{| \rho^0 |} \vee \log \log \frac{1}{\varepsilon} \right)$, the error of EM update at the population level is bounded by $\frac{\|
  \theta^{T + 1} - \tmop{sgn} (\rho^0) \theta^{\ast} \|}{\| \theta^{\ast} \|}
  < \varepsilon$, and $\|\pi^{T+1} -\bar{\pi}^\ast\|_1 =\mathcal{O}(\sqrt{\varepsilon})\cdot \|\frac{1}{2}-\pi^\ast\|_1$
  , where $\bar{\pi}^\ast := \frac{1}{2}-\tmop{sgn} (\rho^0) (\frac{1}{2}-\pi^\ast)$.
\end{theorem}

\textbf{Proof Sketch of Theorem~\ref{thm:population}.}
To estimate the errors of mixing weights and regression parameters, we need to study the trajectory of the regression parameters $\theta^t$. 
Recall, Corollary~\ref{cor:noiseless} demonstrated that for the special case of the noiseless setting, the EM update for $\theta^t$ depends on the angle $\varphi^t$. Additionally, upon recalling $|\tanh(\nu)|=\|\frac{1}{2} - \pi^\ast\|_1$, we directly obtain the following Corollary~\ref{cor:err_mixing} for the error of mixing weights.
Regarding the regression parameters, we derive the recurrence relation in Proposition~\ref{prop:recurrence} and further show the cycloid trajectory in Proposition~\ref{prop:cycloid}. Hence, the error of regression parameters is only determined by the sup-optimality angle $\varphi^t$. Therefore, instead of directly analyzing the error of regression parameters, we could characterize its behavior by studying the evolution of the \textit{sub-optimality angle} $\varphi^t$.
\begin{corollary}{(Error of Mixing Weights $\pi^t$)}\label{cor:err_mixing}
  In the noiseless setting, the error of mixing weights for EM updates at 
  the population level is
  \begin{equation}
      \| \pi^t - \bar{\pi}^{\ast} \|_1 = \left| 1 - \frac{2}{\pi} \varphi^{t - 1}
      \right| \cdot \left\| \frac{1}{2} - \pi^{\ast} \right\|_1
  \end{equation}
where $\bar{\pi}^{\ast} \assign \frac{1}{2} - \tmop{sgn} (\rho^0) 
(\frac{1}{2} - \pi^{\ast}), \varphi^{t - 1} \assign \frac{\pi}{2} - \arccos |
\rho^{t - 1} |$ and $\rho^{t - 1} \assign \frac{\langle \theta^{t - 1},
\theta^{\ast} \rangle}{\| \theta^{t - 1} \| \cdot \| \theta^{\ast} \|}$,
$\rho^0 \assign \frac{\langle \theta^0, \theta^{\ast} \rangle}{\| \theta^0 \|
\cdot \| \theta^{\ast} \|}$.
\end{corollary}

Thus, at the population level, we investigate the evolution of the \textit{sub-optimality angle} $\varphi^t$ to characterize the convergence behavior of EM iterations. In Proposition~\ref{prop:quad} we show the linear convergence rate of $\tan \varphi^t$ when $\varphi^t$ is relatively small and demonstrate the quadratic convergence rate when  $\varphi^t$ is large enough. Hence, we can divide the EM updates into two stages: In the first stage, it takes $T'=\mathcal{O}(\frac{1}{\tan \varphi^0})=\mathcal{O}(\frac{1}{\sin \varphi^0})=\mathcal{O}(\frac{1}{|\rho^0|})$ iterations to ensure $\tan \varphi^{T'}$ is large enough. Subsequently, in the second stage, it takes the other $T''=\mathcal{O}(\log \log \frac{1}{\varepsilon})$ to guarantee $\phi^T=\mathcal{O}(\sqrt{\varepsilon})$. Hence, $\frac{\|\theta^{T + 1} - \tmop{sgn} (\rho^0) \theta^{\ast} \|}{\| \theta^{\ast} \|}= \frac{1}{\pi} \sqrt{\left(\phi^T-\sin \phi^T\right)^2+\left(1-\cos \phi^T\right)^2}\leq \frac{[\phi^T]^2}{2 \pi}=\mathcal{O}(\varepsilon)$.
% {\color{red} Todo: interpret the number of iterations,
% 1. log 1/rho: angle from  rho to arctan 1.5
% 2. log log 1/ eps: angle from arctan 1.5 to eps}

%%%%%%%%%%%%%%%%%%%%%%%%%%%%%%%
The next proposition derives a recurrence equation for the update of the sub-optimality angle $\varphi$.
\begin{proposition}{(Recurrence Relation)}\label{prop:recurrence}
  Assume the  initial sub-optimality cosine satisfies $\rho^0 \assign \frac{\langle \theta^0, \theta^{\ast} \rangle}{\|
  \theta^0 \| \cdot \| \theta^{\ast} \|} \neq \pm 1$, or equivalently $\varphi^0 \assign \frac{\pi}{2} -
  \arccos | \rho^0 | \in [0, \pi/2 )$. 
  Then the recurrence relation for EM updates at
  population level characterized by the sub-optimality angle is
  \begin{equation}
    \tan \varphi^t = \tan \varphi^{t - 1} + \varphi^{t - 1}  (\tan^2
    \varphi^{t - 1} + 1).
  \end{equation}
\end{proposition}

In the following Proposition, we show that the trajectory of iterations $\theta^t, \forall t \in \mathbb{N}_+$ is
on the cycloid ~\citep{harris1998handbook} on the plane $\tmop{span} \{ \theta^0, \theta^{\ast} \}$.
\begin{proposition}{(Cycloid Trajectory)}\label{prop:cycloid}
  If $\rho^0 \assign \frac{\langle \theta^0, \theta^{\ast} \rangle}{\|
  \theta^0 \| \cdot \| \theta^{\ast} \|} \neq \pm 1$, namely $\phi^0 \assign 2
  \arccos | \rho^0 | \in (0, \pi]$. Then the coordinates $\mathtt{x}^t,
  \mathtt{y}^t$ of normalized vector $\frac{\theta^t}{\| \theta^{\ast} \|} =
  \mathtt{x}^t \hat{e}_1 + \mathtt{y}^t \hat{e}_2^t = \mathtt{x}^t \hat{e}_1 +
  \mathtt{y}^t \hat{e}_2^0, \forall t \in \mathbb{N}_+$ for EM updates at
  the Population level can be parameterized with the angle $\phi^{t - 1} \assign 2
  \arccos | \rho^{t - 1} | \in (0, \pi]$ as follows, where $\rho^{t - 1}
  \assign \frac{\langle \theta^{t - 1}, \theta^{\ast} \rangle}{\| \theta^{t -
  1} \| \cdot \| \theta^{\ast} \|}$.
  \begin{eqnarray}
    1 - \tmop{sgn} (\rho^0) \mathtt{x}^t & = & \frac{1}{\pi} [\phi^{t - 1} -
    \sin \phi^{t - 1}] \nonumber\\
    \mathtt{y}^t & = & \frac{1}{\pi} [1 - \cos \phi^{t - 1}] 
  \end{eqnarray}
  Hence, the trajectory of iterations $\theta^t, \forall t \in \mathbb{N}_+$ is
  on the cycloid with a parameter $\frac{\| \theta^{\ast} \|}{\pi}$, on the
  plane $\tmop{span} \{ \theta^0, \theta^{\ast} \}$ (see Fig.~\ref{fig:cycloid}).
\end{proposition}
\begin{figure}[t]
  \centering
  \includegraphics[width=0.4\textwidth]{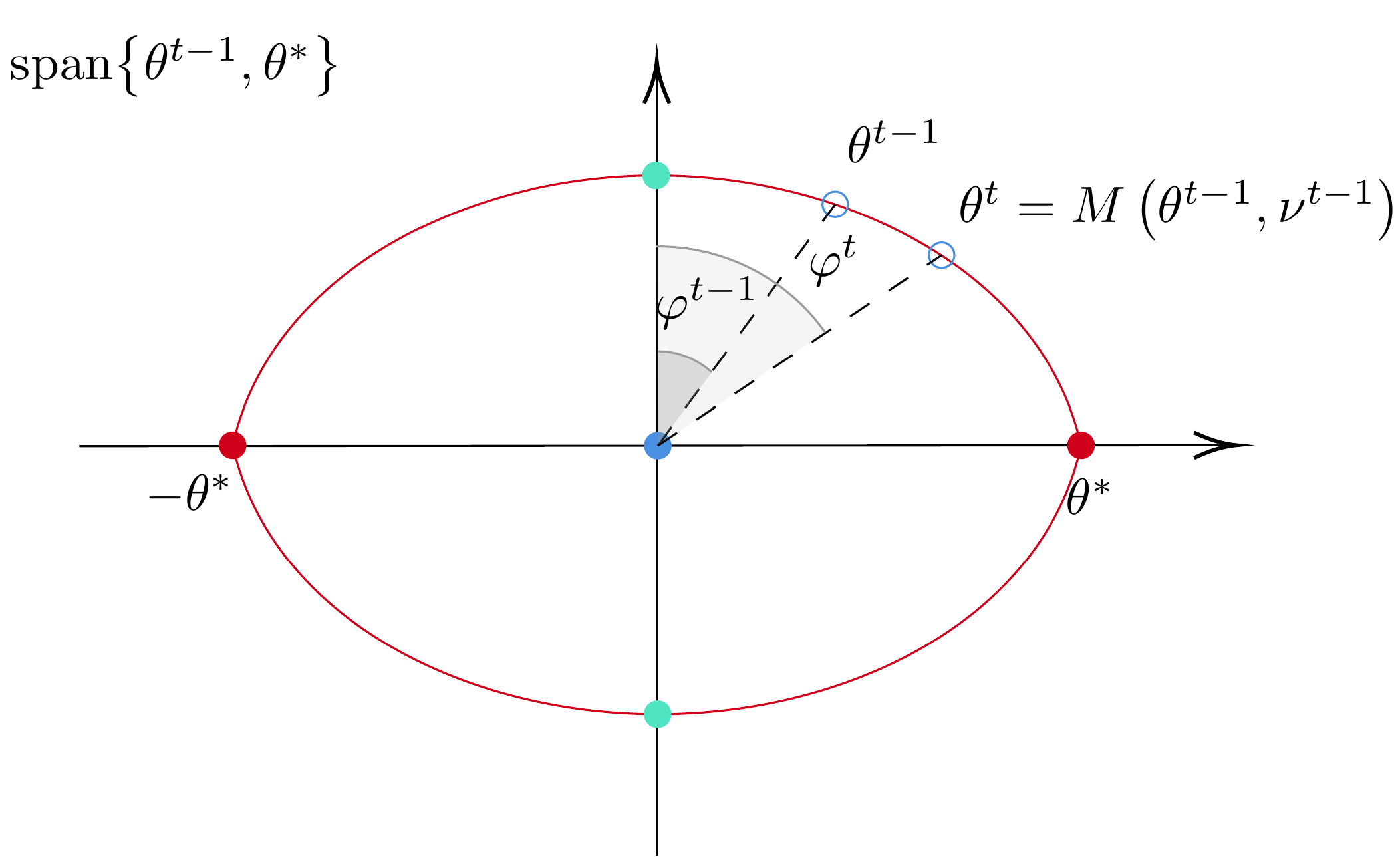}
  \caption{The cycloid trajectory for the EM update $M(\theta, \nu)$ of regression parameters $\theta$. The figure further shows the two global solutions (red dots), the unstable solution (blue dot), and the two saddle points (green dots). As long as the initial suboptimality angle is sufficiently large, $\varphi^t$ and in turn $\theta^t$ super-linearly converge to $\frac{\pi}{2}$ and $\theta^\ast$.}
  \label{fig:cycloid}
\end{figure}
% {\color{red} Todo: add a comment for 5 different stationary points}

\begin{proposition}{(Quadratic Convergence Rate)}\label{prop:quad}
  If $\varphi^0 \assign \frac{\pi}{2} - \arccos | \rho^0 | \in \left( 0,
  \frac{\pi}{2} \right)$, then the EM updates at population level satisfies
  \begin{equation}
    \tan \varphi^t \geq \frac{1 + \sqrt{5}}{2} \cdot \tan \varphi^{t - 1}.
  \end{equation}
  Particularly, if $\varphi^{t-1} \assign \frac{\pi}{2} - \arccos | \rho^{t-1} | \in [ \arctan
  1.5, \frac{\pi}{2} )$, then the EM updates at the Population level
  satisfies
  \begin{equation}\label{eq:quad}
    \frac{\pi}{2} \left( \tan \varphi^t - \frac{\pi}{4} \right) \geq \left\{
    \frac{\pi}{2} \left( \tan \varphi^{t - 1} - \frac{\pi}{4} \right)
    \right\}^2.
  \end{equation}
\end{proposition}
Proposition~\ref{prop:quad} states that as long as the initial sub-optimality angle is nonzero, the angle increases linearly which implies convergence to global optima. 
Note that if $\varphi^0 =0$, then the update may converge to one of the saddle points (see Fig.~\ref{fig:cycloid} for a visualization). 
Let $\phi = 2 (\frac{\pi}{2}- \varphi)$ be small enough, then $\tan(\varphi)=\cot(\frac{\phi}{2})\approx \frac{2}{\phi}$, and by Proposition~\ref{prop:quad}, $\frac{\pi}{2} \tan(\varphi^{t}) \approx [\frac{\pi}{2} \tan(\varphi^{t-1})]^2$, so we show that $\phi^{t}/\pi \approx [\phi^{t-1}/\pi]^2$, demonstrating quadratic convergence in the angle $\phi$ (see Appendix~\ref{sup:finite_sample}, proof of Proposition~\ref{prop:convg_angle}). 
Further, note that at the population level, \citet{kwon2021minimax} (Page 5, Lemma 1) established the quadratic convergence when $C \sqrt{\log \left\|\bar{\theta}^*\right\|}/\left\|\bar{\theta}^*\right\| \leq \frac{\left\|\theta-\theta^*\right\|}{\|\theta^\ast\|} \leq 1 / 10$ for high SNR. Proposition~\ref{prop:quad} then extends the region of quadratic convergence from 1/10 to  $\frac{1}{\pi}\sqrt{(\phi-\sin \phi)^2 + (1-\cos \phi)^2}_{\phi=2(\frac{\pi}{2}-\arctan 1.5)} \approx 0.22$ in the noiseless setting.
% {\color{red} Todo: add comparison results for superlinear convergence:
% lemma 1 in Kwon 2021, (no restriction near the the true parameters)
% AM in Ghosh 2021, (no restriction for the true mixing weights)}

%% file: 4_finite_sample.tex
In this section, we give a finite-sample analysis by coupling the population EM with the corresponding finite-sample EM in the noiseless setting.

\begin{theorem}{(Convergence at Finite-sample Level)}\label{thm:convg_finite}
  In the noiseless setting, suppose any initial mixing weights $\pi^0$ and any initial regression parameters $\theta^0 \in \mathbb{R}^d$ ensuring that $\varphi^0 \geq \Theta \left( \sqrt{\frac{\log
  \frac{1}{\delta}}{n}} \vee \frac{\log \frac{1}{\delta}}{n} \right)$. If we
  run finite-sample Easy EM for at most $T_1=\mathcal{O}\left( \log
  \frac{n}{\log \frac{1}{\delta}}\right)$ iterations followed by the finite-sample standard EM for at most $T' =\mathcal{O} \left(
  \log \frac{n}{d} 
    \wedge \log \frac{n}{\log \frac{1}{\delta}} \right)$
  iterations with all the same $n =
  \Omega \left( d \vee \log \frac{1}{\delta} 
  %\vee \frac{\log^2\frac{1}{\delta}}{d} 
  \right)$ samples, then we have
  \ifdefined\isarxiv
    \normalsize
  \else
    \scriptsize
  \fi  
  \begin{equation*}
      \begin{aligned}
         &\frac{\| \theta^{T + 1} - \mathrm{sgn}(\rho^{T+1}) \theta^{\ast} \|}{\| \theta^{\ast} \|} =\mathcal{O}
    \left( \sqrt{\frac{d}{n}} \vee \frac{\log \frac{1}{\delta}}{n}
    \vee \sqrt{\frac{\log \frac{1}{\delta}}{n}} \right),\\
    &\| \pi^{T + 1} - \bar{\pi}^{\ast} \|_1 = \left\| \frac{1}{2} - \pi^{\ast}
    \right\|_1
     \mathcal{O} \left(  \sqrt{\frac{d}{n}} \vee \frac{\log
    \frac{1}{\delta}}{n} \vee \sqrt{\frac{\log \frac{1}{\delta}}{n}}
    \right)\\
    &\quad+ c (\pi^{\ast}) \cdot \mathcal{O} \left( \sqrt{\frac{\log
    \frac{1}{\delta}}{n}} \right),
      \end{aligned}
  \end{equation*}
  % \begin{equation}
  %   \frac{\| \theta^{T + 1} - \mathrm{sgn}(\rho^{T+1}) \theta^{\ast} \|}{\| \theta^{\ast} \|} =\mathcal{O}
  %   \left( \sqrt{\frac{d}{n}} \vee \frac{\log \frac{1}{\delta}}{n}
  %   \vee \sqrt{\frac{\log \frac{1}{\delta}}{n}} \right)
  % \end{equation}
  % and
  % \begin{eqnarray}
  %   \| \pi^{T + 1} - \bar{\pi}^{\ast} \|_1 &=& \left\| \frac{1}{2} - \pi^{\ast}
  %   \right\|_1 \nonumber\\
  %   &\cdot& \mathcal{O} \left(  \sqrt{\frac{d}{n}} \vee \frac{\log
  %   \frac{1}{\delta}}{n} \vee \sqrt{\frac{\log \frac{1}{\delta}}{n}}
  %   \right) \nonumber\\
  %   &+& c (\pi^{\ast}) \cdot \mathcal{O} \left( \sqrt{\frac{\log
  %   \frac{1}{\delta}}{n}} \right)
  % \end{eqnarray}
  \normalsize
  with probability at least $1 - T\delta$, where $T:=T_1+T',\varphi^0 \assign
  \frac{\pi}{2} - \arccos \left| \frac{\langle \theta^0, \theta^{\ast}
  \rangle}{\| \theta^0 \| \cdot \| \theta^{\ast} \|} \right|, 
  \rho^{T+1}\assign \frac{\langle \theta^{T+1}, \theta^{\ast}
  \rangle}{\| \theta^{T+1} \| \cdot \| \theta^{\ast} \|},
  \bar{\pi}^{\ast} \assign \frac{1}{2} - \tmop{sgn} (\rho^0) 
(\frac{1}{2} - \pi^{\ast})$, and the
  coefficient \ $c (\pi^{\ast}) =\mathcal{O} (1)$, especially $c (\pi^{\ast})
  = 0$ when $\pi^{\ast} = \{1, 0\}$ or $\{0, 1\}$.
\end{theorem}

%% dicussion
We note that \citet{ghosh20a} showed the quadratic convergence at the finite-sample level, only when the error of regression parameters $\lesssim \min \left\{\pi^\ast(1), \pi^\ast(2)\right\}\left\|\theta_1^*-\theta_2^*\right\|$ for a variant of EM in the noiseless setting. Theorem~\ref{thm:convg_finite} on the other hand removes the restriction of mixing weights and still obtains the quadratic convergence rate. 
Our results for convergence rate and statistical error hold when $\frac{d}{n}$ is not greater than some constant, which is less than 1 (see Appendix~\ref{sup:finite_sample}, proof of Proposition~\ref{prop:convg_angle}). However, if $\frac{d}{n}>1$, recovering regression parameters $\theta\in\mathbb{R}^d$ from a reduced number $n$ of measurements invalidates the EM update rules for regression parameters at the finite-sample level (see Eq.~\eqref{eq:finite}), as $\text{rank}(\frac{1}{n} \sum_{i=1}^n x_i x_i^\top)\leq n< d$. Additional restrictions on $\theta$, such as assuming some components of $\theta$ are zeros~\citep{barik2022sparse}, are necessary.

\textbf{Sketch Proof of Theorem~\ref{thm:convg_finite}}
The first step is to upper-bound the statistical error of the finite-sample EM. In the following Proposition~\ref{prop:proj_err}, Proposition~\ref{prop:stat_err}, we give the bounds for the projected statistical error $\mathcal{O}(\sqrt{\frac{\log \frac{1}{\delta}}{n}} \vee \frac{\log \frac{1}{\delta}}{n})$ and the total statistical error $\mathcal{O} ( \sqrt{\frac{d}{n}} \vee \frac{\log\frac{1}{\delta}}{n} \vee \sqrt{\frac{\log \frac{1}{\delta}}{n}})$.
\vspace{1em}
\begin{proposition}{(Projected Statistical Error)}\label{prop:proj_err}
    In the noiseless setting, the projection on $\text{span}\{\theta,\theta^\ast\}$ for the statistical error of $\theta$ satisfies
      \small
    \begin{equation}\nonumber
      \frac{\|P_{\theta,\theta^\ast}
      [M^{\tmop{easy}}_n (\theta, \nu) - M (\theta, \nu)]\|}{\|\theta^\ast\|}
      = \mathcal{O}\left(\sqrt{\frac{\log \frac{1}{\delta}}{n}} \vee \frac{\log \frac{1}{\delta}}{n}\right),
    \end{equation}
    \normalsize
    with probability at least $1 - \delta$, where $M_n (\theta, \nu), M (\theta,
    \nu)$ are the EM update rules for $\theta$ at the Finite-sample level and the
    population level respectively, and the orthogonal projection matrix $P_{\theta,\theta^\ast}$ satisfies
    $\text{span}(P_{\theta,\theta^\ast})=\text{span}\{\theta,\theta^\ast\}$.
\end{proposition}

\begin{proposition}{(Statistical Error)}\label{prop:stat_err}
  In the noiseless setting, the statistical error of $\theta$ for EM
  updates at the Finite-sample level satisfies
  \small
  \begin{equation}\nonumber
    \frac{\| M_n (\theta, \nu) - M (\theta, \nu) \|_2}{\| \theta^{\ast} \|}
    =\mathcal{O} \left( \sqrt{\frac{d}{n}} \vee \frac{\log
    \frac{1}{\delta}}{n} \vee \sqrt{\frac{\log \frac{1}{\delta}}{n}}
    \right),
  \end{equation}
  \normalsize
  with probability at least $1 - \delta$, $M_n (\theta, \nu), M (\theta,
  \nu)$ denote the EM update rules for $\theta$ at the Finite-sample level and the
  Population level.
\end{proposition}

\citet{kwon2021minimax} (Appendix E, page 17, Lemma 11) and \citet{balakrishnan2017statistical} upper-bounded $\mathcal{O}(\sqrt{\frac{d}{n}} \log \frac{n}{\delta})$ for the statistical error, while our finer analysis presents a tighter bound $\mathcal{O} ( \sqrt{\frac{d}{n}} \vee \frac{\log
    \frac{1}{\delta}}{n} \vee \sqrt{\frac{\log \frac{1}{\delta}}{n}}
    ) \lesssim \mathcal{O}(\sqrt{\frac{d}{n}} \log \frac{n}{\delta})$.
The difference arises from the techniques used for bounding the statistical error in the $\ell_2$ norm. We obtain our additive log factor by leveraging the rotational invariance of Gaussians, rewriting the $\ell_2$ norm of the error as the geometric mean of two Chi-square distributions (see Appendix~\ref{sup:finite_sample}, proof of Proposition~\ref{prop:stat_err}). 
The multiplicative log factor is achieved using the standard symmetrization technique and the Ledoux-Talagrand contraction argument (see Appendix E, proof of Lemma 11, page 17 of~\citep{kwon2021minimax}).

Proposition~\ref{prop:init} shows in the first step when using the initialization for $\theta^0,\pi^0$ the angle $\varphi^t$ is larger than the projected statistical error, if we run at most $\mathcal{O}(\log \frac{1}{\delta})$ iterations of  Easy EM  with at most $\Theta(\frac{n}{\log \frac{1}{\delta}})$ fresh samples per iteration. 

\begin{proposition}{(Initialization with Easy EM)}\label{prop:init}
  In the noiseless setting, suppose we run the sample-splitting finite-sample
  Easy EM with \ $n' \assign \Theta \left( \frac{n}{\log \frac{1}{\delta}}
  \wedge \left[ \frac{n}{\log \frac{1}{\delta}} \right]^2 \right)$ fresh
  samples for each iteration, then after at most $T_0 =\mathcal{O} \left( \log
  \frac{1}{\delta} \right)$ iterations, it satisfies $\varphi^{T_0} \geq
  \Theta \left( \sqrt{\frac{\log \frac{1}{\delta}}{n}} \vee \frac{\log
  \frac{1}{\delta}}{n} \right)$ with probability at least $1 - \delta$.
\end{proposition}

This result allows us to divide the analysis of the convergence rate into three stages.
In the first stage, we show that after at most 
$T_1=\mathcal{O}(\log \frac{\text{statistical error}}{\text{projected statistical error}})=\mathcal{O}(\log \frac{n}{\log \frac{1}{\delta}})$ iterations of easy EM, we can ensure the angle $\varphi^t$ is larger than the statistical error.
Then, in the second stage, by using the linear convergence rate for $\varphi^t$ established in Proposition~\ref{prop:quad}, we show that the angle $\varphi^t$ would be larger than $\arctan(1.5)$ after at most $T_2=\mathcal{O}(\log \frac{1}{\text{statistical error}}) = \mathcal{O}(\log \frac{n}{d} \wedge \log \frac{n}{\log \frac{1}{\delta}})$ iterations of standard EM.
In the third stage, with the quadratic convergence rate for $\varphi^t \geq \arctan(1.5)$ in Proposition~\ref{prop:quad}, we further show the distance between $\theta^t$ and $\theta^\ast$ will decrease with the quadratic speed until the distance reaches the statistical error (using Proposition~\ref{prop:convg_angle}). 
Hence, with at most $T_3=\mathcal{O}(\log\log \frac{1}{\text{statistical error}})=\mathcal{O}(\log [\log \frac{n}{d} \wedge \log \frac{n}{\log \frac{1}{\delta}}])$ iterations of standard EM, the error of regression parameters reaches the statistical error.
%\newpage
\begin{proposition}{(Convergence of Angle)}\label{prop:convg_angle}
  In the noiseless setting, suppose $\varphi^0 \geq \Theta \left(
  \sqrt{\frac{\log \frac{1}{\delta}}{n}} \vee \frac{\log \frac{1}{\delta}}{n}
  \right)$, run finite-sample Easy EM for $T_1=\mathcal{O}\left( \log
  \frac{n}{\log \frac{1}{\delta}}\right)$ iterations followed by the finite-sample standard EM for at most $T' =\mathcal{O} \left(
  \log \frac{n}{d} 
    \wedge \log \frac{n}{\log \frac{1}{\delta}} \right)$
  iterations with all the same $n =
  \Omega \left( d \vee \log \frac{1}{\delta} 
  %\vee \frac{\log^2\frac{1}{\delta}}{d} 
  \right)$ samples, then it satisfies
  \begin{equation}
    \varphi^T \geq \frac{\pi}{2} - \Theta \left( \sqrt{\frac{d}{n}} \vee
    \frac{\log \frac{1}{\delta}}{n} \vee \sqrt{\frac{\log
    \frac{1}{\delta}}{n}} \right),
  \end{equation}
  with probability at least $1 - T \delta$, where $T=T_1+T',\varphi^0 \assign
  \frac{\pi}{2} - \arccos \left| \frac{\langle \theta^0, \theta^{\ast}
  \rangle}{\| \theta^0 \| \cdot \| \theta^{\ast} \|} \right|$ and $\varphi^T
  \assign \frac{\pi}{2} - \arccos \left| \frac{\langle \theta^T, \theta^{\ast}
  \rangle}{\| \theta^T \| \cdot \| \theta^{\ast} \|} \right|$.
\end{proposition}
%{\color{red} Todo: compare T, stat error with Kwon 2019, 2021, Ghosh 2021}

Upon taking all these three stages into account, the total number of iterations until convergence with the initialization (i.e., $\varphi_0\geq$ the projected statistical error) is $T=T_1 +T_2 +T_3 = \mathcal{O} ( \log \frac{n}{\log \frac{1}{\delta}} )$. 
In particular, with a good initialization (i.e., $\varphi_0\geq$ the statistical error), the total number of iterations is $T'=T_2+T_3=\mathcal{O} (\log \frac{n}{d}  \wedge \log \frac{n}{\log \frac{1}{\delta}} )$.

For the error of the mixing weights, we first establish the upper bound for the error between the population EM update and the finite-sample EM update for mixing weights $\|N_n(\theta, \nu) - N(\theta, \nu)\|_1=c (\pi^{\ast}) \cdot \mathcal{O} ( \sqrt{\frac{\log\frac{1}{\delta}}{n}} )$. 
We establish this result by estimating the Chernoff bound (the full proof of Theorem~\ref{thm:convg_finite} is described in the supplementary materials, Appendix~\ref{sup:finite_sample}). The final error is obtained by summing up the error stemming from the population EM update (Corollary~\ref{cor:err_mixing}) and the error between the population and finite-sample EM updates.

%% file: 5_experiments.tex
\begin{figure*}[t]
  \centering
  \begin{subfigure}[!htbp]{0.31\textwidth}
      \centering
      \includegraphics[width=\textwidth]{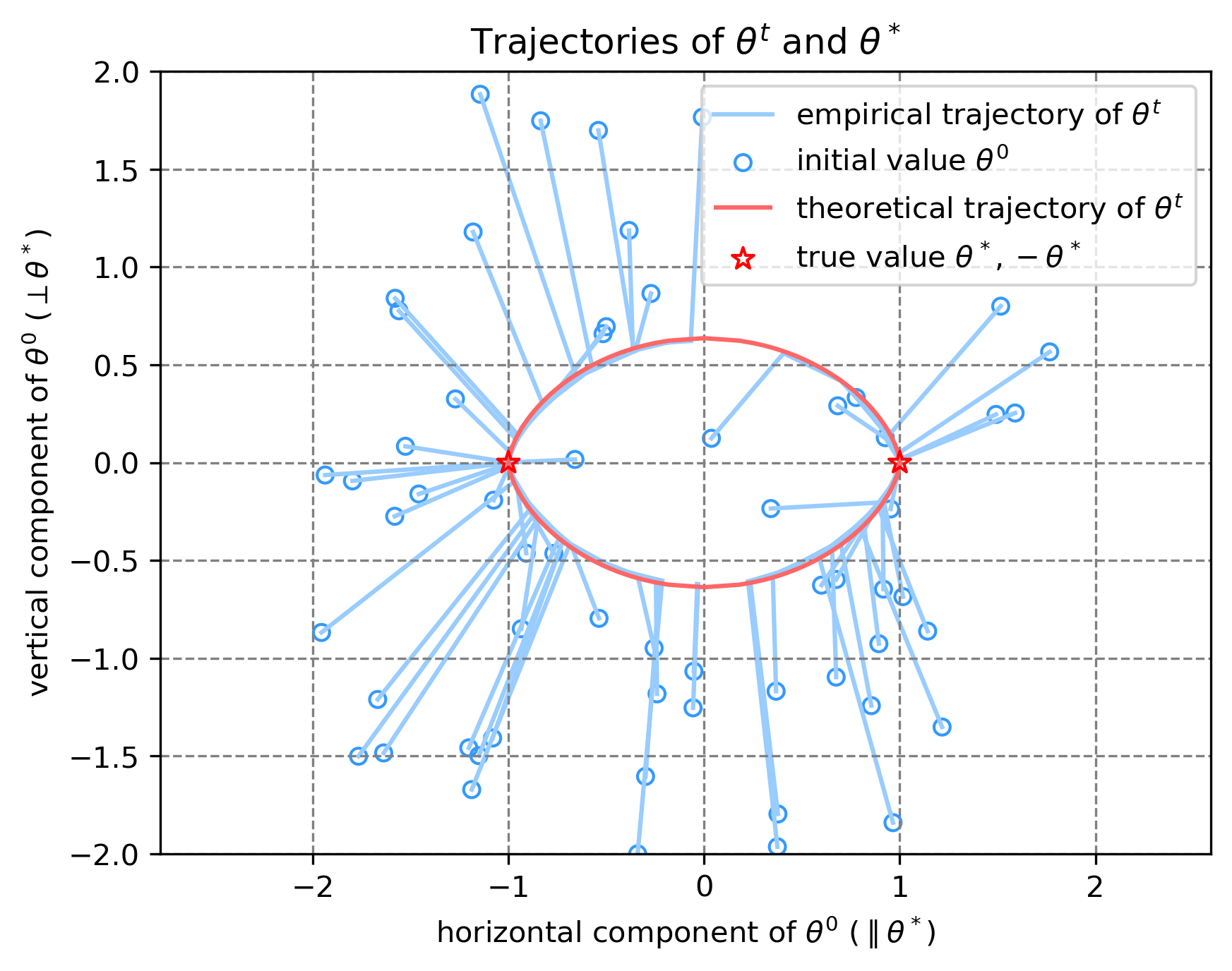}
      \caption{$d=2$, $\theta^\ast=[1, 0]$, $\pi^\ast=[0.7, 0.3]$, 
      trajectories of $\theta^t$ for 60 trials with $\theta^0$ and $\pi^0$ uniformly sampled from $[-2, 2]^2$ and $[0, 1]$, respectively.}
      \label{fig:traj_d2}
  \end{subfigure}
  \hfill
  \begin{subfigure}[!htbp]{0.31\textwidth}
      \centering
      \includegraphics[width=0.75\textwidth,height=0.75\textwidth]{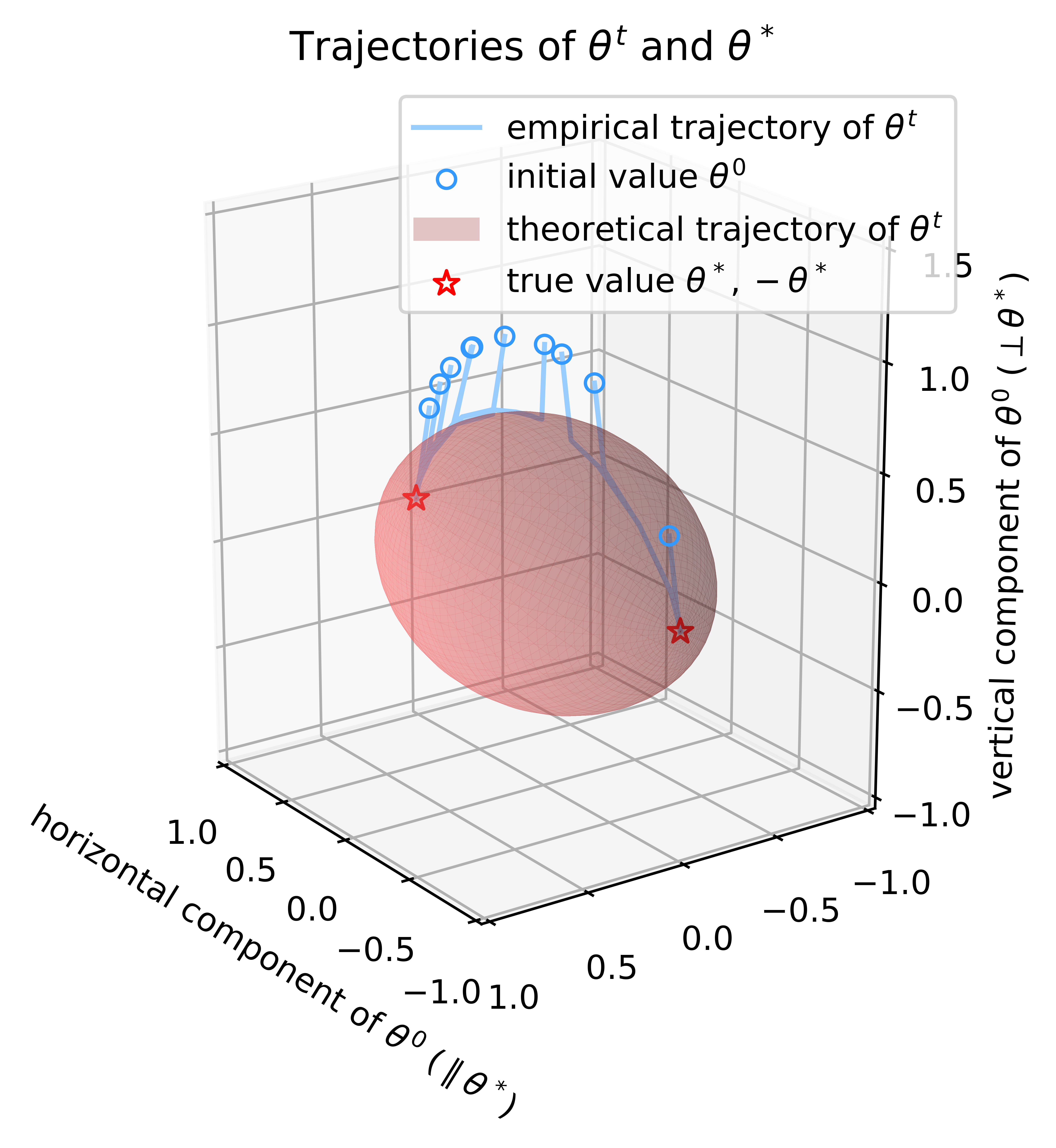}
      \caption{$d=3$, trajectories of $\theta^t$ for 10 trials, with $\theta^\ast, \theta^0$ sampled from three-dimensional unit sphere, $\pi^\ast, \pi^0$ drawn uniformly from $[0, 1]$.}
      \label{fig:traj_d3}
  \end{subfigure}
  \hfill
  \begin{subfigure}[!htbp]{0.31\textwidth}
      \centering
      \includegraphics[width=\textwidth]{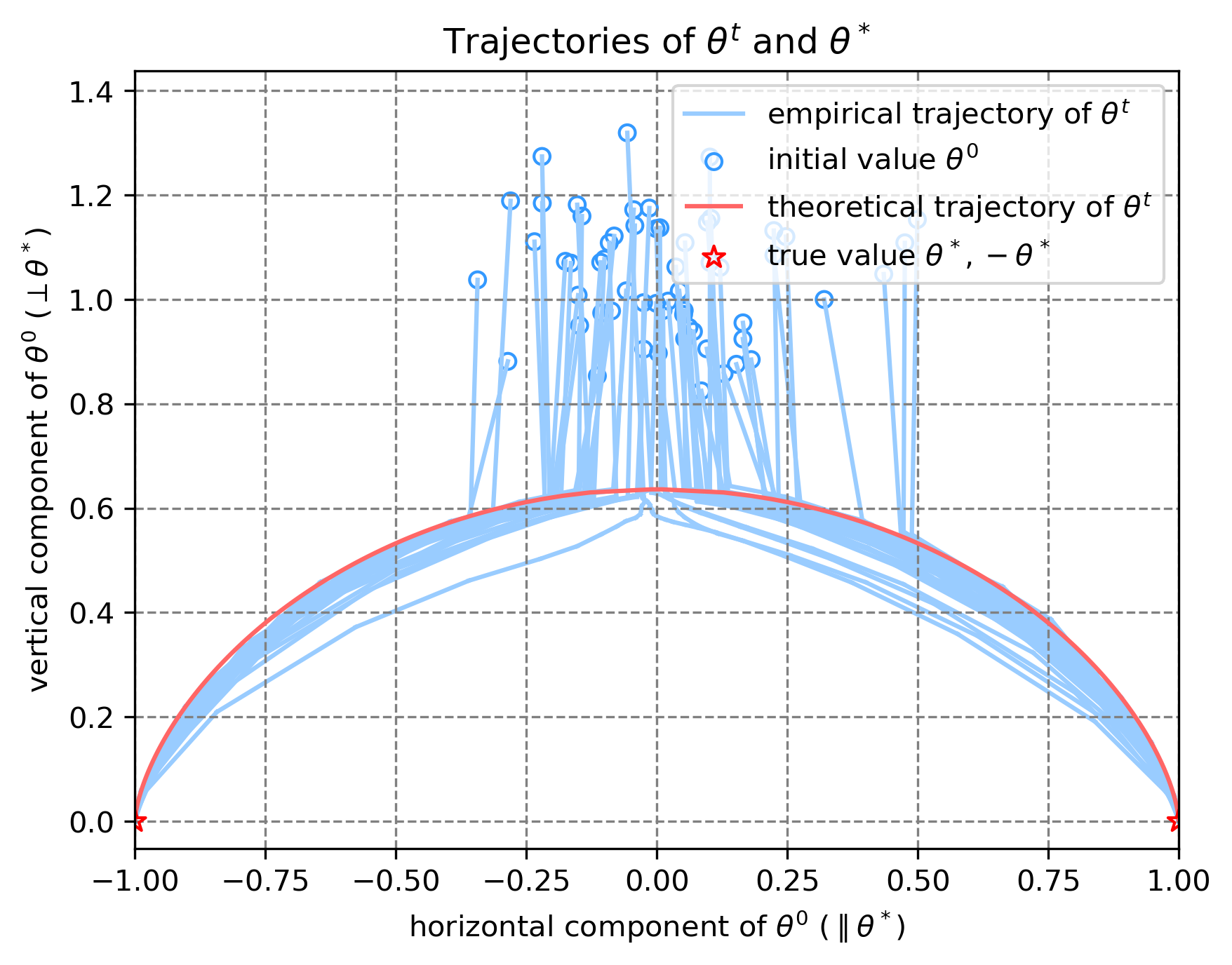}
      \caption{$d=50$, trajectories of $\theta^t$ are displayed across 60 trials, with $\theta^\ast, \theta^0$ sampled from $\mathcal{N}(0, I_d)$, $\pi^\ast, \pi^0$ drawn uniformly from $[0, 1]$.}
      \label{fig:traj_dhigh}
  \end{subfigure}
\caption{Cycloid trajectory of EM iterations $\theta^t$-- 
we perform 100 iterations of Finite-sample EM with SNR=$10^8$, varying dimensions ($d=2,3,50$).
}
\label{fig:traj}
\end{figure*}

\begin{figure*}[t]
  \centering
  \begin{subfigure}[!htbp]{0.31\textwidth}
      \centering
      \includegraphics[width=\textwidth,height=0.8\textwidth]{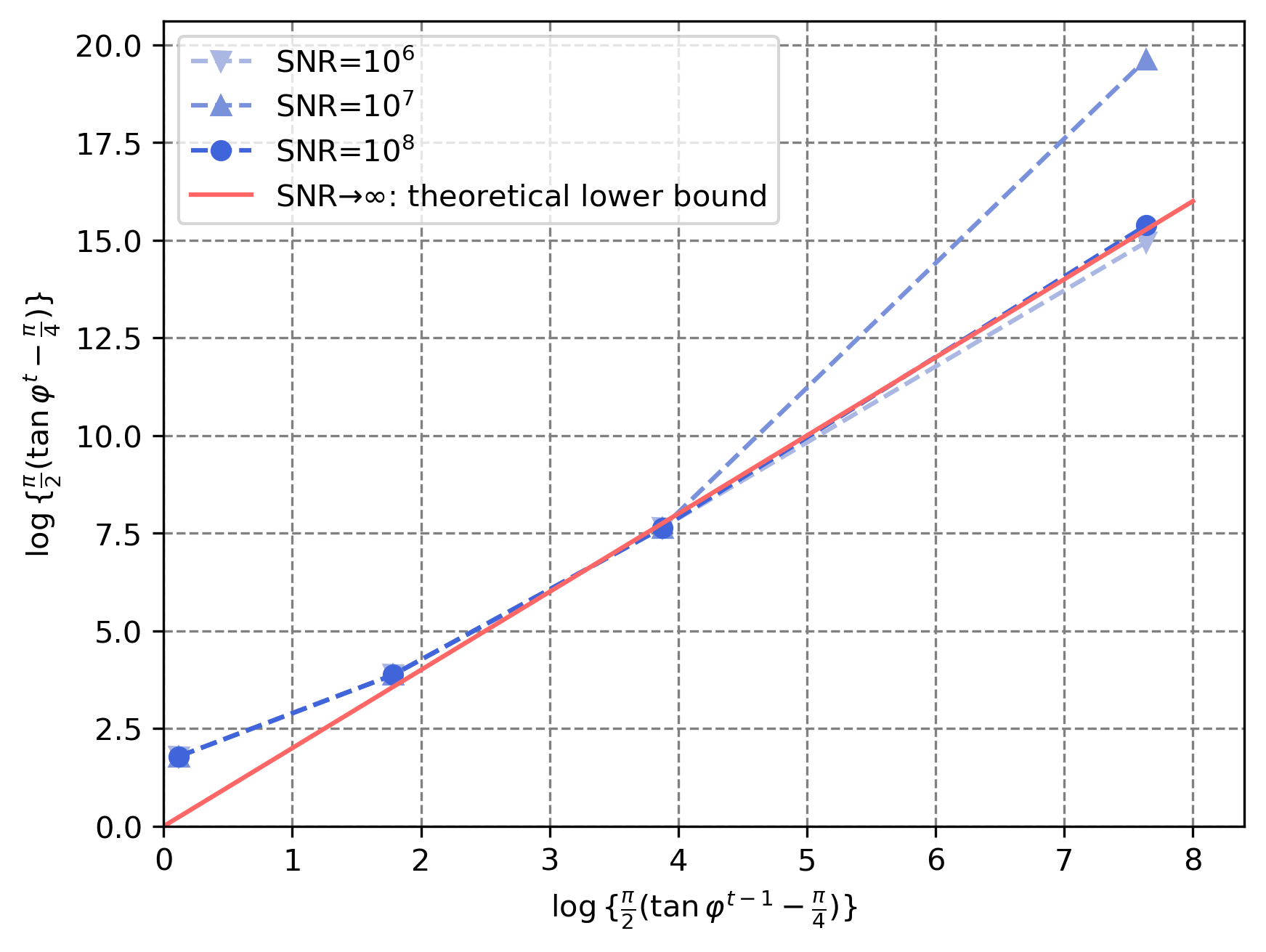}
      \caption{Quadratic convergence of $\frac{\pi}{2}(\tan \varphi^t -\frac{\pi}{4})$ with $\pi^\ast, \pi_0\sim[0, 1]$.} 
      \label{fig:superlinear}
  \end{subfigure}
  \hfill
  \begin{subfigure}[!htbp]{0.31\textwidth}
      \centering
      \includegraphics[width=\textwidth,height=0.8\textwidth]{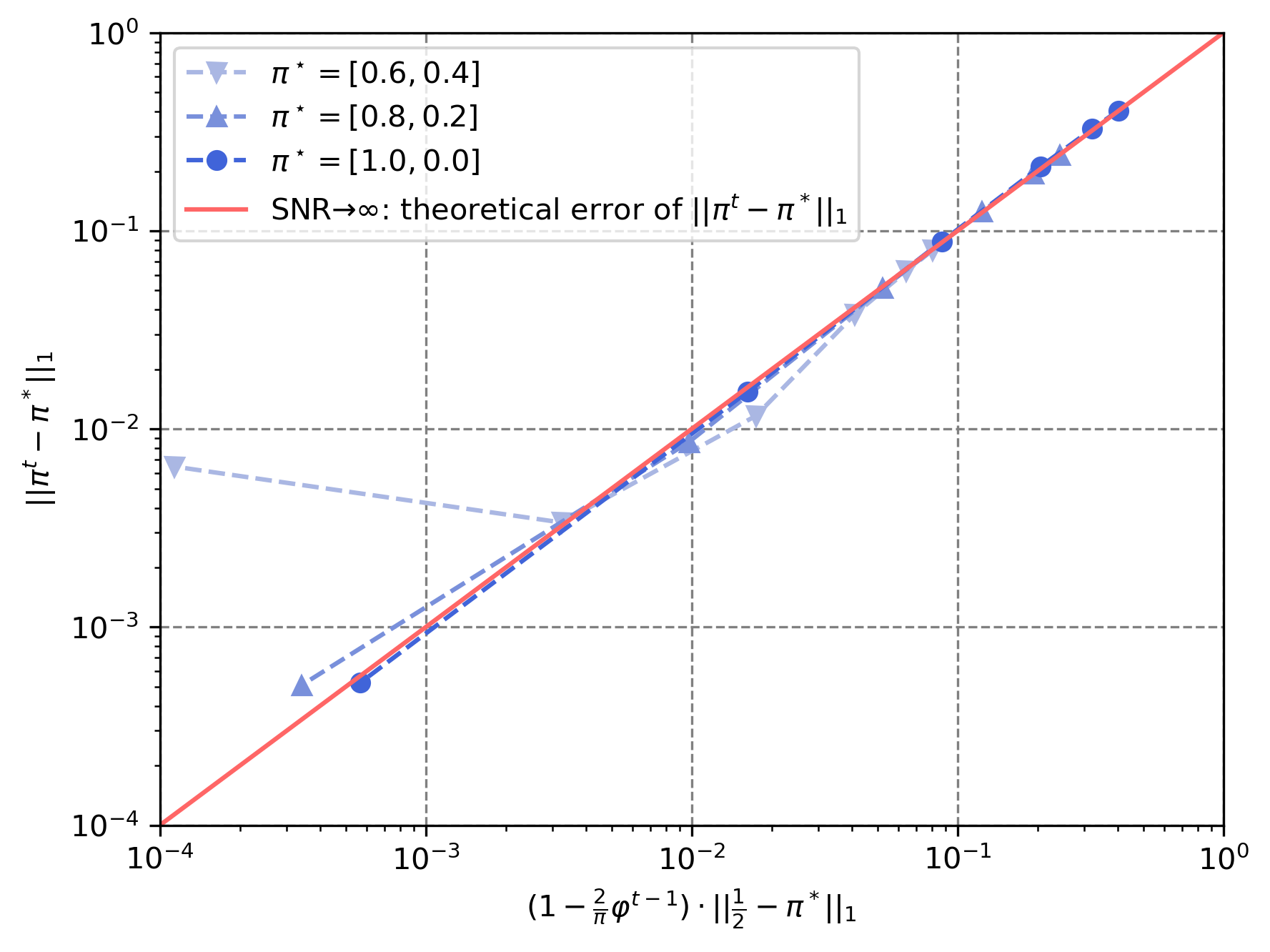}
      \caption{Correlation between $\|\pi^t-\bar{\pi}^*\|_1$ 
      and $\frac{\pi}{2}-\varphi^{t-1}=\arccos \left|\frac{\langle\theta^{t-1}, \theta^\ast \rangle}{\|\theta^{t-1}\|\|\theta^\ast\|}\right|$.}
      \label{fig:mixing}
  \end{subfigure}
  \begin{subfigure}[!htbp]{0.31\textwidth}
      \centering
      \includegraphics[width=\textwidth,height=0.8\textwidth]{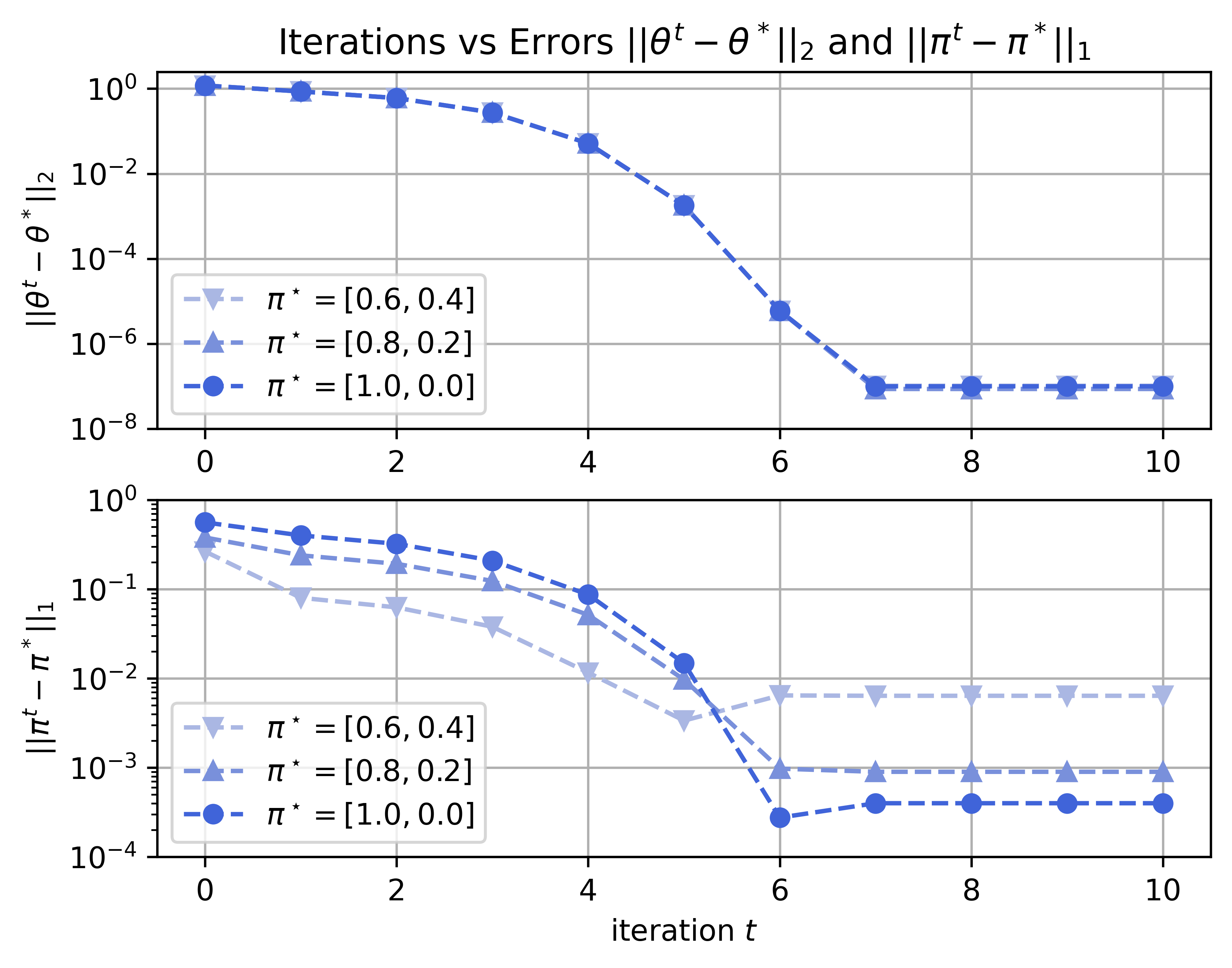}
      \caption{Estimation error of EM v.s. iteration for various mixing weights.}
      \label{fig:dist}
  \end{subfigure}
\caption{Left and Middle: Quadratic convergence and correlation are shown with $\theta^\ast, \theta^0$ from $d=50$ unit sphere, 
s.t. $\varphi^0 = \arctan(1.5)$ in Panel (a), $\varphi^0 = 0.3$ in Panel (b). Right: The errors of regression parameters and mixing weights for ten EM iterations, with $d=50, \varphi^0=0.3$, SNR=$10^8$ and different true mixing weights $\pi^\ast=\{0.6, 0.4\},\{0.8, 0.2\}, \{1, 0\}$.}
\label{fig:convg_mix}
\end{figure*}

In this section of empirical experiments, we validate the theoretical findings established in the preceding sections.
From a normal distribution $\mathcal{N}(0, I_d)$, 
we sample 5,000 independent and identically distributed (i.i.d.) 
$d$-dimensional covariates, denoted as $\{x_i\}_{i=1}^n$. 
The true parameters $\theta^\ast$ are randomly chosen from a $d$-dimensional unit sphere.
We subsequently manually/randomly set the true mixture weights $\pi^\ast$ for two components,
utilizing them to generate latent variable samples $\{z_i\}_{i=1}^n$ from a categorical distribution $\mathcal{CAT}(\pi^\ast)$. 
Following this, we introduce Gaussian noise to the linear regression determined by these latent variables,
yielding output samples $\{y_i\}_{i=1}^n$.
In all experiments, we utilize the entire dataset for EM updates at every iteration. 
Each point on the plots of Fig.~\ref{fig:convg_mix} is an average taken from 50 trials with different initial values for EM updates.
The code for numerical experiments is available at \url{https://github.com/dassein/cycloid_em_mlr}.

\textbf{Cycloid Trajectory of EM Iterations. } 
At the population level, we show that the output of the $t$-th iteration lies on the cycloid of the spanning space $\text{span}\{\theta^{t-1}, \theta^\ast\}$ in the noiseless setting. 
In the corresponding experiments, we choose the signal-to-noise ratio SNR$=10^8$ and consider different values of $d$ (2, 3, and 50). In Fig.~\ref{fig:traj}, all the iterations are near the theoretical cycloid. 
Thus, our experimental results validate our theoretical analysis in Proposition~\ref{prop:cycloid}.

\textbf{Quadratic Convergence for 2 Mixtures. }
We show the super-linear convergence of $\frac{\pi}{2}(\tan \varphi^t -\frac{\pi}{4})$ in Fig.~\ref{fig:superlinear} under high SNR regimes.
We specified the dimension ($d$=50) and considered different high SNR values (SNR=$10^6, 10^7, 10^8$). 
We uniformly choose the initial values for the parameters and the mixing weights from a unit sphere and the interval $[0, 1]$, respectively. 
All the points of 4 EM iterations in Fig.~\ref{fig:superlinear} are the average of 50 trials with different initial values. 
The slope of the plot indicates the convergence rate exponent. Notably, the slopes of lines at different SNR values consistently hover around or slightly exceed 2. 
That aligns with our theoretical result of quadratic convergence rate in Proposition~\ref{prop:quad}. 

\textbf{Error of Mixing Weights and Angle. }
In the noiseless setting, we prove that the error of mixing weights $\|\pi^t-\bar{\pi}^*\|_1$ at the Population level is proportional to 
$\frac{\pi}{2}-\varphi^{t-1}=\arccos \left|\frac{\langle\theta^{t-1}, \theta^\ast \rangle}{\|\theta^{t-1}\| \cdot \|\theta^\ast\|}\right|$ in Corollary~\ref{cor:err_mixing}.
The angle $\frac{\pi}{2}-\varphi^{t-1}$ is determined by the output of EM updates at the $(t-1)$-th iteration and the true value of parameters $\theta^\ast$. 
We demonstrate the linear correlation between the error of mixing weights and the angle in Fig.~\ref{fig:mixing}.
For the setting of experiments, we specify the dimension $d=50$, and consider different high SNR values (SNR=$10^6, 10^7, 10^8$), respectively. 
We note that the error in the mixing weights during the $t$-th iteration is precisely quantified by the angle $\frac{\pi}{2}-\varphi^{t-1}$ at the preceding iteration.
Hence, our experimental results validate Corollary~\ref{cor:err_mixing}.

\textbf{Comparison with Different Mixing Weights. }
In the noiseless setting, we establish in Corollary~\ref{cor:noiseless} and Proposition~\ref{prop:recurrence} that the EM update for regression parameters $\theta^t$ is independent of the true mixing weights $\pi^\ast$.
The first subplot of Fig.~\ref{fig:dist} demonstrates that, at high SNR ($10^8$), the error in regression parameters (measured in $\ell_2$ norm) remains nearly constant for varying true mixing weights $\pi^\ast=\{0.6, 0.4\}$, $\{0.8, 0.2\}$, and $\{1, 0\}$, thus affirming our theoretical analysis.

In Theorem~\ref{thm:convg_finite}, we prove that the final error (in $\ell_1$ norm) in mixing weights depends on the error in regression parameters and true mixing weights.
Specifically, when the error in regression parameters is relatively small, the closer the true mixing weights are to $\{1, 0\}$ or $\{0, 1\}$, the smaller the final mixing weight error.
To validate our theoretical analysis in Theorem~\ref{thm:convg_finite} concerning the statistical errors in regression parameters $\theta$ (measured in $\ell_2$ norm) and mixing weights $\pi$ (measured in $\ell_1$ norm), we experiment with various true mixing weights $\pi^\ast = \{0.6, 0.4\}$, $\{0.8, 0.2\}$, and $\{1, 0\}$. 
The second subplot of Fig.~\ref{fig:dist} illustrates the relationship between the errors and the true mixing weights $\pi^\ast$, further supports our theoretical analysis.
In our experimental setup, $\theta^\ast, \theta^0$ are sampled from $50$-dimensional unit sphere, with $\varphi^0=0.3$.

%% file: 6_conclusion.tex
% Summary of contributions
We derived closed-form expressions for the EM updates in the 2MLR problem. 
Notably, in the noiseless setting  we first showed and then analyzed the cycloid trajectory of EM updates. 
Additionally, we demonstrated the quadratic convergence rate for regression parameters, which is independent of mixing weights. 
We emphasized that errors in mixing weights primarily arise from the angle formed between true and estimated regression parameters. 
Finally, we conducted a detailed analysis of the statistical errors in the estimation of regression parameters and mixing weights.
We investigate the special case of the noiseless setting, namely when SNR tends to infinity. 
Could we propose a more refined analysis using the recurrence relations outlined in Corollary~\ref{cor:no_separa} for weakly separated scenarios? 
These questions could guide our potential future endeavors.

%% file: 6.1_acknowledgements.tex
% All acknowledgments go at the end of the paper, including thanks to reviewers who gave useful comments, to colleagues who contributed to the ideas, and to funding agencies and corporate sponsors that provided financial support. 
% To preserve the anonymity, please include acknowledgments \emph{only} in the camera-ready papers.
We are grateful to the ICML area chair and anonymous reviewers for their insightful input that improved this paper.

%% file: 7_impact_statement.tex
There are several potential applications of our theory in Mixed Linear Regression: 
\textbf{Analysis of Over-specified Model}: 
Corollary~\ref{cor:no_separa} enables a thorough analysis of no separation case as SNR$\to 0$ to obtain a fundamental understanding of EM with over-specified mixture models~\citep{dwivedi2020sharp,dwivedi2020unbalanced}. 
\textbf{Extension to Finite High/Low SNR Cases}: 
Leveraging the insights from Theorem~\ref{thm:em_update}, we can conduct asymptotic expansions of integrals~\citep{wong2001asymptotic,bleistein1986asymptotic}, enabling the extension of results from limit cases (SNR $\to \infty,0$) to practical, finite high and low SNR scenarios, exploring the transition from low SNR to high SNR regimes.
\textbf{Generalization to Multiple Components}: 
We could expand our analysis from a mixture of two components to scenarios involving multiple components, albeit requiring strong separation of regression parameters compared to the noise level and good initialization~\citep{kwon2020converges}. 

This paper presents work whose goal is to advance the field of Machine Learning. There are many potential societal consequences of our work, none of which we feel must be specifically highlighted here.

%% file: 8_supplementary.tex
\onecolumn 
\newpage
\appendix
\ifdefined\isarxiv
    \setlength{\parskip}{0em}
    \newgeometry{top=0.64in, bottom=0.64in, left=0.4in, right=0.4in, columnsep=1cm}
    \begin{center}
        \LARGE{\textbf{Supplementary Materials: Unveiling the Cycloid Trajectory of EM Iterations in Mixed Linear Regression}}
    \end{center}
\else
    \icmltitle{Supplementary Materials: Unveiling the Cycloid Trajectory of EM Iterations in Mixed Linear Regression}
\fi
\part{Appendix}
\parttoc
We organize the Appendix as follows:
\begin{itemize}
    \item In Section~\ref{sup:lemma}, we prepare some lemmas for integrals, convolutions related to Bessel functions, expectations for Gaussians%, and inequalites 
    used in proofs, etc.
    \item In Section~\ref{sup:derive_em}, we derive EM update rules at the population level and the finite-sample level.
    \item In Section~\ref{sup:updates}, we provide the proof for the explicit closed-form expressions with Bessel functions for Population EM Updates.
    \item In Section~\ref{sup:population}, we give the full proof for the results at the population level in the noiseless setting.
    \item In Section~\ref{sup:finite_sample}, we give the full proof for the results at the finite-sample level in the noiseless setting.
\end{itemize}

\include{8_supplementary_0}

\include{8_supplementary_1}

\include{8_supplementary_2}
\include{8_supplementary_3}

\include{8_supplementary_4}

%% file: 8_supplementary_0.tex
\section{Lemmas: Integrals, Convolutions, Expectations}\label{sup:lemma}

\input{8_supplementary_0_1}

\input{8_supplementary_0_2}

%% file: 8_supplementary_0_1.tex
%\newpage
\subsection{\texorpdfstring{Relations between $\theta^{\ast}, \theta$ and
unit vecotors $\vec{e}_1, \vec{e}_2, \hat{e}_1, \hat{e}_2$}
{Relations between theta*, theta and e1, e2, e1, e2}}
The following lemma shows the relations between $\theta^{\ast}, \theta$ and
unit vecotors $\vec{e}_1, \vec{e}_2, \hat{e}_1, \hat{e}_2$ .

\begin{lemma}
  For $\hat{e}_1 \assign \frac{\theta^{\ast}}{\| \theta^{\ast} \|}$,
  $\hat{e}_2 \assign \frac{\theta - \hat{e}_1  \hat{e}_1^{\top} \theta}{\|
  \theta - \hat{e}_1  \hat{e}_1^{\top} \theta \|} = \frac{\frac{\theta}{\|
  \theta \|} - \rho \frac{\theta^{\ast}}{\| \theta^{\ast} \|}}{\sqrt{1 -
  \rho^2}}$, and \ $\vec{e}_1 = \frac{\theta}{\| \theta \|}$ , $\vec{e}_2 =
  \frac{\theta^{\ast} - \vec{e}_1  \vec{e}_1^{\top} \theta^{\ast}}{\|
  \theta^{\ast} - \vec{e}_1  \vec{e}_1^{\top} \theta^{\ast} \|}$, define $\rho
  \assign \frac{\langle \bar{\theta}^{\ast}, \bar{\theta} \rangle}{\|
  \bar{\theta} \| \cdot \| \bar{\theta}^{\ast} \|} = \frac{\langle
  \theta^{\ast}, \theta \rangle}{\| \theta \| \cdot \| \theta^{\ast} \|}$,
  then
  \begin{eqnarray*}
    \vec{e}_2 + \rho \hat{e}_2 & = & \sqrt{1 - \rho^2} 
    \frac{\theta^{\ast}}{\| \theta^{\ast} \|}\\
    \hat{e}_2 + \rho \vec{e}_2 & = & \sqrt{1 - \rho^2}  \frac{\theta}{\|
    \theta \|}
  \end{eqnarray*}
\end{lemma}

\begin{proof}
  Let $\bar{\theta}^{\ast} \assign \frac{\theta^{\ast}}{\sigma}$ and
  $\bar{\theta} \assign \frac{\theta}{\sigma}$
  
  Define $\vec{e}_1 = \frac{\theta}{\| \theta \|}$ and $\vec{e}_2 =
  \frac{\theta^{\ast} - \vec{e}_1  \vec{e}_1^{\top} \theta^{\ast}}{\|
  \theta^{\ast} - \vec{e}_1  \vec{e}_1^{\top} \theta^{\ast} \|}$ thus $\langle
  \vec{e}_1, \vec{e}_2 \rangle = 0, \| \vec{e}_1 \| = \| \vec{e}_2 \| = 1$,
  $\tmop{span} \{\theta, \theta^{\ast} \} = \tmop{span} \{
  \bar{\theta}^{\ast}, \bar{\theta} \} = \tmop{span} \{ \vec{e}_1, \vec{e}_2
  \}$
  
  Let $\rho \assign \frac{\langle \bar{\theta}^{\ast}, \bar{\theta}
  \rangle}{\| \bar{\theta} \| \cdot \| \bar{\theta}^{\ast} \|} = \frac{\langle
  \theta^{\ast}, \theta \rangle}{\| \theta \| \cdot \| \theta^{\ast} \|}$,
  then $\left| \bar{\theta}^{\ast} - \frac{\langle \bar{\theta}^{\ast},
  \bar{\theta} \rangle}{\| \bar{\theta} \|^2}  \bar{\theta} \right| = \sqrt{1
  - \rho^2} \| \bar{\theta}^{\ast} \|$
  
  If we define $\hat{e}_1 \assign \frac{\theta^{\ast}}{\| \theta^{\ast} \|}$,
  and $\hat{e}_2 \assign \frac{\theta - \hat{e}_1  \hat{e}_1^{\top} \theta}{\|
  \theta - \hat{e}_1  \hat{e}_1^{\top} \theta \|} = \frac{\frac{\theta}{\|
  \theta \|} - \rho \frac{\theta^{\ast}}{\| \theta^{\ast} \|}}{\sqrt{1 -
  \rho^2}}$ and $\langle \hat{e}_1, \hat{e}_2 \rangle = 0, \| \hat{e}_1 \| =
  \| \hat{e}_2 \| = 1$
  \[ \left( \begin{array}{c}
       \theta\\
       \theta^{\ast}
     \end{array} \right) = \left( \begin{array}{cc}
       \| \theta \| & \\
       & \| \theta^{\ast} \|
     \end{array} \right) \left( \begin{array}{cc}
       1 & 0\\
       \rho & \sqrt{1 - \rho^2}
     \end{array} \right) \left( \begin{array}{c}
       \vec{e}_1\\
       \vec{e}_2
     \end{array} \right) = \left( \begin{array}{cc}
       \| \theta \| & \\
       & \| \theta^{\ast} \|
     \end{array} \right) \left( \begin{array}{cc}
       \rho & \sqrt{1 - \rho^2}\\
       1 & 0
     \end{array} \right) \left( \begin{array}{c}
       \hat{e}_1\\
       \hat{e}_2
     \end{array} \right) \]
  Therefore
  \begin{eqnarray*}
    \left( \begin{array}{c}
      \hat{e}_1\\
      \hat{e}_2
    \end{array} \right) & = & \left( \begin{array}{cc}
      \rho & \sqrt{1 - \rho^2}\\
      1 & 0
    \end{array} \right)^{- 1} \left( \begin{array}{cc}
      1 & 0\\
      \rho & \sqrt{1 - \rho^2}
    \end{array} \right) \left( \begin{array}{c}
      \vec{e}_1\\
      \vec{e}_2
    \end{array} \right) = \left( \begin{array}{cc}
      \rho & \sqrt{1 - \rho^2}\\
      \sqrt{1 - \rho^2} & - \rho
    \end{array} \right) \left( \begin{array}{c}
      \vec{e}_1\\
      \vec{e}_2
    \end{array} \right)\\
    \left( \begin{array}{c}
      \vec{e}_1\\
      \vec{e}_2
    \end{array} \right) & = & \left( \begin{array}{cc}
      \rho & \sqrt{1 - \rho^2}\\
      \sqrt{1 - \rho^2} & - \rho
    \end{array} \right) \left( \begin{array}{c}
      \hat{e}_1\\
      \hat{e}_2
    \end{array} \right)
  \end{eqnarray*}
%   We can express $\vec{e}\tmrsub{1} - {\frac{{\rho}{\sqrt{1 -
%   {\rho}\tmrsup{2}}} \| \={{\theta}}\tmrsup{{\ast}} \|\tmrsup{2}}{1 + (1 -
%   {\rho}\tmrsup{2}) \| \={{\theta}}\tmrsup{{\ast}} \|\tmrsup{2}}}
%   \vec{e}\tmrsub{2}$as, notethat${\rho}\^{e}\tmrsub{1} + {\sqrt{1 -
%   {\rho}\tmrsup{2}}} \^{e}\tmrsub{2} = {\frac{{\theta}}{\| {\theta}\|}} =
%   {\frac{\={{\theta}}}{\| \={{\theta}} \|}}$
  \begin{eqnarray*}
    \vec{e}_1 - \frac{\rho \sqrt{1 - \rho^2} \| \bar{\theta}^{\ast} \|^2}{1 +
    (1 - \rho^2) \| \bar{\theta}^{\ast} \|^2}  \vec{e}_2 & = & \left(
    \begin{array}{c}
      1\\
      - \frac{\rho \sqrt{1 - \rho^2} \| \bar{\theta}^{\ast} \|^2}{1 + (1 -
      \rho^2)\| \bar{\theta}^{\ast} \|^2}
    \end{array} \right)^{\top} \left( \begin{array}{cc}
      \rho & \sqrt{1 - \rho^2}\\
      \sqrt{1 - \rho^2} & - \rho
    \end{array} \right) \left( \begin{array}{c}
      \hat{e}_1\\
      \hat{e}_2
    \end{array} \right)\\
    & = & \frac{1}{\frac{1}{\| \bar{\theta}^{\ast} \|^2} + (1 - \rho^2)}
    \left( \begin{array}{c}
      \frac{1}{\| \bar{\theta}^{\ast} \|^2} + (1 - \rho^2)\\
      - \rho \sqrt{1 - \rho^2}
    \end{array} \right)^{\top} \left( \begin{array}{cc}
      \rho & \sqrt{1 - \rho^2}\\
      \sqrt{1 - \rho^2} & - \rho
    \end{array} \right) \left( \begin{array}{c}
      \hat{e}_1\\
      \hat{e}_2
    \end{array} \right)\\
    & = & \frac{\left( \begin{array}{cc}
      \rho \frac{1}{\| \bar{\theta}^{\ast} \|^2} & \sqrt{1 - \rho^2}  \left(
      \frac{1}{\| \bar{\theta}^{\ast} \|^2} + 1 \right)
    \end{array} \right)}{\frac{1}{\| \bar{\theta}^{\ast} \|^2} + (1 - \rho^2)}
    \left( \begin{array}{c}
      \hat{e}_1\\
      \hat{e}_2
    \end{array} \right)\\
    & = & \frac{\frac{1}{\| \bar{\theta}^{\ast} \|^2}}{\frac{1}{\|
    \bar{\theta}^{\ast} \|^2} + (1 - \rho^2)} \cdot \frac{\bar{\theta}}{\|
    \bar{\theta} \|} + \frac{\sqrt{1 - \rho^2}}{\frac{1}{\|
    \bar{\theta}^{\ast} \|^2} + (1 - \rho^2)} \cdot \hat{e}_2
  \end{eqnarray*}
%   With {\frac{{\theta}}{\| {\theta}\|}} =
%   {\rho}{\frac{{\theta}\tmrsup{{\ast}}}{\| {\theta}\tmrsup{{\ast}} \|}} +
%   {\sqrt{1 - {\rho}\tmrsup{2}}}
%   \^{e}\tmrsub{2}$and${\frac{{\theta}\tmrsup{{\ast}}}{\|
%   {\theta}\tmrsup{{\ast}} \|}} = {\rho}{\frac{{\theta}}{\| {\theta}\|}} +
%   {\sqrt{1 - {\rho}\tmrsup{2}}} \vec{e}\tmrsub{2}$, hence$
  \[ \frac{1}{\sqrt{1 - \rho^2}} \left( \begin{array}{cc}
       1 & - \rho\\
       - \rho & 1
     \end{array} \right) \left( \begin{array}{c}
       \frac{\theta}{\| \theta \|}\\
       \frac{\theta^{\ast}}{\| \theta^{\ast} \|}
     \end{array} \right) = \left( \begin{array}{c}
       \hat{e}_2\\
       \vec{e}_2
     \end{array} \right) \]
  Therefore, we can show that
  \begin{eqnarray*}
    \vec{e}_2 + \rho \hat{e}_2 = \sqrt{1 - \rho^2} 
    \frac{\theta^{\ast}}{\| \theta^{\ast} \|}
    ,\quad
    \hat{e}_2 + \rho \vec{e}_2 = \sqrt{1 - \rho^2}  \frac{\theta}{\|
    \theta \|}
  \end{eqnarray*}
\end{proof}

%% file: 8_supplementary_0_2.tex
%\newpage
\subsection{Integrals and Expectation with Gaussian}

\subsubsection{integrals with Gaussian}
\begin{lemma}\label{lem:gauss_int}
    For $\forall a >0$, then
    \begin{equation*}
        (2 \pi)^{-\frac{1}{2}} \int_{-\infty}^{\infty} \exp \left[-\frac{a t^2+2 b t}{2}\right] \mathrm{d} t
        % =\frac{\exp \left[\frac{b^2}{2 a}\right]}{\sqrt{a}} \cdot\left[\sqrt{\frac{a}{2 \pi}} \int_{-\infty}^{\infty} \exp \left[-\frac{a\left(t+\frac{b}{a}\right)^2}{2}\right] \mathrm{d} t\right]
        =\frac{\exp \left[\frac{b^2}{2 a}\right]}{\sqrt{a}}
    \end{equation*}
    \begin{eqnarray*}
        (2 \pi)^{- \frac{1}{2}} \int_{- \infty}^{\infty} \exp \left[ - \frac{a t^2
        + 2 b t}{2} \right] t \mathd t 
        % & = & (2 \pi)^{- \frac{1}{2}} \int_{-
        % \infty}^{\infty} \exp \left[ - \frac{a t^2 + 2 b t}{2} \right] \left( t +
        % \frac{b}{a} \right) \mathd \left( t + \frac{b}{a} \right)\\
        % & - & \frac{b}{a} \cdot (2 \pi)^{- \frac{1}{2}} \int_{- \infty}^{\infty}
        % \exp \left[ - \frac{a t^2 + 2 b t}{2} \right] \mathd t\\
        % & = & - \frac{b}{a} \cdot \frac{\exp \left[ \frac{b^2}{2 a }
        % \right]}{\sqrt{a}} 
        = - \frac{b}{a^{\frac{3}{2}}} \exp \left[ \frac{b^2}{2
        a } \right]
      \end{eqnarray*}
\end{lemma}

\subsubsection{expectations with Gaussian}
\begin{lemma}
    Let $\bar{\theta}:=\frac{\theta}{\sigma},\bar{\theta^\ast}:=\frac{\theta^\ast}{\sigma}$, 
    and $\nu\sim \mathcal{N} (- \langle x,
    \bar{\theta}^{\ast} \rangle \langle x, \bar{\theta} \rangle, \langle x,
    \bar{\theta} \rangle^2)$ or 
    $\nu\sim \mathcal{N} (\langle x,
    \bar{\theta}^{\ast} \rangle \langle x, \bar{\theta} \rangle, \langle x,
    \bar{\theta} \rangle^2)$, 
    
    then the expectations for the density functions are
    \begin{eqnarray*}
        & & \mathbb{E}_{x \sim \mathcal{N} (0, I_d)} \mathcal{N} (- \langle x,
        \bar{\theta}^{\ast} \rangle \langle x, \bar{\theta} \rangle, \langle x,
        \bar{\theta} \rangle^2)\\
        &=& (\pi \| \bar{\theta} \|)^{- 1} {(1 + (1 - \rho^2) \|
        \bar{\theta}^{\ast} \|^2)^{- \frac{1}{2}}}  \exp \left[ - \frac{\rho \|
        \bar{\theta}^{\ast} \| \left( \frac{\nu}{\| \bar{\theta} \|} \right)}{[1 +
        (1 - \rho^2) \| \bar{\theta}^{\ast} \|^2]} \right] \cdot K_0 \left(
        \frac{\sqrt{1 + \| \bar{\theta}^{\ast} \|^2} \cdot \left| \frac{\nu}{\|
        \bar{\theta} \|} \right|}{[1 + (1 - \rho^2) \| \bar{\theta}^{\ast} \|^2]}
        \right)\\
        & & \mathbb{E}_{x \sim \mathcal{N} (0, I_d)} \mathcal{N} (\langle x,
        \bar{\theta}^{\ast} \rangle \langle x, \bar{\theta} \rangle, \langle x,
        \bar{\theta} \rangle^2)\\
        &=& (\pi \| \bar{\theta} \|)^{- 1} {(1 + (1 - \rho^2) \|
        \bar{\theta}^{\ast} \|^2)^{- \frac{1}{2}}}  \exp \left[ + \frac{\rho \|
        \bar{\theta}^{\ast} \| \left( \frac{\nu}{\| \bar{\theta} \|} \right)}{[1 +
        (1 - \rho^2) \| \bar{\theta}^{\ast} \|^2]} \right] \cdot K_0 \left(
        \frac{\sqrt{1 + \| \bar{\theta}^{\ast} \|^2} \cdot \left| \frac{\nu}{\|
        \bar{\theta} \|} \right|}{[1 + (1 - \rho^2) \| \bar{\theta}^{\ast} \|^2]}
        \right)
    \end{eqnarray*}
\end{lemma}

\begin{proof}
\
We define $\rho := \frac{\langle \bar{\theta}, \bar{\theta^\ast} \rangle}{\| \bar{\theta}\| \cdot \|\bar{\theta^\ast} \|}$,
and $\lambda_1 := \langle \bar{\theta}, \vec{e}_1 \rangle, \lambda_2 := \langle \bar{\theta}, \vec{e}_2 \rangle$.

% define $\vec{e}_1 = \frac{\theta}{\| \theta \|}$ and $\vec{e}_2 =
% \frac{\theta^{\ast} - \vec{e}_1 \vec{e}_1^{\top} \theta^{\ast}}{\|
% \theta^{\ast} - \vec{e}_1 \vec{e}_1^{\top} \theta^{\ast} \|}$ thus $\langle
% \vec{e}_1, \vec{e}_2 \rangle = 0, \| \vec{e}_1 \| = \| \vec{e}_2 \| = 1$,
% $\tmop{span} \{ \theta, \theta^{\ast} \} = \tmop{span} \{
% \bar{\theta}^{\ast}, \bar{\theta} \} = \tmop{span} \{ \vec{e}_1, \vec{e}_2
% \}$

% let $\lambda_1 \assign \langle x, \vec{e}_1 \rangle, \lambda_2 \assign
% \langle x, \vec{e}_2 \rangle$, let $\rho \assign \frac{\langle
% \bar{\theta}^{\ast}, \bar{\theta} \rangle}{\| \bar{\theta} \| \cdot \|
% \bar{\theta}^{\ast} \| } = \frac{\langle \theta^{\ast}, \theta \rangle}{\|
% \theta \| \cdot \| \theta^{\ast} \| }$, then $\left\| \bar{\theta}^{\ast} -
% \frac{\langle \bar{\theta}^{\ast}, \bar{\theta} \rangle}{\| \bar{\theta}
% \|^2} \bar{\theta} \right\| = \sqrt{1 - \rho^2} \| \bar{\theta}^{\ast} \|$

Then $\langle x, \bar{\theta} \rangle = \lambda_1 \| \bar{\theta} \|,
\langle x, \bar{\theta}^{\ast} \rangle = \left\langle \lambda_1 \vec{e}_1 +
\lambda_2 \vec{e}_2, \rho \| \bar{\theta}^{\ast} \| \vec{e}_1 + \sqrt{1 -
\rho^2} \| \bar{\theta}^{\ast} \| \vec{e}_2 \right\rangle = \lambda_1 \rho
\| \bar{\theta}^{\ast} \| + \lambda_2 \sqrt{1 - \rho^2} \|
\bar{\theta}^{\ast} \|$.

For the evaluation of the first expectation, we let

$v \leftarrow 0, z \leftarrow \frac{\sqrt{1 + \| \bar{\theta}^{\ast}
\|^2} \cdot \left| \frac{\nu}{\| \bar{\theta} \|} \right|}{[1 + (1 - \rho^2)
\| \bar{\theta}^{\ast} \|^2]}$ in \cite{olver2010nist} Chapter 10 %Lemma~\ref{lem:schlafli_int} 
(Schl{\"a}fli's Integral of $K_v (z)$).

$a \leftarrow 1 + (1 - \rho^2) \| \bar{\theta}^{\ast} \|^2$ and $b
\leftarrow \frac{\nu \sqrt{1 - \rho^2} \| \bar{\theta}^{\ast} \|}{\lambda_1
\| \bar{\theta} \|} + \lambda_1 \rho \sqrt{1 - \rho^2} \|
\bar{\theta}^{\ast} \|^2 = \sqrt{1 - \rho^2} \| \bar{\theta}^{\ast} \|
\left( \frac{\nu}{\lambda_1 \| \bar{\theta} \|} + \lambda_1 \rho \|
\bar{\theta}^{\ast} \| \right)$ and $t \leftarrow \lambda_2$ in Lemma~\ref{lem:gauss_int}.

For the evaluation of the second expectation, note that the following relation holds for these two density functions.
\begin{equation*}
    \mathcal{N} (\langle x,
    \bar{\theta}^{\ast} \rangle \langle x, \bar{\theta} \rangle, \langle x,
    \bar{\theta} \rangle^2) 
    = \left[ 
        \mathcal{N} (-\langle x,
        \bar{\theta}^{\ast} \rangle \langle x, \bar{\theta} \rangle, \langle x,
        \bar{\theta} \rangle^2)
    \right]_{\bar{\theta}\to -\bar{\theta}}
\end{equation*}
Since $\bar{\theta}\to -\bar{\theta}$ implies $\rho\to -\rho$, we can obtain such a relation to derive the closed-form expression in this Lemma.
\begin{equation*}
    \mathbb{E}_{x \sim \mathcal{N} (0, I_d)}
    \mathcal{N} (\langle x,
    \bar{\theta}^{\ast} \rangle \langle x, \bar{\theta} \rangle, \langle x,
    \bar{\theta} \rangle^2) 
    = \left[ 
        \mathbb{E}_{x \sim \mathcal{N} (0, I_d)}
        \mathcal{N} (- \langle x,
        \bar{\theta}^{\ast} \rangle \langle x, \bar{\theta} \rangle, \langle x,
        \bar{\theta} \rangle^2)
    \right]_{\rho\to -\rho}
\end{equation*}
% we follow the same procedure but let 
% $a \leftarrow 1 + (1 - \rho^2) \| \bar{\theta}^{\ast} \|^2$ and
%       $b \leftarrow - \frac{\nu \sqrt{1 - \rho^2} \| \bar{\theta}^{\ast}
%       \|}{\lambda_1 \| \bar{\theta} \|} + \lambda_1 \rho \sqrt{1 - \rho^2} \|
%       \bar{\theta}^{\ast} \|^2 = \sqrt{1 - \rho^2} \| \bar{\theta}^{\ast} \|
%       \left( - \frac{\nu}{\lambda_1 \| \bar{\theta} \|} + \lambda_1 \rho \|
%       \bar{\theta}^{\ast} \| \right)$ in Lemma~\ref{lem:gauss_int}
\small
\normalsize

\end{proof}

\begin{lemma}
    Let $\bar{\theta}:=\frac{\theta}{\sigma},\bar{\theta^\ast}:=\frac{\theta^\ast}{\sigma}$, 
    and $\lambda_1 := \langle x, \vec{e}_1 \rangle, \lambda_2 := \langle x, \vec{e}_2 \rangle$.

    $\nu\sim \mathcal{N} (- \langle x,
    \bar{\theta}^{\ast} \rangle \langle x, \bar{\theta} \rangle, \langle x,
    \bar{\theta} \rangle^2)$ or 
    $\nu\sim \mathcal{N} (\langle x,
    \bar{\theta}^{\ast} \rangle \langle x, \bar{\theta} \rangle, \langle x,
    \bar{\theta} \rangle^2)$, 
    
    then the expectations for the the products of $\frac{\lambda_2}{\lambda_1}$ and density functions are
\small
    \begin{eqnarray*}
        &  & \mathbb{E}_{x \sim \mathcal{N} (0, I_d)} \frac{\lambda_2}{\lambda_1}
        \mathcal{N} (- \langle x, \bar{\theta}^{\ast} \rangle \langle x,
        \bar{\theta} \rangle, \langle x, \bar{\theta} \rangle^2)\\
    & = & - (\pi \| \bar{\theta} \|)^{- 1} {(1 + (1 - \rho^2) \|
        \bar{\theta}^{\ast} \|^2)^{- \frac{3}{2}}}  \sqrt{1 - \rho^2} \|
        \bar{\theta}^{\ast} \| \exp \left[ - \frac{\rho \| \bar{\theta}^{\ast} \|
        \left( \frac{\nu}{\| \bar{\theta} \|} \right)}{[1 + (1 - \rho^2) \|
        \bar{\theta}^{\ast} \|^2]} \right]\\
        & \cdot & \left[ \tmop{sgn} (\nu) [1 + \| \bar{\theta}^{\ast}
        \|^2]^{\frac{1}{2}} K_1 \left( \frac{\sqrt{1 + \| \bar{\theta}^{\ast}
        \|^2} \cdot \left| \frac{\nu}{\| \bar{\theta} \|} \right|}{[1 + (1 -
        \rho^2) \| \bar{\theta}^{\ast} \|^2]} \right) + \rho \|
        \bar{\theta}^{\ast} \| K_0 \left( \frac{\sqrt{1 + \| \bar{\theta}^{\ast}
        \|^2} \cdot \left| \frac{\nu}{\| \bar{\theta} \|} \right|}{[1 + (1 -
        \rho^2) \| \bar{\theta}^{\ast} \|^2]} \right) \right]\\
        &  & \mathbb{E}_{x \sim \mathcal{N} (0, I_d)} \frac{\lambda_2}{\lambda_1}
        \mathcal{N} (\langle x, \bar{\theta}^{\ast} \rangle \langle x,
        \bar{\theta} \rangle, \langle x, \bar{\theta} \rangle^2)\\
    & = & (\pi \| \bar{\theta} \|)^{- 1} {(1 + (1 - \rho^2) \|
      \bar{\theta}^{\ast} \|^2)^{- \frac{3}{2}}}  \sqrt{1 - \rho^2} \|
      \bar{\theta}^{\ast} \| \exp \left[ + \frac{\rho \| \bar{\theta}^{\ast} \|
      \left( \frac{\nu}{\| \bar{\theta} \|} \right)}{[1 + (1 - \rho^2) \|
      \bar{\theta}^{\ast} \|^2]} \right]\\
      & \cdot & \left[ \tmop{sgn} (\nu) [1 + \| \bar{\theta}^{\ast}
      \|^2]^{\frac{1}{2}} K_1 \left( \frac{\sqrt{1 + \| \bar{\theta}^{\ast}
      \|^2} \cdot \left| \frac{\nu}{\| \bar{\theta} \|} \right|}{[1 + (1 -
      \rho^2) \| \bar{\theta}^{\ast} \|^2]} \right) - \rho \|
      \bar{\theta}^{\ast} \| K_0 \left( \frac{\sqrt{1 + \| \bar{\theta}^{\ast}
      \|^2} \cdot \left| \frac{\nu}{\| \bar{\theta} \|} \right|}{[1 + (1 -
      \rho^2) \| \bar{\theta}^{\ast} \|^2]} \right) \right]
    \end{eqnarray*}
\normalsize
\end{lemma}

\begin{proof}
    \
    For the first expectation, let
    $v \leftarrow \{0, 1\}, z \leftarrow \frac{\sqrt{1 + \| \bar{\theta}^{\ast}
    \|^2} \cdot \left| \frac{\nu}{\| \bar{\theta} \|} \right|}{[1 + (1 - \rho^2)
    \| \bar{\theta}^{\ast} \|^2]}$ in \cite{olver2010nist} Chapter 10 %Lemma~\ref{lem:schlafli_int} 
    (Schl{\"a}fli's Integral of $K_v (z)$);
    $a \leftarrow 1 + (1 - \rho^2) \| \bar{\theta}^{\ast} \|^2$ and $b
    \leftarrow \frac{\nu \sqrt{1 - \rho^2} \| \bar{\theta}^{\ast} \|}{\lambda_1
    \| \bar{\theta} \|} + \lambda_1 \rho \sqrt{1 - \rho^2} \|
    \bar{\theta}^{\ast} \|^2 = \sqrt{1 - \rho^2} \| \bar{\theta}^{\ast} \|
    \left( \frac{\nu}{\lambda_1 \| \bar{\theta} \|} + \lambda_1 \rho \|
    \bar{\theta}^{\ast} \| \right)$ and $t \leftarrow \lambda_2$ in Lemma~\ref{lem:gauss_int}.

    For the second expectation, note that $\mathcal{N} (\langle x,
        \bar{\theta}^{\ast} \rangle \langle x, \bar{\theta} \rangle, \langle x,
        \bar{\theta} \rangle^2) 
        = \left[ 
            \mathcal{N} (-\langle x,
            \bar{\theta}^{\ast} \rangle \langle x, \bar{\theta} \rangle, \langle x,
            \bar{\theta} \rangle^2)
        \right]_{\bar{\theta}\to -\bar{\theta}}$ holds for these two density functions.
    Since $\bar{\theta}\to -\bar{\theta}$ implies $\rho\to -\rho, \sqrt{1-\rho^2}\to -\sqrt{1-\rho^2}$, we can obtain such a relation.
    \begin{equation*}
        \mathbb{E}_{x \sim \mathcal{N} (0, I_d)} \frac{\lambda_2}{\lambda_1}
        \mathcal{N} (\langle x,
        \bar{\theta}^{\ast} \rangle \langle x, \bar{\theta} \rangle, \langle x,
        \bar{\theta} \rangle^2) 
        = \left[ 
            \mathbb{E}_{x \sim \mathcal{N} (0, I_d)} \frac{\lambda_2}{\lambda_1}
            \mathcal{N} (- \langle x,
            \bar{\theta}^{\ast} \rangle \langle x, \bar{\theta} \rangle, \langle x,
            \bar{\theta} \rangle^2)
        \right]_{\rho\to -\rho, \sqrt{1-\rho^2}\to -\sqrt{1-\rho^2}}
    \end{equation*}
    % let $a \leftarrow 1 + (1 - \rho^2) \| \bar{\theta}^{\ast} \|^2$ and $b
    % \leftarrow \frac{\nu \sqrt{1 - \rho^2} \| \bar{\theta}^{\ast} \|}{\lambda_1
    % \| \bar{\theta} \|} + \lambda_1 \rho \sqrt{1 - \rho^2} \|
    % \bar{\theta}^{\ast} \|^2 = \sqrt{1 - \rho^2} \| \bar{\theta}^{\ast} \|
    % \left( \frac{\nu}{\lambda_1 \| \bar{\theta} \|} + \lambda_1 \rho \|
    % \bar{\theta}^{\ast} \| \right)$ and $t \leftarrow \lambda_2$
\small
\normalsize
\end{proof}
    \begin{lemma}
        Let $\bar{\theta}:=\frac{\theta}{\sigma},\bar{\theta^\ast}:=\frac{\theta^\ast}{\sigma}$, 
        and $\lambda_1 := \langle x, \vec{e}_1 \rangle, \lambda_2 := \langle x, \vec{e}_2 \rangle$.
    
        $\nu\sim \mathcal{N} (- \langle x,
        \bar{\theta}^{\ast} \rangle \langle x, \bar{\theta} \rangle, \langle x,
        \bar{\theta} \rangle^2)$ or 
        $\nu\sim \mathcal{N} (\langle x,
        \bar{\theta}^{\ast} \rangle \langle x, \bar{\theta} \rangle, \langle x,
        \bar{\theta} \rangle^2)$, 
        
        then the expectations for the the products of $\frac{x}{\langle x,
        \bar{\theta} \rangle}$ and density functions are
\small
        \begin{eqnarray*}
            &  & \mathbb{E}_{x \sim \mathcal{N} (0, I_d)} \left[ \frac{x}{\langle x,
            \bar{\theta} \rangle} \right] \mathcal{N} (- \langle x,
            \bar{\theta}^{\ast} \rangle \langle x, \bar{\theta} \rangle, \langle x,
            \bar{\theta} \rangle^2)\\
        & = & \frac{(\pi \| \bar{\theta} \|^2)^{- 1}}{{(1 + (1 - \rho^2) \|
            \bar{\theta}^{\ast} \|^2)^{\frac{1}{2}}} } \exp \left[ - \frac{\rho \|
            \bar{\theta}^{\ast} \| \left( \frac{\nu}{\| \bar{\theta} \|} \right)}{[1 +
            (1 - \rho^2) \| \bar{\theta}^{\ast} \|^2]} \right] \cdot K_0 \left(
            \frac{\sqrt{1 + \| \bar{\theta}^{\ast} \|^2} \cdot \left| \frac{\nu}{\|
            \bar{\theta} \|} \right|}{[1 + (1 - \rho^2) \| \bar{\theta}^{\ast} \|^2]}
            \right) \left[ \vec{e}_1 - \frac{\rho \sqrt{1 - \rho^2} \|
            \bar{\theta}^{\ast} \|^2}{1 + (1 - \rho^2) \| \bar{\theta}^{\ast} \|^2}
            \vec{e}_2 \right]\\
            & - & \tmop{sgn} (\nu) \frac{(\pi \| \bar{\theta} \|^2)^{- 1} \sqrt{1 -
            \rho^2} \| \bar{\theta}^{\ast} \| [1 + \| \bar{\theta}^{\ast}
            \|^2]^{\frac{1}{2}}}{{(1 + (1 - \rho^2) \| \bar{\theta}^{\ast}
            \|^2)^{\frac{3}{2}}} } \exp \left[ - \frac{\rho \| \bar{\theta}^{\ast} \|
            \left( \frac{\nu}{\| \bar{\theta} \|} \right)}{[1 + (1 - \rho^2) \|
            \bar{\theta}^{\ast} \|^2]} \right] K_1 \left( \frac{\sqrt{1 + \|
            \bar{\theta}^{\ast} \|^2} \cdot \left| \frac{\nu}{\| \bar{\theta} \|}
            \right|}{[1 + (1 - \rho^2) \| \bar{\theta}^{\ast} \|^2]} \right) \vec{e}_2\\
            &  & \mathbb{E}_{x \sim \mathcal{N} (0, I_d)} \left[ \frac{x}{\langle x,
            \bar{\theta} \rangle} \right] \mathcal{N} (\langle x, \bar{\theta}^{\ast}
            \rangle \langle x, \bar{\theta} \rangle, \langle x, \bar{\theta}
            \rangle^2)\\
        & = & (\pi \| \bar{\theta} \|)^{- 1} {(1 + (1 - \rho^2) \|
          \bar{\theta}^{\ast} \|^2)^{- \frac{3}{2}}}  \sqrt{1 - \rho^2} \|
          \bar{\theta}^{\ast} \| \exp \left[ + \frac{\rho \| \bar{\theta}^{\ast} \|
          \left( \frac{\nu}{\| \bar{\theta} \|} \right)}{[1 + (1 - \rho^2) \|
          \bar{\theta}^{\ast} \|^2]} \right]\\
          & = & \frac{(\pi \| \bar{\theta} \|^2)^{- 1}}{{(1 + (1 - \rho^2) \|
          \bar{\theta}^{\ast} \|^2)^{\frac{1}{2}}} } \exp \left[ + \frac{\rho \|
          \bar{\theta}^{\ast} \| \left( \frac{\nu}{\| \bar{\theta} \|} \right)}{[1 +
          (1 - \rho^2) \| \bar{\theta}^{\ast} \|^2]} \right] \cdot K_0 \left(
          \frac{\sqrt{1 + \| \bar{\theta}^{\ast} \|^2} \cdot \left| \frac{\nu}{\|
          \bar{\theta} \|} \right|}{[1 + (1 - \rho^2) \| \bar{\theta}^{\ast} \|^2]}
          \right) \left[ \vec{e}_1 - \frac{\rho \sqrt{1 - \rho^2} \|
          \bar{\theta}^{\ast} \|^2}{1 + (1 - \rho^2) \| \bar{\theta}^{\ast} \|^2}
          \vec{e}_2 \right]\\
          & + & \tmop{sgn} (\nu) \frac{(\pi \| \bar{\theta} \|^2)^{- 1} \sqrt{1 -
          \rho^2} \| \bar{\theta}^{\ast} \| [1 + \| \bar{\theta}^{\ast}
          \|^2]^{\frac{1}{2}}}{{(1 + (1 - \rho^2) \| \bar{\theta}^{\ast}
          \|^2)^{\frac{3}{2}}} } \exp \left[ + \frac{\rho \| \bar{\theta}^{\ast} \|
          \left( \frac{\nu}{\| \bar{\theta} \|} \right)}{[1 + (1 - \rho^2) \|
          \bar{\theta}^{\ast} \|^2]} \right] K_1 \left( \frac{\sqrt{1 + \|
          \bar{\theta}^{\ast} \|^2} \cdot \left| \frac{\nu}{\| \bar{\theta} \|}
          \right|}{[1 + (1 - \rho^2) \| \bar{\theta}^{\ast} \|^2]} \right) \vec{e}_2
        \end{eqnarray*}
\normalsize
    \end{lemma}

    \begin{proof}
        \
        We express $x = (\lambda_1 \vec{e}_1 + \lambda_2 \vec{e}_2) + \tilde{x}$ and
        r.v. $\tilde{x} \ind \lambda_1, \lambda_2$ and $\mathbb{E}_{\tilde{x}}
        [\tilde{x}] = 0$. Hence, for any function $\varphi (\lambda_1,
        \lambda_2)$, we have 
        $\mathbb{E}_{x \sim \mathcal{N} (0, I_d)} [\tilde{x} \cdot \varphi
        (\lambda_1, \lambda_2)] =\mathbb{E}_{\lambda_1, \lambda_2
        \overset{\tmop{iid}}{\sim} \mathcal{N} (0, 1)} \varphi (\lambda_1,
        \lambda_2) [\mathbb{E}_{\tilde{x}} [\tilde{x} \mid \lambda_1, \lambda_2] ]
        =\mathbb{E}_{\lambda_1, \lambda_2 \overset{\tmop{iid}}{\sim} \mathcal{N} (0,
        1)} \varphi (\lambda_1, \lambda_2) [\mathbb{E}_{\tilde{x}} [\tilde{x}] ] =
        0$
        
        Note that $\langle x, \bar{\theta} \rangle, \mathcal{N} (- \langle x,
        \bar{\theta}^{\ast} \rangle \langle x, \bar{\theta} \rangle, \langle x,
        \bar{\theta} \rangle^2), \mathcal{N} (\langle x, \bar{\theta}^{\ast} \rangle
        \langle x, \bar{\theta} \rangle, \langle x, \bar{\theta} \rangle^2)$ are
        functions of $\lambda_1, \lambda_2$.  

        Therefore, $\mathbb{E}_{x \sim \mathcal{N} (0, I_d)} \left[
        \frac{\tilde{x}}{\langle x, \bar{\theta} \rangle} \right] 
        \mathcal{N} (- \langle x, \bar{\theta}^{\ast} \rangle \langle x,
        \bar{\theta} \rangle, \langle x, \bar{\theta} \rangle^2) =0,
        \mathbb{E}_{x \sim \mathcal{N} (0, I_d)} \left[
        \frac{\tilde{x}}{\langle x, \bar{\theta} \rangle} \right]
        \mathcal{N} (\langle x, \bar{\theta}^{\ast} \rangle \langle x, \bar{\theta}
        \rangle, \langle x, \bar{\theta} \rangle^2) = 0$.

        Subsequently, we can decompose thees expectations into two terms.
        \begin{eqnarray*}
          &  & \mathbb{E}_{x \sim \mathcal{N} (0, I_d)} \left[ \frac{x}{\langle x,
          \bar{\theta} \rangle} \right] \mathcal{N} (- \langle x,
          \bar{\theta}^{\ast} \rangle \langle x, \bar{\theta} \rangle, \langle x,
          \bar{\theta} \rangle^2)%\\
          = \mathbb{E}_{x \sim \mathcal{N} (0, I_d)} \frac{\lambda_1 \vec{e}_1
          + \lambda_2 \vec{e}_2}{\lambda_1 \| \bar{\theta} \|} \mathcal{N} (-
          \langle x, \bar{\theta}^{\ast} \rangle \langle x, \bar{\theta} \rangle,
          \langle x, \bar{\theta} \rangle^2)\\
          & = & \frac{\vec{e}_1}{\| \bar{\theta} \|} \mathbb{E}_{x \sim \mathcal{N}
          (0, I_d)} \mathcal{N} (- \langle x, \bar{\theta}^{\ast} \rangle \langle x,
          \bar{\theta} \rangle, \langle x, \bar{\theta} \rangle^2)%\\
          + \frac{\vec{e}_2}{\| \bar{\theta} \|} \mathbb{E}_{x \sim \mathcal{N}
          (0, I_d)} \frac{\lambda_2}{\lambda_1} \mathcal{N} (- \langle x,
          \bar{\theta}^{\ast} \rangle \langle x, \bar{\theta} \rangle, \langle x,
          \bar{\theta} \rangle^2)\\
          &  & \mathbb{E}_{x \sim \mathcal{N} (0, I_d)} \left[ \frac{x}{\langle x,
          \bar{\theta} \rangle} \right] \mathcal{N} (\langle x, \bar{\theta}^{\ast}
          \rangle \langle x, \bar{\theta} \rangle, \langle x, \bar{\theta}
          \rangle^2)\\
          & = & \frac{\vec{e}_1}{\| \bar{\theta} \|} \mathbb{E}_{x \sim \mathcal{N}
          (0, I_d)} \mathcal{N} (\langle x, \bar{\theta}^{\ast} \rangle \langle x,
          \bar{\theta} \rangle, \langle x, \bar{\theta} \rangle^2)
          + \frac{\vec{e}_2}{\| \bar{\theta} \|} \mathbb{E}_{x \sim \mathcal{N}
          (0, I_d)} \frac{\lambda_2}{\lambda_1} \mathcal{N} (\langle x,
          \bar{\theta}^{\ast} \rangle \langle x, \bar{\theta} \rangle, \langle x,
          \bar{\theta} \rangle^2)
        \end{eqnarray*}    
        Then, with the previous two Lemmas, we derive the closed-from expressions in this Lemma.
    \end{proof}

%\newpage
\subsubsection{expectations for 2MLR}
\begin{lemma}
    For the 2MLR at the population level, $s:=(x, y)$ 

    and for any function $\psi(y), \forall y\in\mathbb{R}$,
    the operator $\mathcal{F}_{-y}$ are defined by $\mathcal{F}_{-y}[\psi(y)]=\psi(-y)$;

    then $\mathbb{E}_{s \sim p (s \mid \theta^{\ast}, \pi^{\ast})}
    = \mathbb{E}_{x\sim\mathcal{N}(0, 1)}\mathbb{E}_{y\mid x\sim 
    \pi^\ast(1) \mathcal{N}(\langle x, \theta^\ast \rangle, \sigma^2) + \pi^\ast(2) \mathcal{N}(-\langle x, \theta^\ast\rangle, \sigma^2)}$ satisfies
    \begin{equation*}
        \mathbb{E}_{s \sim p (s \mid \theta^{\ast}, \pi^{\ast})}
        =\mathbb{E}_{x \sim \mathcal{N} (0, I_d)} \mathbb{E}_{y \sim \mathcal{N}
        (\langle x, \theta^{\ast} \rangle, \sigma^2)} [\pi^{\ast} (1) + \pi^{\ast}
        (2) \mathcal{F}_{- y}]
    \end{equation*}
\end{lemma}
\begin{proof}
    For any $\psi(y)$, it can be verified by letting $y\leftarrow -y$ in the second term, and note that $\mathcal{F}_{-y}[\psi(y)]=\psi(-y)$
    \begin{eqnarray*}
        & &\mathbb{E}_{s \sim p (s \mid \theta^{\ast}, \pi^{\ast})} \psi(y)
         \mathbb{E}_{x\sim\mathcal{N}(0, 1)}\mathbb{E}_{y\mid x\sim 
        \pi^\ast(1) \mathcal{N}(\langle x, \theta^\ast \rangle, \sigma^2) + \pi^\ast(2) \mathcal{N}(-\langle x, \theta^\ast\rangle, \sigma^2)} \psi(y)\\
        &= & \pi^\ast(1) \mathbb{E}_{x\sim\mathcal{N}(0, 1)} \mathbb{E}_{y\sim \mathcal{N}(\langle x, \theta^\ast \rangle, \sigma^2)} \psi(y)%\\
        + \pi^\ast(2) \mathbb{E}_{x\sim\mathcal{N}(0, 1)} \mathbb{E}_{y\sim \mathcal{N}(-\langle x, \theta^\ast \rangle, \sigma^2)} \psi(y)\\
        & & \pi^\ast(1) \mathbb{E}_{x\sim\mathcal{N}(0, 1)} \mathbb{E}_{y\sim \mathcal{N}(\langle x, \theta^\ast \rangle, \sigma^2)} \psi(y)%\\
        + \pi^\ast(2) \mathbb{E}_{x\sim\mathcal{N}(0, 1)} \mathbb{E}_{y\sim \mathcal{N}(\langle x, \theta^\ast \rangle, \sigma^2)} \psi(-y)\\
        &=&\mathbb{E}_{x \sim \mathcal{N} (0, I_d)} \mathbb{E}_{y \sim \mathcal{N}
        (\langle x, \theta^{\ast} \rangle, \sigma^2)} [\pi^{\ast} (1) + \pi^{\ast}
        (2) \mathcal{F}_{- y}] \cdot \psi(y)
    \end{eqnarray*}
\end{proof}

\begin{lemma}
    Let $\bar{\theta}:=\frac{\theta}{\sigma},\bar{\theta^\ast}:=\frac{\theta^\ast}{\sigma}$
    and $\rho \assign \frac{\langle \theta,
    \theta^{\ast} \rangle}{\| \theta \| \cdot \| \theta^{\ast} \|}, \nu^\ast \assign \frac{\log \pi^\ast(1)-\log \pi^\ast(2)}{2}$, then
    \begin{eqnarray*}
      &  & \mathbb{E}_{x \sim \mathcal{N} (0, I_d)}
      \left[\pi^{\ast}
      (1) \mathcal{N} (- \langle x, \bar{\theta}^{\ast} \rangle \langle x,
      \bar{\theta} \rangle, \langle x, \bar{\theta} \rangle^2) + \pi^{\ast} (2)
      \mathcal{N} (\langle x, \bar{\theta}^{\ast} \rangle \langle x,
      \bar{\theta} \rangle, \langle x, \bar{\theta} \rangle^2)\right]\\
      & = & (\pi \| \bar{\theta} \|)^{- 1}  (1 + (1 -
      \rho^2) \| \bar{\theta}^{\ast} \|^2)^{- \frac{1}{2}}\\
      &  & K_0 \left( \frac{\sqrt{1 + \| \bar{\theta}^{\ast} \|^2} \cdot \left|
      \frac{\nu}{\| \bar{\theta} \|} \right|}{[1 + (1 - \rho^2) \|
      \bar{\theta}^{\ast} \|^2]} \right) \cosh^{- 1} (\nu^{\ast}) \cosh \left(
      \frac{\rho \| \bar{\theta}^{\ast} \| \left( \frac{\nu}{\| \bar{\theta} \|}
      \right)}{[1 + (1 - \rho^2) \| \bar{\theta}^{\ast} \|^2]} - \nu^{\ast}
      \right)
    \end{eqnarray*}
    
\end{lemma}
\begin{proof}
    \ 
    With the previous Lemma and $\pi^\ast(1)=\frac{1+\tanh(\nu^\ast)}{2}, \pi^\ast(2)=\frac{1-\tanh(\nu^\ast)}{2}$, we obtain the follows.
    \begin{eqnarray*}
      &  & \mathbb{E}_{x \sim \mathcal{N} (0, I_d)} \left[\pi^{\ast}
      (1) \mathcal{N} (- \langle x, \bar{\theta}^{\ast} \rangle \langle x,
      \bar{\theta} \rangle, \langle x, \bar{\theta} \rangle^2) + \pi^{\ast} (2)
      \mathcal{N} (\langle x, \bar{\theta}^{\ast} \rangle \langle x,
      \bar{\theta} \rangle, \langle x, \bar{\theta} \rangle^2)\right]\\
      & = & (\pi \| \bar{\theta} \|)^{- 1} {(1 + (1 - \rho^2) \|
      \bar{\theta}^{\ast} \|^2)^{- \frac{1}{2}}}  K_0 \left( \frac{\sqrt{1 + \|
      \bar{\theta}^{\ast} \|^2} \cdot \left| \frac{\nu}{\| \bar{\theta} \|}
      \right|}{[1 + (1 - \rho^2) \| \bar{\theta}^{\ast} \|^2]} \right)\\
      & \cdot & \left[ \cosh \left( \frac{\rho \| \bar{\theta}^{\ast} \| \left(
      \frac{\nu}{\| \bar{\theta} \|} \right)}{[1 + (1 - \rho^2) \|
      \bar{\theta}^{\ast} \|^2]} \right) - \tanh (\nu^{\ast}) \sinh \left(
      \frac{\rho \| \bar{\theta}^{\ast} \| \left( \frac{\nu}{\| \bar{\theta} \|}
      \right)}{[1 + (1 - \rho^2) \| \bar{\theta}^{\ast} \|^2]} \right) \right]\\
    \end{eqnarray*}
    Note that $\cosh(\alpha)- \tanh(\beta) \sinh(\alpha)= \frac{\cosh(\alpha - \beta)}{\cosh(\beta)}$,
    let $\alpha \leftarrow  \frac{\rho \| \bar{\theta}^{\ast} \| \left(
        \frac{\nu}{\| \bar{\theta} \|} \right)}{[1 + (1 - \rho^2) \|
        \bar{\theta}^{\ast} \|^2]}, \beta \leftarrow \nu^{\ast}$, we further prove this Lemma.
\end{proof}

%\newpage
\begin{lemma}
    Let $\bar{\theta}:=\frac{\theta}{\sigma},\bar{\theta^\ast}:=\frac{\theta^\ast}{\sigma}$
    and $\rho \assign \frac{\langle \theta,
    \theta^{\ast} \rangle}{\| \theta \| \cdot \| \theta^{\ast} \|}, \nu^\ast \assign \frac{\log \pi^\ast(1)-\log \pi^\ast(2)}{2}$, then
    \begin{eqnarray*}
      &  & \mathbb{E}_{x \sim \mathcal{N} (0,
      I_d)} \left[ \frac{x}{\langle x, \bar{\theta} \rangle} \right] \left[\pi^{\ast}
      (1) \mathcal{N} (- \langle x, \bar{\theta}^{\ast} \rangle \langle x,
      \bar{\theta} \rangle, \langle x, \bar{\theta} \rangle^2) + \pi^{\ast} (2)
      \mathcal{N} (\langle x, \bar{\theta}^{\ast} \rangle \langle x,
      \bar{\theta} \rangle, \langle x, \bar{\theta} \rangle^2)\right]\\
      & = & \frac{1}{\pi} \cdot \frac{\| \bar{\theta}^{\ast} \|^2}{\|
      \bar{\theta} \|^2} \cdot \frac{\sqrt{1 - \rho^2} \cosh^{- 1}
      (\nu^{\ast})}{{(1 + (1 - \rho^2) \| \bar{\theta}^{\ast}
      \|^2)^{\frac{3}{2}}} } \\
      & \Bigg\{ & \cosh \left[ \frac{\rho \| \bar{\theta}^{\ast} \| \left(
      \frac{\nu}{\| \bar{\theta} \|} \right)}{[1 + (1 - \rho^2) \|
      \bar{\theta}^{\ast} \|^2]} - \nu^{\ast} \right] \cdot K_0 \left(
      \frac{\sqrt{1 + \| \bar{\theta}^{\ast} \|^2} \cdot \left| \frac{\nu}{\|
      \bar{\theta} \|} \right|}{[1 + (1 - \rho^2) \| \bar{\theta}^{\ast} \|^2]}
      \right) \left[ \frac{1}{\sqrt{1 - \rho^2} \| \bar{\theta}^{\ast} \|^2}
      \cdot \frac{\bar{\theta}}{\| \bar{\theta} \|} + \hat{e}_2 \right]\\
      & + & \tmop{sgn} (\nu) \frac{\sqrt{1 + \| \bar{\theta}^{\ast} \|^2}}{\|
      \bar{\theta}^{\ast} \|} \sinh \left[ \frac{\rho \| \bar{\theta}^{\ast} \|
      \left( \frac{\nu}{\| \bar{\theta} \|} \right)}{[1 + (1 - \rho^2) \|
      \bar{\theta}^{\ast} \|^2]} - \nu^{\ast} \right] \cdot K_1 \left(
      \frac{\sqrt{1 + \| \bar{\theta}^{\ast} \|^2} \cdot \left| \frac{\nu}{\|
      \bar{\theta} \|} \right|}{[1 + (1 - \rho^2) \| \bar{\theta}^{\ast} \|^2]}
      \right) \vec{e}_2 \Bigg\}
    \end{eqnarray*}
    
\end{lemma}
\begin{proof}
    Note that $\frac{1+\tanh(\beta)}{2} \exp(-\alpha) + \frac{1-\tanh(\beta)}{2} \exp(\alpha)
    =\cosh(\alpha) - \tanh(\beta) \sinh(\alpha)
    =\frac{\cosh(\alpha-\beta)}{\cosh(\beta)}$ ,

    and  $-\frac{1+\tanh(\beta)}{2} \exp(-\alpha) + \frac{1-\tanh(\beta)}{2} \exp(\alpha)
    =\sinh(\alpha) - \tanh(\beta) \cosh(\alpha)
    =\frac{\sinh(\alpha-\beta)}{\cosh(\beta)}$,

    let $\alpha \leftarrow  \frac{\rho \| \bar{\theta}^{\ast} \| \left(
        \frac{\nu}{\| \bar{\theta} \|} \right)}{[1 + (1 - \rho^2) \|
        \bar{\theta}^{\ast} \|^2]}, \beta \leftarrow \nu^{\ast}$, and use previous Lemma, we give the following.
\small
\normalsize
      With Lemma for relations of unit vectors, we obtain $\hat{e}_2+\rho \vec{e}_2=\sqrt{1-\rho^2} \frac{\theta}{\|\theta\|}$, and by definition $\vec{e}_1=\frac{\theta}{\|\theta\|}$ and $\vec{e}_2=-\hat{e}_2+ \sqrt{1-\rho^2} \frac{\theta}{\|\theta\|}$.
      \begin{eqnarray*}
        \vec{e}_1 - \frac{\rho \sqrt{1 -
        \rho^2} \| \bar{\theta}^{\ast} \|^2}{1 + (1 - \rho^2) \|
        \bar{\theta}^{\ast} \|^2} \vec{e}_2
        % =
        % \frac{\theta}{\|\theta\|} - \frac{\sqrt{1 -
        % \rho^2} \| \bar{\theta}^{\ast} \|^2}{1 + (1 - \rho^2) \|
        % \bar{\theta}^{\ast} \|^2} 
        % \left[-\hat{e}_2+  \sqrt{1-\rho^2} \frac{\theta}{\|\theta\|}\right]%\\
        = \left(1 + (1 - \rho^2) \|
        \bar{\theta}^{\ast} \|^2 \right)^{-1} 
        \left[ \frac{\theta}{\|\theta\|} + \sqrt{1-\rho^2}\|
        \bar{\theta}^{\ast} \|^2 \hat{e}_2 \right]
      \end{eqnarray*}
      Hence, this Lemma is proved by rearranging the terms and using the above relation for vectors.
\end{proof}

%\newpage
\begin{lemma}
        For the 2MLR at the population level, let $\rho \assign \frac{\langle \theta,
        \theta^{\ast} \rangle}{\| \theta \| \cdot \| \theta^{\ast} \|}, \nu^\ast \assign \frac{\log \pi^\ast(1)-\log \pi^\ast(2)}{2}$, 
        \begin{eqnarray*}
          &  & \mathbb{E}_{s \sim p (s \mid \theta^{\ast}, \pi^{\ast})} \log \cosh
          \left( \frac{y \langle x, \theta \rangle}{\sigma^2} + \nu \right)\\
          & = & \log \cosh (\nu) \ast (\pi \| \bar{\theta} \|)^{- 1}  (1 + (1 -
          \rho^2) \| \bar{\theta}^{\ast} \|^2)^{- \frac{1}{2}}\\
          &  & K_0 \left( \frac{\sqrt{1 + \| \bar{\theta}^{\ast} \|^2} \cdot \left|
          \frac{\nu}{\| \bar{\theta} \|} \right|}{[1 + (1 - \rho^2) \|
          \bar{\theta}^{\ast} \|^2]} \right) \cosh^{- 1} (\nu^{\ast}) \cosh \left(
          \frac{\rho \| \bar{\theta}^{\ast} \| \left( \frac{\nu}{\| \bar{\theta} \|}
          \right)}{[1 + (1 - \rho^2) \| \bar{\theta}^{\ast} \|^2]} - \nu^{\ast}
          \right)
        \end{eqnarray*}
\end{lemma}

      \begin{proof}
        % let's analyze this term with both $\theta, \nu$ in $f (\theta, \pi) \assign
        % -\mathbb{E}_{s \sim p (s \mid \theta^{\ast}, \pi^{\ast})} [\log p (s \mid
        % \theta, \pi)]$
        Note that $\mathbb{E}_{s \sim p (s \mid \theta^{\ast}, \pi^{\ast})}
        =\mathbb{E}_{x \sim \mathcal{N} (0, I_d)} \mathbb{E}_{y \sim \mathcal{N}
        (\langle x, \theta^{\ast} \rangle, \sigma^2)} [\pi^{\ast} (1) + \pi^{\ast}
        (2) \mathcal{F}_{- y}]$ in a previous Lemma;

        and for convolution $(f \ast g) (\nu) = \int_{- \infty}^{+ \infty} f
        (\nu - \nu') g (\nu') \mathd \nu'$, we can exchange expectation and
        convolution.

        Let $\nu' \assign - \frac{y \langle x, \theta \rangle}{\sigma^2}$, then
        $\nu' \sim \mathcal{N} \left( - \frac{\langle x, \theta^{\ast} \rangle
        \langle x, \theta \rangle}{\sigma^2}, \frac{\langle x, \theta
        \rangle^2}{\sigma^2} \right)$; and $\nu'' \assign \frac{y \langle x, \theta
        \rangle}{\sigma^2}$, then $\nu'' \sim \mathcal{N} \left( \frac{\langle x,
        \theta^{\ast} \rangle \langle x, \theta \rangle}{\sigma^2}, \frac{\langle x,
        \theta \rangle^2}{\sigma^2} \right)$.
        \begin{eqnarray*}
          &  & \mathbb{E}_{s \sim p (s \mid \theta^{\ast}, \pi^{\ast})} \log \cosh
          \left( \frac{y \langle x, \theta \rangle}{\sigma^2} + \nu \right)
          = \mathbb{E}_{x \sim \mathcal{N} (0, I_d)} \mathbb{E}_{y \sim \mathcal{N}
          (\langle x, \theta^{\ast} \rangle, \sigma^2)} [\pi^{\ast} (1) + \pi^{\ast}
          (2) \mathcal{F}_{- y}] 
          \log \cosh
          \left( \frac{y \langle x, \theta \rangle}{\sigma^2} + \nu \right)
          \\
          & = & \mathbb{E}_{x \sim \mathcal{N} (0, I_d)} \mathbb{E}_{y \sim
          \mathcal{N} (\langle x, \theta^{\ast} \rangle, \sigma^2)} \left[
          \pi^{\ast} (1) \log \cosh \left( \frac{y \langle x, \theta
          \rangle}{\sigma^2} + \nu \right) + \pi^{\ast} (2) \log \cosh \left( -
          \frac{y \langle x, \theta \rangle}{\sigma^2} + \nu \right) \right]\\
          & = & \pi^{\ast} (1) \mathbb{E}_{x \sim \mathcal{N} (0, I_d)}
          {\mathbb{E}_{\nu' \sim \mathcal{N} \left( - \frac{\langle x, \theta^{\ast}
          \rangle \langle x, \theta \rangle}{\sigma^2}, \frac{\langle x, \theta
          \rangle^2}{\sigma^2} \right)}}  \log \cosh (\nu - \nu')\\
          & + & \pi^{\ast} (2) \mathbb{E}_{x \sim \mathcal{N} (0, I_d)}
          {\mathbb{E}_{\nu'' \sim \mathcal{N} \left( \frac{\langle x, \theta^{\ast}
          \rangle \langle x, \theta \rangle}{\sigma^2}, \frac{\langle x, \theta
          \rangle^2}{\sigma^2} \right)}}  \log \cosh (\nu - \nu'')\\
          & = & \log \cosh (\nu) \ast \mathbb{E}_{x \sim \mathcal{N} (0, I_d)}
          \left[\pi^{\ast}
          (1) \mathcal{N} (- \langle x, \bar{\theta}^{\ast} \rangle \langle x,
          \bar{\theta} \rangle, \langle x, \bar{\theta} \rangle^2) + \pi^{\ast} (2)
          \mathcal{N} (\langle x, \bar{\theta}^{\ast} \rangle \langle x,
          \bar{\theta} \rangle, \langle x, \bar{\theta} \rangle^2)\right]
        \end{eqnarray*}
        Then with closed-form expression in the previous Lemma, we further prove this Lemma.
      \end{proof}

    \begin{lemma}
        For the 2MLR at the population level, let $\rho \assign \frac{\langle \theta,
        \theta^{\ast} \rangle}{\| \theta \| \cdot \| \theta^{\ast} \|}, \nu^\ast \assign \frac{\log \pi^\ast(1)-\log \pi^\ast(2)}{2}$, 
        \begin{eqnarray*}
            &  & \mathbb{E}_{s \sim p (s \mid \theta^{\ast}, \pi^{\ast})} \tanh
            \left( \frac{y \langle x, \theta \rangle}{\sigma^2} + \nu \right)\\
            & = & \tanh (\nu) \ast (\pi \| \bar{\theta} \|)^{- 1}  (1 + (1 -
            \rho^2) \| \bar{\theta}^{\ast} \|^2)^{- \frac{1}{2}}\\
            &  & K_0 \left( \frac{\sqrt{1 + \| \bar{\theta}^{\ast} \|^2} \cdot \left|
            \frac{\nu}{\| \bar{\theta} \|} \right|}{[1 + (1 - \rho^2) \|
            \bar{\theta}^{\ast} \|^2]} \right) \cosh^{- 1} (\nu^{\ast}) \cosh \left(
            \frac{\rho \| \bar{\theta}^{\ast} \| \left( \frac{\nu}{\| \bar{\theta} \|}
            \right)}{[1 + (1 - \rho^2) \| \bar{\theta}^{\ast} \|^2]} - \nu^{\ast}
            \right)\\
            & = & \frac{\left(1+\left(1-\rho^2\right)\|\bar{\theta}^\ast\|^2\right)^{-\frac{1}{2}}}{\pi\|\bar{\theta}\| \cosh \left(\nu^*\right)} 
            \int_{\mathbb{R}} \tanh \left(\nu-\nu'\right) 
            K_0\left(\frac{\sqrt{1+\|\bar{\theta}^\ast\|^2} \cdot\left|\frac{\nu'}{\|\bar{\theta}\|}\right|}{1+\left(1-\rho^2\right)\|\bar{\theta}^\ast\|^2}\right) 
            \cosh \left(\frac{\rho\left\|\bar{\theta}^\ast\right\|\left(\frac{\nu'}{\|\bar{\theta}\|}\right)}{1+\left(1-\rho^2\right)\|\bar{\theta}^\ast\|^2}-\nu^*\right) \mathrm{d} \nu'
          \end{eqnarray*}
    \end{lemma}

    \begin{proof}
        We follow the same steps in the proof of the previous Lemma, but substitute $\log \cosh$ with $\tanh$.
    \end{proof}

%\newpage
    \begin{lemma}
        For the 2MLR at the population level, let $\rho \assign \frac{\langle \theta,
        \theta^{\ast} \rangle}{\| \theta \| \cdot \| \theta^{\ast} \|}, \nu^\ast \assign \frac{\log \pi^\ast(1)-\log \pi^\ast(2)}{2}$, 
        \begin{eqnarray*}
          &  & \mathbb{E}_{s \sim p (s \mid \theta^{\ast}, \pi^{\ast})} \tanh
          \left( \frac{y \langle x, \theta \rangle}{\sigma^2} + \nu \right) y x\\
          & = & - \sigma \frac{1}{\pi} \cdot \frac{\| \bar{\theta}^{\ast} \|^2}{\|
          \bar{\theta} \|^2} \cdot \frac{\sqrt{1 - \rho^2} \cosh^{- 1}
          (\nu^{\ast})}{{(1 + (1 - \rho^2) \| \bar{\theta}^{\ast}
          \|^2)^{\frac{3}{2}}} } \tanh (\nu) \ast \nu\\
          & \Bigg\{ &  \cosh \left( \frac{\rho \| \bar{\theta}^{\ast} \| \left(
          \frac{\nu}{\| \bar{\theta} \|} \right)}{[1 + (1 - \rho^2) \|
          \bar{\theta}^{\ast} \|^2]} - \nu^{\ast} \right) \cdot K_0 \left(
          \frac{\sqrt{1 + \| \bar{\theta}^{\ast} \|^2} \cdot \left| \frac{\nu}{\|
          \bar{\theta} \|} \right|}{[1 + (1 - \rho^2) \| \bar{\theta}^{\ast} \|^2]}
          \right) \left[ \frac{1}{\sqrt{1 - \rho^2} \| \bar{\theta}^{\ast} \|^2}
          \cdot \frac{\bar{\theta}}{\| \bar{\theta} \|} + \hat{e}_2 \right]\\
          &  & + \tmop{sgn} (\nu) \frac{\sqrt{1 + \| \bar{\theta}^{\ast} \|^2}}{\|
          \bar{\theta}^{\ast} \|} \sinh \left( \frac{\rho \| \bar{\theta}^{\ast} \|
          \left( \frac{\nu}{\| \bar{\theta} \|} \right)}{[1 + (1 - \rho^2) \|
          \bar{\theta}^{\ast} \|^2]} - \nu^{\ast} \right) \cdot K_1 \left(
          \frac{\sqrt{1 + \| \bar{\theta}^{\ast} \|^2} \cdot \left| \frac{\nu}{\|
          \bar{\theta} \|} \right|}{[1 + (1 - \rho^2) \| \bar{\theta}^{\ast} \|^2]}
          \right) \vec{e}_2 
          \Bigg\}
        \end{eqnarray*}
    \end{lemma}

    \begin{proof}
        Note that $\mathbb{E}_{s \sim p (s \mid \theta^{\ast}, \pi^{\ast})}
        =\mathbb{E}_{x \sim \mathcal{N} (0, I_d)} \mathbb{E}_{y \sim \mathcal{N}
        (\langle x, \theta^{\ast} \rangle, \sigma^2)} [\pi^{\ast} (1) + \pi^{\ast}
        (2) \mathcal{F}_{- y}]$ in a previous Lemma;

        and for convolution $(f \ast g) (\nu) = \int_{- \infty}^{+ \infty} f
        (\nu - \nu') g (\nu') \mathd \nu'$, we can exchange expectation and
        convolution.

        Let $\nu' \assign - \frac{y \langle x, \theta \rangle}{\sigma^2}$, then
        $\nu' \sim \mathcal{N} \left( - \frac{\langle x, \theta^{\ast} \rangle
        \langle x, \theta \rangle}{\sigma^2}, \frac{\langle x, \theta
        \rangle^2}{\sigma^2} \right)$; and $\nu'' \assign \frac{y \langle x, \theta
        \rangle}{\sigma^2}$, then $\nu'' \sim \mathcal{N} \left( \frac{\langle x,
        \theta^{\ast} \rangle \langle x, \theta \rangle}{\sigma^2}, \frac{\langle x,
        \theta \rangle^2}{\sigma^2} \right)$.
        \begin{eqnarray*}
          &  & \mathbb{E}_{s \sim p (s \mid \theta^{\ast}, \pi^{\ast})} \tanh
          \left( \frac{y \langle x, \theta \rangle}{\sigma^2} + \nu \right) x y\\
          & = & \mathbb{E}_{x \sim \mathcal{N} (0, I_d)} x \cdot \mathbb{E}_{y \sim
          \mathcal{N} (\langle x, \theta^{\ast} \rangle, \sigma^2)} [\pi^{\ast} (1)
          + \pi^{\ast} (2) \mathcal{F}_{- y}] \tanh \left( \frac{y \langle x, \theta
          \rangle}{\sigma^2} + \nu \right) y\\
          & = & \mathbb{E}_{x \sim \mathcal{N} (0, I_d)} x\mathbb{E}_{y \sim
          \mathcal{N} (\langle x, \theta^{\ast} \rangle, \sigma^2)} \left[
          \pi^{\ast} (1) \tanh \left( \frac{y \langle x, \theta \rangle}{\sigma^2} +
          \nu \right) y - \pi^{\ast} (2) \tanh \left( - \frac{y \langle x, \theta
          \rangle}{\sigma^2} + \nu \right) y \right]\\
          & = & - \pi^{\ast} (1) \mathbb{E}_{x \sim \mathcal{N} (0, I_d)} \left[
          \frac{x \sigma^2}{\langle x, \theta \rangle} \right] {\mathbb{E}_{\nu'
          \sim \mathcal{N} \left( - \frac{\langle x, \theta^{\ast} \rangle \langle
          x, \theta \rangle}{\sigma^2}, \frac{\langle x, \theta \rangle^2}{\sigma^2}
          \right)}}  \tanh (\nu - \nu') \nu'\\
          &  & - \pi^{\ast} (2) \mathbb{E}_{x \sim \mathcal{N} (0, I_d)} \left[
          \frac{x \sigma^2}{\langle x, \theta \rangle} \right] {\mathbb{E}_{\nu''
          \sim \mathcal{N} \left( \frac{\langle x, \theta^{\ast} \rangle \langle x,
          \theta \rangle}{\sigma^2}, \frac{\langle x, \theta \rangle^2}{\sigma^2}
          \right)}}  \tanh (\nu - \nu'') \nu''\\
          & = & - \sigma \tanh (\nu) \ast \nu \mathbb{E}_{x \sim \mathcal{N} (0,
          I_d)} \left[ \frac{x}{\langle x, \bar{\theta} \rangle} \right] (\pi^{\ast}
          (1) \mathcal{N} (- \langle x, \bar{\theta}^{\ast} \rangle \langle x,
          \bar{\theta} \rangle, \langle x, \bar{\theta} \rangle^2) + \pi^{\ast} (2)
          \mathcal{N} (\langle x, \bar{\theta}^{\ast} \rangle \langle x,
          \bar{\theta} \rangle, \langle x, \bar{\theta} \rangle^2))
        \end{eqnarray*}
    Then, with the previous Lemma of expectaions for 2MLR, we further prove the closed-form expression.
    \end{proof}

%% file: 8_supplementary_1.tex
\section{Derivations for EM Update Rules}\label{sup:derive_em}
\begin{lemma}
  The negative expected log-likelihood $f (\theta, \pi) \assign -\mathbb{E}_{s \sim p (s \mid
  \theta^{\ast}, \pi^{\ast})} [\log p (s \mid \theta, \pi)]$ for the mixture model of $s:=(x, y), z\in[M]$ 
  with the mixing weights $\pi^\ast\in\mathbb{R}^M$ and regression parameters $\theta$ is as follows.
  \begin{eqnarray*}
      -f(\theta, \pi)   & = & -\tmop{KL}_s [p (s \mid \theta^{\ast}, \pi^{\ast}) | | p (s \mid
      \theta, \pi)] -\mathcal{H}_s [p (s \mid \theta^{\ast}, \pi^{\ast})]\\
      & = & \mathbb{E}_{s \sim p (s \mid \theta^{\ast}, \pi^{\ast})} [\log p (s
      \mid \theta, \pi)]\\
      & = & \mathbb{E}_{s \sim p (s \mid \theta^{\ast}, \pi^{\ast})}
      \mathbb{E}_{z \sim q_s (z)} \log p (s, z \mid \theta, \pi) +\mathbb{E}_{s
      \sim p (s \mid \theta^{\ast}, \pi^{\ast})} \mathcal{H}_z [q_s (z)]
      +\mathbb{E}_{s \sim p (s \mid \theta^{\ast}, \pi^{\ast})} \tmop{KL}_z [q_s
      (z) | | p (z \mid s ; \theta, \pi)]
  \end{eqnarray*}
  where $\tmop{KL}_s, \mathcal{H}_s$ are KL divengence and Shannon's
  entropy wrt. $s = (x, y)$;
  
  $\tmop{KL}_z, \mathcal{H}_z, \tmop{softmax}_z$ are KL divengence, Shannon's
  entropy and softmax wrt. $z \in \mathcal{Z}= [M]$;
  
  $\{ q_s (z) \mid s \in \mathcal{X} \times \mathcal{Y}=\mathbb{R}^d \times
  \mathbb{R} \}$ is a family of distributions wrt. $z \in \mathcal{Z}= [M]$,
  namely $\sum_{z \in \mathcal{Z}} q_s (z) = 1$.
  \end{lemma}
  \begin{proof}
  Note that $p(s \mid \theta, \pi)=\frac{p(s, z \mid \theta, \pi)}{p(z \mid s ; \theta, \pi)}$, we obtain the following expression.
  \begin{eqnarray*}
      -f(\theta, \pi)   & = & -\tmop{KL}_s [p (s \mid \theta^{\ast}, \pi^{\ast}) | | p (s \mid
      \theta, \pi)] -\mathcal{H}_s [p (s \mid \theta^{\ast}, \pi^{\ast})]\\
      & = & \mathbb{E}_{s \sim p (s \mid \theta^{\ast}, \pi^{\ast})} [\log p (s
      \mid \theta, \pi)]\\
      & = & \mathbb{E}_{s \sim p (s \mid \theta^{\ast}, \pi^{\ast})} \left[
      \sum_{z \in \mathcal{Z}} q_s (z) \log p (s \mid \theta, \pi) \right]\\
      & = & \mathbb{E}_{s \sim p (s \mid \theta^{\ast}, \pi^{\ast})} \left[ 
      \sum_{z \in \mathcal{Z}} q_s (z) \log \left( \frac{p (s, z \mid \theta,
      \pi)}{q_s (z)} \cdot \frac{q_s (z)}{p (z \mid s ; \theta, \pi)} \right) \right]\\
      & = & \mathbb{E}_{s \sim p (s \mid \theta^{\ast}, \pi^{\ast})}
      \mathbb{E}_{z \sim q_s (z)} \log p (s, z \mid \theta, \pi) +\mathbb{E}_{s
      \sim p (s \mid \theta^{\ast}, \pi^{\ast})} \mathcal{H}_z [q_s (z)]
      +\mathbb{E}_{s \sim p (s \mid \theta^{\ast}, \pi^{\ast})} \tmop{KL}_z [q_s
      (z) | | p (z \mid s ; \theta, \pi)]
  \end{eqnarray*}
  \end{proof}
  
  \begin{lemma}
      The surrogate function $g^t$ of $f (\theta, \pi) \assign -\mathbb{E}_{s \sim p (s \mid
      \theta^{\ast}, \pi^{\ast})} [\log p (s \mid \theta, \pi)]$ at $(t-1)$-th iteration $(\theta^{t-1}, \pi^{t-1})$ be expressed as follows.
      \begin{eqnarray*}
          -g^t(\theta, \pi)   
          & = & \Bigg\{ \mathbb{E}_{s \sim p (s \mid \theta^{\ast}, \pi^{\ast})}
          \mathbb{E}_{z \sim q_s (z)} \log p (s, z \mid \theta, \pi) +\mathbb{E}_{s
          \sim p (s \mid \theta^{\ast}, \pi^{\ast})} \mathcal{H}_z [q_s (z)]
          \Bigg\}_{q_s
          (z) =p (z \mid s ; \theta^{t-1}, \pi^{t-1})}
      \end{eqnarray*}
      that is $g^t(\theta, \pi)\geq f(\theta, \pi)$, and 
      $g^t(\theta, \pi)\mid_{(\theta, \pi)=(\theta^{t-1}, \pi^{t-1})}= f(\theta, \pi)\mid_{(\theta, \pi)=(\theta^{t-1}, \pi^{t-1})}$,
  
      $\nabla_\theta g^t(\theta, \pi)\mid_{(\theta, \pi)=(\theta^{t-1}, \pi^{t-1})}=  \nabla_\theta f(\theta, \pi)\mid_{(\theta, \pi)=(\theta^{t-1}, \pi^{t-1})}$;
      
      where $\tmop{KL}_s, \mathcal{H}_s$ are KL divengence and Shannon's
      entropy wrt. $s = (x, y)$;
      
      $\tmop{KL}_z, \mathcal{H}_z, \tmop{softmax}_z$ are KL divengence, Shannon's
      entropy and softmax wrt. $z \in \mathcal{Z}= [M]$
  \end{lemma}
  \begin{proof}
      Let $r^t \assign g^t -f$, note that $r^t = \mathbb{E}_{s \sim p (s \mid \theta^{\ast}, \pi^{\ast})} \tmop{KL}_z [q_s
      (z) | | p (z \mid s ; \theta, \pi)]_{q_s
      (z) =p (z \mid s ; \theta^{t-1}, \pi^{t-1})}\geq 0$, and 
      \begin{eqnarray*}
          r^t(\theta^{t-1}, \pi^{t-1}) &=& \mathbb{E}_{s \sim p (s \mid \theta^{\ast}, \pi^{\ast})} \tmop{KL}_z [q_s
          (z) | | p (z \mid s ; \theta^{t-1}, \pi^{t-1})]_{q_s
          (z) =p (z \mid s ; \theta^{t-1}, \pi^{t-1})}= 0\\
          \left[\frac{\mathd r^t}{\mathd p (z \mid s ; \theta, \pi)}\right]_{(\theta, \pi)=(\theta^{t-1}, \pi^{t-1})} &=& 
          - \frac{q_s(z) }{ p (z \mid s ; \theta, \pi)}_{q_s
          (z) =p (z \mid s ; \theta^{t-1}, \pi^{t-1}), (\theta, \pi)=(\theta^{t-1}, \pi^{t-1})} = 0
      \end{eqnarray*}
      Hence, the gradients of $r^t$ wrt. $(\theta, \pi)$ at $(t-1)$-th iteration $(\theta^{t-1}, \pi^{t-1})$ are all 0 by the chain rule. 
  \end{proof}

  \newpage
  \begin{lemma}
  Assuming $(z; \pi)\ind \theta$, and $\pi \ind s \mid z$, and $x\ind (z; \theta, \pi)$; then 
  \begin{eqnarray*}
      Q(\theta, \pi\mid \theta^{t-1}, \pi^{t-1}) & \assign& 
      \left[ \mathbb{E}_{s \sim p (s \mid \theta^{\ast}, \pi^{\ast})}
      \mathbb{E}_{z \sim q_s (z)} \log p (s, z \mid \theta, \pi)\right]_{q_s
      (z) \leftarrow p (z \mid s ; \theta^{t-1}, \pi^{t-1})}\\
      &= & \mathbb{E}_{s \sim p\left(s \mid \theta^*, \pi^*\right)} \mathbb{E}_{z \sim q_s(z)} \log p(y \mid x, z ; \theta) \\
      & + &\mathbb{E}_{s \sim p\left(s \mid \theta^*, \pi^*\right)} \log p(x) \\
      & - &\mathrm{KL}_z\left[\pi^t(z)|| \pi(z)\right] \\
      & - &\mathcal{H}_z\left[\pi^t(z)\right]\\
      p (z \mid s ; \theta, \pi) &=& \tmop{softmax}_z
      (\log \pi (z) + \log p (y \mid x, z ; \theta))
  \end{eqnarray*}
  where $\tmop{KL}_s, \mathcal{H}_s$ are KL divengence and Shannon's
  entropy wrt. $s = (x, y)$;
  
  $\tmop{KL}_z, \mathcal{H}_z, \tmop{softmax}_z$ are KL divengence, Shannon's
  entropy and softmax wrt. $z \in \mathcal{Z}= [M]$;
  
  $q_s(z) \leftarrow p\left(z \mid s ; \theta^{t-1}, \pi^{t-1}\right)$ and $\pi^t=\{\pi(z)\}_{z\in \mathcal{Z}},\pi^t(z) \assign \mathbb{E}_{s \sim p\left(s \mid \theta^*, \pi^*\right)} q_s(z)$.
  \end{lemma}
  
  \begin{proof}
      \
  \begin{itemizedot}
  
      \item $(z ; \pi) \ind \theta$ are independent:
      
      therefore $p (z \mid \theta, \pi) = p (z \mid \pi) = \pi (z)$
      
      \item $\pi \ind (s ; \theta) \mid z$ are conditional independent given $z$:
      
      then $p (\pi \mid z) = p (\pi \mid z; \theta) = p (\pi \mid z, s ; \theta)$,
      it implies $\frac{p (s , z ; \theta, \pi)}{p (z ; \theta, \pi)} =\frac{p (s , z ; \theta)}{p (z ; \theta)}$;
      
      hence $p (s \mid z ; \theta, \pi) = p (s \mid z ; \theta)$, $p (s, z \mid
      \theta, \pi) = p (s \mid z ; \theta, \pi) p (z \mid \theta, \pi) = p (s \mid
      z ; \theta) \pi (z)$
      
      \item $x \ind (z ; \theta, \pi)$ are independent:
      
      then $p (x \mid z ; \theta) = p (x)$
      
      hence $p (s \mid z ; \theta) = p (y \mid x, z ; \theta) \cdot p (x \mid z ;
      \theta) = p (y \mid x, z ; \theta) \cdot p (x)$
      
      therefore $p (s \mid \theta, \pi) = \sum_{z \in \mathcal{Z}} p (s, z \mid
      \theta, \pi) = p (x) \cdot \sum_{z \in \mathcal{Z}} \pi (z) p (y \mid x, z ;
      \theta)$
      
      and $p (z \mid s ; \theta, \pi) = \frac{p (s, z \mid \theta, \pi)}{p (s \mid
      \theta, \pi)} = \frac{\pi (z) \cdot p (y \mid x, z ; \theta)}{\sum_{z' \in
      \mathcal{Z}} \pi (z') \cdot p (y \mid x, z' ; \theta)} = \tmop{softmax}_z
      (\log \pi (z) + \log p (y \mid x, z ; \theta))$
  \end{itemizedot}
  With the above assumptions, we obtain that $p(s, z\mid \theta, \pi)=  p (y \mid x, z ; \theta) \cdot p (x) \cdot \pi(z)$, further prove this Lemma.
  \end{proof}

  \begin{lemma}
  For MLR $y= \langle x, \theta_z^\ast\rangle + \varepsilon, z\in\mathcal{Z}=[M]$, $\theta \assign \{\theta_z \}_{z\in\mathcal{Z}}, \pi \assign \{\pi(z)\}_{z\in\mathcal{Z}}$, 
  with assumptions: $(z; \pi)\ind \theta$, and $\varepsilon \ind (x, z; \theta, \pi)$, and $x\ind (z; \theta, \pi)$ and $\varepsilon \sim \mathcal{N}(0, \sigma^2)$
  \begin{eqnarray*}
      f(\theta, \pi) & = & -\mathbb{E}_{\mathrm{s} \sim p\left(s \mid \theta^*, \pi^*\right)} \log \sum_{z \in \mathcal{Z}} \exp \left[-\frac{\left\|y-\left\langle\theta_z, x\right\rangle\right\|^2}{2 \sigma^2}+\log \pi(z)\right]-\mathbb{E}_{s \sim p\left(s \mid \theta^*, \pi^*\right)} \log p(x)-c\\
      g^t(\theta, \pi) & =&\left(2 \sigma^2\right)^{-1} \mathbb{E}_{s \sim p\left(s \mid \theta^*, \pi^*\right)} \mathbb{E}_{z \sim q_x(z)}\left\|y-\left\langle\theta_z, x\right\rangle\right\|^2 \\
      & +&\mathrm{KL}_z\left[\pi^t(z) \| \pi(z)\right] \\
      & +&\mathcal{H}_z\left[\pi^t(z)\right]-\mathbb{E}_{s \sim p\left(s \mid \theta^*, \pi^*\right)} \log p(x)-c \\
      & -&\mathbb{E}_{s \sim p\left(s \mid \theta^*, \pi^*\right)} \mathcal{H}_z\left[q_s(z)\right]
  \end{eqnarray*}
  where $c=-\frac{1}{2} \log \left(2 \pi \sigma^2\right)$ and $q_s(z) \leftarrow p\left(z \mid s ; \theta^{t-1}, \pi^{t-1}\right)=\operatorname{softmax}_z\left(-\frac{\left\|y-\left\langle\theta_z^{t-1}, x\right\rangle\right\|^2}{2 \sigma^2}+\log \pi^{t-1}(z)\right)$,
  and $\pi^t=\{\pi(z)\}_{z\in \mathcal{Z}},\pi^t(z) \assign \mathbb{E}_{s \sim p\left(s \mid \theta^*, \pi^*\right)} q_s(z)$.
  \end{lemma}
  \begin{proof}
      Since $(z; \pi)\ind \theta$, and $\varepsilon \ind (x, z; \theta, \pi)$, and $x\ind (z; \theta, \pi)$, then implies $\pi \ind s \mid z$
      because of $p(\pi \mid z, s)=p(\pi \mid z, x, y)=p(\pi \mid z, x, \varepsilon)=p(\pi \mid z)=\pi(z)$. Hence, we can apply the previous Lemma.
      
      Furthermore, $p (y, x, z ;
      \theta) = p (\varepsilon, x, z ; \theta) \left| \frac{\partial
      \varepsilon}{\partial y} \right| = p (\varepsilon) \cdot p (x, z ; \theta)$
      \[ p (y \mid x, z ; \theta) = p (\tmmathbf{\varepsilon}) = (2 \pi \sigma^2)^{-
         \frac{1}{2}} \exp \left( - \frac{\| \varepsilon \|^2}{2 \sigma^2} \right) =
         (2 \pi \sigma^2)^{- \frac{1}{2}} \exp \left( - \frac{\| y - \langle x,
         \theta_z \rangle \|^2}{2 \sigma^2} \right) =\mathcal{N} (\langle x,
         \theta_z \rangle, \sigma^2) \]
      Hence, we obtain $\log p (y \mid x, z ; \theta) = - \frac{\| y - \langle x,
      \theta_z \rangle \|^2}{2 \sigma^2} + c$ , where $c=-\frac{1}{2} \log \left(2 \pi \sigma^2\right)$.
  
      Subsequently, note that $g^t(\theta, \pi) = Q(\theta, \pi\mid \theta^{t-1}, \pi^{t-1}) - \mathbb{E}_{s \sim p\left(s \mid \theta^*, \pi^*\right)} \mathcal{H}_z\left[q_s(z)\right]$,
      we prove the expression for $g^t$ by substituting $\log p (y \mid x, z ; \theta)$ in the previous Lemma.
  
      As shown in the proof of the previous Lemma, $p(s, z\mid \theta, \pi)=  p (y \mid x, z ; \theta) \cdot p (x) \cdot \pi(z)$.
  
      Hence $p(s\mid \theta, \pi) = \sum_{z\in\mathcal{Z}} p(s, z\mid \theta, \pi)=  p (x) \cdot \sum_{z\in\mathcal{Z}} p (y \mid x, z ; \theta) \cdot  \pi(z)$,
      we prove the expression for $f$ by substituting $\log p (y \mid x, z ; \theta)$ in the previous Lemma.
  \end{proof}

  \begin{theorem}
      For 2MLR, $y= (-1)^{z+1}\langle x, \theta^\ast\rangle + \varepsilon, z\in\mathcal{Z}=\{1, 2\}$, $\theta_1 = \theta, \theta_2 = -\theta, \pi \assign \{\pi(z)\}_{z\in\mathcal{Z}}$, 
      with assumptions: $(z; \pi)\ind \theta$, and $\varepsilon \ind (x, z; \theta, \pi)$, and $x\ind (z; \theta, \pi)$ and $\varepsilon \sim \mathcal{N}(0, \sigma^2)$.
  
      Then the negative expected log-likelihood $f (\theta, \pi) \assign
      -\mathbb{E}_{s \sim p (s \mid \theta^{\ast}, \pi^{\ast})} [\log p (s \mid
      \theta, \pi)]$ and the surrogate function $g^t (\theta, \pi)$ can be expressed as follows.
      \begin{eqnarray*}
        f  & = & (2 \sigma^2)^{- 1}  \langle \theta, \mathbb{E}_{s \sim p (s \mid
          \theta^{\ast}, \pi^{\ast})} x x^{\top} \cdot \theta \rangle + \log \cosh
          \nu -\mathbb{E}_{s \sim p (s \mid \theta^{\ast}, \pi^{\ast})} \log \cosh
          \left( \frac{y \langle x, \theta \rangle}{\sigma^2} + \nu \right)\\
          & + &  (2 \sigma^2)^{- 1} \mathbb{E}_{s \sim p (s \mid \theta^{\ast},
          \pi^{\ast})} y^2 -\mathbb{E}_{s \sim p (s \mid \theta^{\ast},
          \pi^{\ast})} \log p (x) - c\\
        g^t & = & (2 \sigma^2)^{- 1}  \langle \theta, \mathbb{E}_{s \sim p (s \mid
        \theta^{\ast}, \pi^{\ast})} x x^{\top} \cdot \theta \rangle + \log \cosh
        \nu -\mathbb{E}_{s \sim p (s \mid \theta^{\ast}, \pi^{\ast})} \log \cosh
        \left( \frac{y \langle x, \theta \rangle}{\sigma^2} + \nu \right)
        \mid_{(\theta, \nu) \leftarrow (\theta^{t - 1}, \nu^{t - 1})}\\
        & - & \left\langle \nabla_{\theta} \mathbb{E}_{s \sim p (s \mid
        \theta^{\ast}, \pi^{\ast})} \log \cosh \left( \frac{y \langle x, \theta
        \rangle}{\sigma^2} + \nu \right) \mid_{(\theta, \nu) \leftarrow (\theta^{t
        - 1}, \nu^{t - 1})}, \theta - \theta^{t - 1} \right\rangle\\
        & - & \left\langle \nabla_{\nu} \mathbb{E}_{s \sim p (s \mid
        \theta^{\ast}, \pi^{\ast})} \log \cosh \left( \frac{y \langle x, \theta
        \rangle}{\sigma^2} + \nu \right) \mid_{(\theta, \nu) \leftarrow (\theta^{t
        - 1}, \nu^{t - 1})}, \nu - \nu^{t - 1} \right\rangle\\
        & + & (2 \sigma^2)^{- 1} \mathbb{E}_{s \sim p (s \mid \theta^{\ast},
        \pi^{\ast})} y^2 -\mathbb{E}_{s \sim p (s \mid \theta^{\ast},
        \pi^{\ast})} \log p (x) - c
      \end{eqnarray*}
      where $\nu \assign \frac{\log \pi(1) -\log \pi(2)}{2}$.
  \end{theorem}

    \begin{proof}
      \
      Note that $\tmop{sigmoid} (2 t) + \tmop{sigmoid} (- 2 t) = 1$ and
      $\tmop{sigmoid} (2 t) - \tmop{sigmoid} (- 2 t) = \tanh (t)$ for $\forall t$,
  
      $q_s(z) \leftarrow p\left(z \mid s ; \theta^{t-1}, \pi^{t-1}\right)
      =\tmop{softmax}_z(-\frac{\left\|y-\left\langle\theta_z^{t-1}, x\right\rangle\right\|^2}{2 \sigma^2}+\log \pi^{t-1}(z))
      =\tmop{sigmoid}(2(-1)^{z+1}\left[\frac{y \langle x, \theta^{t-1} \rangle}{\sigma^2} +\nu^{t-1}\right])$,
  
      $\nu^{t-1}\assign \frac{\log \pi^{t-1}(1) -\log \pi^{t-1}(2)}{2}$.
      \begin{eqnarray*}
        &   & \left(2 \sigma^2\right)^{-1} \mathbb{E}_{s \sim p\left(s \mid \theta^*, \pi^*\right)} 
        \mathbb{E}_{z \sim q_x(z)}\left\|y-\left\langle\theta_z, x\right\rangle\right\|^2 \\
        & = & \mathbb{E}_{s \sim p (s \mid \theta^{\ast}, \pi^{\ast})}
        \mathbb{E}_{z \sim q_s (z)} \frac{\| (- 1)^{z + 1} y - \langle x, \theta
        \rangle \|^2}{2 \sigma^2}\\
        & = & \mathbb{E}_{s \sim p (s \mid \theta^{\ast}, \pi^{\ast})} \frac{y^2 +
        \langle x, \theta \rangle^2}{2 \sigma^2} -\mathbb{E}_{s \sim p (s \mid
        \theta^{\ast}, \pi^{\ast})} \frac{y \langle x, \theta \rangle}{\sigma^2}
        \tanh \left( \frac{y \langle x, \theta^{t-1} \rangle}{\sigma^2} +\nu^{t-1} \right)\\
        & = & (2 \sigma^2)^{- 1}  \langle \theta, \mathbb{E}_{s \sim p (s \mid
        \theta^{\ast}, \pi^{\ast})} x x^{\top} \cdot \theta \rangle
        -  \left\langle \mathbb{E}_{s \sim p (s \mid \theta^{\ast},
        \pi^{\ast})} \tanh \left( \frac{y \langle x, \theta^{t - 1}
        \rangle}{\sigma^2} + \nu^{t - 1} \right) \frac{y x}{\sigma^2}, \theta \right\rangle
        + (2 \sigma^2)^{- 1} \mathbb{E}_{s \sim p (s \mid \theta^{\ast},
        \pi^{\ast})} y^2
      \end{eqnarray*}

      Consider the other terms for $g^t$, note $\pi^t(z) \assign \mathbb{E}_{s \sim p\left(s \mid \theta^*, \pi^*\right)} q_s(z)$, and 
      $\pi(z) = \tmop{sigmoid}(2(-1)^{z+1}\nu)$.
  
      Note that $\log 2 + \log \cosh \nu = - [\frac{\log \pi(1)+ \log \pi(2)}{2}]$ and $\tmop{sigmoid} (2 t) - \tmop{sigmoid} (- 2 t) = \tanh (t)$ for $\forall t$.
      \begin{eqnarray*}
      & &\mathrm{KL}_z\left[\pi^t(z) \| \pi(z)\right]
      + \mathcal{H}_z\left[\pi^t(z)\right]\\
      & = & - \mathbb{E}_{s \sim p\left(s \mid \theta^*, \pi^*\right)} \sum_{z\in\mathcal{Z}} q_s(z) \log \pi(z)\\
      & = & \log 2 + \log \cosh \nu -\mathbb{E}_{s \sim p (s \mid
      \theta^{\ast}, \pi^{\ast})} \tanh \left( \frac{y \langle x, \theta^{t - 1}
      \rangle}{\sigma^2} + \nu^{t - 1} \right) \cdot \nu\\
      \end{eqnarray*}
      To sum up, we obtain the following.
      \begin{eqnarray*}
          g^t &  = & (2 \sigma^2)^{- 1}  \langle \theta, \mathbb{E}_{s \sim p (s \mid
          \theta^{\ast}, \pi^{\ast})} x x^{\top} \cdot \theta \rangle -
          \left\langle \mathbb{E}_{s \sim p (s \mid \theta^{\ast},
          \pi^{\ast})} \tanh \left( \frac{y \langle x, \theta^{t - 1}
          \rangle}{\sigma^2} + \nu^{t - 1} \right) \frac{y x}{\sigma^2}, \theta \right\rangle\\
          &  & + \log 2 + \log \cosh \nu -\mathbb{E}_{s \sim p (s \mid
          \theta^{\ast}, \pi^{\ast})} \tanh \left( \frac{y \langle x, \theta^{t - 1}
          \rangle}{\sigma^2} + \nu^{t - 1} \right) \cdot \nu\\
          &  & -\mathbb{E}_{s \sim p (s \mid \theta^{\ast}, \pi^{\ast})}
          \mathcal{H}_z [q_s (z)]\\
          &  & + (2 \sigma^2)^{- 1} \mathbb{E}_{s \sim p (s \mid \theta^{\ast},
          \pi^{\ast})} y^2 -\mathbb{E}_{s \sim p (s \mid \theta^{\ast},
          \pi^{\ast})} \log p (x) - c
      \end{eqnarray*}
      For the negative expectation of log-likelihood $f$, we show the following.
      \begin{eqnarray*}
        f & = & -\mathbb{E}_{s \sim p (s \mid \theta^{\ast}, \pi^{\ast})} \log
        \sum_{z \in \mathcal{Z}} \exp \left[ - \frac{\| y - \langle \theta_z, x
        \rangle \|^2}{2 \sigma^2} + \log \pi (z) \right] -\mathbb{E}_{s \sim p (s
        \mid \theta^{\ast}, \pi^{\ast})} \log p (x) - c\\
        & = & (2 \sigma^2)^{- 1} \mathbb{E}_{s \sim p (s \mid \theta^{\ast},
        \pi^{\ast})} [y^2 + \langle x, \theta \rangle^2] - \frac{\log \pi
        (1) + \log \pi (2)}{2}\\
        &  & - \log 2 -\mathbb{E}_{s \sim p (s \mid \theta^{\ast}, \pi^{\ast})}
        \log \cosh \left( \frac{y \langle x, \theta \rangle}{\sigma^2} + \nu
        \right) -\mathbb{E}_{s \sim p (s \mid \theta^{\ast}, \pi^{\ast})} \log p
        (x) - c\\
        & = & (2 \sigma^2)^{- 1}  \langle \theta, \mathbb{E}_{s \sim p (s \mid
        \theta^{\ast}, \pi^{\ast})} x x^{\top} \cdot \theta \rangle + \log \cosh
        \nu -\mathbb{E}_{s \sim p (s \mid \theta^{\ast}, \pi^{\ast})} \log \cosh
        \left( \frac{y \langle x, \theta \rangle}{\sigma^2} + \nu \right)\\
        &  & + (2 \sigma^2)^{- 1} \mathbb{E}_{s \sim p (s \mid \theta^{\ast},
        \pi^{\ast})} y^2 -\mathbb{E}_{s \sim p (s \mid \theta^{\ast},
        \pi^{\ast})} \log p (x) - c
      \end{eqnarray*}

      Note $g^t =f$ at $(\theta, \nu)=(\theta^{t-1}, \nu^{t-1})$, by comparing the expressions  $f, g^t$,
      and use $\frac{\mathd \log \cosh(t)}{\mathd t} = \tanh (t)$.
      \begin{eqnarray*}
        g^t & = & (2 \sigma^2)^{- 1}  \langle \theta, \mathbb{E}_{s \sim p (s \mid
        \theta^{\ast}, \pi^{\ast})} x x^{\top} \cdot \theta \rangle + \log \cosh
        \nu -\mathbb{E}_{s \sim p (s \mid \theta^{\ast}, \pi^{\ast})} \log \cosh
        \left( \frac{y \langle x, \theta \rangle}{\sigma^2} + \nu \right)
        \mid_{(\theta, \nu) \leftarrow (\theta^{t - 1}, \nu^{t - 1})}\\
        & - & \left\langle \nabla_{\theta} \mathbb{E}_{s \sim p (s \mid
        \theta^{\ast}, \pi^{\ast})} \log \cosh \left( \frac{y \langle x, \theta
        \rangle}{\sigma^2} + \nu \right) \mid_{(\theta, \nu) \leftarrow (\theta^{t
        - 1}, \nu^{t - 1})}, \theta - \theta^{t - 1} \right\rangle\\
        & - & \left\langle \nabla_{\nu} \mathbb{E}_{s \sim p (s \mid
        \theta^{\ast}, \pi^{\ast})} \log \cosh \left( \frac{y \langle x, \theta
        \rangle}{\sigma^2} + \nu \right) \mid_{(\theta, \nu) \leftarrow (\theta^{t
        - 1}, \nu^{t - 1})}, \nu - \nu^{t - 1} \right\rangle\\
        & + & (2 \sigma^2)^{- 1} \mathbb{E}_{s \sim p (s \mid \theta^{\ast},
        \pi^{\ast})} y^2 -\mathbb{E}_{s \sim p (s \mid \theta^{\ast},
        \pi^{\ast})} \log p (x) - c
      \end{eqnarray*}
  \end{proof}
  \newpage
  
  \begin{theorem}
      For 2MLR, $y= (-1)^{z+1}\langle x, \theta^\ast\rangle + \varepsilon, z\in\mathcal{Z}=\{1, 2\}$, $\theta_1 = \theta, \theta_2 = -\theta, \pi \assign \{\pi(z)\}_{z\in\mathcal{Z}}$, 
      with assumptions: $(z; \pi)\ind \theta$, and $\varepsilon \ind (x, z; \theta, \pi)$, and $x\ind (z; \theta, \pi)$ and $\varepsilon \sim \mathcal{N}(0, \sigma^2)$.
  
      Then the negative Maximum Likelihood Estimat (MLE) $f_n (\theta, \pi) \assign
      -\frac{1}{n}\sum_{i=1}^n [\log p (s_i \mid
      \theta, \pi)]$ and the surrogate function $g_n^t (\theta, \pi)$ for the dataset $\mathcal{S}\assign \{s_i\}_{i=1}^n=\{(x_i, y_i)\}_{i=1}^n$ of $n$ i.i.d. samples can be expressed as follows.
      \begin{eqnarray*}
        f_n  & = & (2 \sigma^2)^{- 1}  \langle \theta, \frac{1}{n}\sum_{i=1}^n x_i x_i^{\top} \cdot \theta \rangle + \log \cosh
          \nu -\frac{1}{n}\sum_{i=1}^n \log \cosh
          \left( \frac{y_i \langle x_i, \theta \rangle}{\sigma^2} + \nu \right)\\
          & + &  (2 \sigma^2)^{- 1} \frac{1}{n}\sum_{i=1}^n y_i^2 -\frac{1}{n}\sum_{i=1}^n \log p (x_i) - c\\
        g_n^t & = & (2 \sigma^2)^{- 1}  \langle \theta, \frac{1}{n}\sum_{i=1}^n x_i x_i^{\top} \cdot \theta \rangle + \log \cosh
        \nu -\frac{1}{n}\sum_{i=1}^n \log \cosh
        \left( \frac{y_i \langle x_i, \theta \rangle}{\sigma^2} + \nu \right)
        \mid_{(\theta, \nu) \leftarrow (\theta^{t - 1}, \nu^{t - 1})}\\
        & - & \left\langle \nabla_{\theta} \frac{1}{n}\sum_{i=1}^n \log \cosh \left( \frac{y_i \langle x_i, \theta
        \rangle}{\sigma^2} + \nu \right) \mid_{(\theta, \nu) \leftarrow (\theta^{t
        - 1}, \nu^{t - 1})}, \theta - \theta^{t - 1} \right\rangle\\
        & - & \left\langle \nabla_{\nu} \frac{1}{n}\sum_{i=1}^n \log \cosh \left( \frac{y_i \langle x_i, \theta
        \rangle}{\sigma^2} + \nu \right) \mid_{(\theta, \nu) \leftarrow (\theta^{t
        - 1}, \nu^{t - 1})}, \nu - \nu^{t - 1} \right\rangle\\
        & + & (2 \sigma^2)^{- 1} \frac{1}{n}\sum_{i=1}^n y_i^2 -\frac{1}{n}\sum_{i=1}^n \log p (x_i) - c
      \end{eqnarray*}
      where $\nu \assign \frac{\log \pi(1) -\log \pi(2)}{2}$.
  \end{theorem}
  \begin{proof}
  This is proved by susbstituting $\frac{1}{n}\sum_{i=1}^n, s_i \assign (x_i, y_i)$ for $\mathbb{E}_{s \sim p (s \mid \theta^{\ast},
  \pi^{\ast})}, s\assign (x, y)$ in the previous Theorem.
  \end{proof}
  
  \begin{theorem}{(Derivation for Eq.~\eqref{eq:theta},~\eqref{eq:nu} in Section~\ref{sec:setup})}
  For 2MLR, $y= (-1)^{z+1}\langle x, \theta^\ast\rangle + \varepsilon, z\in\mathcal{Z}=\{1, 2\}$, $\theta_1 = \theta, \theta_2 = -\theta, \pi \assign \{\pi(z)\}_{z\in\mathcal{Z}}$, 
  with assumptions: $(z; \pi)\ind \theta$, and $\varepsilon \ind (x, z; \theta, \pi)$, and $x\ind (z; \theta, \pi)$ and $\varepsilon \sim \mathcal{N}(0, \sigma^2), x\sim \mathcal{N}(0, I_d)$.
  
  The EM update rules $M(\theta^{t-1}, \nu^{t-1}), N(\theta^{t-1}, \nu^{t-1})$ for $\theta, \tanh(\nu)$ at the population level, namely the minizer of the surrogate $g^t$ / the maximizer of $Q$, are the following.
  \begin{eqnarray*}
      M(\theta^{t-1}, \nu^{t-1}) & = & \mathbb{E}_{s \sim p (s \mid \theta^{\ast},
      \pi^{\ast})} \tanh\left(\frac{y\langle x, \theta^{t-1}\rangle}{\sigma^2}+\nu^{t-1}\right) y x\\
      N(\theta^{t-1}, \nu^{t-1}) & = & \mathbb{E}_{s \sim p (s \mid \theta^{\ast},
      \pi^{\ast})} \tanh\left(\frac{y\langle x, \theta^{t-1}\rangle}{\sigma^2}+\nu^{t-1}\right)
  \end{eqnarray*}
  \end{theorem}
  \begin{proof}
  Take the gradients of $g^t$ wrt. $\theta, \nu$, we obtain the following.
  \begin{eqnarray*}
      \nabla_\theta g^t & = & (\sigma^2)^{- 1} \mathbb{E}_{s \sim p (s \mid
      \theta^{\ast}, \pi^{\ast})} x x^{\top} \cdot \theta
      - (\sigma^2)^{- 1} \mathbb{E}_{s \sim p (s \mid \theta^{\ast},
      \pi^{\ast})} \tanh\left(\frac{y\langle x, \theta^{t-1}\rangle}{\sigma^2}+\nu^{t-1}\right) y x\\
      \nabla_\nu g^t & = & \tanh \nu - \mathbb{E}_{s \sim p (s \mid \theta^{\ast},
      \pi^{\ast})} \tanh\left(\frac{y\langle x, \theta^{t-1}\rangle}{\sigma^2}+\nu^{t-1}\right)
  \end{eqnarray*}
  Furthermore, the Hessian of $g^t$ wrt. $\theta, \nu$ are positive-definite, we show that the solution to $\nabla_\theta g^t=0, \nabla_\nu g^t=0$ must be the minimizer of $g^t$.
  \begin{eqnarray*}
      \nabla^2_\theta g^t & = & (\sigma^2)^{- 1} \mathbb{E}_{s \sim p (s \mid
      \theta^{\ast}, \pi^{\ast})} x x^{\top}\\
      \nabla^2_\nu g^t & = & \cosh^{-2} \nu 
  \end{eqnarray*}
  Note that $\mathbb{E}_{s \sim p (s \mid
  \theta^{\ast}, \pi^{\ast})} x x^{\top} = I_d$ for $x\sim \mathcal{N}(0, I_d)$, we derive the expressions for EM update rules.
  \end{proof}

  \begin{theorem}{(Derivation for Eq.~\eqref{eq:finite} in Section~\ref{sec:setup})}
  For 2MLR, $y= (-1)^{z+1}\langle x, \theta^\ast\rangle + \varepsilon, z\in\mathcal{Z}=\{1, 2\}$, $\theta_1 = \theta, \theta_2 = -\theta, \pi \assign \{\pi(z)\}_{z\in\mathcal{Z}}$, 
  with assumptions: $(z; \pi)\ind \theta$, and $\varepsilon \ind (x, z; \theta, \pi)$, and $x\ind (z; \theta, \pi)$ and $\varepsilon \sim \mathcal{N}(0, \sigma^2), x\sim \mathcal{N}(0, I_d)$.
      
  The EM update rules $M_n(\theta^{t-1}, \nu^{t-1}), N_n(\theta^{t-1}, \nu^{t-1})$ for $\theta, \tanh(\nu)$ at the finite-sample level, namely the minizer of the surrogate $g_n^t$, are the following.
  \begin{eqnarray*}
      M_n(\theta^{t-1}, \nu^{t-1}) & = & \left(\frac{1}{n}\sum_{i=1}^n x_i x_i^\top\right)^{-1}\cdot
       \frac{1}{n}\sum_{i=1}^n \tanh\left(\frac{y\langle x, \theta^{t-1}\rangle}{\sigma^2}+\nu^{t-1}\right) y x\\
      N_n(\theta^{t-1}, \nu^{t-1}) & = & \frac{1}{n}\sum_{i=1}^n \tanh\left(\frac{y\langle x, \theta^{t-1}\rangle}{\sigma^2}+\nu^{t-1}\right)
  \end{eqnarray*}
  \end{theorem}
  \begin{proof}
  This is proved by susbstituting $\frac{1}{n}\sum_{i=1}^n, s_i \assign (x_i, y_i)$ for $\mathbb{E}_{s \sim p (s \mid \theta^{\ast},
  \pi^{\ast})}, s\assign (x, y)$ in the previous Theorem, but note that $\frac{1}{n}\sum_{i=1}^n x_i x_i^\top \not\equiv I_d$.
  \end{proof}

%% file: 8_supplementary_2.tex
\section{Proof for Results of Population EM Updates}\label{sup:updates}

\begin{theorem}{(Theorem~\ref{thm:em_update} in Section~\ref{sec:updates})}
  Let $\rho \equiv \rho^{t - 1} \assign \frac{\langle \theta^{t - 1},
  \theta^{\ast} \rangle}{\| \theta^{t - 1} \| \cdot \| \theta^{\ast} \|},
  \bar{\theta} \assign \frac{\theta^{t - 1}}{\sigma}, \bar{\theta}^{\ast}
  \assign \frac{\theta^{\ast}}{\sigma}$, then the EM update rules for
  $\theta^t, \tanh(\nu^t)$ at Population level are as follows.
  \begin{eqnarray*}
    & & \theta^t  \assign M(\theta^{t-1}, \nu^{t-1})\\
    & = & \nabla_{\theta} \mathbb{E}_{s \sim p (s \mid \theta^{\ast},
    \pi^{\ast})} \log \cosh \left( \frac{y \langle x, \theta
    \rangle}{\sigma^2} + \nu \right) \mid_{(\theta, \nu) = (\theta^{t - 1},
    \nu^{t - 1})}\\
    & = & \mathbb{E}_{s \sim p (s \mid \theta^{\ast}, \pi^{\ast})} \tanh
    \left( \frac{y \langle x, \theta^{t - 1} \rangle}{\sigma^2} + \nu^{t - 1}
    \right) y x\\
    & = & \left[ - \frac{\sigma}{\pi} \cdot \frac{\| \bar{\theta}^{\ast}
    \|^2}{\| \bar{\theta} \|^2} \cdot \frac{\sqrt{1 - \rho^2}}{\left( 1 + (1 -
    \rho^2) \| \bar{\theta}^{\ast} \|^2 \right)^{\frac{3}{2}}}  \cosh^{-1}
    (\nu^{\ast}) \right]\\
    & \cdot & \left\{ \tanh (\nu) \ast \nu \left[ \alpha_{\rho, \| \bar{\theta}
    \|, \| \bar{\theta}^{\ast} \|, \nu^{\ast}} (\nu) \left(
    \frac{\frac{\bar{\theta}}{\| \bar{\theta} \|}}{\sqrt{1 - \rho^2} \|
    \bar{\theta}^{\ast} \|^2} + \hat{e}_2 \right) + \beta_{\rho, \|
    \bar{\theta} \|, \| \bar{\theta}^{\ast} \|, \nu^{\ast}} (\nu)  \vec{e}_2
    \right] \right\}_{\nu \leftarrow \nu^{t - 1}} \\
    &  & \tanh (\nu^t) \assign N(\theta^{t-1}, \nu^{t-1})\\
    & = & \nabla_{\nu} \mathbb{E}_{s \sim p (s \mid \theta^{\ast},
    \pi^{\ast})} \log \cosh \left( \frac{y \langle x, \theta
    \rangle}{\sigma^2} + \nu \right) \mid_{(\theta, \nu) = (\theta^{t - 1},
    \nu^{t - 1})} \\
    & = & \mathbb{E}_{s \sim p (s \mid \theta^{\ast}, \pi^{\ast})} \tanh
    \left( \frac{y \langle x, \theta^{t - 1} \rangle}{\sigma^2} + \nu^{t - 1}
    \right)\\
    & = & \frac{{(1 + (1 - \rho^2) \| \bar{\theta}^{\ast} \|^2)^{-
    \frac{1}{2}}} }{\pi \| \bar{\theta} \| \cosh (\nu^{\ast})}
    \underset{\mathbb{R}}{\int}  \tanh (\nu^{t - 1} - \nu') K_0 \left(
    \frac{\sqrt{1 + \| \bar{\theta}^{\ast} \|^2} \cdot \left| \frac{\nu'}{\|
    \bar{\theta} \|} \right|}{[1 + (1 - \rho^2) \| \bar{\theta}^{\ast} \|^2]}
    \right) \cosh \left( \frac{\rho \| \bar{\theta}^{\ast} \| \left(
    \frac{\nu'}{\| \bar{\theta} \|} \right)}{[1 + (1 - \rho^2) \|
    \bar{\theta}^{\ast} \|^2]} - \nu^{\ast} \right) \mathd \nu' 
  \end{eqnarray*}
  where $\alpha_{\rho, \| \bar{\theta} \|, \| \bar{\theta}^{\ast} \|,
  \nu^{\ast}} (\nu) \assign \cosh \left( \frac{\rho \| \bar{\theta}^{\ast} \|
  \left( \frac{\nu}{\| \bar{\theta} \|} \right)}{[1 + (1 - \rho^2) \|
  \bar{\theta}^{\ast} \|^2]} - \nu^{\ast} \right) \cdot K_0 \left(
  \frac{\sqrt{1 + \| \bar{\theta}^{\ast} \|^2} \cdot \left| \frac{\nu}{\|
  \bar{\theta} \|} \right|}{[1 + (1 - \rho^2) \| \bar{\theta}^{\ast} \|^2]}
  \right)$
  
  and $\beta_{\rho, \| \bar{\theta} \|, \| \bar{\theta}^{\ast} \|,
  \nu^{\ast}} (\nu) \assign \tmop{sgn} (\nu) \frac{\sqrt{1 + \|
  \bar{\theta}^{\ast} \|^2}}{\| \bar{\theta}^{\ast} \|} \sinh \left(
  \frac{\rho \| \bar{\theta}^{\ast} \| \left( \frac{\nu}{\| \bar{\theta} \|}
  \right)}{[1 + (1 - \rho^2) \| \bar{\theta}^{\ast} \|^2]} - \nu^{\ast}
  \right) \cdot K_1 \left( \frac{\sqrt{1 + \| \bar{\theta}^{\ast} \|^2} \cdot
  \left| \frac{\nu}{\| \bar{\theta} \|} \right|}{[1 + (1 - \rho^2) \|
  \bar{\theta}^{\ast} \|^2]} \right)$
\end{theorem}

\begin{proof}
  \
With Lemma in Section~\ref{sup:derive_em}, we show that the EM update rules at population level are as follows.
\begin{eqnarray*}
  \theta^t  \assign M(\theta^{t-1}, \nu^{t-1})
    & = & \nabla_{\theta} \mathbb{E}_{s \sim p (s \mid \theta^{\ast},
    \pi^{\ast})} \log \cosh \left( \frac{y \langle x, \theta
    \rangle}{\sigma^2} + \nu \right) \mid_{(\theta, \nu) = (\theta^{t - 1},
    \nu^{t - 1})}\\
    & = & \mathbb{E}_{s \sim p (s \mid \theta^{\ast}, \pi^{\ast})} \tanh
    \left( \frac{y \langle x, \theta^{t - 1} \rangle}{\sigma^2} + \nu^{t - 1}
    \right) y x\\
  \tanh (\nu^t) \assign N(\theta^{t-1}, \nu^{t-1})
    & = & \nabla_{\nu} \mathbb{E}_{s \sim p (s \mid \theta^{\ast},
    \pi^{\ast})} \log \cosh \left( \frac{y \langle x, \theta
    \rangle}{\sigma^2} + \nu \right) \mid_{(\theta, \nu) = (\theta^{t - 1},
    \nu^{t - 1})} \\
    & = & \mathbb{E}_{s \sim p (s \mid \theta^{\ast}, \pi^{\ast})} \tanh
    \left( \frac{y \langle x, \theta^{t - 1} \rangle}{\sigma^2} + \nu^{t - 1}
    \right)\\
\end{eqnarray*}
Then, by using the Lemma for the evaluation of expectations in Section~\ref{sup:lemma}, we complete the proof for these closed-from expression in this Theorem.
\end{proof}

\newpage

\begin{corollary}{(Corollary~\ref{cor:no_separa} in Section~\ref{sec:updates}: EM Updates for No Separation Case)}
  For the special case of no separation of parameters, namely $\tmop{SNR}
  \assign \frac{\| \theta^{\ast} \|}{\sigma} \rightarrow 0$, the EM update
  rules for $\theta^t, \tanh(\nu^t)$ at Population level are 
  \begin{eqnarray*}
    \bar{\theta}^t &=&
  \frac{\bar{\theta}^0}{\| \bar{\theta}^0 \|} \cdot \frac{1}{\pi} 
  \int_{\mathbb{R}} \tanh (\| \bar{\theta}^{t - 1} \| x - \nu^{t - 1}) xK_0
  (|x|) \mathrm{d} x\\
  \tanh (\nu^t) &=& \frac{1}{\pi}  \int_{\mathbb{R}}
  \tanh (\nu^{t - 1} - \| \bar{\theta}^{t - 1} \| x) K_0 (|x|) \mathrm{d} x
  \end{eqnarray*}
  where $\bar{\theta}^t \assign \frac{\theta^t}{\sigma}$ and $\bar{\theta}^0
  \assign \frac{\theta^0}{\sigma}$.
\end{corollary}
% Commented by Abolfazl, Move to supp
\begin{remark}
Note that 
  %For the ease of theoretical analysis
  % With the identity $\tanh (\nu^{t - 1} - \| \bar{\theta}^{t - 1} \| x) +
  % \tanh (\nu^{t - 1} + \| \bar{\theta}^{t - 1} \| x) = \frac{2 \sinh (2 \nu^{t
  % - 1})}{\cosh (2 \nu^{t - 1}) + \cosh (2 \| \bar{\theta}^{t - 1} \| x)}$
we can rewrite the EM update rule as 
\begin{equation*}
  \tanh (\nu^t) = \tanh (\nu^{t - 1})
\cdot \frac{2}{\pi}  \int_{\mathbb{R}_{\geq 0}} \frac{\cosh (2 \nu^{t - 1}) +
1}{\cosh (2 \nu^{t - 1}) + \cosh (2 \| \bar{\theta}^{t - 1} \| x)} K_0 (|x|)
\mathrm{d} x.
\end{equation*}

Since $\frac{2}{\pi}  \int_{\mathbb{R}_{\geq 0}} K_0 (|x|) \mathrm{d} x = 1$
and $\cosh (2 \| \bar{\theta}^{t - 1} \| x) \geq 1$, the EM update rule implies 
\begin{equation*}
  | \nu^t |
\leq | \nu^{t - 1} |,
\quad 
\tmop{sgn} (\nu^t) = \tmop{sgn} (\nu^{t - 1}).
\end{equation*}
If we take the $\ell_2$ norm on both sides of the EM update rule for regression parameters, it follows that 
\begin{equation*}
  \| \bar{\theta}^t
\| = \frac{1}{\pi}  \int_{\mathbb{R}} \tanh (\| \bar{\theta}^{t - 1} \| x -
\nu^{t - 1}) xK_0 (|x|) \mathrm{d} x \leq \frac{1}{\pi}  \int_{\mathbb{R}} | x
| K_0 (|x|) \mathrm{d} x = \frac{2}{\pi}
\end{equation*}
is bounded.
\end{remark}

\begin{proof}
As $\tmop{SNR}
\assign \frac{\| \theta^{\ast} \|}{\sigma} \rightarrow 0$ in previous Theorem, then coefficients in the EM update rule for $\theta$ are as follows.
\begin{eqnarray*}
  \| \bar{\theta}^{\ast} \|^2 \cdot \alpha_{\rho, \| \bar{\theta} \|, \| \bar{\theta}^{\ast} \|,
  \nu^{\ast}} (\nu) 
  &\to &
  0^2 \cdot \cosh(\nu^\ast) \cdot K_0\left(\frac{|\nu|}{\| \bar{\theta} \|}\right) =0\\
  \| \bar{\theta}^{\ast} \|^2 \cdot \beta_{\rho, \| \bar{\theta} \|, \| \bar{\theta}^{\ast} \|,
  \nu^{\ast}} (\nu) 
  & \to & 0\cdot  \tmop{sgn}(\nu) \sinh(-\nu^\ast) \cdot K_1\left(\frac{|\nu|}{\| \bar{\theta} \|}\right) =0 
\end{eqnarray*}
Hence, $\theta^t$ has no $\hat{e}_2, \bar{e}_2$ components, and it only contains the $\frac{\bar{\theta}^{t-1}}{\| \bar{\theta}^{t-1} \|}$ component.
\begin{eqnarray*}
  \bar{\theta}^t \assign \frac{\theta^t}{\sigma} 
  & = & - \frac{\cosh^{-1}
  (\nu^{\ast})}{\pi\| \bar{\theta} \|^2}
  \cdot \left\{ \tanh (\nu) \ast \nu \left[ \cosh(\nu^\ast) \cdot K_0\left(\frac{|\nu|}{\| \bar{\theta} \|}\right) \frac{\bar{\theta}^{t-1}}{\| \bar{\theta}^{t-1} \|}
  \right] \right\}_{\nu \leftarrow \nu^{t - 1}} \\
  & = & \frac{\bar{\theta}^{t-1}}{\| \bar{\theta}^{t-1} \|} \cdot \frac{1}{\pi} 
  \int_{\mathbb{R}} \tanh(\nu'-\nu^{t-1}) \frac{\nu'}{\| \bar{\theta} \|} \cdot K_0\left(\frac{|\nu'|}{\| \bar{\theta} \|}\right)\mathd  \frac{\nu'}{\| \bar{\theta} \|}\\
  & = & \frac{\bar{\theta}^{t-1}}{\| \bar{\theta}^{t-1} \|} \cdot \frac{1}{\pi} 
  \int_{\mathbb{R}} \tanh(\| \bar{\theta}^{t-1} \| x-\nu^{t-1}) x \cdot K_0\left(|x|\right)\mathd  x
\end{eqnarray*} 
Note that $\tanh(\nu^{t-1} + \| \bar{\theta}^{t-1} \| x) - \tanh(\nu^{t-1} - \| \bar{\theta}^{t-1} \| x) > 0$ for $\| \bar{\theta}^{t-1} \|\neq 0, x>0$.
\begin{eqnarray*}
  \int_{\mathbb{R}} \tanh(\| \bar{\theta}^{t-1} \| x-\nu^{t-1}) x \cdot K_0\left(|x|\right)\mathd  x
  & = & \left[\int_{0}^\infty + \int_{-\infty}^0\right]  \tanh(\| \bar{\theta}^{t-1} \| x-\nu^{t-1}) x \cdot K_0\left(|x|\right)\mathd  x\\
  & = & \int_{0}^\infty \left(\tanh(\nu^{t-1} + \| \bar{\theta}^{t-1} \| x) - \tanh(\nu^{t-1} - \| \bar{\theta}^{t-1} \| x)\right)x \cdot K_0\left(|x|\right)\mathd  x\\
  & > & 0
\end{eqnarray*}
Hence, we conclude that $\frac{\bar{\theta}^{t}}{\| \bar{\theta}^{t} \|}=\frac{\bar{\theta}^{t-1}}{\| \bar{\theta}^{t-1} \|}=\cdots =\frac{\bar{\theta}^{0}}{\| \bar{\theta}^{0} \|}$.

Consequently, we prove the closed-form expression for the EM update rule for $\theta$.

As $\tmop{SNR}
\assign \frac{\| \theta^{\ast} \|}{\sigma} \rightarrow 0$ in previous Theorem, the EM update rule for $\tanh(\nu)$ is as follows.
\begin{eqnarray*}
  \tanh(\nu^t) & =& \frac{1}{\pi \| \bar{\theta} \| \cosh (\nu^{\ast})}
  \underset{\mathbb{R}}{\int}  \tanh (\nu^{t - 1} - \nu') K_0\left(\frac{|\nu'|}{\| \bar{\theta} \|}\right)
   \cosh \left( \nu^{\ast} \right) \mathd \nu' 
   = \frac{1}{\pi}  \int_{\mathbb{R}}
  \tanh (\nu^{t - 1} - \| \bar{\theta}^{t - 1} \| x) K_0 (|x|) \mathrm{d} x
\end{eqnarray*}
\end{proof}

\begin{corollary}{(Corollary~\ref{cor:noiseless} in Section~\ref{sec:updates}: EM Updates in Noiseless Setting)}
  In the noiseless setting, namely $\tmop{SNR} \assign \frac{\| \theta^{\ast}
  \|}{\sigma} \rightarrow \infty$, the EM update rules for $\theta^t, \tanh(\nu^t)$ 
  at the Population level are
  \begin{equation}\nonumber
  \begin{aligned}
    &\frac{\theta^t}{\| \theta^{\ast} \|} = \frac{2}{\pi} \left[ \mathrm{sgn}
    (\rho^{t - 1}) \varphi^{t-1}  \frac{\theta^{\ast}}{\| \theta^{\ast} \|} + \cos
    \varphi^{t - 1} \frac{\theta^{t - 1}}{\| \theta^{t - 1} \|} \right]\\
    &
    \tanh (\nu^t) = \mathrm{sgn} (\rho^{t - 1}) \left( \frac{2}{\pi}
    \varphi^{t - 1} \right) \cdot \tanh (\nu^{\ast}),
    \end{aligned}
  \end{equation}
  where $\rho^{t - 1} \assign \frac{\langle \theta^{t - 1}, \theta^{\ast}
  \rangle}{\| \theta^{t - 1} \| \| \theta^{\ast} \|},
  \varphi^{t -
  1} \assign \frac{\pi}{2} - \arccos | \rho^{t - 1} |$.
\end{corollary}

\begin{proof}
  \
  For brevity, we let $\rho\assign \rho^{t-1}, \varphi^{t-1} = \frac{\pi}{2}- \arccos |\rho^{t-1}|$, thus $\sqrt{1-\rho^2}=\cos \varphi^{t-1}$

  Denote $k \assign \frac{\| \bar{\theta}^{\ast} \|}{\| \bar{\theta}^{t - 1} \|}$ and 
  $\alpha^2 \assign \frac{1}{(1 + (1 - \rho^2) \| \bar{\theta}^{\ast}
  \|^2)} \rightarrow 0_+ , x \assign k \alpha^2 \cdot \nu$, when $\| \bar{\theta}^{t - 1} \|\to \infty$.

  Hence, $\sqrt{1 + \| \bar{\theta}^{\ast} \|^2} \sim \|
  \bar{\theta}^{\ast} \| \sim \frac{\alpha^{- 1}}{\sqrt{1 - \rho^2}}$,
  $\| \bar{\theta}^{t - 1} \|^{- 1} = k \| \bar{\theta}^{\ast} \|^{- 1} \sim k
  \sqrt{1 - \rho^2} \alpha$.

  Furthermore, $\tanh (\nu^{t - 1} - \nu) = \tanh \left(
  \nu^{t - 1} - \frac{x}{k \alpha^2} \right) \rightarrow - \tmop{sgn} (x)$ as $\alpha \to 0_+$.

  Evaluating the integral involving $K_0(|x|)$, we obtain the following expression as $\| \bar{\theta}^{t - 1} \|\to \infty$, namely $\alpha\to 0_+$.
  \begin{eqnarray*}
    &  & \tanh (\nu^t) =\mathbb{E}_{s \sim p (s \mid \theta^{\ast},
    \pi^{\ast})} \tanh \left( \frac{y \langle x, \theta^{t - 1}
    \rangle}{\sigma^2} + \nu^{t - 1} \right)\\
    & =  & \lim_{\alpha\to 0_+} \frac{1}{\pi \cosh (\nu^{\ast})} k \sqrt{1 - \rho^2} \alpha^2
    \underset{\mathbb{R}}{\int}  \tanh (\nu^{t - 1} - \nu) K_0 (k \alpha^2 \cdot
    | \nu |) \cosh (\rho k \alpha^2 \cdot \nu - \nu^{\ast}) \mathd \nu\\
    & =  &  \frac{\sqrt{1 - \rho^2}}{\pi \cosh (\nu^{\ast})}
    \underset{\mathbb{R}}{\int}  - \tmop{sgn} (x) K_0 (| x |) \cosh (\rho x -
    \nu^{\ast}) \mathd x\\
    & = & \frac{\sqrt{1 - \rho^2}}{\pi \cosh (\nu^{\ast})} \int_0^{+ \infty} 2
    \sinh (\nu^{\ast}) \sinh (\rho x) K_0 (| x |) \mathd x\\
    & = & \frac{2 \sqrt{1 - \rho^2}}{\pi} \tanh (\nu^{\ast}) \tmop{sgn} (\rho)
    \int_0^{+ \infty} \sinh (| \rho | x) K_0 (| x |) \mathd x\\
    & = & \tmop{sgn} (\rho) \tanh (\nu^{\ast}) \left[ 1 - \frac{2}{\pi} \arccos
    | \rho | \right]\\
    & = & \mathrm{sgn} (\rho^{t - 1}) \left( \frac{2}{\pi}
    \varphi^{t - 1} \right) \cdot \tanh (\nu^{\ast})
  \end{eqnarray*}
  Evaluating the integrals involving $K_0(|x|), K_1(|x|)$ and the Lemma for the relations of unit vectors, we obtain the following expression as $\| \bar{\theta}^{t - 1} \|\to \infty$, namely $\alpha\to 0_+$.

  \begin{eqnarray*}
    \frac{\theta^t}{\| \theta^{\ast} \|} & = & \frac{\bar{\theta}^t}{\|
    \bar{\theta}^{\ast} \|} = \frac{1}{\| \bar{\theta}^{\ast} \|} \cdot
    \frac{\theta^t}{\sigma}\\
    & =  & \lim_{\alpha\to 0_+} - \frac{\sqrt{1 - \rho^2} \cosh^{- 1} (\nu^{\ast})}{\pi} k^2
    \left( \sqrt{1 - \rho^2} \alpha \right) \alpha^3 \tanh (\nu) \ast \nu\\
    &  & \left[ \cosh (\rho k \alpha^2 \cdot \nu - \nu^{\ast}) K_0 (k \alpha^2
    \cdot | \nu |) \left( \frac{\alpha^2}{\sqrt{1 - \rho^2}} \cdot
    \frac{\bar{\theta}}{\| \bar{\theta} \|} + \hat{e}_2 \right) + \tmop{sgn}
    (\nu) \sinh (\rho k \alpha^2 \cdot \nu - \nu^{\ast}) K_1 (k \alpha^2 \cdot |
    \nu |) \vec{e}_2 \right]\\
    & =  &  - \frac{(1 - \rho^2)}{\pi} \cosh^{- 1} (\nu^{\ast})
    \int_{\mathbb{R}} - \tmop{sgn} (x) x [\cosh (\rho x - \nu^{\ast}) K_0
    (| x |) \hat{e}_2 + \tmop{sgn} (x) \sinh (\rho x - \nu^{\ast}) K_1 (| x |)
    \vec{e}_2] \mathd x\\
    & = & \frac{(1 - \rho^2)}{\pi} \cosh^{- 1} (\nu^{\ast}) \left\{ \left[
    \int_{\mathbb{R}} | x | \cosh (\rho x - \nu^{\ast}) K_0 (| x |) \mathd x
    \right] \hat{e}_2 + \left[ \int_{\mathbb{R}} x \sinh (\rho x - \nu^{\ast})
    K_1 (| x |) \mathd x \right] \vec{e}_2 \right\}\\
    & = & \frac{(1 - \rho^2)}{\pi} \left\{ \left( 2 \left[ \frac{1}{1 - \rho^2}
    + \frac{| \rho |  \left( \frac{\pi}{2} - \arccos | \rho | \right)}{(1 -
    \rho^2)^{\frac{3}{2}}} \right] \right) \hat{e}_2 + \left( 2 \tmop{sgn}
    (\rho) \left[ \frac{ \left( \frac{\pi}{2} - \arccos | \rho | \right)}{(1 -
    \rho^2)^{\frac{3}{2}}} + \frac{| \rho |}{1 - \rho^2} \right] \right)
    \vec{e}_2 \right\}\\
    & = & \tmop{sgn} (\rho) \frac{ \left( \frac{\pi}{2} - \arccos | \rho |
    \right)}{\frac{\pi}{2} \sqrt{1 - \rho^2}} (\vec{e}_2 + \rho \hat{e}_2) +
    \left( \frac{\pi}{2} \right)^{- 1} (\hat{e}_2 + \rho \vec{e}_2)\\
    & = & \tmop{sgn} (\rho) \frac{ \left( \frac{\pi}{2} - \arccos | \rho |
    \right)}{\frac{\pi}{2} \sqrt{1 - \rho^2}} \cdot \sqrt{1 - \rho^2}
    \frac{\theta^{\ast}}{\| \theta^{\ast} \|} + \left( \frac{\pi}{2} \right)^{-
    1} \sqrt{1 - \rho^2} \frac{\theta }{\| \theta \|}\\
    & = & \frac{2}{\pi} \left[ \mathrm{sgn}
    (\rho^{t - 1}) \varphi^{t-1}  \frac{\theta^{\ast}}{\| \theta^{\ast} \|} + \cos
    \varphi^{t - 1} \frac{\theta^{t - 1}}{\| \theta^{t - 1} \|} \right]
  \end{eqnarray*}
\end{proof}

\begin{lemma}
In the noiseless setting, the EM update rules at the population level for 2MLR are
\begin{eqnarray*}
  M(\theta^{t-1}, \nu^{t-1}) & = & \mathbb{E}_{x\sim p(x)} |\langle x, \theta^\ast\rangle| \tmop{sgn}\langle x, \theta^{t-1}\rangle x\\
  N(\theta^{t-1}, \nu^{t-1}) & = & \mathbb{E}_{x\sim p(x)} \mathbb{E}_{z\sim \mathcal{CAT}(\pi^\ast)} (-1)^{z+1} 
  \tmop{sgn}\langle x, \theta^\ast\rangle \tmop{sgn}\langle x, \theta^{t-1}\rangle
\end{eqnarray*}
In the noiseless setting, the EM update rules at the finite-sample level for 2MLR are
\begin{eqnarray*}
  M_n(\theta^{t-1}, \nu^{t-1}) & = & \left(\frac{1}{n}\sum_{i=1}^n x_i x_i^\top\right)^{-1}\frac{1}{n}\sum_{i=1}^n |\langle x_i, \theta^\ast\rangle| \tmop{sgn}\langle x_i, \theta^{t-1}\rangle x_i\\
  N_n(\theta^{t-1}, \nu^{t-1}) & = & \frac{1}{n}\sum_{i=1}^n (-1)^{z_i+1} 
  \tmop{sgn}\langle x_i, \theta^\ast\rangle \tmop{sgn}\langle x_i, \theta^{t-1}\rangle
\end{eqnarray*}
\end{lemma}
\begin{proof}
By letting $\sigma\to 0_+$, SNR $\assign \frac{\|\theta^\ast \|}{\sigma}\to 0_+$, then $y\to (-1)^{z+1} \langle x, \theta^\ast\rangle, 
\tanh \left( \frac{y \langle x, \theta^{t - 1}
    \rangle}{\sigma^2} + \nu^{t - 1} \right) \to \tmop{sgn}\langle x, \theta^\ast\rangle \tmop{sgn}\langle x, \theta^{t-1}\rangle$. 
Hence, these expressions are proved by taking the limits.
\end{proof}

\begin{lemma}
In the noiseless setting, the easy EM update rule for $\theta$ at the finite-sample level for 2MLR is
\begin{eqnarray*}
  M_n^{\tmop{easy}}(\theta^{t-1}, \nu^{t-1}) & = &\frac{1}{n}\sum_{i=1}^n |\langle x_i, \theta^\ast\rangle| \tmop{sgn}\langle x_i, \theta^{t-1}\rangle x_i\\
\end{eqnarray*}
\end{lemma}
\begin{proof}
  This Lemma is proved by taking the limit $\sigma\to 0_+$ for easy EM update.
\end{proof}
\newpage
\begin{lemma}
  Let $\rho := \frac{\langle  \theta,  \theta^\ast \rangle}{\| \theta\| \cdot \|\theta^\ast \|}$ and $\varphi = \frac{\pi}{2} -\arccos |\rho|$, then the identity holds.
\begin{eqnarray*}
  \mathbb{E}_{x \sim \mathcal{N} (0, I_d)} | \theta^{\ast \top} x 
    x^{\top} \theta |
    & = &
    \| \theta^{\ast} \| \| \theta \| \cdot \left\{  \left[ 1 -
    \frac{\arccos | \rho |}{\frac{\pi}{2}} \right] | \rho | + \left(
    \frac{\pi}{2} \right)^{- 1} \sqrt{1 - \rho^2} \right\}\\
    & = &
    \| \theta^{\ast} \| \| \theta \| \cdot \frac{2}{\pi} 
    (\varphi \sin\varphi + \cos \varphi  )
\end{eqnarray*}
\end{lemma}
\begin{proof}
  Decompose $x= \tilde{x} + \in\mathbb{R}^d$, where $\tilde{x}\in\tmop{span}\{\theta, \theta^\ast\}$.

  Let $\theta = \| \theta \| \left( \rho \hat{e}_1 + \sqrt{1 - \rho^2}
  \hat{e}_2 \right), \theta^{\ast} = \| \theta^{\ast} \| \hat{e}_1$ and
  $\tilde{x} = \lambda_1 \hat{e}_1 + \lambda_2 \hat{e}_2$ , since $\lambda_1,
  \lambda_2 \overset{\tmop{iid}}{\sim} \mathcal{N} (0, 1)$
  
  and $r = \sqrt{\lambda^2_1 + \lambda^2_2}, \lambda_1 + \mathtt{i} \lambda_2 =
  r \exp (\mathtt{i} \alpha)$ and $\mathd x \mathd y = r \mathd r \mathd \alpha$
  
  Note $(\pi - 2 \arccos \rho) \rho$ is an even function
  \begin{eqnarray*}
    \mathbb{E}_{x \sim \mathcal{N} (0, I_d)} | \theta^{\ast \top} x 
    x^{\top} \theta |
    & = & \mathbb{E}_{x \sim \mathcal{N} (0, I_d)} | \theta^{\ast \top} \tilde{x} 
    \tilde{x}^{\top} \theta | \\
    & = & \| \theta^{\ast} \| \| \theta \|
    \mathbb{E}_{\lambda_1, \lambda_2 \overset{\tmop{iid}}{\sim} \mathcal{N} (0,
    1)} \left| \lambda_1 \left( \lambda_1 \rho + \lambda_2 \sqrt{1 - \rho^2}
    \right) \right|\\
    & = & \| \theta^{\ast} \| \| \theta \| \underset{\mathbb{R}^2}{\iint}
    \left| \lambda_1 \left( \lambda_1 \rho + \lambda_2 \sqrt{1 - \rho^2} \right)
    \right| \frac{1}{2 \pi} \exp \left( - \frac{\lambda^2_1 + \lambda^2_2}{2}
    \right) \mathd x \mathd y\\
    & = & \| \theta^{\ast} \| \| \theta \| \left[ \frac{1}{2 \pi}
    \underset{\mathbb{R}_{\geq 0}}{\int} r^3 \exp \left( - \frac{r^2}{2} \right)
    \mathd r \right] \cdot \left[ \underset{[0, 2 \pi]}{\int} | \cos \alpha
    \cdot \cos (\alpha - \arccos \rho) | \mathd \alpha \right]\\
    & = & \| \theta^{\ast} \| \| \theta \| \cdot \frac{1}{\pi} \cdot \left[
    (\pi - 2 \arccos \rho) \rho + 2 \sqrt{1 - \rho^2} \right]\\
    & = & \| \theta^{\ast} \| \| \theta \| \cdot \left\{  \left[ 1 -
    \frac{\arccos \rho}{\frac{\pi}{2}} \right] \rho + \left( \frac{\pi}{2}
    \right)^{- 1} \sqrt{1 - \rho^2} \right\}\\
    & = & \| \theta^{\ast} \| \| \theta \| \cdot \left\{  \left[ 1 -
    \frac{\arccos | \rho |}{\frac{\pi}{2}} \right] | \rho | + \left(
    \frac{\pi}{2} \right)^{- 1} \sqrt{1 - \rho^2} \right\}\\
    & = &  \| \theta^{\ast} \| \| \theta \| \cdot \frac{2}{\pi} 
    (\varphi \sin\varphi + \cos \varphi  )
  \end{eqnarray*}
\end{proof}

\begin{lemma}{(Grothendieck's Identity)}\label{lem:Grothendieck}
Let $\rho := \frac{\langle  \theta,  \theta^\ast \rangle}{\| \theta\| \cdot \|\theta^\ast \|}$ and $\varphi = \frac{\pi}{2} -\arccos |\rho|$, then the identity holds.
\begin{eqnarray*}
  \mathbb{E}_{x \sim \mathcal{N} (0, I_d)} \tmop{sgn} \langle x, \theta^{\ast} \rangle
  \tmop{sgn} \langle x, \theta \rangle
    & = &
    \frac{2}{\pi} \tmop{sgn}\langle  \theta,  \theta^\ast \rangle \varphi
\end{eqnarray*}
Note that $\frac{1}{2} + \frac{1}{2}\tmop{sgn} \langle x, \theta^{\ast} \rangle
\tmop{sgn} \langle x, \theta \rangle \in\{1, 0\}$ is a Binomial random variable, hence
\begin{eqnarray*}
\mathbb{P}\left[\tmop{sgn} \langle x, \theta^{\ast} \rangle
\tmop{sgn} \langle x, \theta \rangle = +1\right] 
& = & \frac{1}{2} + \frac{1}{\pi} \tmop{sgn}\langle  \theta,  \theta^\ast \rangle \varphi \\
\mathbb{P}\left[\tmop{sgn} \langle x, \theta^{\ast} \rangle
\tmop{sgn} \langle x, \theta \rangle = -1\right] 
& = & \frac{1}{2} - \frac{1}{\pi} \tmop{sgn}\langle  \theta,  \theta^\ast \rangle \varphi
\end{eqnarray*}
\end{lemma}

%% file: 8_supplementary_3.tex
\section{Proof for Results at the Population Level}\label{sup:population}
\begin{theorem}{(Proposition~\ref{prop:recurrence} in Section~\ref{sec:population}: Recurrence Relation)}
   Assume the  initial sub-optimality cosine satisfies $\rho^0 \assign \frac{\langle \theta^0, \theta^{\ast} \rangle}{\|
   \theta^0 \| \cdot \| \theta^{\ast} \|} \neq \pm 1$, or equivalently $\varphi^0 \assign \frac{\pi}{2} -
   \arccos | \rho^0 | \in [ 0, \frac{\pi}{2} )$. 
   % Let $\varphi^t \assign \frac{\pi}{2} -
   % \arccos | \rho^t | \in \left[ 0, \frac{\pi}{2} \right), \rho^t \assign
   % \frac{\langle \theta^t, \theta^{\ast} \rangle}{\| \theta^t \| \cdot \|
   % \theta^{\ast} \|}$, 
   Then the recurrence relation for EM updates at
   population level characterized by the sub-optimality angle is
   \begin{equation*}
     \tan \varphi^t = \tan \varphi^{t - 1} + \varphi^{t - 1}  (\tan^2
     \varphi^{t - 1} + 1).
   \end{equation*}
\end{theorem}
\begin{proof}
  As $\| \bar{\theta}^{\ast} \| \rightarrow \infty$, we can
  obtain the EM update rule for $\theta^t$ in previous Corollary~\ref{cor:noiseless}.
  \[ \frac{\theta^t}{\| \theta^{\ast} \|} = \frac{2}{\pi} \left[ \mathrm{sgn}
  (\rho^{t - 1}) \varphi^{t-1}  \frac{\theta^{\ast}}{\| \theta^{\ast} \|} + \cos
  \varphi^{t - 1} \frac{\theta^{t - 1}}{\| \theta^{t - 1} \|} \right] \]
  Let $\varphi \assign \frac{\pi}{2} - \arccos | \rho | \in [ 0,
  \frac{\pi}{2} )$, then since $\rho^0, \rho^{t - 1}$ have the same sign
  (validated by checking the sign of $\langle \theta^t, \theta^{\ast} \rangle$)
  \[ \frac{\theta^t}{\| \theta^{\ast} \|} = \left( \frac{\pi}{2} \right)^{- 1}
     \left[ \varphi^{t - 1}  \frac{\tmop{sgn} (\rho^0) \theta^{\ast}}{\|
     \theta^{\ast} \|} + \cos \varphi^{t - 1}  \frac{\theta^{t - 1}}{\|
     \theta^{t - 1} \|} \right] \]
  With $\left\langle \frac{\tmop{sgn} (\rho^0) \theta^{\ast}}{\| \theta^{\ast}
  \|}, \frac{\theta^{t - 1}}{\| \theta^{t - 1} \|} \right\rangle = | \rho^{t -
  1} | = \sin \varphi^{t - 1}$
  \[ \sin \varphi^t  \frac{\| \theta^t \|}{\| \theta^{\ast} \|} = | \rho^t |
     \frac{\| \theta^t \|}{\| \theta^{\ast} \|} = \left\langle
     \frac{\theta^t}{\| \theta^{\ast} \|}, \frac{\tmop{sgn} (\rho^0)
     \theta^{\ast}}{\| \theta^{\ast} \|} \right\rangle = \left( \frac{\pi}{2}
     \right)^{- 1}  [\varphi^{t - 1} + \cos \varphi^{t - 1} \sin \varphi^{t -
     1}] \]
  \[ \frac{\| \theta^t \|}{\| \theta^{\ast} \|} = \left( \frac{\pi}{2}
     \right)^{- 1} \sqrt{[\varphi^{t - 1}]^2 + \cos^2 \varphi^{t - 1} + 2
     \varphi^{t - 1} \cos \varphi^{t - 1} \sin \varphi^{t - 1}} \]
  Therefore
  \[ \sin \varphi^t = \frac{\varphi^{t - 1} + \cos \varphi^{t - 1} \sin
     \varphi^{t - 1}}{\sqrt{[\varphi^{t - 1}]^2 + \cos^2 \varphi^{t - 1} + 2
     \varphi^{t - 1} \cos \varphi^{t - 1} \sin \varphi^{t - 1}}} \]
  Hence
  \[ \cos \varphi^t = \sqrt{1 - \sin^2 \varphi^t} = \frac{\cos^2 \varphi^{t -
     1}}{\sqrt{[\varphi^{t - 1}]^2 + \cos^2 \varphi^{t - 1} + 2 \varphi^{t -
     1} \cos \varphi^{t - 1} \sin \varphi^{t - 1}}} \]
  Thus, we obtain the recurrence relation for $\varphi^t$
  \[ \tan \varphi^t = \frac{\sin \varphi^t}{\cos \varphi^t} = \frac{\varphi^{t
     - 1}}{\cos^2 \varphi^{t - 1}} + \frac{\sin \varphi^{t - 1}}{\cos
     \varphi^{t - 1}} = \tan \varphi^{t - 1} + \varphi^{t - 1}  [\tan^2
     \varphi^{t - 1} + 1] \]
\end{proof}

\begin{theorem}{(Proposition~\ref{prop:cycloid} in Section~\ref{sec:population}: Cycloid Trajectory)}
   If $\rho^0 \assign \frac{\langle \theta^0, \theta^{\ast} \rangle}{\|
   \theta^0 \| \cdot \| \theta^{\ast} \|} \neq \pm 1$, namely $\phi^0 \assign 2
   \arccos | \rho^0 | \in (0, \pi]$. Then the coordinates $\mathtt{x}^t,
   \mathtt{y}^t$ of normalized vector $\frac{\theta^t}{\| \theta^{\ast} \|} =
   \mathtt{x}^t \hat{e}_1 + \mathtt{y}^t \hat{e}_2^t = \mathtt{x}^t \hat{e}_1 +
   \mathtt{y}^t \hat{e}_2^0, \forall t \in \mathbb{N}_+$ for EM updates at
   the Population level can be parameterized with the angle $\phi^{t - 1} \assign 2
   \arccos | \rho^{t - 1} | \in (0, \pi]$ as follows, where $\rho^{t - 1}
   \assign \frac{\langle \theta^{t - 1}, \theta^{\ast} \rangle}{\| \theta^{t -
   1} \| \cdot \| \theta^{\ast} \|}$.
   \begin{eqnarray*}
     1 - \tmop{sgn} (\rho^0) \mathtt{x}^t & = & \frac{1}{\pi} [\phi^{t - 1} -
     \sin \phi^{t - 1}]\\
     \mathtt{y}^t & = & \frac{1}{\pi} [1 - \cos \phi^{t - 1}] 
   \end{eqnarray*}
   Hence, the trajectory of iterations $\theta^t, \forall t \in \mathbb{N}_+$ is
   on the cycloid with a parameter $\frac{\| \theta^{\ast} \|}{\pi}$, on the
   plane $\tmop{span} \{ \theta^0, \theta^{\ast} \}$.
\end{theorem}
\begin{proof}
  Let's prove this, by using the recurrence relation in Proposition~\ref{prop:recurrence}.
  
  Since $\tan \varphi^t = \tan \varphi^{t - 1} + \varphi^{t - 1}  (\tan^2
  \varphi^{t - 1} + 1)$ in Proposition~\ref{prop:recurrence}, it shows that $\tan \varphi^t \geq
  \tan \varphi^{t - 1} \geq 0$, therefore $0 \leq \varphi^0 \leq \varphi^1
  \leq \cdots \leq \varphi^{t - 1} \leq \varphi^t < \frac{\pi}{2}$.
  
  Let $\hat{e}_1 \assign \frac{\theta^{\ast}}{\| \theta^{\ast} \|}$, and
  $\hat{e}^t_2 \assign \hat{e}_2 \mid_{\theta = \theta^t} = \frac{\theta -
  \hat{e}_1  \hat{e}_1^{\top} \theta}{\| \theta - \hat{e}_1  \hat{e}_1^{\top}
  \theta \|} \mid_{\theta = \theta^t} = \frac{\frac{\theta^t}{\| \theta^t \|}
  - [\rho^t] \frac{\theta^{\ast}}{\| \theta^{\ast} \|}}{\sqrt{1 -
  [\rho^t]^2}}$ and $\langle \hat{e}_1, \hat{e}^t_2 \rangle = 0, \| \hat{e}_1
  \| = \| \hat{e}^t_2 \| = 1$
  \begin{eqnarray*}
    \frac{\left( \frac{\pi}{2} \right)}{\| \theta^{\ast} \| \| \theta^{t - 1}
    \|}  \langle \theta^{t - 1} - \hat{e}_1  \hat{e}_1^{\top} \theta^{t - 1},
    \theta^t - \hat{e}_1  \hat{e}_1^{\top} \theta^t \rangle & = & \frac{\left(
    \frac{\pi}{2} \right)}{\| \theta^{\ast} \| \| \theta^{t - 1} \|} 
    \{\langle \theta^{t - 1}, \theta^t \rangle - \langle \theta^{t - 1},
    \hat{e}_1 \rangle \langle \hat{e}_1, \theta^t \rangle\}\\
    & = & [\varphi^{t - 1} \sin \varphi^{t - 1} + \cos \varphi^{t - 1}] -
    \sin \varphi^{t - 1}  [\varphi^{t - 1} + \cos \varphi^{t - 1} \sin
    \varphi^{t - 1}]\\
    & = & \cos^3 \varphi^{t - 1} > 0
  \end{eqnarray*}
  Hence, we conclude that $\langle \hat{e}^{t - 1}_2, \hat{e}^t_2 \rangle >
  0$, With $\hat{e}^{t - 1}_2, \hat{e}^t_2 \perp \hat{e}_1$ and $\hat{e}^{t -
  1}_2, \hat{e}^t_2 \in \tmop{span} \{\theta^t, \theta^{t - 1}, \theta^{\ast}
  \} \subset \tmop{span} \{\theta^0, \theta^{\ast} \}$, $\| \hat{e}^{t - 1}_2
  \| = \| \hat{e}^t_2 \| = 1$, we validate $\hat{e}^0_2 = \cdots = \hat{e}^{t - 1}_2 = \hat{e}^t_2$.

  By the definition of $\hat{e}^t_2$, we obtain $\theta^t = \| \theta^t \| 
  \{\tmop{sgn} (\rho^0) \sin \varphi^t  \hat{e}_1 + \cos \varphi^t 
  \hat{e}^t_2 \} = \| \theta^t \|  \{\tmop{sgn} (\rho^0) \sin \varphi^t 
  \hat{e}_1 + \cos \varphi^t  \hat{e}^0_2 \}$
  
  Since $\theta^t \in \tmop{span} \{\theta^{t - 1}, \theta^{\ast} \}$, then
  $\theta^t \in \tmop{span} \{\theta^0, \theta^{\ast} \}$, we can express
  $\frac{\theta^t}{\| \theta^{\ast} \|} = \mathtt{x}^t \hat{e}_1 +
  \mathtt{y}^t \hat{e}^t_2 = \mathtt{x}^t \hat{e}_1 + \mathtt{y}^t
  \hat{e}^0_2$.
  
  Comparing the expressions for $\theta^t$, we derive the following result.
  \[ \frac{\theta^t}{\| \theta^{\ast} \|} = \mathtt{x}^t \hat{e}_1 +
     \mathtt{y}^t \hat{e}^0_2 = \left\{ \tmop{sgn} (\rho^0) \sin \varphi^t
     \cdot \frac{\| \theta^t \|}{\| \theta^{\ast} \|} \right\} \hat{e}_1 +
     \left\{ \cos \varphi^t \cdot \frac{\| \theta^t \|}{\| \theta^{\ast} \|}
     \right\} \hat{e}^0_2 \]
  With the recurrence relation $\sin \varphi^t \cdot \frac{\| \theta^t \|}{\|
  \theta^{\ast} \|} = \left( \frac{\pi}{2} \right)^{- 1}  [\varphi^{t - 1} +
  \cos \varphi^{t - 1} \sin \varphi^{t - 1}], \cos \varphi^t \cdot \frac{\|
  \theta^t \|}{\| \theta^{\ast} \|} = \left( \frac{\pi}{2} \right)^{- 1}
  \cos^2 \varphi^{t - 1}$, which we showed in the proof of Proposition~\ref{prop:recurrence}, we
  derive the implicit equation of $\mathtt{x}^t, \mathtt{y}^t (t \geq 1)$
  \begin{eqnarray*}
    \mathtt{x}^t & = & \left\langle \frac{\theta^t}{\| \theta^{\ast} \|},
    \hat{e}_1 \right\rangle = \left( \frac{\pi}{2} \right)^{- 1} \tmop{sgn}
    (\rho^0)  [\varphi^{t - 1} + \cos \varphi^{t - 1} \sin \varphi^{t - 1}]\\
    \mathtt{y}^t & = & \left\langle \frac{\theta^t}{\| \theta^{\ast} \|},
    \hat{e}^0_2 \right\rangle = \left\langle \frac{\theta^t}{\| \theta^{\ast}
    \|}, \hat{e}^t_2 \right\rangle = \left( \frac{\pi}{2} \right)^{- 1} \cos^2
    \varphi^{t - 1}
  \end{eqnarray*}
  Let's cancel out the parameter $\varphi^{t - 1}$ in the parameterized curve
  $\varphi^{t - 1} \mapsto (\mathtt{x}^t, \mathtt{y}^t)$
  \[ \tmop{sgn} (\rho^0) \frac{\pi}{2} \mathtt{x}^t = \sqrt{\left(
     \frac{\pi}{2} \mathtt{y}^t \right)  \left( 1 - \frac{\pi}{2} \mathtt{y}^t
     \right)} + \arccos \sqrt{\frac{\pi}{2} \mathtt{y}^t} \]
  Let $\phi \assign 2 \left( \frac{\pi}{2} - \varphi \right) \in (0,
  \pi]$, then we rewrite the implicit equations of $\mathtt{x}^t,
  \mathtt{y}^t (t \geq 1)$
  \citep{harris1998handbook}.
  \begin{eqnarray*}
    1 - \tmop{sgn} (\rho^0) \mathtt{x}^t & = & \pi^{- 1}  [\phi - \sin
    \phi]_{\phi = \phi^{t - 1}}\\
    \mathtt{y}^t & = & \pi^{- 1}  [1 - \cos \phi]_{\phi = \phi^{t - 1}}
  \end{eqnarray*}
\end{proof}
\begin{theorem}{(Proposition~\ref{prop:quad} in Section~\ref{sec:population}: Quadratic Convergence Rate)}
   If $\varphi^0 \assign \frac{\pi}{2} - \arccos | \rho^0 | \in \left( 0,
   \frac{\pi}{2} \right)$, then the EM updates at population level satisfies
   \begin{equation*}
     \tan \varphi^t \geq \frac{1 + \sqrt{5}}{2} \cdot \tan \varphi^{t - 1}.
   \end{equation*}
   Particularly, if $\varphi^{t-1} \assign \frac{\pi}{2} - \arccos | \rho^{t-1} | \in [ \arctan
   1.5, \frac{\pi}{2} )$, then the EM updates at the Population level
   satisfies
   \begin{equation*}
     \frac{\pi}{2} \left( \tan \varphi^t - \frac{\pi}{4} \right) \geq \left\{
     \frac{\pi}{2} \left( \tan \varphi^{t - 1} - \frac{\pi}{4} \right)
     \right\}^2.
   \end{equation*}
\end{theorem}
\newpage
 \begin{proof}
  Let's prove the inequalites in the Propostion~\ref{prop:quad} as follows.
  
  Let's define $a^t \assign \tan \varphi^t$, then start from the recurrence
  relation in Proposition~\ref{prop:recurrence}, note that $a^0 > 0$.
  \[ a^t - a^{t - 1} = \arctan a^{t - 1}  ([a^{t - 1}]^2 + 1) > 0 \]
  {\tmstrong{Firstly}}, for the case of \ $\varphi^0 \assign \frac{\pi}{2} -
  \arccos | \rho^0 | \in \left( 0, \frac{\pi}{2} \right)$, since
  $\frac{\mathd^2 \arctan (a)}{\mathd a^2} = - \frac{2 a}{(a^2 + 1)^2} < 0$,
  it is a concave function.
  \[ \frac{\arctan a^{t + 1} - \arctan a^t}{a^{t + 1} - a^t} > \frac{\mathd
     \arctan (a)}{\mathd a} \mid_{a = a^{t + 1}} = \frac{1}{1 + [a^{t + 1}]^2}
     = \frac{\arctan a^{t + 1}}{a^{t + 2} - a^{t + 1}} \]
  Then for $\forall t \geq 0$, we obtain the following inequality.
  \[ \frac{a^{t + 2} - a^{t + 1}}{a^{t + 1} - a^t} > \frac{1}{1 -
     \frac{\arctan a^t}{\arctan a^{t + 1}}} > 1 + \frac{\arctan a^t}{\arctan
     a^{t + 1}} = 1 + \left[ \frac{a^{t + 2} - a^{t + 1}}{a^{t + 1} - a^t}
     \right]^{- 1} \frac{1 + [a^{t + 2}]^2}{1 + [a^{t + 1}]^2} > 1 + \left[
     \frac{a^{t + 2} - a^{t + 1}}{a^{t + 1} - a^t} \right]^{- 1} \]
  Hence, by solving $x > 1 + x^{- 1}  (x > 0)$, we show that
  \[ \frac{a^{t + 2} - a^{t + 1}}{a^{t + 1} - a^t} = \left( \frac{a^{t +
     2}}{a^{t + 1}} - 1 \right) \left\{ 1 + \frac{1}{\frac{a^{t + 1}}{a^t} -
     1} \right\} > \frac{\sqrt{5} + 1}{2} \quad \forall t \geq 0 \]
  With $a^1 - a^0 = \arctan a^0  ([a^0]^2 + 1) > a^0$, we conclude that
  $\frac{a^1}{a^0} - 1 > 1 > \frac{\sqrt{5} - 1}{2}$.
  
  If we assume that $\frac{a^{t + 1}}{a^t} - 1 > \frac{\sqrt{5} - 1}{2}$,
  then the inequality below shows that $\frac{a^{t + 2}}{a^{t + 1}} - 1 >
  \frac{\sqrt{5} - 1}{2}$.
  \[ \left( \frac{a^{t + 2}}{a^{t + 1}} - 1 \right) \cdot \frac{\frac{\sqrt{5}
     + 1}{2}}{\frac{\sqrt{5} - 1}{2}} = \tmmathbf{} \left( \frac{a^{t +
     2}}{a^{t + 1}} - 1 \right) \left\{ 1 + \frac{1}{\frac{\sqrt{5} - 1}{2}}
     \right\} > \left( \frac{a^{t + 2}}{a^{t + 1}} - 1 \right) \left\{ 1 +
     \frac{1}{\frac{a^{t + 1}}{a^t} - 1} \right\} > \frac{\sqrt{5} + 1}{2} \]
  By mathematical induction, $\frac{a^{t + 1}}{a^t} > \frac{\sqrt{5} + 1}{2}$
  for $\forall t \geq 0$. Therefore, we have proved the following inequality.
  \[ \tan \varphi^t \geq \frac{1 + \sqrt{5}}{2} \cdot \tan \varphi^{t - 1}
     \quad \forall t \in \mathbb{N}_+ \]

  {\tmstrong{Secondly}}, for the case of $\varphi^0 \assign \frac{\pi}{2} -
  \arccos | \rho^0 | \in \left( \arctan 1.5, \frac{\pi}{2} \right)$, then we
  have $a^0 = \tan \varphi^0 \geq 1.5$.
  
  Applying the elementary inequality $\arctan a > \frac{\pi a}{2 a + \pi}$, $\forall a > 0$, 
  and noting the fact that $a^t \geq a^0 \geq 1.5$, and $1.5 - \frac{4 \pi^2 + \pi^4}{8 (2 \cdot
  1.5 + \pi)} + \frac{1}{8} \pi (4 + \pi^2) \approx 4.16 > \frac{\pi}{4} +
  \frac{\pi^3}{32} \approx 1.75$
  \begin{eqnarray*}
    a^{t + 1} & = & a^t + \arctan a^t  (1 + [a^t]^2)\\
    & > & a^t + \frac{\pi a^t}{2 a^t + \pi} \cdot (1 + [a^t]^2)\\
    & = & \frac{\pi}{2} [a^t]^2 - \frac{\pi^2}{4} a^t + a^t - \frac{4 \pi^2 +
    \pi^4}{8 (2 a^t + \pi)} + \frac{1}{8} \pi (4 + \pi^2)\\
    & \geq & \frac{\pi}{2} [a^t]^2 - \frac{\pi^2}{4} a^t + 1.5 - \frac{4
    \pi^2 + \pi^4}{8 (2 \cdot 1.5 + \pi)} + \frac{1}{8} \pi (4 + \pi^2)\\
    & > & \frac{\pi}{2} [a^t]^2 - 2 \cdot \frac{\pi}{2} \cdot \frac{\pi}{4}
    a^t + \frac{\pi}{4} + \frac{\pi^3}{32}\\
    & = & \frac{\pi}{4} + \frac{\pi}{2}  \left[ a^t - \frac{\pi}{4} \right]^2
  \end{eqnarray*}
  Therefore, we have proved the following inequality.
  \[ \frac{\pi}{2} \left( \tan \varphi^t - \frac{\pi}{4} \right) \geq \left\{
     \frac{\pi}{2} \left( \tan \varphi^{t - 1} - \frac{\pi}{4} \right)
     \right\}^2 \quad \forall t \in \mathbb{N}_+ \]
  
\end{proof}

\begin{corollary}{(Corollary~\ref{cor:err_mixing} in Section~\ref{sec:population}: Error of Mixing Weights $\pi^t$)}
In the noiseless setting, the error of mixing weights for EM updates at the population level is
   \begin{equation*}
       \| \pi^t - \bar{\pi}^{\ast} \|_1 = \left| 1 - \frac{2}{\pi} \varphi^{t - 1}
       \right| \cdot \left\| \frac{1}{2} - \pi^{\ast} \right\|_1
   \end{equation*}
 where $\bar{\pi}^{\ast} \assign \frac{1}{2} - \tmop{sgn} (\rho^0) 
 (\frac{1}{2} - \pi^{\ast}), \varphi^{t - 1} \assign \frac{\pi}{2} - \arccos |
 \rho^{t - 1} |$ and $\rho^{t - 1} \assign \frac{\langle \theta^{t - 1},
 \theta^{\ast} \rangle}{\| \theta^{t - 1} \| \cdot \| \theta^{\ast} \|}$,
 $\rho^0 \assign \frac{\langle \theta^0, \theta^{\ast} \rangle}{\| \theta^0 \|
 \cdot \| \theta^{\ast} \|}$.
\end{corollary}
\begin{proof}
   Using Corollary~\ref{cor:noiseless}, and note that $\mathrm{sgn} (\rho^{t - 1}) = \mathrm{sgn} (\rho^{0})$, we obtain that equation.
\begin{equation*}
   \tanh (\nu^t) = \mathrm{sgn} (\rho^{0}) \left( \frac{2}{\pi}
   \varphi^{t - 1} \right) \cdot \tanh (\nu^{\ast})
\end{equation*}
Since $\pi^t(1)=\frac{1+ \tanh(\nu^t)}{2}, \pi^t(2)=\frac{1- \tanh(\nu^t)}{2}$ and 
$\bar{\pi}^{\ast}(1)=\frac{1+ \mathrm{sgn} (\rho^{0})\tanh(\nu^\ast)}{2}, \bar{\pi}^{\ast}(2)=\frac{1- \mathrm{sgn} (\rho^{0})\tanh(\nu^\ast)}{2}$.
\begin{equation*}
   \| \pi^t - \bar{\pi}^{\ast} \|_1 = |\pi^t(1)-\bar{\pi}^{\ast}(1)| + |\pi^t(2)-\bar{\pi}^{\ast}(2)|
   = | \tanh(\nu^t) - \mathrm{sgn} (\rho^{0})\tanh(\nu^\ast)| = \left| 1 - \frac{2}{\pi} \varphi^{t - 1}
   \right| \cdot \left\| \frac{1}{2} - \pi^{\ast} \right\|_1
\end{equation*}
In the above equation, we use such an identity $\tanh (\nu^{\ast}) = \left\| \frac{1}{2} - \pi^{\ast} \right\|_1$.
\end{proof}

\begin{theorem}{(Theorem~\ref{thm:population} in Section~\ref{sec:population}: Population Level Convergence)}
   If the initial sup-optimality cosine $\rho^0 \assign \frac{\langle \theta^0,
   \theta^{\ast} \rangle}{\| \theta^0 \| \cdot \| \theta^{\ast} \|} \neq 0$,
   then with the number of total iterations at most $T =\mathcal{O} \left( \log
   \frac{1}{| \rho^0 |} \vee \log \log \frac{1}{\varepsilon} \right)$, the error of EM update at the population level is bounded by $\frac{\|
   \theta^{T + 1} - \tmop{sgn} (\rho^0) \theta^{\ast} \|}{\| \theta^{\ast} \|}
   < \varepsilon$, and $\|\pi^{T+1} -\bar{\pi}^\ast\|_1 =\mathcal{O}(\sqrt{\varepsilon})\cdot \|\frac{1}{2}-\pi^\ast\|_1$
   , where $\bar{\pi}^\ast := \frac{1}{2}-\tmop{sgn} (\rho^0) (\frac{1}{2}-\pi^\ast)$.
\end{theorem}
\begin{proof}
  Let's prove Theorem~\ref{thm:population}, and consider the convergence rate of EM updates at
  population level.
  
  Let $a^t \assign \tan \varphi^t$, and $\varphi^0 \assign \frac{\pi}{2} -
  \arccos | \rho^0 | \in \left( 0, \frac{\pi}{2} \right)$.
  
  {\tmstrong{Step 1}}. Determine the minmum iteration number $T'$ required to
  ensure $a^{T'} \geq 1.5$
  
  If $a^0 \geq 1.5$, then $T' = 0$.
  
  Otherwise $a^0 < 1.5$, with the inequality in Proposition~\ref{prop:quad}, we obtain the
  following result for $\forall t \in \mathbb{N}_+$.
  \[ a^t = \left[ \prod^t_{t' = 1} \frac{a^{t'}}{a^{t' - 1}} \right] \cdot a^0
     \geq \left( \frac{\sqrt{5} + 1}{2} \right)^t a^0 \]
  Hence, if $T' \geq \left\lceil \frac{\log \frac{1.5}{a^0}}{\log
  \frac{\sqrt{5} + 1}{2}} \right\rceil$, we shows that $a^{T'} \geq \left(
  \frac{\sqrt{5} + 1}{2} \right)^{T'} a^0 \geq 1.5$
  
  {\tmstrong{Step 2}}. Determine the minimum iteration number $T''$ required
  to ensure $a^{T' + T''} > N_{\varepsilon}$, where $N_{\varepsilon}$ is a big
  number that replies on $\varepsilon$.
  
  With the inequality in Proposition~\ref{prop:quad}, we obtain the following result for
  $\forall t \in \mathbb{N}_+$ and $t \geq T'$
  \[ \frac{\pi}{2}  \left[ a^{t + 1} - \frac{\pi}{4} \right] > \left(
     \frac{\pi}{2}  \left[ a^t - \frac{\pi}{4} \right] \right)^2 \]
  Hence, if $T'' \geq \left\lceil \frac{\log \left[ \log \left(
  N_{\varepsilon} - \frac{\pi}{4} \right) + \log \frac{\pi}{2} \right] - \log
  \left( \log \left( \frac{\pi}{4}  \left( 1.5 - \frac{\pi}{4} \right) \right)
  \right)}{\log 2} \right\rceil =\mathcal{O} (\log [\log N_{\varepsilon}])$
  \[ \frac{\pi}{2}  \left[ a^{T' + T''} - \frac{\pi}{4} \right] > \left(
     \frac{\pi}{2}  \left[ a^{T'} - \frac{\pi}{4} \right] \right)^{2^{T''}}
     \geq \left( \frac{\pi}{2}  \left[ 1.5 - \frac{\pi}{4} \right]
     \right)^{2^{T''}} \geq \frac{\pi}{2}  \left[ N_{\varepsilon} -
     \frac{\pi}{4} \right] \]
  Hence $T'' \leq \mathcal{O} (\log [\log N_{\varepsilon}])$, then $T \assign
  T' + T'' =\mathcal{O} \left( \log \frac{1}{a^0} \vee \log [\log
  N_{\varepsilon}] \right)$ ensures $a^T > N_{\varepsilon}$ .
  
  {\tmstrong{Step 3}}. Determine the relationship between $N_{\varepsilon}$
  and the desired relative error $\varepsilon$.
  
  With Proposition~\ref{prop:cycloid}, we can write the expression for relative error of
  $\theta$ as follows.
  \begin{eqnarray*}
    \frac{\| \theta^{T + 1} - \tmop{sgn} (\rho^0) \theta^{\ast} \|}{\|
    \theta^{\ast} \|} & = & \pi^{- 1}  \sqrt{(\phi^T - \sin \phi^T)^2 + (1 -
    \cos \phi^T)^2}\\
    & = & \pi^{- 1}  \left[ \frac{\phi^2}{2} - \frac{\phi^4}{72} +
    \frac{\phi^6}{6480} - o (\phi^6) \right]_{\phi = \phi^T}
  \end{eqnarray*}
  By letting $\phi^T =\mathcal{O} \left( \sqrt{\varepsilon} \right)$, we
  ensure that $\frac{\| \theta^{T+1} - \tmop{sgn} (\rho^0) \theta^{\ast} \|}{\|
  \theta^{\ast} \|} < \varepsilon$.

  $\phi^T =\mathcal{O} \left( \sqrt{\varepsilon} \right)$, that is $\frac{\pi}{2}-\varphi^T=\frac{\phi^T}{2}=\mathcal{O} \left( \sqrt{\varepsilon} \right)$.

  With Corollary~\ref{cor:err_mixing}, we show that $ \| \pi^{T+1} - \bar{\pi}^{\ast} \|_1 = \frac{2}{\pi}\left| \frac{\pi}{2}-\varphi^T
  \right| \cdot \left\| \frac{1}{2} - \pi^{\ast} \right\|_1 = \mathcal{O} \left( \sqrt{\varepsilon} \right) \left\| \frac{1}{2} - \pi^{\ast} \right\|_1$.
  
  Note that $\phi^T \assign 2 \left( \frac{\pi}{2} - \varphi^T \right)$ and
  $a^T \assign \tan \varphi^T = \tan \left( \frac{\pi}{2} - \frac{\phi^T}{2}
  \right)$, expand $\tan \left( \frac{\pi}{2} - x \right) = \frac{1}{x} -
  \frac{x}{3} - \frac{x^3}{45} + o (x^3)$.
  \[ a^T = \frac{1}{\mathcal{O} \left( \sqrt{\varepsilon} \right)}
     -\mathcal{O} \left( \sqrt{\varepsilon} \right) -\mathcal{O} \left(
     \sqrt{\varepsilon} \right)^3 + o \left( \mathcal{O} \left(
     \sqrt{\varepsilon} \right)^3 \right) = \Omega \left(
     \frac{1}{\sqrt{\varepsilon}} \right) \]
  Let $N_{\varepsilon} \assign \Theta \left( \frac{1}{\sqrt{\varepsilon}}
  \right)$, with $T =\mathcal{O} \left( \log \frac{1}{\arctan a^0} \vee \log
  [\log N_{\varepsilon}] \right)$, we ensure $a^T > N_{\varepsilon} = \Theta
  \left( \frac{1}{\sqrt{\varepsilon}} \right)$, then we obtain
  \[ T =\mathcal{O} \left( \log \frac{1}{\arctan a^0} \vee \log \left[ \log
     \left( \frac{1}{\sqrt{\varepsilon}} \right) \right] \right) =\mathcal{O}
     \left( \log \frac{1}{\arctan a^0} \vee \log \log \frac{1}{\varepsilon}
     \right) \]
  Note that $\varphi^0 \assign \frac{\pi}{2} - \arccos | \rho^0 |$, and $a^0
  \assign \tan \varphi^0$, then $| \rho^0 | = \sin (\varphi^0) = \Theta
  (\varphi^0) = \Theta (a^0)$ when $a^0 < 1.5$, thus
  \[ T =\mathcal{O} \left( \log \frac{1}{a^0} \vee \log \log
     \frac{1}{\varepsilon} \right) =\mathcal{O} \left( \log \frac{1}{| \rho^0
     |} \vee \log \log \frac{1}{\varepsilon} \right) \]
  
\end{proof}

%% file: 8_supplementary_4.tex
\newpage
\section{Proof for Results at the Finite-sample Level}\label{sup:finite_sample}

\input{8_supplementary_4_0.tex}
\input{8_supplementary_4_1.tex}

\input{8_supplementary_4_2.tex}

%% file: 8_supplementary_4_0.tex
\subsection{Upper-bound for Statistical Error}

\begin{theorem}{(Proposition~\ref{prop:proj_err} in Section~\ref{sec:finite}: Projected Statistical Error)}
  In the noiseless setting, the projection on $\text{span}\{\theta,\theta^\ast\}$ for the statistical error of $\theta$ satisfies
  \begin{equation}\nonumber
    \frac{\|P_{\theta,\theta^\ast}
    [M_n^{\tmop{easy}} (\theta, \nu) - M (\theta, \nu)]\|}{\|\theta^\ast\|}
    = \mathcal{O}\left(\sqrt{\frac{\log \frac{1}{\delta}}{n}} \vee \frac{\log \frac{1}{\delta}}{n}\right),
  \end{equation}
  \normalsize
  with probability at least $1 - \delta$, where $M_n (\theta, \nu), M (\theta,
  \nu)$ are the EM update rules for $\theta$ at the Finite-sample level and the
  population level respectively, and the orthogonal projection matrix $P_{\theta,\theta^\ast}$ satisfies
  $\text{span}(P_{\theta,\theta^\ast})=\text{span}\{\theta,\theta^\ast\}$ 
  .%and $P^2_{\theta,\theta^\ast}=P_{\theta,\theta^\ast}, P^\top_{\theta,\theta^\ast}=P_{\theta,\theta^\ast}$.
\end{theorem}
\begin{proof}
In the noiseless setting, the statistical error is
\begin{eqnarray*}
  M^{\tmop{easy}}_n (\theta^t, \nu^t) - M (\theta^t, \nu^t) & = & \left\{ 
  \frac{1}{n}  \sum_{i \in [n]} -\mathbb{E}_{s \sim p (s \mid \theta^{\ast},
  \pi^{\ast})} \right\} \tanh \left( \frac{y_i  \langle x_i, \theta^t
  \rangle}{\sigma^2} + \nu^t \right) y_i x_i\\
  & \rightarrow & \left\{ \frac{1}{n}  \sum_{i \in [n]} -\mathbb{E}_{\{ x_i
  \}_{i \in [n]} \overset{\tmop{iid}}{\sim} \mathcal{N} (0, I_d)} \right\} |
  \langle x_i, \theta^{\ast} \rangle | \tmop{sgn} \langle x_i, \theta^t
  \rangle x_i
\end{eqnarray*}
  Let $x_i = \tilde{x}_i + x_i^{\perp}$ and $\tilde{x}_i \in \tmop{span} \{
    \theta^{\ast}, \theta^t \}, x_i^{\perp} \perp \tmop{span} \{ \theta^{\ast},
    \theta^t \}$, we may assume $\dim \tmop{span} \{ \theta^{\ast},
    \theta^t \} = 2$ without loss of generality, and decompose the space into $\mathbb{R}^d = \tmop{span} \{ \theta^{\ast}, \theta^t \} \oplus \tmop{span}
    \{ \theta^{\ast}, \theta^t \}^{\perp}$.
   
    It implies that 2nd term of statistical error of (Easy EM) doesn't depends on
    dimension $d$ of the space $\mathbb{R}^d$.
    \begin{eqnarray*}
      &  & M_n^{\tmop{easy}} (\theta^t, \nu^t) - M (\theta^t, \nu^t)
      =\left\{ \frac{1}{n}  \sum_{i \in [n]} -\mathbb{E}_{\{ x_i \}_{i \in
      [n]} \overset{\tmop{iid}}{\sim} \mathcal{N} (0, I_d)} \right\} | \langle
      x_i, \theta^{\ast} \rangle | \tmop{sgn} \langle x_i, \theta^t \rangle x_i\\
      & = & \frac{1}{n}  \sum_{i \in [n]} | \langle \tilde{x}_i,
      \theta^{\ast} \rangle | \tmop{sgn} \langle \tilde{x}_i, \theta^t \rangle
      x_i^{\perp} + \left\{ \frac{1}{n}  \sum_{i \in [n]} -\mathbb{E}_{\{ x_i
      \}_{i \in [n]} \overset{\tmop{iid}}{\sim} \mathcal{N} (0, I_d)} \right\} |
      \langle \tilde{x}_i, \theta^{\ast} \rangle | \tmop{sgn} \langle \tilde{x}_i,
      \theta^t \rangle \tilde{x}_i
    \end{eqnarray*}
Let the orthogonal projection matrix $P_{\theta,\theta^\ast}$ satisfy
  $\text{span}(P_{\theta,\theta^\ast})=\text{span}\{\theta,\theta^\ast\}$, the 2nd term is the projected statistical error.
\begin{equation*}
    \|P_{\theta,\theta^\ast}
  [M_n^{\tmop{easy}} (\theta^t, \nu^t) - M (\theta^t, \nu^t)]\|
  = \left\| \left\{ \frac{1}{n}  \sum_{i
  \in [n]} -\mathbb{E}_{\{ x_i \}_{i \in [n]} \overset{\tmop{iid}}{\sim}
  \mathcal{N} (0, I_d)} \right\} | \langle \tilde{x}_i, \theta^{\ast} \rangle |
  \tmop{sgn} \langle \tilde{x}_i, \theta^t \rangle \tilde{x}_i \right\|_2
\end{equation*}
To estimate the projected statistical error, we begin with decompose it into two parts in terms of $\hat{e}_1, \hat{e}_2$.

\input{8_supplementary_4_0_bound_2nd_term_}
\end{proof}

\begin{theorem}{(Proposition~\ref{prop:stat_err} in Section~\ref{sec:finite}: Statistical Error)}
In the noiseless setting, the statistical error of $\theta$ for EM
updates at the Finite-sample level satisfies
\small
\begin{equation}\nonumber
  \frac{\| M_n (\theta, \nu) - M (\theta, \nu) \|_2}{\| \theta^{\ast} \|}
  =\mathcal{O} \left( \sqrt{\frac{d}{n}} \vee \frac{\log
  \frac{1}{\delta}}{n} \vee \sqrt{\frac{\log \frac{1}{\delta}}{n}}
  \right),
\end{equation}
\normalsize
with probability at least $1 - \delta$, $M_n (\theta, \nu), M (\theta,
\nu)$ denote the EM update rules for $\theta$ at the Finite-sample level and the
Population level.
\end{theorem}
\begin{proof}
%In the noiseless setting, the statistical error is
\begin{eqnarray*}
  M^{\tmop{easy}}_n (\theta^t, \nu^t) - M (\theta^t, \nu^t) & = & \left\{ 
  \frac{1}{n}  \sum_{i \in [n]} -\mathbb{E}_{s \sim p (s \mid \theta^{\ast},
  \pi^{\ast})} \right\} \tanh \left( \frac{y_i  \langle x_i, \theta^t
  \rangle}{\sigma^2} + \nu^t \right) y_i x_i\\
  & \rightarrow & \left\{ \frac{1}{n}  \sum_{i \in [n]} -\mathbb{E}_{\{ x_i
  \}_{i \in [n]} \overset{\tmop{iid}}{\sim} \mathcal{N} (0, I_d)} \right\} |
  \langle x_i, \theta^{\ast} \rangle | \tmop{sgn} \langle x_i, \theta^t
  \rangle x_i
\end{eqnarray*}
  Let $x_i = \tilde{x}_i + x_i^{\perp}$ and $\tilde{x}_i \in \tmop{span} \{
    \theta^{\ast}, \theta^t \}, x_i^{\perp} \perp \tmop{span} \{ \theta^{\ast},
    \theta^t \}$, we may assume $\dim \tmop{span} \{ \theta^{\ast},
    \theta^t \} = 2$ without loss of generality, and decompose the space into $\mathbb{R}^d = \tmop{span} \{ \theta^{\ast}, \theta^t \} \oplus \tmop{span}
    \{ \theta^{\ast}, \theta^t \}^{\perp}$.
   
    It implies that 2nd term of statistical error of (Easy EM) doesn't depends on
    dimension $d$ of the space $\mathbb{R}^d$
    \begin{eqnarray*}
      &  & M_n^{\tmop{easy}} (\theta^t, \nu^t) - M (\theta^t, \nu^t)
      =\left\{ \frac{1}{n}  \sum_{i \in [n]} -\mathbb{E}_{\{ x_i \}_{i \in
      [n]} \overset{\tmop{iid}}{\sim} \mathcal{N} (0, I_d)} \right\} | \langle
      x_i, \theta^{\ast} \rangle | \tmop{sgn} \langle x_i, \theta^t \rangle x_i\\
      & = & \frac{1}{n}  \sum_{i \in [n]} | \langle \tilde{x}_i,
      \theta^{\ast} \rangle | \tmop{sgn} \langle \tilde{x}_i, \theta^t \rangle
      x_i^{\perp} + \left\{ \frac{1}{n}  \sum_{i \in [n]} -\mathbb{E}_{\{ x_i
      \}_{i \in [n]} \overset{\tmop{iid}}{\sim} \mathcal{N} (0, I_d)} \right\} |
      \langle \tilde{x}_i, \theta^{\ast} \rangle | \tmop{sgn} \langle \tilde{x}_i,
      \theta^t \rangle \tilde{x}_i
    \end{eqnarray*}
In the previous Proposition~\ref{prop:proj_err}, we bound the $\ell_2$ norm of the second term with
\begin{eqnarray*}
  \frac{\|P_{\theta,\theta^\ast}
  [M_n^{\tmop{easy}} (\theta^t, \nu^t) - M (\theta^t, \nu^t)]\|}{\|\theta^\ast\|}
  &=& \frac{1}{\| \theta^{\ast} \|} \left\| \left\{ \frac{1}{n}  \sum_{i
  \in [n]} -\mathbb{E}_{\{ x_i \}_{i \in [n]} \overset{\tmop{iid}}{\sim}
  \mathcal{N} (0, I_d)} \right\} | \langle \tilde{x}_i, \theta^{\ast} \rangle |
  \tmop{sgn} \langle \tilde{x}_i, \theta^t \rangle \tilde{x}_i \right\|_2\\
  &=& \mathcal{O}\left(\sqrt{\frac{\log \frac{1}{\delta}}{n}} \vee \frac{\log \frac{1}{\delta}}{n}\right),
\end{eqnarray*}
Let's focus on the first term, we start by rewriting the $\ell_2$ norm of the first term in a different notation.

%{\color{red}  insert lemma for bounding the first term}
\input{8_supplementary_4_0_bound_1st_term_}

\textbf{Bound for Easy EM}
For easy EM in the noiseless setting, 
% we update $\theta$ with
% \[ \begin{array}{lll}
%      \theta^{t + 1} \leftarrow M^{\tmop{easy}}_n (\theta^t, \nu^t) & = &
%      \frac{1}{n}  \sum_{i\in [n]}|
%      \langle x_i, \theta^{\ast} \rangle | \tmop{sgn} \langle x_i, \theta^t
%      \rangle x_i
%    \end{array} \]
we show the following upper bound for the statistical error (in $\ell_2$ norm) by combining the upper-bounds for the 1st and 2nd terms.
\begin{eqnarray*}
  &  & \frac{\| M^{\tmop{easy}}_n (\theta^t, \nu^t) - M (\theta^t, \nu^t)
  \|_2}{\| \theta^{\ast} \|}\\
  % & = & \frac{1}{\| \theta^{\ast} \|} \left\| \left\{ \frac{1}{n}  \sum_{i
  % \in [n]} -\mathbb{E}_{\{ x_i \}_{i \in [n]} \overset{\tmop{iid}}{\sim}
  % \mathcal{N} (0, I_d)} \right\} | \langle x_i, \theta^{\ast} \rangle |
  % \tmop{sgn} \langle x_i, \theta^t \rangle x_i \right\|_2\\
  & \leq & \frac{1}{\| \theta^{\ast} \|} \left\| \frac{1}{n}  \sum_{i \in
  [n]} | \langle \tilde{x}_i, \theta^{\ast} \rangle | \tmop{sgn} \langle
  \tilde{x}_i, \theta^t \rangle x_i^{\perp} \right\|_2 + \frac{1}{\|
  \theta^{\ast} \|} \left\| \left\{ \frac{1}{n}  \sum_{i \in [n]}
  -\mathbb{E}_{\{ x_i \}_{i \in [n]} \overset{\tmop{iid}}{\sim} \mathcal{N}
  (0, I_d)} \right\} | \langle \tilde{x}_i, \theta^{\ast} \rangle | \tmop{sgn}
  \langle \tilde{x}_i, \theta^t \rangle \tilde{x}_i \right\|_2\\
  & \leq & \left(2\sqrt{\frac{(d-2)}{n}} + 2 \frac{\log\frac{2}{\delta}}{n} + 4 \sqrt{\frac{\log\frac{2}{\delta}}{n}}\right)
  + \max \left( 20 \sqrt{\frac{\log \frac{8}{\delta}}{n}}, 45 \frac{\log
  \frac{8}{\delta}}{n} \right)
  %\\& = & 
  =\mathcal{O} \left( \sqrt{\frac{d}{n}} \vee \frac{\log \frac{1}{\delta}}{n} \vee \sqrt{\frac{\log \frac{1}{\delta}}{n}}
  \right)
\end{eqnarray*}

\textbf{Bound for EM}
For the standard EM, We update parameters $\theta$ with
\[ \begin{array}{lll}
     \theta^{t + 1} \leftarrow M_n (\theta^t, \nu^t) & = & \left[ \frac{1}{n} 
     \sum_{i \in [n]} x_i x_i^{\top} \right]^{- 1} \frac{1}{n}  \sum_{i \in
     [n]} \tanh \left( \frac{y_i  \langle x_i, \theta^t \rangle}{\sigma^2} +
     \nu^t \right) y_i x_i
   \end{array} \]
In noiseless setting ($\sigma \rightarrow 0$, SNR$\rightarrow \infty$),
consider the difference between EM updates at the finite-sample/population level.
\begin{eqnarray*}
  M_n (\theta^t, \nu^t) - M (\theta^t, \nu^t) & \rightarrow & \left[
  \frac{1}{n}  \sum_{i \in [n]} x_i x_i^{\top} \right]^{- 1} \left\{
  \frac{1}{n}  \sum_{i \in [n]} -\mathbb{E}_{s \sim p (s \mid \theta^{\ast},
  \pi^{\ast})} \right\} | \langle x_i, \theta^{\ast} \rangle | \tmop{sgn}
  \langle x_i, \theta^t \rangle x_i\\
  &  & + \left[ \frac{1}{n}  \sum_{i \in [n]} x_i x_i^{\top} \right]^{- 1}
  \left\{ I_d - \left[ \frac{1}{n}  \sum_{i \in [n]} x_i x_i^{\top} \right]
  \right\} \mathbb{E}_{s \sim p (s \mid \theta^{\ast}, \pi^{\ast})} | \langle
  x_i, \theta^{\ast} \rangle | \tmop{sgn} \langle x_i, \theta^t \rangle x_i
\end{eqnarray*}
By using \citep{wainwright2019high} 
%[\href{https://people.eecs.berkeley.edu/~wainwrig/BibPapers/Wai19.pdf}{High-Dimensional
%Statistics: A Non-Asymptotic Viewpoint}] by Martin J. Wainwright, 
page 162,
equation (6.9), where $\gamma_{\min}$ is the minimum eigen value.
\[ \mathbb{P} \left[ \sqrt{\gamma_{\min} \left[ \frac{1}{n}  \sum_{i \in [n]}
   x_i x_i^{\top} \right]} \leq (1 - \delta) - \sqrt{\frac{d}{n}} \right] \leq
   \mathrm{e}^{- n \delta^2 / 2} \]
Let $\mathrm{e}^{- n \delta^2 / 2} \leftarrow \delta, \delta \leftarrow
\sqrt{2} \sqrt{\frac{\log \frac{1}{\delta}}{n}}$, with probability at least $1
- \delta$
\[ \sqrt{\gamma_{\min} \left[ \frac{1}{n}  \sum_{i \in [n]} x_i x_i^{\top}
   \right]} \geq 1 - \sqrt{2} \sqrt{\frac{\log \frac{1}{\delta}}{n}} -
   \sqrt{\frac{d}{n}} =\mathcal{O} (1) \]
Using $\frac{1}{\| \theta^{\ast} \|} \mathbb{E}_{s \sim p (s \mid
\theta^{\ast}, \pi^{\ast})} | \langle x, \theta^{\ast} \rangle | \tmop{sgn}
\langle x, \theta^t \rangle x = \mathrm{sgn} (\rho)  \left[ 1 - \frac{\arccos
| \rho |}{\frac{\pi}{2}} \right]  \frac{\theta^{\ast}}{\| \theta^{\ast} \|} +
\left( \frac{\pi}{2} \right)^{- 1} \sqrt{1 - \rho^2}  \frac{\theta^t}{\|
\theta^t \|}$

Hence $\frac{\| \mathbb{E}_{s \sim p (s \mid \theta^{\ast}, \pi^{\ast})} |
\langle x, \theta^{\ast} \rangle | \tmop{sgn} \langle x, \theta^t \rangle x
\|_2}{\| \theta^{\ast} \|} = \sqrt{\left[ 1 - \frac{\arccos | \rho
|}{\frac{\pi}{2}} \right]^2 + \left( \frac{\pi}{2} \right)^{- 2} (1 - \rho^2)
+ \frac{4}{\pi} | \rho | \sqrt{1 - \rho^2} \left[ 1 - \frac{\arccos | \rho
|}{\frac{\pi}{2}} \right]} \in [0, 1]$

By using \citep{wainwright2019high} 
%[\href{https://people.eecs.berkeley.edu/~wainwrig/BibPapers/Wai19.pdf}{High-Dimensional
%Statistics: A Non-Asymptotic Viewpoint}] by Martin J. Wainwright, 
page 162,
equation (6.12)
\[ \mathbb{P} \left[ \left\| \left[ \frac{1}{n}  \sum_{i \in [n]} x_i
   x_i^{\top} \right] - I_d \right\|_2 \geq 2 \sqrt{\frac{d}{n}} + 2 \delta +
   \left( \sqrt{\frac{d}{n}} + \delta \right)^2 \right] \leq 2\mathrm{e}^{- n
   \delta^2 / 2} \]
let $2\mathrm{e}^{- n \delta^2 / 2} \leftarrow \delta, \delta \leftarrow
\sqrt{2} \sqrt{\frac{\log \frac{2}{\delta}}{n}}$, with probability at least $1
- \delta$
\[ \left\| \left[ \frac{1}{n}  \sum_{i \in [n]} x_i x_i^{\top} \right]  - I_d
   \right\|_2 \leq 2 \left( \sqrt{\frac{d}{n}} + 2 \sqrt{2} \sqrt{\frac{\log
   \frac{2}{\delta}}{n}} \right) + \left( \sqrt{\frac{d}{n}} + \sqrt{2}
   \sqrt{\frac{\log \frac{2}{\delta}}{n}} \right)^2 \]

\begin{eqnarray*}
  &  & \frac{\| M_n (\theta^t, \nu^t) - M (\theta^t, \nu^t) \|_2}{\|
  \theta^{\ast} \|}\\
  & \leq & \left\{ \gamma_{\min} \left[ \frac{\sum_{i \in [n]}}{n} x_i
  x_i^{\top} \right] \right\}^{- 1} \cdot \left\{ \frac{1}{\| \theta^{\ast}
  \|} \left\| \left\{ \frac{\sum_{i \in [n]}}{n} -\mathbb{E} \right\} |
  \langle x_i, \theta^{\ast} \rangle | \tmop{sgn} \langle x_i, \theta^t
  \rangle x_i \right\|_2 + \left\| \left[ \frac{\sum_{i \in [n]}}{n} x_i
  x_i^{\top} \right] - I_d \right\|_2 \right\}\\
  & = & \frac{\mathcal{O} \left( \sqrt{\frac{d}{n}} \vee \frac{\log
  \frac{1}{\delta}}{n} \vee \sqrt{\frac{\log \frac{1}{\delta}}{n}}
  \right) + \left\{ 2 \left( \sqrt{\frac{d}{n}} + 2 \sqrt{2} \sqrt{\frac{\log
  \frac{2}{\delta}}{n}} \right) + \left( \sqrt{\frac{d}{n}} + \sqrt{2}
  \sqrt{\frac{\log \frac{2}{\delta}}{n}} \right)^2 \right\}}{\mathcal{O}
  (1)}\\
  & = & 
  % \mathcal{O} \left( \sqrt{\frac{d}{n}} \vee \frac{\log
  % \frac{1}{\delta}}{n} \vee \sqrt{\frac{\log \frac{1}{\delta}}{n}}
  % \right) +\mathcal{O} \left( \left( \sqrt{\frac{d}{n}} \vee \sqrt{\frac{\log
  % \frac{1}{\delta}}{n}} \right) \vee \left( \sqrt{\frac{d}{n}} \vee
  % \sqrt{\frac{\log \frac{1}{\delta}}{n}} \right)^2 \right)
  %\\
  % & = & 
  % =\mathcal{O} \left( \frac{\log \frac{1}{\delta}}{n} \vee
  % \left( \sqrt{\frac{d}{n}} \vee \sqrt{\frac{\log \frac{1}{\delta}}{n}}
  % \right) \vee \left( \sqrt{\frac{d}{n}} \vee \sqrt{\frac{\log
  % \frac{1}{\delta}}{n}} \right)^2 \right) =
  \mathcal{O} \left(
  \sqrt{\frac{d}{n}} \vee \frac{\log \frac{1}{\delta}}{n} \vee
  \sqrt{\frac{\log \frac{1}{\delta}}{n}} \right)
\end{eqnarray*}
\end{proof}

\begin{lemma}{(Convergence of $\theta^t$ in Single Iteration)}\label{lem:single_iter}
  For $\vartheta \assign \sin \varphi \hat{e}_1 + \cos \varphi \hat{e}_2$,
  where $\{ \hat{e}_1, \hat{e}_2 \}$ is an orthonormal basis for the subspace
  $\tmop{span} \{ \hat{e}_1, \hat{e}_2 \} \subset \mathbb{R}^d$, and $\varphi
  \in \left( 0, \frac{\pi}{2} \right)$; with a pertubation vector $\varrho \in
  \mathbb{R}^d$ with lengh $\| \varrho \| = r \in (0, \sin \varphi)$; then the
  angle $\varphi' \assign
  \arcsin \frac{| \langle \vartheta + \varrho, \hat{e}_1 \rangle |}{\|
  \vartheta + \varrho \|}$, satisfies $\varphi' \geq \varphi - \arcsin r$
\end{lemma}

\begin{proof}
  Note that with $\| \varrho \| = r \in (0, \sin \varphi)$, $\langle \vartheta
  + \varrho, \hat{e}_1 \rangle = \langle \vartheta, \hat{e}_1 \rangle +
  \langle \varrho, \hat{e}_1 \rangle \geq \sin \varphi - \| \varrho \| = \sin
  \varphi - r > 0$, thus
  \[ \sin \varphi' = \frac{\langle \vartheta + \varrho, \hat{e}_1 \rangle}{\|
     \vartheta + \varrho \|} > 0 \]
  Express the pertubation vector by $\varrho = - r' \cos (\varphi - \Delta)
  \hat{e}_1 + r' \sin (\varphi - \Delta) \hat{e}_2 + \sqrt{r^2 - \left[ {r'} 
  \right]^2} \hat{e}$, where $r' \in [0, r], \Delta \in (- \pi, \pi]$ and
  $\hat{e} \in \tmop{span} \{ \hat{e}_1, \hat{e}_2 \}^{\perp}, \| \hat{e} \| =
  1$
  \begin{eqnarray*}
    \langle \vartheta + \varrho, \hat{e}_1 \rangle & = & \langle \vartheta,
    \hat{e}_1 \rangle + \langle \varrho, \hat{e}_1 \rangle = \sin \varphi - r'
    \cos (\varphi - \Delta)\\
    \| \vartheta + \varrho \| & = & \left\| [\sin \varphi - r' \cos (\varphi -
    \Delta)] \hat{e}_1 + [\cos \varphi + r' \sin (\varphi - \Delta)] \hat{e}_2
    + \sqrt{r^2 - \left[ {r'}  \right]^2} \hat{e} \right\|\\
    & = & \sqrt{[\sin \varphi - r' \cos (\varphi - \Delta)]^2 + [\cos \varphi
    + r' \sin (\varphi - \Delta)] + \left[ r^2 - \left[ {r'}  \right]^2
    \right]}\\
    & = & \sqrt{[1 + r^2] - 2 r' \sin \Delta}
  \end{eqnarray*}
  Hence, let $p : = \frac{r' | \sin \Delta |}{r}, - r' \cos \Delta \geq -
  \sqrt{[r']^2 - [r p]^2} \geq - \sqrt{r^2 - [r p]^2} = - r \sqrt{1 - p^2}$
  and $p = \frac{r'}{r} | \sin \Delta | \leq 1$
  \begin{eqnarray*}
    \sin \varphi' & = & \frac{\sin \varphi - r' \cos (\varphi -
    \Delta)}{\sqrt{[1 + r^2] - 2 r' \sin \Delta}} r\\
    & = & \frac{[1 - r' \sin \Delta] \sin \varphi - r' \cos \Delta \cos
    \varphi}{\sqrt{[1 + r^2] - 2 r' \sin \Delta}}\\
    & \geq & \frac{[1 - r p] \sin \varphi - r \sqrt{1 - p^2} \cos
    \varphi}{\sqrt{[1 + r^2] - 2 r p}} \assign \psi (p)
  \end{eqnarray*}
  For $\psi (p) \assign \frac{[1 - r p] \sin \varphi - r \sqrt{1 - p^2} \cos
  \varphi}{\sqrt{[1 + r^2] - 2 r p}}$ for $p \in [0, 1]$, note that $\cos
  \left( \varphi + \left[ \frac{\pi}{2} - \arcsin p \right] \right) \leq \cos
  \varphi$
  
  Thus $\cos \varphi - r \cos \left( \varphi + \left[ \frac{\pi}{2} - \arcsin
  p \right] \right) > 0$
  \begin{eqnarray*}
    \frac{\mathd}{\mathd p} \log \psi & = & \frac{- r \sin \varphi + r
    \frac{p}{\sqrt{1 - p^2}} \cos \varphi}{[1 - r p] \sin \varphi - r \sqrt{1
    - p^2} \cos \varphi} - \frac{1}{2} \cdot \frac{- 2 r}{[1 + r^2] - 2 r p}\\
    & = & \frac{r \cdot \left\{ \cos \varphi - r \left[ p \cos \varphi -
    \sqrt{1 - p^2} \sin \varphi \right] \right\} (p - r)}{\sqrt{1 - p^2}
    \left\{ [1 - r p] \sin \varphi - r \sqrt{1 - p^2} \cos \varphi \right\}
    \cdot \{ [1 + r^2] - 2 r p \}}\\
    & = & \frac{r \cdot \left\{ \cos \varphi - r \cos \left( \varphi + \left[
    \frac{\pi}{2} - \arcsin p \right] \right) \right\} (p - r)}{\sqrt{1 - p^2}
    \left\{ [1 - r p] \sin \varphi - r \sqrt{1 - p^2} \cos \varphi \right\}
    \cdot \{ [1 + r^2] - 2 r p \}}
  \end{eqnarray*}
  therefore, $\frac{\mathd}{\mathd p} \log \psi < 0, \forall p \in (0, r) ;
  \frac{\mathd}{\mathd p} \log \psi > 0, \forall p \in (r, 1)$, hence $\psi
  (p) \geq \psi (p) \mid_{p = r}$ for $p \in [0, 1]$
  \begin{eqnarray*}
    \sin \varphi' & \geq & \psi (p)\\
    & \geq & \psi (p) \mid_{p = r}\\
    & = & \sqrt{1 - r^2} \sin \varphi - r \cos \varphi\\
    & = & \sin (\varphi - \arcsin r)
  \end{eqnarray*}
  Note that $r \in (0, \sin \varphi)$, that is $\frac{\pi}{2} > \varphi >
  \varphi - \arcsin r > 0$, and we show that
  \[ \varphi' \geq \varphi - \arcsin r \]
  
\end{proof}

%% file: 8_supplementary_4_0_bound_2nd_term_.tex
%     It implies that 2nd term of statistical error of (Easy EM) doesn't depends on
%     dimension $d$ of the space $\mathbb{R}^d$
%     \begin{eqnarray*}
%       &  & \left\{ \frac{1}{n}  \sum_{i \in [n]} -\mathbb{E}_{\{ x_i \}_{i \in
%       [n]} \overset{\tmop{iid}}{\sim} \mathcal{N} (0, I_d)} \right\} | \langle
%       x_i, \theta^{\ast} \rangle | \tmop{sgn} \langle x_i, \theta^t \rangle x_i\\
%       & = & \frac{1}{n}  \sum_{i \in [n]} | \langle \tilde{x}_i,
%       \theta^{\ast} \rangle | \tmop{sgn} \langle \tilde{x}_i, \theta^t \rangle
%       x_i^{\perp} + \left\{ \frac{1}{n}  \sum_{i \in [n]} -\mathbb{E}_{\{ x_i
%       \}_{i \in [n]} \overset{\tmop{iid}}{\sim} \mathcal{N} (0, I_d)} \right\} |
%       \langle \tilde{x}_i, \theta^{\ast} \rangle | \tmop{sgn} \langle \tilde{x}_i,
%       \theta^t \rangle \tilde{x}_i
%     \end{eqnarray*}

% We can decompose the 2nd term into two part in terms of $\hat{e}_1, \hat{e}_2$

Let $\tilde{x}_i = \lambda_{1i} \hat{e}_1 + \lambda_{2i} \hat{e}_2,\lambda_{1
i}, \lambda_{2 i} \overset{\tmop{iid}}{\sim} \mathcal{N} (0, 1)$, so $\tilde{x}_i = \left\langle \left(\begin{array}{c}
  \rho\\
  \sqrt{1 - \rho^2}
\end{array}\right), \left(\begin{array}{c}
  \lambda_{1 i}\\
  \lambda_{2 i}
\end{array}\right) \right\rangle \vec{e}_1 + \left\langle
\left(\begin{array}{c}
  \sqrt{1 - \rho^2}\\
  - \rho
\end{array}\right), \left(\begin{array}{c}
  \lambda_{1 i}\\
  \lambda_{2 i}
\end{array}\right) \right\rangle \vec{e}_2$.

Let $Z_i \assign | \lambda_{1i} | \cdot |\rho \lambda_{1i}+ \sqrt{1 - \rho^2}\lambda_{2i}|$,
$Z_i' \assign | \lambda_{1i} | \tmop{sgn}(\rho \lambda_{1i}+ \sqrt{1 - \rho^2}\lambda_{2i})\cdot (-\sqrt{1 - \rho^2}  \lambda_{1i}+ \rho \lambda_{2i})$.

% Let $Z_i \assign | \lambda_{1i} | \left| \left\langle \left(\begin{array}{c}
%     \rho\\
%     \sqrt{1 - \rho^2}
%   \end{array}\right), \left(\begin{array}{c}
%     \lambda_{1i}\\
%     \lambda_{2i}
%   \end{array}\right) \right\rangle \right|$, 
% $Z_i' \assign | \lambda_{1i} | \tmop{sgn}
% \left\langle \left(\begin{array}{c}
%   \rho\\
%   \sqrt{1 - \rho^2}
% \end{array}\right), \left(\begin{array}{c}
%   \lambda_{1i}\\
%   \lambda_{2i}
% \end{array}\right) \right\rangle \left\langle \left(\begin{array}{c}
%   - \sqrt{1 - \rho^2}\\
%   \rho
% \end{array}\right), \left(\begin{array}{c}
%   \lambda_{1i}\\
%   \lambda_{2i}
% \end{array}\right) \right\rangle$.

\begin{eqnarray*}
  \frac{\|P_{\theta,\theta^\ast}
    [M_n^{\tmop{easy}} (\theta, \nu) - M (\theta, \nu)]\|}{\|\theta^\ast\|}
  =\frac{1}{\| \theta^{\ast} \|} \sqrt{\left| \frac{1}{n}  \sum_{i \in
  [n]} (Z_i -\mathbb{E} [Z_i]) \right|^2 + \left| \frac{1}{n}  \sum_{i \in
  [n]} (Z'_i -\mathbb{E} [Z'_i]) \right|^2}\\
\end{eqnarray*}
Note that both $Z_i, Z_i', i\in[n]$ are sub-exponential with parameters $\left( 2 \cdot
2.91\mathrm{e} (2 \pi)^{- \frac{1}{4}}, 2 \cdot 2.91\mathrm{e} \right)$ (we explain the reason later).
By using the concentration inequality in \citep{wainwright2019high} page 29,
equation (2.18) for iid sub-exponential r.v. with parameters, with
probability at least $1 - 2 \delta$
\[ \left| \frac{1}{n}  \sum_{i \in [n]} (Z_i -\mathbb{E} [Z_i]) \right|,
\left| \frac{1}{n}  \sum_{i \in [n]} (Z_i' -\mathbb{E} [Z_i']) \right| \leq
   \max \left( 2^{\frac{1}{2}} \left( 2 \cdot 2.91\mathrm{e} (2 \pi)^{-
   \frac{1}{4}} \right) \sqrt{\frac{\log \frac{1}{\delta}}{n}}, 2 (2 \cdot
   2.91\mathrm{e}) \frac{\log \frac{1}{\delta}}{n} \right) \]
Hence, the proof is complete.
\begin{eqnarray*}
  \frac{\|P_{\theta,\theta^\ast}
    [M_n^{\tmop{easy}} (\theta, \nu) - M (\theta, \nu)]\|}{\|\theta^\ast\|}
  & \leq & \sqrt{2} \max \left( 2^{\frac{1}{2}} \left( 2 \cdot 2.91\mathrm{e}
  (2 \pi)^{- \frac{1}{4}} \right) \sqrt{\frac{\log \frac{1}{\delta}}{n}}, 2 (2
  \cdot 2.91\mathrm{e}) \frac{\log \frac{1}{\delta}}{n} \right)\\
  & < & \max \left( 20 \sqrt{\frac{\log \frac{1}{\delta}}{n}}, 45 \frac{\log
  \frac{1}{\delta}}{n} \right)
  =\mathcal{O}\left(\sqrt{\frac{\log \frac{1}{\delta}}{n}} \vee \frac{\log \frac{1}{\delta}}{n}\right)
\end{eqnarray*}
Now, let's explain why both $Z_i, Z_i', i\in[n]$ are sub-exponential. For brevity, we write $Z, Z'$ for $Z_i, Z_i'$ instead.
\begin{equation*}
    Z  \leq \frac{1 + | \rho |}{2} (\lambda_1^2 + \lambda_2^2) \leq \lambda_1^2 + \lambda_2^2,\quad
    Z' \leq  \frac{1 + \sqrt{1 -  \rho^2}}{2} (\lambda_1^2 + \lambda_2^2) \leq \lambda_1^2 + \lambda_2^2,\quad |\mathbb{E}[Z]|\leq  1, \quad |\mathbb{E}[Z']|\leq  1
\end{equation*}

Let $W \assign \lambda_1^2 + \lambda_2^2 \sim \chi^2 (2)$, then the $q$-th
moment of $| Z |$ for $q \geq 2$, $\mathbb{E} [| Z |^q] \leq \mathbb{E}_{W \sim \chi^2  (2)} [W^q] = 2^q q!$.
By Minkowski's Inequality and Stirling's approximation, $- x \log x \leq \mathrm{e}^{- 1}$ and $\exp \left( \frac{\log q}{q} \right)
\leq \exp \left( \frac{\log 3}{3} \right) = 3^{\frac{1}{3}}$ for $q \geq 2$
and $\left\{ \frac{2 \sqrt{2 \pi} \mathrm{e}^{\frac{1}{12}}}{\mathrm{e}}
3^{\frac{1}{6}} + \frac{1}{2} \right\} \approx 2.907 < 2.91$
\begin{eqnarray*}
  \mathbb{E} [| Z -\mathbb{E} [Z] |^q]^{\frac{1}{q}} & \leq & \mathbb{E} [| Z
  |^q]^{\frac{1}{q}} +\mathbb{E} [| \mathbb{E} [Z] |^q]^{\frac{1}{q}}
  =\mathbb{E} [| Z |^q]^{\frac{1}{q}} + | \mathbb{E} [Z ] |
  \leq [2^q q!]^{\frac{1}{q}} + 1\\
  % & \leq & 
  % 2 \left[ \sqrt{2 \pi} \mathrm{e}^{\frac{1}{12}} q^{q +
  % \frac{1}{2}} \mathrm{e}^{- q} \right]^{\frac{1}{q}} + 1\\
  & \leq & 
  \frac{2 \left[ \sqrt{2 \pi} \mathrm{e}^{\frac{1}{12}}
  \right]^{\frac{1}{q}}}{\mathrm{e}} \sqrt{\exp \left( \frac{\log q}{q}
  \right)} q + 1
  \leq \frac{2 \sqrt{2 \pi} \mathrm{e}^{\frac{1}{12}}}{\mathrm{e}}
  3^{\frac{1}{6}} q + 1
  \leq \left\{ \frac{2 \sqrt{2 \pi}
  \mathrm{e}^{\frac{1}{12}}}{\mathrm{e}} 3^{\frac{1}{6}} + \frac{1}{2}
  \right\} q
  < 2.91 q
\end{eqnarray*}
Then for $0 \leq | \lambda | \leq \frac{1}{2 \cdot 2.91\mathrm{e}}$
\begin{eqnarray*}
  & & \mathbb{E} [\exp (\lambda \{ Z -\mathbb{E} [Z] \})] \leq \mathbb{E}
  [\exp (| \lambda | | Z -\mathbb{E} [Z] |)]
  = 1 + \sum_{q = 2}^{\infty} \frac{| \lambda |^q \mathbb{E} [| Z
  -\mathbb{E} [Z] |^q]}{q!}\\
  & \leq & 
  % 1 + \sum_{q = 2}^{\infty} \frac{| \lambda |^q \cdot 2.91^q \cdot
  % q^q}{\sqrt{2 \pi} q^{q + \frac{1}{2}} \mathrm{e}^{- q}}\\
  % & = & 
  1 + \frac{1}{\sqrt{2 \pi}} \sum_{q = 2}^{\infty}
  \frac{(2.91\mathrm{e} | \lambda |)^q}{q^{\frac{1}{2}} }
  \leq 1 + \frac{1}{\sqrt{2 \pi}} \frac{(2.91\mathrm{e} | \lambda |)^2}{1
  - 2.91\mathrm{e} | \lambda |}
  \leq 1 + \frac{2}{\sqrt{2 \pi}} (2.91\mathrm{e} | \lambda |)^2
  % \leq \exp \left( \frac{2}{\sqrt{2 \pi}} (2.91\mathrm{e} \lambda)^2
  % \right)\\
  % & = & 
  \leq \exp \left( \frac{\left( 2 \cdot 2.91\mathrm{e} (2 \pi)^{-
  \frac{1}{4}} \right)^2 \lambda^2}{2} \right)
\end{eqnarray*}
Therefore, $Z$ and $Z'$ (the same reason) are sub-exponential with parameters $\left( 2 \cdot
2.91\mathrm{e} (2 \pi)^{- \frac{1}{4}}, 2 \cdot 2.91\mathrm{e} \right)$.

%% file: 8_supplementary_4_0_bound_1st_term_.tex
% \textbf{distribution of 1st term}
Let the projection matrix $P \assign \hat{e}_1 \hat{e}_1^{\top} + \hat{e}_2
\hat{e}^{\top}_2$, select an orthonormal basis 
$\{ \hat{e}_3, \cdots, \hat{e}_d \}$ to form $\tmop{span} \{ \theta^{\ast},
\theta^t \}^{\perp}$, and let $Q \assign (\hat{e}_1, \hat{e}_2, \hat{e}_3, \cdots, \hat{e}_d)$.

Let $\left( \tilde{x}_i' {, x_i'}^{\perp} \right) \assign Q^{\top} x_i$, where
$\tilde{x}_i' \in \mathbb{R}^2 {, x_i'}^{\perp} \in \mathbb{R}^{d - 2}$ are
independent, and $\tilde{x}_i' \overset{\tmop{iid}}{\sim} \mathcal{N} (0, I_2)
{, x_i'}^{\perp} \overset{\tmop{iid}}{\sim} \mathcal{N} (0, I_{d - 2})$.

Note that ${x_i''}^{\perp} \assign \tmop{sgn} \langle \tilde{x}_i, \theta^t
\rangle \tmop{sgn} \langle \tilde{x}_i, \theta^t \rangle Q^{\top} x_i^{\perp}
= \tmop{sgn} \langle \tilde{x}_i', (1, 0)^{\top} \rangle \tmop{sgn}
\left\langle \tilde{x}_i', \left( \rho, \sqrt{1 - \rho^2} \right)^{\top}
\right\rangle {x_i'}^{\perp}$ are indepent from $\tilde{x}_i'$
and $\tilde{x}_i'' \assign \langle \tilde{x}_i', (1, 0)^{\top} \rangle
\overset{\tmop{iid}}{\sim} \mathcal{N} (0, 1) {, x_i''}^{\perp}
\overset{\tmop{iid}}{\sim} \mathcal{N} (0, I_{d - 2})$.
We define $\tilde{x}'' \assign \{ \tilde{x}_i'' \}_{i \in [n]} \in \mathbb{R}^n {,
x_j''}^{\perp} : = \left\{ {x_{i j}''}^{\perp} \right\}_{i \in [n]} \in
\mathbb{R}^n$ for $j \in [d - 2]$ (Remark: ${x_{i j}''}^{\perp}$ is the $j$-th
component of ${x_i''}^{\perp}$).

The components of projected vector $\alpha_j = \left\langle
\frac{\tilde{x}''}{\| \tilde{x}'' \|} {, x_j''}^{\perp} \right\rangle
\overset{\tmop{iid}}{\sim} \mathcal{N} (0, 1)$ are independent from each other,
and let $Z_1 \assign \left[ \sum_{i \in [n]} (\tilde{x}_i'')^2 \right] \sim
\chi^2 (n), Z_2 \assign \left[ \sum_{j \in [d - 2]} \alpha_j^2 \right] \sim
\chi^2 (d - 2)$ are independent from each other.
\begin{eqnarray*}
  \left\| \frac{1}{n}  \sum_{i \in [n]} | \langle \tilde{x}, \theta^{\ast}
  \rangle | \tmop{sgn} \langle \tilde{x}, \theta^t \rangle x^{\perp}
  \right\|_2 & = & \frac{\| \theta^{\ast} \|}{n}  \left\|  \sum_{i \in [n]}
  \tilde{x}_i'' {x_i''}^{\perp} \right\|_2
  = \frac{\| \theta^{\ast} \|}{n}  \sqrt{\sum_{j \in [d - 2]} \left(
  \sum_{i \in [n]} \tilde{x}_i'' {x_{i j}''}^{\perp} \right)^2}\\
  & = & \frac{\| \theta^{\ast} \|}{n}  \sqrt{\sum_{j \in [d - 2]} \|
  \tilde{x}'' \|^2 \alpha_j^2}
  = \frac{\| \theta^{\ast} \|}{n}  \sqrt{\left[ \sum_{i \in [n]}
  (\tilde{x}_i'')^2 \right] \cdot \left[ \sum_{j \in [d - 2]} \alpha_j^2
  \right]}
  = \frac{\| \theta^{\ast} \|}{n} \cdot \sqrt{Z_1 Z_2}
\end{eqnarray*}
By using the concentration inequality for Chi-square distribution (see Lemma 1, page 1325 in \cite{laurent2000adaptive}), then with at least probability at least $\left(1-\frac{\delta}{2}\right)^2\geq 1-\delta$
\begin{eqnarray*}
  Z_1 \leq \left(\sqrt{n} + \sqrt{\log\frac{2}{\delta}}\right)^2 + \log\frac{2}{\delta}
  ,\quad
  Z_2 \leq \left(\sqrt{d-2} + \sqrt{\log\frac{2}{\delta}}\right)^2 + \log\frac{2}{\delta}
\end{eqnarray*}
Therefore
\begin{eqnarray*}
  \sqrt{Z_1 Z_2} 
  \leq 2 \left(\sqrt{n} + \sqrt{\log\frac{2}{\delta}}\right)
  \left(\sqrt{d-2} + \sqrt{\log\frac{2}{\delta}}\right)
  = 2\sqrt{n(d-2)} + 2 \log\frac{2}{\delta} + 2(\sqrt{n} + \sqrt{d-2}) \sqrt{\log\frac{2}{\delta}}
\end{eqnarray*}
Note that $d\leq n$, hence we upper-bound the $\ell_2$ norm for the first term.
\begin{equation*}
  \frac{1}{\| \theta^{\ast} \|} \left\| \frac{1}{n}  \sum_{i \in
  [n]} | \langle \tilde{x}_i, \theta^{\ast} \rangle | \tmop{sgn} \langle
  \tilde{x}_i, \theta^t \rangle x_i^{\perp} \right\|_2
  \leq 2\sqrt{\frac{(d-2)}{n}} + 2 \frac{\log\frac{2}{\delta}}{n} + 4 \sqrt{\frac{\log\frac{2}{\delta}}{n}}
\end{equation*}

%% file: 8_supplementary_4_1.tex
\subsection{\texorpdfstring{Initialization and Convergence of $\theta$}{Initialization and Convergence of theta}}
\begin{theorem}{(Proposition~\ref{prop:init} in Section~\ref{sec:finite}: Initialization with Easy EM)}
  In the noiseless setting, suppose we run the sample-splitting finite-sample
  Easy EM with \ $n' \assign \Theta \left( \frac{n}{\log \frac{1}{\delta}}
  \wedge \left[ \frac{n}{\log \frac{1}{\delta}} \right]^2 \right)$ fresh
  samples for each iteration, then after at most $T_0 =\mathcal{O} \left( \log
  \frac{1}{\delta} \right)$ iterations, it satisfies $\varphi^{T_0} \geq
  \Theta \left( \sqrt{\frac{\log \frac{1}{\delta}}{n}} \vee \frac{\log
  \frac{1}{\delta}}{n} \right)$ with probability at least $1 - \delta$.
\end{theorem}

  \begin{proof}
    Suppose we run finite-sample easy EM with refresh samples $n'$ for each iterations, then after some iterations $\varphi \geq \sqrt{\frac{1}{n'}}$.
    We will prove this in the followings. Let's denote $\hat{\theta}, \hat{\varphi}$ the EM update at population level.
    With the EM update $\hat{\theta}^{t + 1} \assign M (\theta^t)$ and let $\tilde{x}_i =
    \lambda_{1 i} \vec{e}_1 + \lambda_{2 i} \vec{e}_2$ then $\theta^{\ast} = \rho
    \vec{e}_1 + \sqrt{1 - \rho^2} \vec{e}_2$. 
%{\color{red} simplified version for initialization}
\input{8_supplementary_4_1_initialization_}
  \end{proof}

\newpage
\begin{theorem}{(Proposition~\ref{prop:convg_angle} in Section~\ref{sec:finite}: Convergence of Angle)}
  In the noiseless setting, suppose $\varphi^0 \geq \Theta \left(
  \sqrt{\frac{\log \frac{1}{\delta}}{n}} \vee \frac{\log \frac{1}{\delta}}{n}
  \right)$, run Easy finite-sample EM for $T_1=\mathcal{O}\left( \log
  \frac{n}{\log \frac{1}{\delta}}\right)$ iterations followed by the standard finite-sample EM for at most $T' =\mathcal{O} \left(
  \log \frac{n}{d} 
    \wedge \log \frac{n}{\log \frac{1}{\delta}} \right)$
  iterations with all the same $n =
  \Omega \left( d \vee \log \frac{1}{\delta} 
  %\vee \frac{\log^2\frac{1}{\delta}}{d} 
  \right)$ samples, then it satisfies
  \begin{equation}
    \varphi^T \geq \frac{\pi}{2} - \Theta \left( \sqrt{\frac{d}{n}} \vee
    \frac{\log \frac{1}{\delta}}{n} \vee \sqrt{\frac{\log
    \frac{1}{\delta}}{n}} \right),
  \end{equation}
  with probability at least $1 - T \delta$, where $T:=T_1+T',\varphi^0 \assign
  \frac{\pi}{2} - \arccos \left| \frac{\langle \theta^0, \theta^{\ast}
  \rangle}{\| \theta^0 \| \cdot \| \theta^{\ast} \|} \right|$ and $\varphi^T
  \assign \frac{\pi}{2} - \arccos \left| \frac{\langle \theta^T, \theta^{\ast}
  \rangle}{\| \theta^T \| \cdot \| \theta^{\ast} \|} \right|$.
\end{theorem}
\begin{proof}
  In this proof, we assume that $\Theta \left( \sqrt{\frac{\log \frac{1}{\delta}}{n}} \vee
  \frac{\log \frac{1}{\delta}}{n} \right) \leq \Theta \left( \sqrt{\frac{d}{n}}
  \vee \frac{\log \frac{1}{\delta}}{n} \vee \sqrt{\frac{\log
  \frac{1}{\delta}}{n}} \right) < 0.1$, and denote by $\Theta \assign \Theta \left( \sqrt{\frac{d}{n}} \vee \frac{\log
  \frac{1}{\delta}}{n} \vee \sqrt{\frac{\log \frac{1}{\delta}}{n}}
  \right)$ the threshold for $\varphi$.
  
  Besides, we denote $\bar{\theta}, \bar{\varphi}$ for the EM update at population level.
  
  We divide the whole procedure inito three stages.
  
  In Stage 1, $\varphi \geq \tmop{const} \cdot \Theta \left( \sqrt{\frac{\log
  \frac{1}{\delta}}{n}} \vee \frac{\log \frac{1}{\delta}}{n} \right) \Rightarrow
  \varphi \geq \Theta \left( \sqrt{\frac{d}{n}} \vee \frac{\log
  \frac{1}{\delta}}{n} \vee \sqrt{\frac{\log \frac{1}{\delta}}{n}}
  \right)$ after at most $T_1$ iterations of Easy EM.
  
  In Stage 2, $\varphi \geq 4 \Theta \left( \sqrt{\frac{d}{n}} \vee \frac{\log
  \frac{1}{\delta}}{n} \vee \sqrt{\frac{\log \frac{1}{\delta}}{n}}
  \right) \Rightarrow \varphi > \arctan 1.5$
  after at most $T_2$ iterations of standard EM.
  
  In Stage 3, $\varphi > \arctan 1.5 \Rightarrow \varphi^t > \frac{\pi}{2} -
  1.775 \Theta \left( \sqrt{\frac{\log \frac{1}{\delta}}{n}} \vee \frac{\log
  \frac{1}{\delta}}{n} \right) $
  after at most $T_2$ iterations of standard EM.
  
  \textbf{Stage 1: $\varphi \geq \tmop{const} \cdot \Theta \left( \sqrt{\frac{\log
  \frac{1}{\delta}}{n}} \vee \frac{\log \frac{1}{\delta}}{n} \right) \Rightarrow
  \varphi \geq \Theta \left( \sqrt{\frac{d}{n}} \vee \frac{\log
  \frac{1}{\delta}}{n} \vee \sqrt{\frac{\log \frac{1}{\delta}}{n}}
  \right)$}
  
  In the first Stage, we run Easy EM $\theta^{t+1} \leftarrow M_n^{\text{easy}}(\theta^t, \nu^t)$, and note that the length of the projected vector is less than or equal to the length of the original one 
  $$\frac{1}{\| \theta^{\ast} \|} \langle \theta^{t + 1} - M
  (\theta^t), \hat{e}_1 \rangle \leq \frac{1}{\| \theta^{\ast} \|} \left\|
  \left\{ \frac{1}{n}  \sum_{i \in [n]} -\mathbb{E}_{\{ x_i \}_{i \in [n]} \sim
  \mathcal{N} (0, I_d)} \right\} | \langle \tilde{x}, \theta^{\ast} \rangle |
  \mathrm{sgn} \langle \tilde{x}, \theta^t \rangle \tilde{x} \right\|_2.$$
  
  With probability at least $1 - 4 \delta$
  \[ \frac{1}{| \theta^{\ast} |} \left\| \left\{ \frac{1}{n}  \sum_{i \in [n]}
     -\mathbb{E}_{\{ x_i \}_{i \in [n]} \sim \mathcal{N} (0, I_d)} \right\} |
     \langle \tilde{x}, \theta^{\ast} \rangle | \mathrm{sgn} \langle \tilde{x},
     \theta^t \rangle \tilde{x} \right\|_2 < \max \left( 20 \sqrt{\frac{\log
     \frac{1}{\delta}}{n}}, 45 \frac{\log \frac{1}{\delta}}{n} \right) : =
     \Theta \left( \sqrt{\frac{\log \frac{1}{\delta}}{n}} \vee \frac{\log
     \frac{1}{\delta}}{n} \right) \]
  By using the assumption $\Theta \left( \sqrt{\frac{\log \frac{1}{\delta}}{n}} \vee
  \frac{\log \frac{1}{\delta}}{n} \right) \leq \Theta \left( \sqrt{\frac{d}{n}}
  \vee \frac{\log \frac{1}{\delta}}{n} \vee \sqrt{\frac{\log
  \frac{1}{\delta}}{n}} \right) < 0.1$
  \[ \frac{1}{\| \theta^{\ast} \|} \langle \theta^{t + 1} - M (\theta^t),
     \hat{e}_1 \rangle < 0.1 \]
  Then $\frac{\| \theta^{t + 1} \|}{\| \theta^{\ast} \|} \leq 1 + \Theta \left(
  \sqrt{\frac{\log \frac{1}{\delta}}{n}} \vee \frac{\log \frac{1}{\delta}}{n}
  \right) < 1.1$
  
  Use $\left| \frac{1}{\| \theta^{\ast} \|} \langle \theta^{t + 1}, \vec{e}_1
  \rangle \right| \leq 2$, and with assumption $\varphi^t < \Theta \left(
  \sqrt{\frac{d}{n}} \vee \frac{\log \frac{1}{\delta}}{n} \vee
  \sqrt{\frac{\log \frac{1}{\delta}}{n}} \right) \leq 0.1$
  
  We denote $\Theta \assign \Theta \left( \sqrt{\frac{d}{n}} \vee \frac{\log
  \frac{1}{\delta}}{n} \vee \sqrt{\frac{\log \frac{1}{\delta}}{n}}
  \right)$ the threshold for $\varphi$.
  \begin{eqnarray*}
    \sin \varphi^{t + 1} & = & \left( \frac{\| \theta^{t + 1} \|}{\|
    \theta^{\ast} \|} \right)^{- 1} \left| \frac{1}{\| \theta^{\ast} \|^2}
    \langle \theta^{t + 1}, \theta^{\ast} \rangle \right| = \left( \frac{\|
    \theta^{t + 1} \|}{\| \theta^{\ast} \|} \right)^{- 1} \left| \frac{1}{\|
    \theta^{\ast} \|} \langle \theta^{t + 1}, \hat{e}_1 \rangle \right|\\
    & \geq & 1.1^{- 1} \left| \frac{1}{\| \theta^{\ast} \|} \langle \theta^{t +
    1}, \hat{e}_1 \rangle \right|\\
    & = & 1.1^{- 1} \left| \frac{1}{\| \theta^{\ast} \|} \langle \theta^{t + 1}
    - M (\theta^t), \hat{e}_1 \rangle + \frac{1}{\| \theta^{\ast} \|} \langle M
    (\theta^t), \hat{e}_1 \rangle \right|\\
    & \geq & 1.1^{- 1} \left| \frac{1}{\| \theta^{\ast} \|} \langle M
    (\theta^t), \hat{e}_1 \rangle \right| - 1.1^{- 1} \left| \frac{1}{\|
    \theta^{\ast} \|} \langle \theta^{t + 1} - M (\theta^t), \hat{e}_1 \rangle
    \right|\\
    & \geq & 1.1^{- 1} \left| \frac{1}{\| \theta^{\ast} \|} \langle M
    (\theta^t), \hat{e}_1 \rangle \right| - 1.1^{- 1} \Theta \left(
    \sqrt{\frac{\log \frac{1}{\delta}}{n}} \vee \frac{\log \frac{1}{\delta}}{n}
    \right)\\
    & \geq & 1.1^{- 1} \cdot 1.239 \cdot \left| \frac{1}{\| \theta^{\ast} \|}
    \langle \theta^t, \hat{e}_1 \rangle \right| - 1.1^{- 1} \Theta \left(
    \sqrt{\frac{\log \frac{1}{\delta}}{n}} \vee \frac{\log \frac{1}{\delta}}{n}
    \right)\\
    & \geq & 1.1263 \cdot \sin \varphi^t - 1.1^{- 1} \Theta \left(
    \sqrt{\frac{\log \frac{1}{\delta}}{n}} \vee \frac{\log \frac{1}{\delta}}{n}
    \right)
  \end{eqnarray*}
  That is
  \[ 1.1263 \left( \sin \varphi^{t + 1} - \frac{1}{1.1 \cdot 0.1263} \cdot
     \Theta \left( \sqrt{\frac{\log \frac{1}{\delta}}{n}} \vee \frac{\log
     \frac{1}{\delta}}{n} \right) \right) \geq \sin \varphi^t - \frac{1}{1.1
     \cdot 0.1263} \cdot \Theta \left( \sqrt{\frac{\log \frac{1}{\delta}}{n}}
     \vee \frac{\log \frac{1}{\delta}}{n} \right) \]
  For $\varphi^t < \Theta < 0.1$, we have $\varphi^t \geq \sin \varphi^t \geq (1
  - 0.002) \varphi^t$
  
  When $\varphi^0 \geq \frac{\left( 1 + \frac{1}{1.1 \cdot 0.1263} \right)}{(1 -
  0.002)} \Theta \left( \sqrt{\frac{\log \frac{1}{\delta}}{n}} \vee \frac{\log
  \frac{1}{\delta}}{n} \right) \approx 8.2143 \Theta \left( \sqrt{\frac{\log
  \frac{1}{\delta}}{n}} \vee \frac{\log \frac{1}{\delta}}{n} \right)$
  \begin{eqnarray*}
    \sin \varphi^0 - \frac{1}{1.1 \cdot 0.1263} \cdot \Theta \left(
    \sqrt{\frac{\log \frac{1}{\delta}}{n}} \vee \frac{\log \frac{1}{\delta}}{n}
    \right) & \geq & (1 - 0.002) \varphi^0 - \frac{1}{1.1 \cdot 0.1263} \cdot
    \Theta \left( \sqrt{\frac{\log \frac{1}{\delta}}{n}} \vee \frac{\log
    \frac{1}{\delta}}{n} \right)\\
    & \geq & \Theta \left( \sqrt{\frac{\log \frac{1}{\delta}}{n}} \vee
    \frac{\log \frac{1}{\delta}}{n} \right)
  \end{eqnarray*}
  We could {\tmstrong{assume}} that $\varphi^0 \geq 8.3 \Theta \left(
  \sqrt{\frac{\log \frac{1}{\delta}}{n}} \vee \frac{\log \frac{1}{\delta}}{n}
  \right)$, with at most $T_1$ iterations
  \begin{eqnarray*}
    \varphi^{T_1} \geq \sin \varphi^{T_1} & \geq & 1.1263^{T_1} \left[ \sin
    \varphi^0 - \frac{1}{1.1 \cdot 0.1263} \cdot \Theta \left( \sqrt{\frac{\log
    \frac{1}{\delta}}{n}} \vee \frac{\log \frac{1}{\delta}}{n} \right) \right] +
    \frac{1}{1.1 \cdot 0.1263} \cdot \Theta \left( \sqrt{\frac{\log
    \frac{1}{\delta}}{n}} \vee \frac{\log \frac{1}{\delta}}{n} \right)\\
    & \geq & 1.1263^{T_1} \Theta \left( \sqrt{\frac{\log \frac{1}{\delta}}{n}}
    \vee \frac{\log \frac{1}{\delta}}{n} \right)\\
    & \geq & \Theta \left( \sqrt{\frac{d}{n}} \vee \frac{\log
    \frac{1}{\delta}}{n} \vee \sqrt{\frac{\log \frac{1}{\delta}}{n}}
    \right) \equiv \Theta
  \end{eqnarray*}
  Let $T_1 =\mathcal{O} \left( 
  \log \frac{1}{ \Theta \left(
  \sqrt{\frac{\log \frac{1}{\delta}}{n}} \vee \frac{\log \frac{1}{\delta}}{n}
  \right)} 
  -\log\frac{1}{ \Theta \left(
    \sqrt{\frac{d}{n}} \vee \frac{\log \frac{1}{\delta}}{n} \vee
    \sqrt{\frac{\log \frac{1}{\delta}}{n}} \right) }
  \right) 
  =\mathcal{O} \left( \log \frac{1}{ \Theta \left(
    \sqrt{\frac{\log \frac{1}{\delta}}{n}} \vee \frac{\log \frac{1}{\delta}}{n}
    \right)}  \right)
  =\mathcal{O} \left(\log \frac{n}{\log
  \frac{1}{\delta}} \right)$.
  
  We use this Lemma in the previous analysis in Stage 1.
  
  {\tmstrong{Lemma}}
  
  with $\frac{1}{\| \theta^{\ast} \|} | \langle \theta^t, \hat{e}_1 \rangle | =
  \sin \varphi^t \leq \varphi^t < \Theta < 0.1$, \ $\frac{| \langle M
  (\theta^t), \hat{e}_1 \rangle |}{\| M (\theta^t) \|} = \sqrt{[1 - \pi^{- 1}
  (\phi^t - \sin \phi^t)]^2 + [\pi^{- 1} (1 - \cos \phi^t)]^2} \leq 0.643$
  
  for $\phi^t \assign 2 \left( \frac{\pi}{2} - \varphi^t \right) \in [\pi - 0.2,
  \pi]$
  \[ \frac{\frac{| \langle M (\theta^t), \hat{e}_1 \rangle |}{\| M (\theta^t)
     \|}}{\sqrt{1 - \left( \frac{| \langle M (\theta^t), \hat{e}_1 \rangle |}{\|
     M (\theta^t) \|} \right)^2}} \geq \frac{1 + \sqrt{5}}{2} \cdot
     \frac{\frac{| \langle \theta^t, \hat{e}_1 \rangle |}{\| \theta^t
     \|}}{\sqrt{1 - \left( \frac{| \langle \theta^t, \hat{e}_1 \rangle |}{\|
     \theta^t \|} \right)^2}} \]
  we conclude that
  \[ \frac{| \langle M (\theta^t), \hat{e}_1 \rangle |}{\| M (\theta^t) \|}
     \cdot \frac{1}{\sqrt{1 - 0.643^2}} \geq \frac{1 + \sqrt{5}}{2} \cdot
     \frac{| \langle \theta^t, \hat{e}_1 \rangle |}{\| \theta^t \|} \]
  with $\frac{\| M (\theta^t) \|}{\| \theta^t \|} \geq 1$
  \begin{eqnarray*}
    \left| \frac{1}{\| \theta^{\ast} \|} \langle M (\theta^t), \hat{e}_1 \rangle
    \right| & \geq & \left[ \frac{1 + \sqrt{5}}{2} \cdot \sqrt{1 - 0.643^2}
    \cdot \frac{\| M (\theta^t) \|}{\| \theta^t \|} \right] \cdot \frac{1}{\|
    \theta^{\ast} \|} | \langle \theta^t, \hat{e}_1 \rangle |\\
    & \geq & \left[ \frac{1 + \sqrt{5}}{2} \cdot \sqrt{1 - 0.643^2} \right]
    \cdot \frac{1}{\| \theta^{\ast} \|} | \langle \theta^t, \hat{e}_1 \rangle
    |\\
    & \geq & 1.239 \cdot \left| \frac{1}{\| \theta^{\ast} \|} \langle \theta^t,
    \hat{e}_1 \rangle \right|
  \end{eqnarray*}

  \textbf{Stage 2: $\varphi \geq 4 \Theta \left( \sqrt{\frac{d}{n}} \vee \frac{\log
  \frac{1}{\delta}}{n} \vee \sqrt{\frac{\log \frac{1}{\delta}}{n}}
  \right) \Rightarrow \varphi > \arctan 1.5$}
  
  We denote $r^{t + 1} \assign \frac{\| M_n (\theta^t) - M (\theta^t) \|}{\|
  \theta^{\ast} \|} < \Theta \left( \sqrt{\frac{d}{n}} \vee \frac{\log
  \frac{1}{\delta}}{n} \vee \sqrt{\frac{\log \frac{1}{\delta}}{n}}
  \right)$, and $\varphi^t \assign \arcsin \frac{| \langle \theta^t, \hat{e}_1
  \rangle |}{\| \theta^t \|}, \bar{\varphi}^{t + 1} \assign \arcsin
  \frac{\langle M (\theta^t), \hat{e}_1 \rangle}{\| \theta^t \|}$
  and the update rule $\theta^{t + 1} = M_n (\theta^t)$.
  \[ \varphi^{t + 1} \geq \bar{\varphi}^{t + 1} - \arcsin r^{t + 1} \]
  Note that $\arcsin r^{t + 1} \leq \frac{\pi}{2} r^{t + 1}$ and the Lemma~\ref{lem:single_iter} that
  we just proved
  \begin{eqnarray*}
    \tan \bar{\varphi}^{t + 1} & > & \frac{1 + \sqrt{5}}{2} \tan \varphi^t \quad
    \forall \varphi^t > 0\\
    \varphi^{t + 1} & \geq & \bar{\varphi}^{t + 1} - \arcsin r^{t + 1}
  \end{eqnarray*}
  With the assumption $r^{t + 1} \leq \Theta \assign \Theta \left(
  \sqrt{\frac{d}{n}} \vee \frac{\log \frac{1}{\delta}}{n} \vee
  \sqrt{\frac{\log \frac{1}{\delta}}{n}} \right) \leq 0.1$, and $\tan (\arcsin
  \Theta) = \frac{\Theta}{\sqrt{1 - \Theta^2}}$
  
  Note that $1 - \frac{4}{5} \left[ \frac{\Theta}{\sqrt{1 - \Theta^2}} \cdot
  \frac{1 + \sqrt{5}}{2} \tan \varphi^t \right] > \frac{1}{1 +
  \frac{\Theta}{\sqrt{1 - \Theta^2}} \cdot \frac{1 + \sqrt{5}}{2} \tan
  \varphi^t}$ for $\tan \varphi^t \leq 1.5$, since $\frac{\Theta}{\sqrt{1 -
  \Theta^2}} \cdot \frac{1 + \sqrt{5}}{2} \tan \varphi^t \leq \frac{0.1}{\sqrt{1
  - 0.1^2}} \cdot \frac{1 + \sqrt{5}}{2} \cdot 1.5 < \frac{1}{4} = \frac{1 -
  \frac{4}{5}}{\frac{4}{5}}$
  
  If $\tan \varphi^t \geq \frac{1}{5} \Theta$, then $\frac{4 \left( 1 + \sqrt{5}
  \right)}{5} \tan \varphi^t \geq \frac{4 \left( 1 + \sqrt{5} \right)}{5} \cdot
  \frac{1}{5} \Theta \cdot \frac{\sqrt{1 - 0.1^2}}{\sqrt{1 - \Theta^2}} \geq
  \frac{8 \left( 1 + \sqrt{5} \right) \sqrt{1 - 0.1^2}}{25} \cdot \frac{1}{2}
  \left[ \frac{\Theta}{\sqrt{1 - \Theta^2}} \right] > \frac{1}{2} \left[
  \frac{\Theta}{\sqrt{1 - \Theta^2}} \right]$
  
  Then, with $\frac{x}{\sqrt{1 + x^2}} > x - \frac{x^3}{2}$
  \begin{eqnarray*}
    \tan \varphi^{t + 1} & \geq & \tan (\bar{\varphi}^{t + 1} - \arcsin \Theta)
    = \frac{\tan \bar{\varphi}^{t + 1} - \frac{\Theta}{\sqrt{1 - \Theta^2}}}{1 +
    \frac{\Theta}{\sqrt{1 - \Theta^2}} \cdot \tan \bar{\varphi}^{t + 1}}\\
    & = & \left[ \frac{\Theta}{\sqrt{1 - \Theta^2}} \right]^{- 1} -
    \frac{\frac{\Theta}{\sqrt{1 - \Theta^2}} + \left[ \frac{\Theta}{\sqrt{1 -
    \Theta^2}} \right]^{- 1}}{1 + \frac{\Theta}{\sqrt{1 - \Theta^2}} \cdot \tan
    \bar{\varphi}^{t + 1}}\\
    & > & \left[ \frac{\Theta}{\sqrt{1 - \Theta^2}} \right]^{- 1} -
    \frac{\frac{\Theta}{\sqrt{1 - \Theta^2}} + \left[ \frac{\Theta}{\sqrt{1 -
    \Theta^2}} \right]^{- 1}}{1 + \frac{\Theta}{\sqrt{1 - \Theta^2}} \cdot
    \frac{1 + \sqrt{5}}{2} \tan \varphi^t}\\
    & > & \left[ \frac{\Theta}{\sqrt{1 - \Theta^2}} \right]^{- 1} - \left\{
    \frac{\Theta}{\sqrt{1 - \Theta^2}} + \left[ \frac{\Theta}{\sqrt{1 -
    \Theta^2}} \right]^{- 1} \right\} \left\{ 1 - \frac{4}{5} \left[
    \frac{\Theta}{\sqrt{1 - \Theta^2}} \cdot \frac{1 + \sqrt{5}}{2} \tan
    \varphi^t \right] \right\}\\
    & = & - \frac{\Theta}{\sqrt{1 - \Theta^2}} + \frac{4}{5} \cdot \frac{1 +
    \sqrt{5}}{2} \left[ 1 + \left[ \frac{\Theta}{\sqrt{1 - \Theta^2}} \right]^2
    \right] \tan \varphi^t\\
    & \geq & \frac{2 \left( 1 + \sqrt{5} \right)}{5} \tan \varphi^t - \left[
    \frac{\Theta}{\sqrt{1 - \Theta^2}} - \frac{1}{2} \left[
    \frac{\Theta}{\sqrt{1 - \Theta^2}} \right]^3 \right]\\
    & > & \frac{2 \left( 1 + \sqrt{5} \right)}{5} \tan \varphi^t - \Theta\\
    & > & \frac{22}{17} \tan \varphi^t - \Theta
  \end{eqnarray*}
  Thus, we conclude that if $\frac{1}{5} \Theta \leq \tan \varphi^t \leq 1.5$,
  then
  \begin{eqnarray*}
    {}[\tan \varphi^{t + 1} - 3.4 \Theta] & > & \frac{22}{17} [\tan \varphi^t -
    3.4 \Theta]
  \end{eqnarray*}
  We could assume $\varphi^0 \geq 4 \Theta$ then $\tan \varphi^t - 3.4 \Theta >
  \tan \varphi^0 - 3.4 \Theta \geq 0.6 \Theta$, thus
  \begin{eqnarray*}
    {}[\tan \varphi^t - 3.4 \Theta] & > & \left( \frac{22}{17} \right)^t [\tan
    \varphi^0 - 3.4 \Theta] \geq \left( \frac{22}{17} \right)^t \cdot 0.6 \Theta
  \end{eqnarray*}
  Therefore, after runing at most $T_2 = \lceil \frac{\log \frac{1.5}{0.6
  \Theta}}{\log \frac{22}{17}} \rceil = \lceil \frac{\log \frac{1}{\Theta} +
  \log 2.5}{\log \frac{22}{17}} \rceil =\mathcal{O} \left( \log \frac{1}{\Theta}
  \right)$
  \[ \tan \varphi^{T_2} > 3.4 \Theta + \left( \frac{22}{17} \right)^{T_2} \cdot
     0.6 \Theta > 3.4 \Theta + 1.5 > 1.5 \]

  \textbf{Stage 3: $\varphi > \arctan 1.5 \Rightarrow \varphi^t > \frac{\pi}{2} -
  1.775 \Theta \left( \sqrt{\frac{\log \frac{1}{\delta}}{n}} \vee \frac{\log
  \frac{1}{\delta}}{n} \right) $}
  
  Let's start from the following relation
  \[ \tan \bar{\varphi}^{t + 1} = \tan \varphi^t + \varphi^t ([\tan \varphi^t]^2
     + 1) \]
  We denote $\phi \assign 2 \left( \frac{\pi}{2} - \varphi \right) \in (0, \pi)$,
  note that $\frac{2}{\phi} - \frac{\phi}{4.93} < \tan \left( \frac{\pi}{2} -
  \frac{\phi}{2} \right) < \frac{2}{\phi} - \frac{\phi}{6}$
  \begin{eqnarray*}
    \tan \left( \frac{\pi}{2} - \frac{\bar{\phi}^{t + 1}}{2} \right) & = & \tan
    \left( \frac{\pi}{2} - \frac{\phi^t}{2} \right) + \left( \frac{\pi}{2} -
    \frac{\phi^t}{2} \right) \left[ \tan^2  \left( \frac{\pi}{2} -
    \frac{\phi^t}{2} \right) + 1 \right]
  \end{eqnarray*}
  \begin{eqnarray*}
    \frac{2}{\bar{\phi}^{t + 1}} & > & \frac{2}{\bar{\phi}^{t + 1}} -
    \frac{\bar{\phi}^{t + 1}}{6}\\
    & > & \tan \left( \frac{\pi}{2} - \frac{\bar{\phi}^{t + 1}}{2} \right)\\
    & = & \tan \left( \frac{\pi}{2} - \frac{\phi^t}{2} \right) + \left(
    \frac{\pi}{2} - \frac{\phi^t}{2} \right) \left[ \tan^2  \left( \frac{\pi}{2}
    - \frac{\phi^t}{2} \right) + 1 \right]\\
    & > & \frac{2}{\phi^t} - \frac{\phi^t}{4.93} + \left( \frac{\pi}{2} -
    \frac{\phi^t}{2} \right) \left[ \left( \frac{2}{\phi^t} -
    \frac{\phi^t}{4.93} \right)^2 + 1 \right]
  \end{eqnarray*}
  For $0 < \phi^t \leq 1.18805$, i.e. $\varphi^t = \frac{\pi}{2} -
  \frac{\phi^t}{2} \geq 0.976772, \tan (\varphi^t) \geq 1.48061$, then \ $45849
  \pi - 144449 \phi^t + 10000 \pi [\phi^t]^2 - 10000 [\phi^t]^3 \geq 0$
  \begin{eqnarray*}
    \bar{\phi}^{t + 1} & < & \frac{2}{\frac{2}{\phi^t} - \frac{\phi^t}{4.93} +
    \left( \frac{\pi}{2} - \frac{\phi^t}{2} \right) \left[ \left(
    \frac{2}{\phi^t} - \frac{\phi^t}{4.93} \right)^2 + 1 \right]}\\
    & = & \frac{[\phi^t]^2}{\pi + \frac{[\phi^t]^2}{972196} \{ 45849 \pi -
    144449 \phi^t + 10000 \pi [\phi^t]^2 - 10000 [\phi^t]^3 \}}\\
    & \leq & \frac{2}{\frac{2 \pi}{[\phi^t]^2}} = \frac{[\phi^t]^2}{\pi}
  \end{eqnarray*}
  By using the previous Lemma, $\varphi^{t + 1} \geq \bar{\varphi}^{t + 1} - \arcsin
  r^{t + 1}$
  \[ \phi^{t + 1} - 2 \arcsin r^{t + 1} \leq \bar{\phi}^{t + 1} \]
  Hence, with $\arcsin r^{t + 1} \leq \frac{\pi}{2} r^{t + 1} $, for $0 <
  \frac{\phi^t}{\pi} \leq \frac{1.18805}{\pi} \approx 0.378167$
  \begin{eqnarray*}
    \phi^{t + 1} - 2 \arcsin r^{t + 1} \leq \bar{\phi}^{t + 1} & < &
    \frac{[\phi^t]^2}{\pi}\\
    \frac{\phi^{t + 1}}{\pi} < \left[ \frac{\phi^t}{\pi} \right]^2 +
    \frac{2}{\pi} \arcsin r^{t + 1} & \leq & \left[ \frac{\phi^t}{\pi} \right]^2
    + r^{t + 1}
  \end{eqnarray*}
  Suppose $r^{t + 1} \leq \Theta \assign \Theta \left( \sqrt{\frac{d}{n}} \vee
  \frac{\log \frac{1}{\delta}}{n} \vee \sqrt{\frac{\log
  \frac{1}{\delta}}{n}} \right) \leq 0.1$, then
  \[ \frac{\phi^{t + 1}}{\pi} < \left[ \frac{\phi^t}{\pi} \right]^2 + \Theta \]
  Suppose $0 < \frac{\phi^t}{\pi} \leq \frac{1.18805}{\pi} \approx 0.378167$,
  with $\Theta < 0.1$ then in one iteration
  \begin{eqnarray*}
    \frac{\phi^{t + 1}}{\pi} & < & 0.3782^2 + 0.1 < \frac{1}{4}
  \end{eqnarray*}
  Thus, we could assume $0 < \frac{\phi^t}{\pi} \leq \frac{\phi^0}{\pi} <
  \frac{1}{4}$, then
  \begin{eqnarray*}
    \left[ \frac{\phi^{t + 1}}{\pi} - 2 \Theta \right] & < & \left[
    \frac{\phi^t}{\pi} \right]^2 - \Theta\\
    & = & \left[ \frac{\phi^t}{\pi} - 2 \Theta \right]^2 - \Theta \left( \left[
    1 - 4 \frac{\phi^t}{\pi} \right] + 4 \Theta \right)\\
    & < & \left[ \frac{\phi^t}{\pi} - 2 \Theta \right]^2
  \end{eqnarray*}
  If $\frac{\phi^0}{\pi} < 2 \Theta$, then $\frac{\phi^0}{\pi} < 3 \Theta$;
  otherwise, with at most $T_3 = \lceil \frac{\log \log \frac{1}{\Theta} - \log
  \log 4}{\log 2} \rceil =\mathcal{O} \left( \log \log \frac{1}{\Theta} \right)$
  iterations, $\frac{\phi^{T_3}}{\pi} < 3 \Theta$
  \begin{eqnarray*}
    \left[ \frac{\phi^{T_3}}{\pi} - 2 \Theta \right] & \leq & \left[
    \frac{\phi^0}{\pi} - 2 \Theta \right]^{2^{T_3}} < \left[ \frac{1}{4} - 2
    \Theta \right]^{2^{T_3}} < \exp (- \log 4 \cdot 2^{T_3})\\
    & = & \frac{1}{\exp \left( \log 4 \cdot 2^{\lceil \frac{\log \log
    \frac{1}{\Theta} - \log \log 4}{\log 2} \rceil} \right)} \leq \frac{1}{\exp
    \left( \log 4 \cdot 2^{\frac{\log \log \frac{1}{\Theta} - \log \log 4}{\log
    2}} \right)}\\
    & = & \Theta
  \end{eqnarray*}
  After at most $T_3'$ interations, $0 < \frac{\phi^t}{\pi} < 3 \Theta$ is
  satisfied, then run three more iteration, with $\Theta < 0.1$
  \begin{eqnarray*}
    \frac{\phi^{t + 1}}{\pi} & < & \left[ \frac{\phi^t}{\pi} \right]^2 + \Theta
    < (9 \Theta + 1) \Theta < 1.9 \Theta\\
    \frac{\phi^{t + 2}}{\pi} & < & \left[ \frac{\phi^{t + 1}}{\pi} \right]^2 +
    \Theta < (1.9^2 \Theta + 1) \Theta < 1.361 \Theta\\
    \frac{\phi^{t + 3}}{\pi} & < & \left[ \frac{\phi^{t + 2}}{\pi} \right]^2 +
    \Theta < (1.361^2 \Theta + 1) \Theta < 1.13 \Theta
  \end{eqnarray*}
  To sum up, after running for at most $T_3 = 1 + T_3' + 3 = T_3 + 4
  =\mathcal{O} \left( \log \log \frac{1}{\Theta} \right)$ iterations, then $0 <
  \phi^t < (1.13 \pi) \Theta$
  
  with $\phi^t = 2 \left( \frac{\pi}{2} - \varphi^t \right)$, we conclude that
  $\varphi^t \in \left( 0, \frac{\pi}{2} \right)$ satisfies
  \[ \varphi^t > \frac{\pi}{2} - \left( 1.13 \frac{\pi}{2} \right) \Theta >
     \frac{\pi}{2} - 1.775 \Theta \]
  
  \textbf{Number of Iterations, Statistical Error}
  
  We denote $\Theta \assign \Theta \left( \sqrt{\frac{d}{n}} \vee \frac{\log
  \frac{1}{\delta}}{n} \vee \sqrt{\frac{\log \frac{1}{\delta}}{n}}
  \right) < 0.1$.
  
  Note that 
  \begin{eqnarray*}
    \mathcal{O} \left( \log \frac{1}{\Theta} \right) 
    & = & \mathcal{O} \left( \log \frac{1}{\Theta \left( \sqrt{\frac{d}{n}} \vee \frac{\log
    \frac{1}{\delta}}{n} \vee \sqrt{\frac{\log \frac{1}{\delta}}{n}}
    \right)} \right)
     =  \mathcal{O} \left(\log \frac{1}{\sqrt{\frac{d}{n}}} 
    \wedge \log \frac{1}{\frac{\log
    \frac{1}{\delta}}{n}} 
    \wedge \log \frac{1}{\sqrt{\frac{\log \frac{1}{\delta}}{n}}}\right)\\
    & = & \mathcal{O}\left(
    \log \frac{n}{d} 
    \wedge \log \frac{n}{\log \frac{1}{\delta}}
    \right)
  \end{eqnarray*}
  
  In Stage 1, it taks at most $T_1 =\mathcal{O} \left(\log \frac{n}{\log
  \frac{1}{\delta}} \right)$.
  
  In Stage 2, it taks at most $T_2 =\mathcal{O} \left( \log \frac{1}{\Theta} \right) =\mathcal{O}\left(
    \log \frac{n}{d} 
    \wedge \log \frac{n}{\log \frac{1}{\delta}}
    \right)$.
  
  In Stage 3, it taks at most $T_3 =\mathcal{O} \left( \log \log \frac{1}{\Theta} \right)
  =\mathcal{O} \left( \log \left[ \log \frac{n}{d} 
  \wedge \log \frac{n}{\log \frac{1}{\delta}} \right] \right)$.
  
  Hence, the iteration numbers at most to ensure the convergence $\varphi^T \geq \frac{\pi}{2} - \Theta$ with probability at least $1-T \delta$ is 
  \begin{eqnarray*}
    T  =  T_1 + T_2 + T_3 =
    \mathcal{O} \left(\log \frac{n}{\log
  \frac{1}{\delta}} \right)
  \end{eqnarray*}

  For a good initialization $\varphi^0 \geq 4 \Theta$
  \[ T' = T_2 + T_3 =\mathcal{O}\left(
    \log \frac{n}{d} 
    \wedge \log \frac{n}{\log \frac{1}{\delta}}
    \right) \]
  
\end{proof}

%% file: 8_supplementary_4_1_initialization_.tex
We begin with evaluating or estimating some expectations of Gaussian.
%    For the second term $T_2$, note that with Stein's Lemma.
    \begin{eqnarray*}
      \mathbb{E}_{\lambda_{1 i}, \lambda_{2 i} \overset{\tmop{iid}}{\sim}
      \mathcal{N} (0, 1)} \left| \rho \lambda_{1 i} + \sqrt{1 - \rho^2} \lambda_{2
      i} \right| \tmop{sgn} (\lambda_{1 i}) \lambda_{2 i}%\\
      =\left( \frac{\pi}{2} \right)^{- 1} \tmop{sgn} (\rho) \sqrt{1 -
      \rho^2} \left[ \frac{\pi}{2} - \arccos | \rho | \right] \in [- 0.357205,
      0.357205]
    \end{eqnarray*}
    The 2nd moment is $\mathbb{E}_{\lambda_{1 i}, \lambda_{2 i} \overset{\tmop{iid}}{\sim}
      \mathcal{N} (0, 1)} \left| \rho \lambda_{1 i} + \sqrt{1 - \rho^2} \lambda_{2
      i} \right|^2 \tmop{sgn} (\lambda_{1 i})^2 \lambda^2_{2 i}
      =3 - 2 \rho^2$. % calculated as follows.
    % \begin{eqnarray*}
    %   &  & \mathbb{E}_{\lambda_{1 i}, \lambda_{2 i} \overset{\tmop{iid}}{\sim}
    %   \mathcal{N} (0, 1)} \left| \rho \lambda_{1 i} + \sqrt{1 - \rho^2} \lambda_{2
    %   i} \right|^2 \tmop{sgn} (\lambda_{1 i})^2 \lambda^2_{2 i}\\
    %   & = & \mathbb{E}_{\lambda_{1 i}, \lambda_{2 i} \overset{\tmop{iid}}{\sim}
    %   \mathcal{N} (0, 1)} \left( \rho \lambda_{1 i} + \sqrt{1 - \rho^2} \lambda_{2
    %   i} \right)^2 \lambda^2_{2 i}\\
    %   & = & \rho^2 \mathbb{E} [\lambda_{1 i}^2] \mathbb{E} [\lambda_{2 i}^2] + (1
    %   - \rho^2) \mathbb{E} [\lambda^4_{2 i}]\\
    %   & = & \rho^2 + 3 (1 - \rho^2)\\
    %   & = & 3 - 2 \rho^2
    % \end{eqnarray*}
    The 3rd moment is bounded by
    \begin{eqnarray*}
      & &\mathbb{E}_{\lambda_{1 i}, \lambda_{2 i} \overset{\tmop{iid}}{\sim}
      \mathcal{N} (0, 1)} \left\{ \left| \left| \rho \lambda_{1 i} + \sqrt{1 -
      \rho^2} \lambda_{2 i} \right| \tmop{sgn} (\lambda_{1 i}) \lambda_{2 i}
      \right|^3 \right\}%\\
      % & = & \mathbb{E}_{\lambda_{1 i}, \lambda_{2 i} \overset{\tmop{iid}}{\sim}
      % \mathcal{N} (0, 1)} \left| \rho \lambda_{1 i} + \sqrt{1 - \rho^2} \lambda_{2
      % i} \right|^3 | \lambda_{2 i} |^3\\
      \leq \left[ \mathbb{E}_{\lambda_{1 i}, \lambda_{2 i}
      \overset{\tmop{iid}}{\sim} \mathcal{N} (0, 1)} \left| \rho \lambda_{1 i} +
      \sqrt{1 - \rho^2} \lambda_{2 i} \right|^4 | \lambda_{2 i} |^4
      \right]^{\frac{3}{4}}\\
      % & = & [\rho^4 \mathbb{E} [\lambda_{1 i}^4] \mathbb{E} [\lambda_{2 i}^4] + 6
      % \rho^2 (1 - \rho^2) \mathbb{E} [\lambda_{1 i}^2] \mathbb{E} [\lambda_{2
      % i}^6] + (1 - \rho^2)^2 \mathbb{E} [\lambda_{2 i}^8]]^{\frac{3}{4}}\\
      % & = & [3^2 \rho^4 + 6 \cdot 15 \rho^2 (1 - \rho^2) + 105 (1 -
      % \rho^2)^2]^{\frac{3}{4}}\\
      &=& [105 - 120 \rho^2 + 24 \rho^4]^{\frac{3}{4}}%\\
      \leq 105^{\frac{3}{4}}
    \end{eqnarray*}
    Let $X \assign \left| \rho \lambda_{1 i} + \sqrt{1 - \rho^2} \lambda_{2 i}
    \right| \tmop{sgn} (\lambda_{1 i}) \lambda_{2 i}$, thus $3 \geqslant
    \tmop{Var} [X] \geq 3 - 2 \rho^2 - 0.357205^2 \geq 2.8724 - 2 \rho^2$
    
    By Minkovski inequality, $\mathbb{E}
    [| X -\mathbb{E} [X] |^3] \leq (\mathbb{E} [| X
    |^3]^{\frac{1}{3}} + | \mathbb{E} [X] |)^3\leq \left( \left[ 105^{\frac{3}{4}} \right]^{\frac{1}{3}} + 0.357205
      \right)^3 \leq 45.054$.

    Then, let's decompose the statistical error into three terms.
    \begin{eqnarray*}
      &  & \theta^{t + 1} - \hat{\theta}^{t + 1}%\\
      = M^{\tmop{easy}}_n (\theta^t) - M (\theta^t)\\
      % & = & \left\{ \frac{1}{n'}  \sum_{i \in [n']} -\mathbb{E}_{\{ x_i \}_{i \in
      % [n']} \overset{\tmop{iid}}{\sim} \mathcal{N} (0, I_d)} \right\} | \langle
      % x_i, \theta^{\ast} \rangle | \tmop{sgn} \langle x_i, \theta^t \rangle x_i\\
      % & = & \frac{1}{n'}  \sum_{i \in [n']} | \langle \tilde{x}_i, \theta^{\ast}
      % \rangle | \tmop{sgn} \langle \tilde{x}_i, \theta^t \rangle x_i^{\perp} +
      % \left\{ \frac{1}{n'}  \sum_{i \in [n']} -\mathbb{E}_{\{ x_i \}_{i \in [n']}
      % \overset{\tmop{iid}}{\sim} \mathcal{N} (0, I_d)} \right\} | \langle
      % \tilde{x}_i, \theta^{\ast} \rangle | \tmop{sgn} \langle \tilde{x}_i,
      % \theta^t \rangle \tilde{x}_i\\
      % & = & \frac{1}{n'}  \sum_{i \in [n']} | \langle \tilde{x}_i, \theta^{\ast}
      % \rangle | \tmop{sgn} \langle \tilde{x}_i, \theta^t \rangle x_i^{\perp} + \|
      % \theta^{\ast} \| \left\{ \frac{1}{n'}  \sum_{i \in [n']}
      % -\mathbb{E}_{\lambda_{1 i}, \lambda_{2 i} \overset{\tmop{iid}}{\sim}
      % \mathcal{N} (0, 1)} \right\} \left| \rho \lambda_{1 i} + \sqrt{1 - \rho^2}
      % \lambda_{2 i} \right| \tmop{sgn} (\lambda_{1 i}) (\lambda_{1 i} \vec{e}_1 +
      % \lambda_{2 i} \vec{e}_2)\\
      & = & \frac{1}{n'}  \sum_{i \in [n']} | \langle \tilde{x}_i, \theta^{\ast}
      \rangle | \tmop{sgn} \langle \tilde{x}_i, \theta^t \rangle x_i^{\perp} + \|
      \theta^{\ast} \| \underset{T_2}{\underbrace{\left[  \left\{ \frac{1}{n'} 
      \sum_{i \in [n']} -\mathbb{E}_{\lambda_{1 i}, \lambda_{2 i}
      \overset{\tmop{iid}}{\sim} \mathcal{N} (0, 1)} \right\} \left| \rho
      \lambda_{1 i} + \sqrt{1 - \rho^2} \lambda_{2 i} \right| \tmop{sgn}
      (\lambda_{1 i}) \lambda_{2 i} \right]}} \vec{e}_2\\
      &  & + \| \theta^{\ast} \| \underset{T_1}{\underbrace{\left[ \left\{
      \frac{1}{n'}  \sum_{i \in [n']} -\mathbb{E}_{\lambda_{1 i}, \lambda_{2 i}
      \overset{\tmop{iid}}{\sim} \mathcal{N} (0, 1)} \right\} \left| \rho
      \lambda_{1 i} + \sqrt{1 - \rho^2} \lambda_{2 i} \right| | \lambda_{1 i} |
      \right]}} \vec{e}_1
    \end{eqnarray*}
    Consider such events $\mathcal{E}_1 \assign \left\{ | T_1 | \leq \frac{2}{\pi} \right\},
    \mathcal{E}_2 \assign \left\{ | T_2 | \geq c \sqrt{\frac{1}{n'}} \right\}$.
    Note that the variance of $T_1$
    \begin{eqnarray*}
      \tmop{Var} [T_1] = \frac{1}{n'} \tmop{Var} \left[ \left| \rho
      \lambda_{1 i} + \sqrt{1 - \rho^2} \lambda_{2 i} \right| | \lambda_{1 i} |
      \right]
      \leq \frac{1}{n'} [3 - 1^2] = \frac{2}{n'}.
    \end{eqnarray*}
    Hence, $\mathbb{P} [\mathcal{E}_1] = 1 -\mathbb{P} \left\{ | T_1 | \geq \frac{2}{\pi}
    \right\} \geq 1 - 2 \frac{\tmop{Var} [T_1]}{\left( \frac{2}{\pi} \right)^2}
    \geq 1 - 2 \frac{\frac{2}{n'}}{\left( \frac{2}{\pi} \right)^2} = 1 -
    \frac{\pi^2}{n'}$.
    
    Consider $\mathcal{E}_2$, with {\tmstrong{Berry-Esseen}} bound for the central limit
    theorem, \cite{ross2011fundamentals} theorem 1.1, where $T_2 = \frac{1}{n'} \sum_{i \in [n']}.
    [X_i -\mathbb{E} [X_i]]$
    \begin{eqnarray*}
      \left| \mathbb{P} \left[ \frac{\sqrt{n'} T_2}{\sqrt{\tmop{Var} [X]}} \leq
      \frac{c}{\sqrt{\tmop{Var} [X]}} \right] - \Phi \left(
      \frac{c}{\sqrt{\tmop{Var} [X]}} \right) \right| & \leq &
      \frac{0.4785}{(\tmop{Var} [X])^{\frac{3}{2}}} \mathbb{E} [| X -\mathbb{E}
      [X] |^3]
    \end{eqnarray*}
    Thus, with $3 \geqslant \tmop{Var} [X] \geq 3 - 2 \rho^2 - 0.357205^2 \geq
    2.8724 - 2 \rho^2 = 2.8724 - \sin^2 \varphi^t \geq 2.8724 - \frac{1}{n'}$,
    
    and $\mathbb{E} [| X -\mathbb{E} [X] |^3] \leq 45.054$, $2.8724 - \frac{1}{n'}
    \geq 2.8724 - \frac{1}{3} \geq 2.539$.
    \begin{eqnarray*}
      \mathbb{P} \left[ T_2 \geq c \sqrt{\frac{1}{n'}} \right] & = & 1 -\mathbb{P}
      \left[ \frac{\sqrt{n'} T_2}{\sqrt{\tmop{Var} [X]}} \leq
      \frac{c}{\sqrt{\tmop{Var} [X]}} \right]\\
      & \geq & 1 - \Phi \left( \frac{c}{\sqrt{\tmop{Var} [X]}} \right) -
      \frac{0.4785}{(\tmop{Var} [X])^{\frac{3}{2}}} \frac{\mathbb{E} [| X
      -\mathbb{E} [X] |^3]}{\sqrt{n'}}\\
      & \geq & 1 - \Phi \left( \frac{c}{\sqrt{3}} \right) - \frac{0.4785}{\left(
      2.8724 - \frac{1}{n'} \right)^{\frac{3}{2}}} \cdot \frac{45.054}{\sqrt{n'}}\\
      & \geqslant & 1 - \Phi \left( \frac{c}{\sqrt{3}} \right) - \left(
      \frac{0.4785 \cdot 45.054}{2.539^{\frac{3}{2}}} \right) \frac{1}{\sqrt{n'}}\\
      & \geq & 1 - \Phi \left( \frac{c}{\sqrt{3}} \right) - 5.3287 \cdot
      \frac{1}{\sqrt{n'}}
    \end{eqnarray*}
    \[ \mathbb{P} [\mathcal{E}_2] =\mathbb{P} \left[ T_2 \geq c \sqrt{\frac{1}{n'}} \right]
       +\mathbb{P} \left[ - T_2 \geq c \sqrt{\frac{1}{n'}} \right] \geq 2 \left[ 1
       - \Phi \left( \frac{c}{\sqrt{3}} \right) \right] - 2 \cdot 5.3287 \cdot
       \frac{1}{\sqrt{n'}} \]
    \begin{eqnarray*}
      & &\mathbb{P} [\mathcal{E}_1 \wedge \mathcal{E}_2] = \mathbb{P} [\mathcal{E}_1] +\mathbb{P} [\mathcal{E}_2]
      -\mathbb{P} [\mathcal{E}_1 \vee \mathcal{E}_2]%\\
      \geq \mathbb{P} [\mathcal{E}_1] +\mathbb{P} [\mathcal{E}_2] - 1\\
      &\geq& \left( 1 - \frac{\pi^2}{n'} \right) + \left( 2 \left[ 1 - \Phi
      \left( \frac{c}{\sqrt{3}} \right) \right] - 2 \cdot 5.3287 \cdot
      \frac{1}{\sqrt{n'}} \right) - 1
      \geq 2 \left[ 1 - \Phi \left( \frac{c}{\sqrt{3}} \right) \right] -
      \left[ \frac{\pi^2}{n'} + \frac{11}{\sqrt{n'}} \right]
    \end{eqnarray*}
    After runing Easy-EM for $\mathcal{T}$ times with independent batches (batch size is
    $n'$)
    
    To ensure that the probability of $\overline{\mathcal{E}_1 \wedge \mathcal{E}_2}$ happens for
    $\mathcal{T}$ times is less than $\delta$, and let $\mathcal{T} = \Theta \left( \log
    \frac{1}{\delta} \right)$ and $\left\{ 2 \Phi \left( \frac{c}{\sqrt{3}}
    \right) - 1 + \left[ \frac{\pi^2}{n'} + \frac{11}{\sqrt{n'}} \right] \right\} <
    1$ for large enough $n' > \left\{ \frac{11 + \sqrt{11^2 + 8 \pi^2 \left[ 1 -
    \Phi \left( \frac{c}{\sqrt{3}} \right) \right]}}{4 \left[ 1 - \Phi \left(
    \frac{c}{\sqrt{3}} \right) \right]} \right\}^2$
    \begin{eqnarray*}
      (1 -\mathbb{P} [\mathcal{E}_1 \wedge \mathcal{E}_2])^{\mathcal{T}} & \leq & \left\{ 1 - 2 \left[ 1 -
      \Phi \left( \frac{c}{\sqrt{3}} \right) \right] + \left[ \frac{\pi^2}{n'} +
      \frac{11}{\sqrt{n'}} \right] \right\}^{\mathcal{T}}%\\
      = \left\{ 2 \Phi \left( \frac{c}{\sqrt{3}} \right) - 1 + \left[
      \frac{\pi^2}{n'} + \frac{11}{\sqrt{n'}} \right] \right\}^{\mathcal{T}} \leq \delta
    \end{eqnarray*}
    Otherwise, if $\mathcal{E}_1 \wedge \mathcal{E}_2 = \left\{ | T_1 | \leq \frac{2}{\pi} \right\}
    \wedge \left\{ | T_2 | \geq c \sqrt{\frac{1}{n'}} \right\}$ happens in the $\mathcal{T}$-th
    iteration
    \begin{eqnarray*}
      \theta^{\mathcal{T} + 1} & = & \hat{\theta}^{\mathcal{T} + 1} + \frac{1}{n'}  \sum_{i \in
      [n']} | \langle \tilde{x}_i, \theta^{\ast} \rangle | \tmop{sgn} \langle
      \tilde{x}_i, \theta^{\mathcal{T}} \rangle x_i^{\perp} + \| \theta^{\ast} \| [T_1
      \vec{e}_1 + T_2 \vec{e}_2]\\
      & = & M (\theta^{\mathcal{T}}) + \frac{1}{n'}  \sum_{i \in [n']} | \langle
      \tilde{x}_i, \theta^{\ast} \rangle | \tmop{sgn} \langle \tilde{x}_i,
      \theta^{\mathcal{T}} \rangle x_i^{\perp} + \| \theta^{\ast} \| [T_1 \vec{e}_1 + T_2
      \vec{e}_2]
    \end{eqnarray*}
    Then, since $\left\| \frac{M (\theta^{\mathcal{T}})}{\| \theta^{\ast} \|} \right\| \leq
    1$, with probability at least $1 - \delta$,
    
    \[\frac{1}{\| \theta^{\ast} \|} \left\| \frac{1}{n'}  \sum_{i \in [n']} |
    \langle \tilde{x}_i, \theta^{\ast} \rangle | \tmop{sgn} \langle \tilde{x}_i,
    \theta^{\mathcal{T}} \rangle x_i^{\perp} \right\| \leq 2 \sqrt{\frac{d - 2}{n'}} 
    \left[ 1 + \frac{1}{(d - 2) - 1}  \left( 2 \log \frac{1}{\delta} + 1.62
    \right) \right] =\mathcal{O} \left( \sqrt{\frac{d}{n'}} \vee \frac{\log
    \frac{1}{\delta}}{n} \right)\]
    
    If we assume $2 \sqrt{\frac{d - 2}{n'}}  \left[ 1 + \frac{1}{(d - 2) - 1} 
    \left( 2 \log \frac{1}{\delta} + 1.62 \right) \right] \leq 0.1$ for large $n'$
    
    Thus, we conclude that $n' \geq \frac{8 \cdot 2.62^2}{0.1^2}$
    \begin{eqnarray*}
      \frac{\| \theta^{\mathcal{T} + 1} \|}{\| \theta^{\ast} \|} & = & \sqrt{\left\|
      \frac{M (\theta^{\mathcal{T}})}{\| \theta^{\ast} \|} + T_1 \vec{e}_1 + T_2 \vec{e}_2
      \right\|^2 + \frac{1}{\| \theta^{\ast} \|} \left\| \frac{1}{n'}  \sum_{i \in
      [n']} | \langle \tilde{x}_i, \theta^{\ast} \rangle | \tmop{sgn} \langle
      \tilde{x}_i, \theta^{\mathcal{T}} \rangle x_i^{\perp} \right\|^2}\\
      & \leq & \left\| \frac{M (\theta^{\mathcal{T}})}{\| \theta^{\ast} \|} + T_1
      \vec{e}_1 + T_2 \vec{e}_2 \right\| + 0.1%\\
      \leq \left( 1 + \frac{2}{\pi} + | T_2 | \right) + 0.1%\\
      = \left( 1.1 + \frac{2}{\pi} \right) + | T_2 |
    \end{eqnarray*}
    Using the results, $\left\langle \frac{M (\theta^{\mathcal{T}}) - \theta^{\mathcal{T}}}{\|
    \theta^{\ast} \|}, \vec{e}_2 \right\rangle = \frac{2}{\pi} \varphi^{\mathcal{T}} \cos
    (\varphi^{\mathcal{T}}), \langle \theta^{\mathcal{T}}, \vec{e}_2 \rangle = 0$ and
    $\frac{\pi}{2} \leq \frac{1}{\| \theta^{\ast} \|} \langle \theta^{\mathcal{T}},
    \vec{e}_1 \rangle = \frac{\theta^{\mathcal{T}}}{\| \theta^{\ast} \|} \leq 1$ and $|
    T_1 | \leq \frac{2}{\pi}$
    \begin{eqnarray*}
      \left| \frac{1}{\| \theta^{\ast} \|} \langle \theta^{\mathcal{T} + 1}, \vec{e}_1
      \rangle \right| & = & \left| \left\langle \frac{M (\theta^{\mathcal{T}}) -
      \theta^{\mathcal{T}}}{\| \theta^{\ast} \|}, \vec{e}_1 \right\rangle + \frac{1}{\|
      \theta^{\ast} \|} \langle \theta^{\mathcal{T}}, \vec{e}_1 \rangle + T_1 \right|%\\
      \leq \left| T_1 + \frac{1}{\| \theta^{\ast} \|} \langle \theta^{\mathcal{T}},
      \vec{e}_1 \rangle \right| + \left| \left\langle \frac{M (\theta^{\mathcal{T}}) -
      \theta^{\mathcal{T}}}{\| \theta^{\ast} \|}, \vec{e}_1 \right\rangle \right|\\
      & \leq & \left( \frac{2}{\pi} + 1 \right) + \frac{\| M (\theta^{\mathcal{T}}) -
      \theta^{\mathcal{T}} \|}{\| \theta^{\ast} \|}%\\
      \leq \left( \frac{2}{\pi} + 1 \right) + \left( 1 - \frac{2}{\pi}
      \right)%\\
      = 2
    \end{eqnarray*}
    \begin{eqnarray*}
      \left| \frac{1}{\| \theta^{\ast} \|} \langle \theta^{\mathcal{T} + 1}, \vec{e}_2
      \rangle \right| & = & \left| \left\langle \frac{M (\theta^{\mathcal{T}}) -
      \theta^{\mathcal{T}}}{\| \theta^{\ast} \|}, \vec{e}_2 \right\rangle + \frac{1}{\|
      \theta^{\ast} \|} \langle \theta^{\mathcal{T}}, \vec{e}_2 \rangle + T_2 \right|%\\
      \geq | T_2 | - \left\langle \frac{M (\theta^{\mathcal{T}}) - \theta^{\mathcal{T}}}{\|
      \theta^{\ast} \|}, \vec{e}_2 \right\rangle\\
      & \geq & | T_2 | - \frac{2}{\pi} \varphi^{\mathcal{T}} \cos (\varphi^{\mathcal{T}})%\\
      \geq | T_2 | - \frac{2}{\pi} \sqrt{\frac{1}{n'}}
    \end{eqnarray*}
    Then, with $n' \geq \frac{8 \cdot 2.62^2}{0.1^2}$, we conclude that
    $\sqrt{\frac{1}{n'}} \leq \frac{0.1}{2 \sqrt{2} \cdot 2.62} \leq 0.0135,
    \sqrt{1 - \frac{1}{n'}} \geq \sqrt{1 - \frac{0.1^2}{8 \cdot 2.62^2}} \geq 1 -
    0.1^4$
    \begin{eqnarray*}
      & &\left| \frac{1}{\| \theta^{\ast} \|^2} \langle \theta^{\mathcal{T} + 1},
      \theta^{\ast} \rangle \right| 
      =\left| \frac{1}{\| \theta^{\ast} \|}
      \langle \theta^{\mathcal{T} + 1}, \hat{e}_1 \rangle \right|%\\
      = \left| \sin \varphi^t \frac{1}{\| \theta^{\ast} \|} \langle
      \theta^{\mathcal{T} + 1}, \vec{e}_1 \rangle + \cos \varphi^t \frac{1}{\|
      \theta^{\ast} \|} \langle \theta^{\mathcal{T} + 1}, \vec{e}_2 \rangle \right|\\
      & \geq & \cos \varphi^t \left| \frac{1}{\| \theta^{\ast} \|} \langle
      \theta^{\mathcal{T} + 1}, \vec{e}_2 \rangle \right| - \sin \varphi^t \left|
      \frac{1}{\| \theta^{\ast} \|} \langle \theta^{\mathcal{T} + 1}, \vec{e}_1 \rangle
      \right|%\\
      \geq \sqrt{1 - \frac{1}{n'}} \cdot \left[ | T_2 | - \frac{2}{\pi}
      \sqrt{\frac{1}{n'}} \right] - \sqrt{\frac{1}{n'}} \cdot 2\\
      & \geq & \left[ (1 - 0.1^4) \left( \frac{| T_2 |}{\sqrt{\frac{1}{n'}}} -
      \frac{2}{\pi} \right) - 2 \right] \sqrt{\frac{1}{n'}}
    \end{eqnarray*}
    and $| \rho | = \sin \varphi^{\mathcal{T}} \leq \varphi^{\mathcal{T}} < \sqrt{\frac{1}{n'}}$,
    with $\mathcal{E}_1 \wedge \mathcal{E}_2 = \left\{ | T_1 | \leq \frac{2}{\pi} \right\} \wedge
    \left\{ | T_2 | \geq c \sqrt{\frac{1}{n'}} \right\}$.
    
    By solving $\frac{\left[ (1 - 0.1^4) \left( c - \frac{2}{\pi} \right) - 2
    \right]}{\left( 1.1 + \frac{2}{\pi} \right) + 0.0135 c} \geq 1$, we obtain $c
    \geq 4.43347$, then
    \begin{eqnarray*}
      & &\varphi^{\mathcal{T} + 1} \geq \sin \varphi^{\mathcal{T} + 1} = \frac{| \langle
      \theta^{\mathcal{T} + 1}, \theta^{\ast} \rangle |}{\| \theta^{\mathcal{T} + 1} \| \cdot \|
      \theta^{\ast} \|}%\\
      = \frac{\left| \frac{1}{\| \theta^{\ast} \|^2} \langle \theta^{\mathcal{T} +
      1}, \theta^{\ast} \rangle \right|}{\frac{\| \theta^{\mathcal{T} + 1} \|}{\|
      \theta^{\ast} \|}}%\\
      \geq \frac{\left[ (1 - 0.1^4) \left( \frac{| T_2 |}{\sqrt{\frac{1}{n'}}}
      - \frac{2}{\pi} \right) - 2 \right] \sqrt{\frac{1}{n'}}}{\left( 1.1 +
      \frac{2}{\pi} \right) + | T_2 |}\\
      & \geq & \frac{\left[ (1 - 0.1^4) \left( c - \frac{2}{\pi} \right) - 2
      \right]}{\left( 1.1 + \frac{2}{\pi} \right) + c \sqrt{\frac{1}{n'}}}
      \sqrt{\frac{1}{n'}}
      \geq \frac{\left[ (1 - 0.1^4) \left( c - \frac{2}{\pi} \right) - 2
      \right]}{\left( 1.1 + \frac{2}{\pi} \right) + 0.0135 c} \sqrt{\frac{1}{n'}}%\\
      \geq \sqrt{\frac{1}{n'}}
    \end{eqnarray*}
    By choosing $c = 4.43347$, for large $n' \geq 1.103582 \times 10^6$, we have
    $(1 - 0.01048) + \left[ \frac{\pi^2}{n'} + \frac{11}{\sqrt{n'}} \right] < 1$
    \begin{eqnarray*}
      \mathbb{P} [\mathcal{E}_1 \wedge \mathcal{E}_2] & \geq & 2 \left[ 1 - \Phi \left(
      \frac{c}{\sqrt{3}} \right) \right] - \left[ \frac{\pi^2}{n'} +
      \frac{11}{\sqrt{n'}} \right]%\\
      = 0.01048 - \left[ \frac{\pi^2}{n'} + \frac{11}{\sqrt{n'}} \right]\\
      1 -\mathbb{P} [\mathcal{E}_1 \wedge \mathcal{E}_2] & \leq & (1 - 0.01048) + \left[
      \frac{\pi^2}{n'} + \frac{11}{\sqrt{n'}} \right] < 1
    \end{eqnarray*}
    With at most $\mathcal{T} = \frac{\log \frac{1}{\delta}}{- \log (1 -\mathbb{P} [\mathcal{E}_1
    \wedge \mathcal{E}_2])} =\mathcal{O} \left( \log \frac{1}{\delta} \right)$ iterations,
    we can ensure $\varphi^{T_0} > \sqrt{\frac{1}{n'}}$ for some $T_0 \in
    [\mathcal{T}+1]$.
    Hence, $T_0 =\mathcal{O}(\mathcal{T}) = \mathcal{O}(\log \frac{1}{\delta})$ 
    and by choosing $n'=\Theta(\frac{n}{\log \frac{1}{\delta}} \wedge (\frac{n}{\log \frac{1}{\delta}})^2)$
    \begin{equation*}
      \varphi^{T_0} > \sqrt{\frac{1}{n'}} = \Theta\left(\sqrt{\frac{\log \frac{1}{\delta}}{n}} \vee \frac{\log \frac{1}{\delta}}{n}\right)
    \end{equation*}
    The proof is complete.

%% file: 8_supplementary_4_2.tex
\subsection{\texorpdfstring{Error of Mixing Weights $\pi$ and Convergence at the Finite-Sample Level}{Error of Mixing Weights pi and Convergence at the Finite-Sample Level}}
\begin{lemma}\label{lem:prob_inequality}
Let $q \assign \max (p, 1 - p)$ for $V_i \overset{\tmop{iid}}{\sim}
  \tmop{Bern} (p), \forall i \in [n]$, then

for $t\in \mathbb{R}_{\geq 0}$
    \[ \mathbb{P} \left( \frac{1}{n} \sum_{i \in [n]} (V_i -\mathbb{E} [V_i])
          \geq t \right) \leq \exp(-2n t^2) \]
% for $t\in[0, \frac{3}{2}(1-q))$
%     \[ \mathbb{P} \left( \frac{1}{n} \sum_{i \in [n]} (V_i -\mathbb{E} [V_i])
%        \geq t \right) \leq \exp \left\{ - \frac{n t^2}{2 q (1 - q) \left( 1 -
%        \frac{t}{3 (1 - q)} \right)} \right\} \]
for $t \in [\mathe (1 - q), q)\neq \varnothing$
    \[ \mathbb{P} \left( \frac{1}{n} \sum_{i \in [n]} (V_i -\mathbb{E} [V_i])
       \geq t \right) \leq \exp \left\{ - n \left\{ \frac{t}{q}  \left[ \log
       \frac{t}{(1 - q)} - 1 \right] + \frac{t^2}{2 q^2} \right\} \right\} \]
for $t \in[q, \infty)$
    \[ \mathbb{P} \left( \frac{1}{n} \sum_{i \in [n]} (V_i -\mathbb{E} [V_i])
          \geq t \right) = 0 \]

\end{lemma}
\begin{proof}

%{\color{red} lemma for the concentration inequality for Bernoulli r.v.}
\input{8_supplementary_4_2_Bernoulli_}
\end{proof}

% for the second term of the error $| N_n (\theta^t, \nu^t) - N (\theta^t,
% \nu^t) | = 2 \left| \frac{1}{n} \sum_{i \in [n]} (V_i -\mathbb{E} [V_i])
% \right|$, we could obtain tighter upper bounds

% hence
% \[ \mathbb{P} \left( \frac{1}{n} \sum_{i \in [n]} (V_i -\mathbb{E} [V_i])
%    \geq \frac{\sqrt{q (1 - q)}}{2 \sqrt{2}} \cdot \sqrt{\frac{\log
%    \frac{1}{\delta}}{n}} \right) \leq \delta \]

\begin{theorem}{(Theorem~\ref{thm:convg_finite} in Section~\ref{sec:finite}: Convergence at the Finite-Sample Level)}
In the noiseless setting, suppose any initial mixing weights $\pi^0$ and any initial regression parameters $\theta^0 \in \mathbb{R}^d$ ensuring that $\varphi^0 \geq \Theta \left( \sqrt{\frac{\log
  \frac{1}{\delta}}{n}} \vee \frac{\log \frac{1}{\delta}}{n} \right)$. If we
  run finite-sample Easy EM for at most $T_1=\mathcal{O}\left( \log
  \frac{n}{\log \frac{1}{\delta}}\right)$ iterations followed by the finite-sample standard EM for at most $T' =\mathcal{O} \left(
  \log \frac{n}{d} 
    \wedge \log \frac{n}{\log \frac{1}{\delta}} \right)$
  iterations with all the same $n =
  \Omega \left( d \vee \log \frac{1}{\delta} 
  %\vee \frac{\log^2\frac{1}{\delta}}{d} 
  \right)$ samples, then we have
  \begin{equation*}
      \begin{aligned}
         \frac{\| \theta^{T + 1} - \mathrm{sgn}(\rho^{T+1}) \theta^{\ast} \|}{\| \theta^{\ast} \|} 
         &=\mathcal{O}
    \left( \sqrt{\frac{d}{n}} \vee \frac{\log \frac{1}{\delta}}{n}
    \vee \sqrt{\frac{\log \frac{1}{\delta}}{n}} \right),\\
    \| \pi^{T + 1} - \bar{\pi}^{\ast} \|_1 
    &= \left\| \frac{1}{2} - \pi^{\ast}
    \right\|_1 
    \cdot \mathcal{O} \left(  \sqrt{\frac{d}{n}} \vee \frac{\log
    \frac{1}{\delta}}{n} \vee \sqrt{\frac{\log \frac{1}{\delta}}{n}}
    \right) \\
    &+ c (\pi^{\ast}) \cdot \mathcal{O} \left( \sqrt{\frac{\log
    \frac{1}{\delta}}{n}} \right),
      \end{aligned}
  \end{equation*}
  % \begin{equation}
  %   \frac{\| \theta^{T + 1} - \mathrm{sgn}(\rho^{T+1}) \theta^{\ast} \|}{\| \theta^{\ast} \|} =\mathcal{O}
  %   \left( \sqrt{\frac{d}{n}} \vee \frac{\log \frac{1}{\delta}}{n}
  %   \vee \sqrt{\frac{\log \frac{1}{\delta}}{n}} \right)
  % \end{equation}
  % and
  % \begin{eqnarray}
  %   \| \pi^{T + 1} - \bar{\pi}^{\ast} \|_1 &=& \left\| \frac{1}{2} - \pi^{\ast}
  %   \right\|_1 \nonumber\\
  %   &\cdot& \mathcal{O} \left(  \sqrt{\frac{d}{n}} \vee \frac{\log
  %   \frac{1}{\delta}}{n} \vee \sqrt{\frac{\log \frac{1}{\delta}}{n}}
  %   \right) \nonumber\\
  %   &+& c (\pi^{\ast}) \cdot \mathcal{O} \left( \sqrt{\frac{\log
  %   \frac{1}{\delta}}{n}} \right)
  % \end{eqnarray}
  with probability at least $1 - T\delta$, where $T:=T_1+T',\varphi^0 \assign
  \frac{\pi}{2} - \arccos \left| \frac{\langle \theta^0, \theta^{\ast}
  \rangle}{\| \theta^0 \| \cdot \| \theta^{\ast} \|} \right|, 
  \rho^{T+1}\assign \frac{\langle \theta^{T+1}, \theta^{\ast}
  \rangle}{\| \theta^{T+1} \| \cdot \| \theta^{\ast} \|},
  \bar{\pi}^{\ast} \assign \frac{1}{2} - \tmop{sgn} (\rho^0) 
(\frac{1}{2} - \pi^{\ast})$, and the
  coefficient \ $c (\pi^{\ast}) =\mathcal{O} (1)$, especially $c (\pi^{\ast})
  = 0$ when $\pi^{\ast} = \{1, 0\}$ or $\{0, 1\}$.
\end{theorem}
\begin{proof}

In the proof of Proposition~\ref{prop:convg_angle} for the converge of angle $\varphi$, we show that EM upate rules ensure $\varphi^T >
\frac{\pi}{2} - 1.775 \Theta$ after enough $T$ iterations, where $\Theta \assign\Theta \left( \sqrt{\frac{d}{n}} \vee \frac{\log
\frac{1}{\delta}}{n} \vee \sqrt{\frac{\log \frac{1}{\delta}}{n}}
\right)$ is the threshold for $\varphi$.

Using the relation $\phi^T \assign 2 \left( \frac{\pi}{2} - \varphi^T \right) \in (0,
\pi)$, namely
\[ \phi^T < 3.55 \Theta \]

In the following proof, we use $\bar{\nu}^{T+1}$ for the EM update at the population level.
\newpage
\textbf{Final Statistical Error in Regression Parameters $\theta$}

Note that $\sqrt{(\phi - \sin \phi)^2 + (1 - \cos \phi)^2} \leq \frac{\phi^2}{2}$
for $\forall \phi \in \left[ 0, \frac{\pi}{2} \right]$, and use tha assumption $\Theta <0.1$ in Proposition~\ref{prop:convg_angle}, the upper bound for
the relative error is
\begin{eqnarray*}
  \frac{\left\| M_n (\theta^T) - \tmop{sgn} \langle \theta^T, \theta^{\ast}
  \rangle {\theta^{\ast}}^{} \right\|}{\| \theta^{\ast} \|} & \leq &
  \frac{\left\| M (\theta^T) - \tmop{sgn} \langle \theta^T, \theta^{\ast}
  \rangle {\theta^{\ast}}^{} \right\|}{\| \theta^{\ast} \|} + \frac{\| M_n
  (\theta^T) - M (\theta^T) \|}{\| \theta^{\ast} \|}\\
  & \leq & \frac{1}{\pi} \sqrt{(\phi^T - \sin \phi^T)^2 + (1 - \cos
  \phi^T)^2} + \Theta\\
  & \leq & \frac{[\phi^T]^2}{2 \pi} + \Theta\\
  & < & \frac{3.55^2}{2 \pi} \Theta^2 + \Theta\\
  & < & \left( \frac{3.55^2}{2 \pi} \cdot 0.1 + 1 \right) \Theta\\
  & < & 1.21 \Theta
\end{eqnarray*}
Hence, $\frac{\| \theta^{T + 1} - \mathrm{sgn}(\rho^{T+1}) \theta^{\ast} \|}{\| \theta^{\ast} \|} 
=\frac{\left\| M_n (\theta^T) - \tmop{sgn} \langle \theta^T,
\theta^{\ast} \rangle {\theta^{\ast}}^{} \right\|}{\| \theta^{\ast} \|}
=\mathcal{O} \left( \sqrt{\frac{d}{n}} \vee \frac{\log
\frac{1}{\delta}}{n} \vee \sqrt{\frac{\log \frac{1}{\delta}}{n}}
\right)$

\textbf{Final Statistical Error in Mixing Weights $\pi$}

By using the Corollary~\ref{cor:noiseless} in Section~\ref{sec:updates}, and note that $\tanh (\bar{\nu}^{T + 1})\assign N(\theta^T, \nu^T)$.
\begin{eqnarray*}
  \tanh (\bar{\nu}^{T + 1}) = N(\theta^T, \nu^T)
  =\tanh (\nu^{\ast}) \cdot \tmop{sgn} \langle \theta^T, \theta^{\ast}
  \rangle \left[ \frac{2}{\pi} \varphi^T \right]
\end{eqnarray*}
Note that $\mathbb{E} \left[ N_n(\theta^T, \nu^T) \right] = N(\theta^T, \nu^T) = \tanh
(\bar{\nu}^{T + 1})$, in the noiseless setting, Lemma in Appendix~\ref{sup:updates} gives that.
\begin{eqnarray*}
  \tanh (\nu^{t + 1})  =  N_n (\theta^T, \nu^T) = \frac{1}{n}  \sum_{i \in [n]} \mathrm{sgn} \langle x_i,
  \theta^{\ast} \rangle \mathrm{sgn} \langle x_i, \theta^T \rangle \cdot (-
  1)^{z_i + 1}
\end{eqnarray*}
Note that $x_i, z_i$ are independent, let $W_i \assign \mathrm{sgn} \langle x_i,
\theta^{\ast} \rangle \mathrm{sgn} \langle x_i, \theta^T \rangle \cdot (-
1)^{z_i + 1}$.

with $\mathbb{P} (z_i = 1) = \frac{1}{2} + \frac{1}{2} \tanh (\nu^{\ast}),
\mathbb{P} (z_i = 2) = \frac{1}{2} - \frac{1}{2} \tanh (\nu^{\ast})$

Using Lemma~\ref{lem:Grothendieck} (Grothendieck's Identity), we show that
\begin{eqnarray*}
  \mathbb{P} [\mathrm{sgn} \langle x_i, \theta^{\ast} \rangle \mathrm{sgn}
\langle x_i, \theta^T \rangle = + 1] 
&= &\left[ \frac{1}{2} + \tmop{sgn} \langle
\theta^T, \theta^{\ast} \rangle \frac{\varphi^T}{\pi} \right]\\
\mathbb{P}
[\mathrm{sgn} \langle x_i, \theta^{\ast} \rangle \mathrm{sgn} \langle x_i,
\theta^T \rangle = - 1] 
&=& \left[ \frac{1}{2} - \tmop{sgn} \langle \theta^T,
\theta^{\ast} \rangle \frac{\varphi^T}{\pi} \right]
\end{eqnarray*}
Therefore, we obtain that
\begin{eqnarray*}
  \mathbb{P} [W_i = + 1] & = & \mathbb{P} [\mathrm{sgn} \langle x_i,
  \theta^{\ast} \rangle \mathrm{sgn} \langle x_i, \theta^T \rangle = + 1]
  \cdot \mathbb{P} (z_i = 1) +\mathbb{P} [\mathrm{sgn} \langle x_i,
  \theta^{\ast} \rangle \mathrm{sgn} \langle x_i, \theta^T \rangle = - 1]
  \cdot \mathbb{P} (z_i = 2)\\
  & = & \frac{1}{2} + \tmop{sgn} \langle \theta^T, \theta^{\ast} \rangle
  \frac{\varphi^T}{\pi} \cdot \tanh (\nu^{\ast})\\
  \mathbb{P} [W_i = - 1] & = & \frac{1}{2} - \tmop{sgn} \langle \theta^T,
  \theta^{\ast} \rangle \frac{\varphi^T}{\pi} \cdot \tanh (\nu^{\ast})
\end{eqnarray*}
Let $V_i \assign \frac{1}{2} (W_i + 1) \overset{\tmop{iid}}{\sim}
\tmop{Bern} (p)$ be Bernoulli distribution with the parameter $p
\assign \frac{1}{2} + \tmop{sgn} \langle \theta^T, \theta^{\ast} \rangle
\frac{\varphi^T}{\pi} \cdot \tanh (\nu^{\ast})$.

Therefore, with $\frac{1}{2} N_n (\theta^T, \nu^T) + \frac{1}{2} = \frac{1}{n}
\sum_{i \in [n]} V_i$ and $\frac{1}{2} N (\theta^T, \nu^T) + \frac{1}{2}
=\mathbb{E} [V_i] = \frac{1}{n} \sum_{i \in [n]} \mathbb{E} [V_i]$.

By using the probability inequalities in Lemma~\ref{lem:prob_inequality}
\begin{eqnarray*}
  \mathbb{P} [| N_n (\theta^T, \nu^T) - N (\theta^T, \nu^T) | \geq 2 t] & = &
  \mathbb{P} \left[ \left| \frac{1}{n} \sum_{i \in [n]} (V_i -\mathbb{E}
  [V_i]) \right| \geq t \right]\\
  & \leq & 2 \exp (- 2 n t^2)
\end{eqnarray*}
Consequently, by letting $\exp (- 2 n t^2) \leftarrow \delta$, with
probability at least $1 - 2 \delta$
\[ | N_n (\theta^T, \nu^T) - N (\theta^T, \nu^T) | < \sqrt{\frac{2 \log
   \frac{1}{\delta}}{n}} \]
Furthermore, when $\pi^\ast=[1, 0]$ or $[0, 1]$, namely $\nu^\ast\to \pm \infty$, 
then $q\assign \max(p, 1-p)\to \frac{1}{2} +
\frac{\varphi^T}{\pi}$.
Since after enough iterations, EM updates ensure that $\varphi^T > \frac{\pi}{2} - 1.775$.
Hence, $q\assign\frac{1}{2} +\frac{\varphi^T}{\pi} >  1- 0.565 \Theta, 1-q <  0.565 \Theta$.

Using Lemma~\ref{lem:prob_inequality}, we show that for $t \in [\mathe (1 - q), q)\neq \varnothing$
\[ \mathbb{P} \left( \frac{1}{n} \sum_{i \in [n]} (V_i -\mathbb{E} [V_i])
   \geq t \right) \leq \exp \left\{ - n \left\{ \frac{t}{q}  \left[ \log
   \frac{t}{(1 - q)} - 1 \right] + \frac{t^2}{2 q^2} \right\} \right\} \]
By choosing $t=0.565 \mathe \Theta \in[\mathe (1-q), \infty)$, note that $\Theta^2\geq \frac{\log \frac{1}{\delta}}{n}$, then with probability at least 
\begin{eqnarray*}
  1-\exp \left\{ - n \left\{ \frac{t}{q}  \left[ \log
  \frac{t}{(1 - q)} - 1 \right] + \frac{t^2}{2 q^2} \right\} \right\}
& \geq & 
1-\exp \left\{ - n \frac{0.565^2 \mathe^2}{2} \Theta^2  \right\}\\
& \geq & 1-\exp \left\{ -1.179 n\cdot \frac{\log \frac{1}{\delta}}{n}   \right\}\\
& \geq & 1- \delta
\end{eqnarray*}

Hence $| N_n (\theta^T, \nu^T) - N (\theta^T, \nu^T) | 
= 2 \frac{1}{n} \sum_{i \in [n]} |V_i -\mathbb{E} [V_i]|
\geq  2 \cdot 0.565 \mathe \Theta = 1.13 \mathe \Theta$ with probability at least $1- 2\delta$, when $\pi^\ast=[1, 0]$ or $[0, 1]$.

\begin{eqnarray*}
  | N_n (\theta^T, \nu^T) - \tmop{sgn} \langle \theta^T, \theta^{\ast} \rangle
  \tanh (\nu^{\ast}) | & \leq & | N (\theta^T, \nu^T) - \tmop{sgn} \langle
  \theta^T, \theta^{\ast} \rangle \tanh (\nu^{\ast}) | + | N_n (\theta^T,
  \nu^T) - N (\theta^T, \nu^T) |\\
  & \leq & \left| 1 - \frac{2}{\pi} \varphi^T \right| \cdot | \tanh
  (\nu^{\ast}) | + \sqrt{\frac{2 \log \frac{1}{\delta}}{n}}\\
  & < & 1.775 \frac{2}{\pi} \Theta \cdot | \tanh (\nu^{\ast}) | +
  \sqrt{\frac{2 \log \frac{1}{\delta}}{n}}\\
  & = & 1.13 \Theta \cdot | \tanh (\nu^{\ast}) | + \sqrt{\frac{2 \log
  \frac{1}{\delta}}{n}}\\
  & = & 1.13 | \tanh (\nu^{\ast}) | \Theta \left( \sqrt{\frac{d}{n}} \vee
  \frac{\log \frac{1}{\delta}}{n} \vee \sqrt{\frac{\log
  \frac{1}{\delta}}{n}} \right) + \sqrt{\frac{2 \log \frac{1}{\delta}}{n}}
\end{eqnarray*}
In the proof, we use $| N (\theta^T, \nu^T) - \tmop{sgn}
\langle \theta^T, \theta^{\ast} \rangle \tanh (\nu^{\ast}) | = \left| 1 -
\frac{2}{\pi} \varphi^T \right| \cdot | \tanh (\nu^{\ast}) |$, which is provided in Corollary~\ref{cor:err_mixing}.

Particularly, when $\pi^\ast=\{1, 0\}$ or $\{0, 1\}$, then $| \tanh (\nu^{\ast}) | = 1$
\begin{eqnarray*}
  | N_n (\theta^T, \nu^T) - \tmop{sgn} \langle \theta^T, \theta^{\ast} \rangle
  \tanh (\nu^{\ast}) | & \leq & | N (\theta^T, \nu^T) - \tmop{sgn} \langle
  \theta^T, \theta^{\ast} \rangle \tanh (\nu^{\ast}) | + | N_n (\theta^T,
  \nu^T) - N (\theta^T, \nu^T) |\\
  & < & 1.13 \Theta \cdot | \tanh (\nu^{\ast}) | + 1.13 \mathe \Theta\\
  & = & 1.13 (1+\mathe)  | \tanh (\nu^{\ast}) | \Theta \left( \sqrt{\frac{d}{n}} \vee
  \frac{\log \frac{1}{\delta}}{n} \vee \sqrt{\frac{\log
  \frac{1}{\delta}}{n}} \right)
\end{eqnarray*}
Therefore $\| \pi^{T + 1} - \bar{\pi}^{\ast} \|_1 = | N_n (\theta^T, \nu^T) -
\tmop{sgn} \langle \theta^T, \theta^{\ast} \rangle \tanh (\nu^{\ast}) |
=\left\| \frac{1}{2} - \pi^{\ast} \right\|_1 
\cdot
\mathcal{O} \left(  \sqrt{\frac{d}{n}} \vee \frac{\log
\frac{1}{\delta}}{n} \vee \sqrt{\frac{\log \frac{1}{\delta}}{n}}
\right) + c (\pi^{\ast}) \cdot \mathcal{O} \left( \sqrt{\frac{\log
\frac{1}{\delta}}{n}} \right)$, where $c (\pi^{\ast})=0$ when $\pi^\ast=\{1, 0\}$ or $\{0, 1\}$, and $c (\pi^{\ast})=\mathcal{O}(1)$.

\end{proof}

%% file: 8_supplementary_4_2_Bernoulli_.tex
Let's denote $q \assign \max (p, 1 - p), V' \assign V_i -\mathbb{E} [V_i], i \in
  [n]$ , thus $\mathbb{E} \left[ {V'}^2 \right] = \tmop{Var} [V_i] = p (1 - p) =
  q (1 - q), |V'|\leq q$.%, and for $\lambda > 0$.
  With Chenorff bound and let $\psi (\lambda)\assign \mathbb{E} [\exp
  (\lambda (V_i -\mathbb{E} [V_i]))] = - \lambda p + \log (1 + p (\exp
  (\lambda) - 1)) \leq \frac{\lambda^2}{8}$
  \begin{eqnarray*}
    \log \mathbb{P} [\sum_{i \in [n]} (V_i -\mathbb{E} [V_i]) \geq n t] & \leq
    & \inf_{\lambda > 0}  \left\{ \log \mathbb{E} \left[ \exp \left(
    \lambda \sum_{i \in [n]} (V_i -\mathbb{E} [V_i]) \right) \right] - \lambda
    n t \right\}\\
    & \leq & \inf_{\lambda > 0}  \left\{ n \left[ \frac{\lambda^2}{8} -
    \lambda t \right] \right\} = - 2 n t^2
  \end{eqnarray*}
  Hence, the first probability inequality is proved.
Let's focus on next concentration inequality. We begin with bounding $\psi (\lambda)$, note that $2 \sinh (x) > \exp (x) - x - 1, \forall x > 0$.
  \begin{eqnarray*}
    \psi (\lambda) %& \assign & \mathbb{E} [\exp (\lambda (V_i -\mathbb{E} [V_i]))] 
    %= p \exp (+ \lambda [1 - p]) + (1 - p) \exp (- \lambda p)\\
    = 1 + \sum_{k \geq 2} \frac{\mathbb{E} \left[ {V'}^k \right]}{k!}
    \lambda^k \leq 1 + \sum_{k \geq 2} \frac{\mathbb{E} \left[ {V'}^2 \right]
    \cdot q^{k - 2}}{k!} \lambda^k
    %\\& \leq & 
    \leq 1 + \frac{(1 - q)}{q} \{ \exp (q \lambda) - q \lambda - 1 \}
    \leq 1 + 2 \frac{(1 - q)}{q} \sinh(q \lambda)
  \end{eqnarray*}
%Since $\cosh x > \exp (x) - x - 1, \forall x \geq \frac{5}{3}$, 
Let $\mu \assign 2 \frac{(1 - q)}{q} \sinh (q \lambda)$, then $\lambda =
\frac{1}{q} \tmop{arcsinh} \left( \frac{q}{2 (1 - q)} \mu \right)$,
and $\mu' \assign \frac{q}{(1 - q)} \mu$, $\gamma \assign \frac{(1 - q)}{q}
\in (0, 1], \tau \assign \frac{t}{(1 - q)}$.
\begin{eqnarray*}
  \log \mathbb{P} \left( \frac{1}{n} \sum_{i \in [n]} (V_i -\mathbb{E}
  [V_i]) \geq t \right) & \leq & \inf_{\lambda > 0} \sum_{i\in[n]} \log \mathbb{E} [\exp
  (\lambda (V_i -\mathbb{E} [V_i]))] - \log \exp (n \lambda t)\\
  & = & n \inf_{\lambda > 0}  [\log \psi (\lambda) - \lambda t]\\
  % & \leq & n \inf_{\lambda > 0}  \left\{ \log \left[ 1 + \frac{(1 - q)}{q}
  % \{ \exp (q \lambda) - q \lambda - 1 \} \right] - \lambda t \right\}\\
  & \leq & n \inf_{\lambda > 0}  \left\{ \log \left[ 1 + 2 \frac{(1 -
  q)}{q} \sinh (q \lambda) \right] - \lambda t \right\}\\
  % & = & n \inf_{\mu > 0}  \left\{ \log (1 + \mu) - \frac{t}{q}
  % \tmop{arcsinh} \left( \frac{q}{2 (1 - q)} \mu \right) \right\}\\
  & = & n \gamma \inf_{\mu' > 0}  \left\{ \frac{1}{\gamma} \log (1 + \gamma
  \mu') - \tau \tmop{arcsinh} \left( \frac{\mu'}{2} \right) \right\}\\
  & \leq & n \gamma \inf_{\mu' > 0}  \left\{ \frac{1}{\gamma} \log (1 +
  \gamma \mu') - \tau \log (\mu') \right\}
\end{eqnarray*}
If $t \geq q$, then $\tau \assign \frac{t}{(1 - q)} \in [
\frac{1}{\gamma}, \infty )$.
\[ n \gamma \inf_{\mu' > 0}  \left\{ \frac{1}{\gamma} \log (1 + \gamma \mu')
   - \tau \log (\mu') \right\} = - \infty \]
If $t \in [\mathe (1 - q), q)$, then $\tau \assign \frac{t}{(1 - q)} \in [ \mathe,
\frac{1}{\gamma} )$.
\begin{eqnarray*}
  &  & n \gamma \inf_{\mu' > 0}  \left\{ \frac{1}{\gamma} \log (1 + \gamma
  \mu') - \tau \log (\mu') \right\}\\
  & = & n \gamma \left\{ \frac{1}{\gamma} \log (1 + \gamma \mu') - \tau
  \log (\mu') \right\}_{\mu' = \frac{\tau}{1 - \gamma \tau}}\\
  & = & n \gamma \left\{ - \tau \log \tau + \frac{1}{\gamma} [- (1 - \gamma
  \tau) \log (1 - \gamma \tau)] \right\}\\
  & \leq & n \gamma \left\{ - \tau \log \tau + \frac{\gamma \tau}{\gamma}
  \left( 1 - \frac{\gamma \tau}{2} \right) \right\}\\
  & \leq & n \gamma \left\{ - \tau [\log \tau - 1] - \frac{\gamma}{2}
  \tau^2 \right\}\\
  & = & n \left\{ - \frac{t}{q}  \left[ \log \frac{t}{(1 - q)} - 1 \right]
  - \frac{t^2}{2 q^2} \right\}
\end{eqnarray*}
Therefore, the third probability inequality is proved, and we show the probability is 0 when $t\geq q$. 